\documentclass[11pt]{article} 
\usepackage[margin=1in]{geometry}
\usepackage[utf8]{inputenc} %
\usepackage[T1]{fontenc}    %
\usepackage{lmodern}

\usepackage[utf8]{inputenc} %
\usepackage[T1]{fontenc}    %
\usepackage{lmodern}

\usepackage{parskip}
\usepackage{xargs}

\usepackage{amsmath}
\usepackage{amsxtra,amsfonts,amscd,amssymb,bm,bbm,mathrsfs}

\usepackage{amsthm}
\usepackage{mathtools}

\usepackage{enumitem}
\usepackage{multirow}
\usepackage{multicol}
\usepackage{booktabs}
\usepackage[toc,page,header]{appendix}
\usepackage{minitoc}
\usepackage{amssymb,amsmath,amsfonts}
\usepackage{amsthm}
\usepackage{mathtools}
\usepackage{graphicx}
\usepackage{epstopdf}
\usepackage{bm}
\usepackage{bbm}
\usepackage{color}
\usepackage{algorithm}
\usepackage[noend]{algorithmic}
\usepackage{hyperref}
\hypersetup{
  colorlinks,
  breaklinks=true,
  citecolor=blue!70!black,
  linkcolor=blue!70!black,
  urlcolor=blue!70!black
}

\usepackage{multirow}
\usepackage{multicol}
\usepackage{latexsym}
\usepackage{amsmath}
\usepackage{amssymb}
\usepackage{amsthm}
\usepackage{epsfig}
\usepackage{xcolor}
\usepackage{graphicx}
\usepackage{bm}
\usepackage{dsfont}
\usepackage{enumitem}
\usepackage{natbib}

\usepackage[nameinlink,capitalize]{cleveref}

\newcommand{\pref}[1]{\cref{#1}}

\renewcommand{\eqref}[1]{\texorpdfstring{\hyperref[#1]{Eq. (\ref*{#1})}}{Eq. (\ref*{#1})}}
\crefformat{equation}{#2(#1)#3}
\Crefformat{equation}{#2(#1)#3}

\Crefformat{figure}{#2Figure #1#3}
\Crefname{assumption}{Assumption}{Assumptions}
\Crefformat{assumption}{#2Assumption #1#3}

\usepackage{crossreftools}
\newtheorem{theorem}{Theorem}
\newtheorem{lemma}[theorem]{Lemma}
\newtheorem{corollary}[theorem]{Corollary}
\newtheorem{assumption}{Assumption}
\newtheorem{proposition}[theorem]{Proposition}

\newtheorem{definition}{Definition}
\newtheorem{example}{Example}
\theoremstyle{definition}
\newtheorem{remark}[theorem]{Remark}

\usepackage[algo2e,linesnumbered,vlined,ruled]{algorithm2e}
\usepackage{algorithm}
\usepackage[noend]{algorithmic}

\usepackage{url}
\usepackage{breakcites}
\usepackage{cancel}
\usepackage{enumitem}
\usepackage{comment}
\crefformat{equation}{#2(#1)#3}
\Crefformat{equation}{#2(#1)#3}

\usepackage{booktabs}       %
\usepackage{multirow}
\usepackage{multicol}
\usepackage{makecell}
\usepackage{array}
\newcolumntype{H}{>{\setbox0=\hbox\bgroup}c<{\egroup}@{}}
\newcolumntype{Z}{>{\setbox0=\hbox\bgroup}c<{\egroup}@{\hspace*{-\tabcolsep}}}

\usepackage[normalem]{ulem}
\usepackage{exscale}
\usepackage{tikz}
\usepackage{float}
\usepackage{graphicx,graphics,subfig}

\allowdisplaybreaks

\usepackage{chngcntr}
\usepackage{apptools}

\let\oldparagraph=\paragraph
\renewcommand\paragraph[1]{\oldparagraph{#1.}}

\usepackage{crossreftools}
\usepackage{setspace}

\newcommand{\sups}[1]{^{{\scriptscriptstyle#1}}}
\newcommand{\subs}[1]{_{{\scriptscriptstyle#1}}}

\newcommand{\alg}{\Alg}

\newcommand{\R}{\mathbb{R}} %

\newcommand{\Hy}{\mathcal{H}}

\newcommand{\hth}{\hat{\theta}}
\newcommand{\lsim}{{\;\raise0.3ex\hbox{$<$\kern-0.75em\raise-1.1ex\hbox{$\sim$}}\;}}
\newcommand{\gsim}{{\;\raise0.3ex\hbox{$>$\kern-0.75em\raise-1.1ex\hbox{$\sim$}}\;}}

\newcommand{\eps}{\varepsilon} 
\newcommand{\RNum}[1]{\uppercase\expandafter{\romannumeral #1\relax}}

\DeclareMathOperator*{\argmin}{arg\,min}
\DeclareMathOperator*{\argmax}{arg\,max}
\newcommand{\KL}{D_{\mathrm{KL}}}

\newcommand{\cA}{\mathcal{A}}
\newcommand{\cB}{\mathcal{B}}

\newcommand{\cD}{\mathcal{D}}
\newcommand{\cE}{\mathcal{E}}
\newcommand{\cF}{\mathcal{F}}

\newcommand{\cH}{\mathcal{H}}
\newcommand{\cI}{\mathcal{I}}

\newcommand{\cL}{\mathcal{L}}
\newcommand{\cM}{\mathcal{M}}

\newcommand{\cO}{\mathcal{O}}
\newcommand{\cP}{\mathcal{P}}
\newcommand{\cQ}{\mathcal{Q}}

\newcommand{\cT}{\mathcal{T}}
\newcommand{\cU}{\mathcal{U}}
\newcommand{\cV}{\mathcal{V}}

\newcommand{\cX}{\mathcal{X}}

\newcommand{\cZ}{\mathcal{Z}}

\newcommand{\En}{\mathbb{E}}

\newcommand{\vf}{\mathbf{f}}

\DeclareFontFamily{U}{mathx}{\hyphenchar\font45}
\DeclareFontShape{U}{mathx}{m}{n}{<-> mathx10}{}
\DeclareSymbolFont{mathx}{U}{mathx}{m}{n}
\DeclareMathAccent{\widebar}{0}{mathx}{"73}

\newcommand{\wb}[1]{\widebar{#1}}

\newcommand{\veps}{\varepsilon}
\newcommand{\ldef}{\vcentcolon=}

\newcommand{\Unif}{\mathrm{Unif}}

\newcommand{\Dhels}[2]{D^{2}_{\mathrm{H}}\prn*{#1,#2}}

\newcommand{\poly}{\mathrm{poly}}

\newcommand{\supp}{\mathrm{supp}}

\DeclarePairedDelimiter{\brk}{[}{]}

\DeclarePairedDelimiter{\prn}{(}{)}
\DeclarePairedDelimiter{\norm}{\|}{\|}

\DeclarePairedDelimiter{\set}{\{}{\}}

\DeclarePairedDelimiter{\floor}{\lfloor}{\rfloor}

\def\medskip{\vskip 10 pt}
\def\bigskip{\vskip 15 pt}

\def\texitem#1{\par\vspace{5pt}
\noindent\hangindent 20pt
\hbox to 20pt {\hss #1 ~}\ignorespaces}

\newcommand{\co}{\operatorname{co}}
\newcommand{\regdm}{\mathbf{Reg}_{\mathsf{DM}}}

\newcommand{\rdecc}{\normalfont{\textsf{r-dec}}^{\rm c}}
\newcommand{\rdeco}{\normalfont{\textsf{r-dec}}^{\rm o}}
\newcommand{\pdeco}{\normalfont{\textsf{p-dec}}^{\rm o}}
\newcommand{\pdecc}{\normalfont{\textsf{p-dec}}^{\rm c}}
\newcommand{\LDPtag}{{\scriptscriptstyle\mathsf{LDP}}}
\newcommand{\Gtag}{{\scriptscriptstyle\mathsf{H}}}
\newcommand{\rdecl}{{\normalfont \textsf{r-dec}}^{\LDPtag}}
\newcommand{\pdecl}{{\normalfont \textsf{p-dec}}^{\LDPtag}}
\newcommand{\rdecol}{{\normalfont \textsf{r-dec}}^{\rm o, \LDPtag}}
\newcommand{\pdecol}{{\normalfont \textsf{p-dec}}^{\rm o, \LDPtag}}

\newcommand{\pdecq}{{\normalfont \textsf{p-dec}}^{\rm q}}
\newcommand{\pdecql}{{\normalfont \textsf{p-dec}}^{\rm q, \LDPtag}}
\newcommand{\pdecr}{{\normalfont \textsf{p-dec}}^{\scriptscriptstyle\mathsf{R}}}
\newcommand{\pdecg}{{\normalfont \textsf{p-dec}}^{\Gtag}}
\newcommand{\rdecg}{{\normalfont \textsf{r-dec}}^{\Gtag}}
\newcommand{\pdecqg}{{\normalfont \textsf{p-dec}}^{{\rm q}, \Gtag}}
\newcommand{\pdecog}{\normalfont{\textsf{p-dec}}^{{\rm o},\Gtag}}
\newcommand{\SQtag}{{\normalfont\scriptscriptstyle \textsf{-SQ}}}
\newcommandx{\pdecltau}[1][1=\tau]{{\normalfont \textsf{p-dec}}^{#1\SQtag}}

\newcommand{\PDEC}{Private PAC-DEC}
\newcommand{\RDEC}{Private regret-DEC}
\newcommand{\pDEC}{private PAC-DEC}
\newcommand{\rDEC}{private regret-DEC}

\newcommandx{\Whp}[1][1=\delta]{With probability at least $1-#1$}
\newcommandx{\whp}[1][1=\delta]{with probability at least $1-#1$}

\newcommand{\Lipr}{C_{V}}

\newcommand{\DPi}{\Delta(\Pi)}
\newcommand{\DbPi}{\Delta(\bPi)}
\newcommand{\DM}{\Delta(\cM)}
\newcommand{\DA}{\Delta(\cA)}
\newcommand{\coM}{\co(\cM)}

\newcommand{\Rdd}{\R^{d\times d}}

\newcommand{\Mstar}{M^\star}
\newcommand{\Mbar}{\wb{M}}

\usepackage{pifont}

\newcommand{\id}{I}

\newcommand{\EE}{\mathbb{E}}
\newcommand{\PP}{\mathbb{P}}
\newcommand{\QQ}{\mathbb{Q}}

\newcommand{\DD}{\mathbb{D}}
\renewcommand{\DD}{\Pi}
\newcounter{cnt}
\foreach \num in {1,2,...,26}{%
  \stepcounter{cnt}%
  \expandafter\xdef \csname c\Alph{cnt}\endcsname {\noexpand\mathcal{\Alph{cnt}}}%
  \expandafter\xdef \csname b\Alph{cnt}\endcsname {\noexpand\mathbb{\Alph{cnt}}}%
}

\newcommand{\tr}{\mathrm{tr}}

\newcommand{\diag}{\operatorname{diag}}

\newcommand{\leqsim}{\lesssim}
\newcommand{\geqsim}{\gtrsim}

\newcommand{\nrm}[1]{\left\|#1\right\|}

\newcommand{\abs}[1]{\left|#1\right|}

\newcommand{\Om}[1]{\Omega\left(#1\right)}

\DeclarePairedDelimiterX{\ddiv}[2]{(}{)}{%
  #1\;\delimsize\|\;#2%
}
\newcommand{\KLd}{\KL\ddiv}
\newcommand{\chis}{D_{\chi^2}\ddiv}

\newcommand{\be}{\mathbf{e}}

\newcommand{\linf}[1]{\norm{#1}_\infty} %

\newcommand{\indic}[1]{\mathbf{1}\left\{#1\right\}} %

\newcommand{\<}{\left\langle}
\renewcommand{\>}{\right\rangle}

\newcommand{\dH}{D_{\rm H}}
\newcommand{\DHr}[1]{D_{\rm H}\paren{#1}}
\renewcommand{\DH}[1]{D_{\mathrm{H}}^2\left(#1\right)}

\newcommand{\DTV}[1]{D_{\mathrm{TV}}\left(#1\right)}

\newcommand{\Alg}{\mathsf{Alg}}
\newcommand{\env}{\mathsf{Env}}

\newcommand{\Vmax}{V_{\max}}

\newcommand{\hM}{\widehat{M}}
\newcommand{\oM}{\Mbar}
\newcommand{\Ms}{\Mstar}
\newcommand{\pis}{\pi^\star}

\newcommand{\oeps}{\bar{\eps}}
\newcommand{\ueps}{\uline{\eps}}

\newcommand{\defeq}{\mathrel{\mathop:}=}

\newcommand{\CKL}{C_{\rm KL}}

\newcommand{\paren}[1]{{\left( #1 \right)}}
\newcommand{\brac}[1]{{\left[ #1 \right]}}

\newcommand{\normal}[1]{\mathsf{N}\paren{#1}}

\newcommand{\sset}[1]{\left\{#1\right\}}

\newcommandx{\VM}[1][1=M]{V\sups{#1}}
\newcommand{\VMs}[1][1=M]{V\sups{#1}(\pi\subs{#1})}
\newcommandx{\fm}[1][1=M]{f\sups{#1}}

\newcommandx{\muM}[1][1=M]{\nu\subs{#1}}

\newcommand{\sumt}{\sum_{t=1}^T}
\newcommand{\Err}{\mathrm{Err}}

\newcommand{\cMp}{\cM\cup\set{\oM}}

\newcommand{\constr}[2]{\left\{\left.#1 ~\right|~#2\right\}}
\newcommand{\onu}{\overline{\nu}}

\newcommand{\DX}{\Delta(\cX)}
\newcommand{\DZ}{\Delta(\cZ)}
\newcommand{\DO}{\Delta(\cO)}

\newcommand{\DPP}{\Delta(\Pi\times\Pc)}

\newcommand{\Power}[1]{\mathscr{P}(#1)}

\newcommandx{\PM}[3][1=M,2=\pi,3=\pr]{#3\!\circ\!#1(#2)}

\newcommand{\Pc}{\cQ}
\newcommand{\Pcp}{\cQ_{\alpha}}

\newcommand{\Pcbin}{\Pc_{\alpha,\sf bin}}

\newcommand{\pr}{\mathsf{Q}}

\newcommand{\pLDP}{$\alpha$-LDP}
\newcommand{\pDP}{$\alpha$-DP}
\newcommand{\aLDP}{$(\alpha,\beta)$-LDP}
\newcommand{\aDP}{$(\alpha,\beta)$-DP}

\newcommand{\gfunc}{g}
\newcommandx{\gm}[1][1=M]{\gfunc\sups{#1}}

\newcommand{\ldim}{\mathsf{LDim}}

\newcommand{\rdim}{\mathsf{RDim}}

\newcommand{\cFp}{\cF^+}

\newcommandx{\risk}[1][1=M]{\mathsf{Risk}^{#1}}

\newcommand{\bpi}{\boldsymbol{\pi}}

\newcommand{\vvec}{\mathrm{vec}}
\newcommand{\Bone}{\mathbb{B}^d(1)}

\newcommand{\nrmF}[1]{\nrm{#1}_F}

\newcommand{\ps}{p^\star}

\newcommand{\otheta}{\Bar{\theta}}
\newcommand{\oq}{\Bar{q}}

\newcommand{\op}{\Bar{p}}

\newcommand{\piom}{\pim[\oM]}

\newcommand{\cuto}[1]{\min\set{#1,1}}

\newcommand{\sq}{\sups{1/2}}

\newcommand{\ths}{\theta^\star}
\newcommand{\la}{\langle}
\newcommand{\ra}{\rangle}

\newcommand{\bz}{\mathbf{0}}

\newcommand{\Riskdm}{\mathbf{Risk}_{\mathsf{DM}}}
\newcommand{\riskdm}{\Riskdm}

\newcommand{\Phc}{\Phi}

\newcommandx{\Mcxt}[1][1=M]{#1_{\sf cxt}}

\newcommand{\ea}{e^{\alpha}}
\newcommand{\eai}{e^{-\alpha}}

\newcommand{\DPL}{\Delta(\Pi\times\Lc)}

\newcommand{\lr}{\langle}
\newcommand{\rr}{\rangle}
\newcommand{\llr}{\left \langle}
\newcommand{\rrr}{\right \rangle}
\newcommand{\bigO}[1]{O\paren{#1}}
\newcommand{\tbO}[1]{\tilde{O}\paren{#1}}

\newcommand{\tO}{\Tilde{O}}

\newcommand{\Bern}[1]{\mathrm{Bern}\paren{#1}}
\newcommand{\Rad}[1]{\mathrm{Rad}\paren{#1}}

\newcommand{\cMagn}{\cM_{\sf agnostic}}
\newcommand{\cMreal}{\cM_{\cF,\sf realizable}}
\newcommand{\cMwell}{\cM_{\cF}}
\newcommand{\fs}{f^\star}

\newcommand{\mus}{\mu^\star}

\newcommand{\infpqb}{\inf_{\substack{p\in\DDD\\q\in\Delta(\bPi)}}}
\newcommand{\infpql}{\inf_{\substack{p\in\DDD\\q\in\DPL}}}

\newcommand{\infpb}{\inf_{p\in\DbPi}}
\newcommand{\infpl}{\inf_{p\in\DPL}}
\newcommand{\pqll}{{\substack{p\in\DDD\\q\in\DL}}}
\newcommand{\infpqll}{\inf_\pqll}

\newcommand{\ZZ}{\mathbb{Z}}
\newcommand{\cFpar}{\cF_{\sf parity}}
\newcommand{\cMpar}{\cM_{\sf parity}}

\newcommandx{\Dl}[1][1=\lf]{\mathsf{D}_{#1}}
\newcommand{\lf}{\ell}
\newcommand{\Lc}{\cL}
\newcommand{\DL}{\Delta(\Lc)}

\newcommandx{\DC}[2][1=\Delta]{N_{\mathsf{frac}}(#2,#1)}
\newcommandx{\pds}[1][1=\Delta]{p_{#1}^\star}

\newcommand{\dct}{fractional covering number}
\newcommand{\Dct}{Fractional covering number}

\newcommand{\DF}{\Delta(\cF)}
\newcommand{\cMF}{\cM_{\cF}}

\newcommand{\DDD}{\Delta(\DD)}
\newcommandx{\NM}[2][1=\cM]{N(#1,#2)}
\newcommand{\pio}{\pi}

\newcommand{\pihat}{\hpi}

\newcommand{\phat}{\widehat{p}}
\newcommand{\statp}{statistical problems}
\newcommandx{\pim}[1][1=M]{\pi\sups{#1}}
\newcommandx{\pip}[1][1=\cP]{\pi\sups{#1}}
\newcommandx{\pims}{\pim[\Mstar]}
\newcommandx{\Vm}[1][1=M]{V\sups{#1}}
\newcommandx{\Vmm}[1][1=M]{V\sups{#1}(\pi\sups{#1})}

\newcommand{\LOSS}{L}
\newcommand{\losst}{loss function}
\newcommand{\loss}{\mathrm{L}}
\newcommandx{\LM}[2][1=M]{L(#1,#2)}

\newcommand{\ind}[1]{_{#1}}
\newcommand{\hpi}{\pi\subs{T+1}}
\newcommandx{\pit}[1][1=t]{\pi_{#1}}
\newcommandx{\ppt}[1][1=t]{p_{#1}}
\newcommandx{\qt}[1][1=t]{q_{#1}}
\newcommandx{\prt}[1][1=t]{\pr_{#1}}
\newcommandx{\ot}[1][1=t]{o_{#1}}
\newcommandx{\zt}[1][1=t]{z_{#1}}
\newcommandx{\act}[1][1=t]{a_{#1}}
\newcommandx{\rt}[1][1=t]{r_{#1}}

\newcommand{\algcommentbig}[1]{\textcolor{blue!70!black}{\footnotesize{\texttt{\textbf{/*
          #1~*/}}}}}

\newcommand{\etod}{{\normalfont\textsf{E2D}}}
\newcommand{\LDPetod}{{\normalfont\textsf{LDP-E2D}}}
\newcommand{\SQetod}{{\normalfont\textsf{SQ-E2D}}}
\newcommand{\sqetod}{\SQetod}
\newcommand{\ExO}{{Exploration-by-Optimization}}
\newcommand{\ExOp}{\ensuremath{ \mathsf{ExO}^+}}
\newcommand{\LDPexo}{{\normalfont\textsf{LDP-ExO}}}
\newcommandx{\Enmpi}[3][1=M,2=\pi]{\En\sups{#1,#2}\brac{#3}}
\newcommandx{\Emalg}[3][1=M,2=\alg]{\EE\sups{#1,#2}\brac{#3}}
\newcommandx{\Pmalg}[3][1=M,2=\alg]{\PP\sups{#1,#2}\paren{#3}}

\newcommand{\pJDP}{$\alpha$-JDP}

\newcommand{\piout}{\pi^{\sf out}}

\newcommand{\bPP}{\Bar{\PP}}
\newcommand{\bEE}{\Bar{\EE}}
\newcommand{\bv}{\mathbf{v}}

\newcommandx{\bpr}[1][1=\lf]{\pr_{#1}}
\newcommand{\ca}{c_\alpha}
\newcommandx{\Mpara}[1][1=M]{\theta(#1)}
\newcommandx{\RISK}[2][1=T,2=\xspace]{\mathfrak{M}_{#1}^{#2}}
\newcommandx{\RISKob}[1][1=T]{\mathfrak{M}_{#1}^{\mathsf{obl}}}
\newcommandx{\SC}[2][1=\Delta,2=\xspace]{\mathfrak{C}_{#1}^{#2}}
\newcommandx{\SCob}[1][1=\Delta]{\mathfrak{C}_{#1}^{\mathsf{ob}}}
\newcommand{\LP}{\cM}
\newcommandx{\Ncov}[3][1=\xspace,2=\Delta]{N_{#1}(#3,#2)}

\newcommand{\MPow}{\mathscr{P}}
\newcommandx{\cMPow}[1][1=\MPow]{\cM_{#1}}

\newcommand{\SQDEC}{SQ DEC}
\newcommandx{\SQ}[1][1=M]{\mathsf{STAT}_{#1}^{\tau}}
\newcommandx{\VSTAT}[1][1=M]{\mathsf{VSTAT}_{#1}^{\tau}}
\newcommandx{\GSQ}[1][1=M]{\mathsf{GQ}_{#1}^{\tau}}
\newcommandx{\phq}[2][1=\bpi]{#2(#1)}
\newcommandx{\Dph}[2][1={\phi}]{\mathsf{D}_{#1}\paren{#2}}

\newcommand{\pqb}{{\substack{p\in\DDD\\q\in\DbPi}}}

\newcommand{\bcM}{\cM^+}

\newcommand{\MPowiid}{\MPow_{\mathsf{sto}}}
\newcommand{\MPowadv}{\MPow_{\mathsf{adv}}}
\newcommand{\MPowcxt}{\MPow_{\mathsf{cxt}}}
\newcommand{\MPowsq}{\MPow_{\tau\text{-}\mathsf{query}}}
\newcommand{\MPowdp}{\MPow_{\LDPtag}}
\newcommand{\Hubertag}{\text{-}\mathsf{Huber}}
\newcommand{\MPowrob}{\MPow_{\beta\Hubertag}}
\newcommand{\cMrob}{\cM_{\beta\Hubertag}}

\newcommand{\ext}[1]{{#1}^{\sharp}}
\newcommand{\tPi}{\ext{\Pi}}
\newcommand{\tcM}{\ext{\cM}}
\newcommand{\tM}{\ext{M}}
\newcommand{\toM}{\ext{\oM}}

\newcommand{\tpi}{\bpi}

\newcommand{\lfdiv}{$\lf$-divergence}
\newcommand{\lfdivs}{$\lf$-divergences}

\newcommand{\creg}{c_{\rm reg}}
\newcommand{\rmd}{\mathsf{d}}
\newcommand{\OsqrtT}{O(\sqrt{\log T})}
\newcommand{\finite}{is compact (\cref{asmp:finite})}

\newcommandx{\hgm}[3][1=M,2=\delta]{\widehat{L}_{#2}(#1,#3)}

\renewcommand{\perp}{\mathrm{id}}

\newcommand{\tq}{\Tilde{q}}

\newcommandx{\Tdec}[2][1=\Delta]{\mathfrak{C}^{\,\sf dec}_{#1}(#2)}

\newcommand{\cDMSO}{hybrid DMSO}
\newcommand{\qbDMSO}{SQ DMSO}

\newcommand{\rDMSO}{robust DMSO}
\newcommand{\RDMSO}{Robust DMSO}
\newcommand{\pDMSO}{private DMSO}
\newcommand{\PDMSO}{Private DMSO}
\newcommand{\gDEC}{hybrid DEC}

\title{Decision Making in Changing Environments: \\
Robustness, Query-Based Learning, and Differential Privacy}

\author{Fan Chen\\{\small \texttt{fanchen@mit.edu}} \and    Alexander Rakhlin\\{\small \texttt{rakhlin@mit.edu}}
}

\begin{document}

\maketitle

\begin{abstract}
    We study the problem of interactive decision making in which the underlying environment changes over time subject to given constraints. We propose a framework, which we call \textit{hybrid Decision Making with Structured Observations} (hybrid DMSO), that provides an interpolation between the stochastic and adversarial settings of decision making. Within this framework, we can analyze local differentially private (LDP) decision making, query-based learning (in particular, SQ learning), and robust and smooth decision making under the same umbrella, deriving upper and lower bounds based on variants of the Decision-Estimation Coefficient (DEC). We further establish strong connections between the DEC's behavior, the SQ dimension, local minimax complexity, learnability, and joint differential privacy. To showcase the framework's power, we provide new results for contextual bandits under the LDP constraint.
\end{abstract}

\section{Introduction}\label{sec:intro}
The Decision-Estimation Coefficient (DEC) \citep{foster2021statistical,foster2023tight} has been recently shown to capture the difficulty of exploration in a wide range of problems in which a learning agent interacts with an unknown environment by making decisions and observing outcomes. Such problems include structured bandits, contextual bandits, and reinforcement learning, among others. 
The interaction protocol, termed \textit{Decision Making with Structured Observations} (DMSO) in  \citep{foster2021statistical}, assumes that the unknown model is fixed over the length of the interaction, i.e. the learning agent faces a stationary environment. This is often referred to as a \textit{stochastic setting}, or \textit{stochastic DMSO}. In contrast, the \textit{adversarial DMSO}, studied in \citep{foster2022complexity}, is a more complex task where the model may change arbitrarily between the rounds of the interaction. 

In this paper, we study a setting that interpolates between the stochastic and adversarial DMSO. This interpolation is achieved by placing constraints on the way the model may change over time. Within the constraint set, the model is allowed to change arbitrarily, and we refer to the setting as that of \textit{constrained adversaries}, or \textit{hybrid DMSO}. In parallel with such constraints on the adversary, we additionally study constraints placed on the information received by the decision-maker, for instance due to privacy requirements or a specific oracle model of computation. The specification of constraints allows us to study---under the same umbrella---decision making with Statistical Queries (SQ)~\citep{kearns1998efficient}, local differential privacy (LDP)~\citep{kasiviswanathan2011can,duchi2013local}, robustness with respect to model corruption~\citep{huber1965robust,huber2011robust}, and smooth decision making~\citep{rakhlin2011online}. For example, in SQ learning, the decision-maker obtains information by issuing queries; since the response to these queries is only approximately correct, it is natural to model it as a response of an adversary that has limited power in providing misleading information. Similarly, for robust decision making, we can model corruption (for instance, as in Huber's contamination model~\citep{huber1965robust}) or mis-specification (as in agnostic learning) directly as a constraint on the environment to be close to a ground-truth model. In turn, the local privacy constraint can be  formulated as a restriction on the decision-maker to only observe information through differentially private channels.

Our approach begins with the \emph{\gDEC} formulation that yields both lower and upper bounds for PAC learning and no-regret learning under \cDMSO. Then, by investigating the specific information structures imposed by the constraints on the adversary and the decision-maker, we derive the corresponding DECs and the statistical guarantees for the aforementioned (and seemingly disparate) settings. 
As such, the unified viewpoint leads to a systematic ``recipe'' for analyzing new problems under the hybrid DMSO setting; this is illustrated on numerous examples throughout the paper. 
What is perhaps even more surprising, all the upper bounds are achieved by only two algorithmic approaches: a generalization of the Exploration-by-Optimization Algorithm~\citep{lattimore2020exploration,lattimore2021mirror,foster2022complexity} and a variant of the Estimation-to-Decision Algorithm~\citep{foster2021statistical,foster2023tight}.

The fact that DMSO provides such a unified viewpoint on disparate problems is a testament to the power of the framework, with DEC as the central notion of inherent problem complexity.

\subsection{Contributions}

We formulate decision making in the setting of hybrid DMSO, generalizing the Decision-Estimation Coefficient framework~\citep{foster2021statistical,foster2023tight}. Our proposed notion of hybrid DEC allows us to understand, under the same umbrella, minimax behavior of statistical estimation and interactive decision making under such seemingly different settings as local differential privacy, query-based learning (in particular, statistical queries), robust learning, and smoothness. In particular, hybrid DECs for PAC learning and no-regret learning yield both lower and upper bounds for the corresponding learning goals. Our upper bounds are achieved by the unified Exploration-by-Optimization Algorithm (\ExOp, cf. \citet{lattimore2020exploration,lattimore2021mirror,foster2022complexity}).

As instantiations of our framework, we derive the hybrid DECs and corresponding upper and lower bounds for query-based learning (\cref{ssec:query-demo}), locally private learning (\cref{ssec:LDP-demo}), robust decision making (\cref{ssec:robust}), and decision making against smooth adversaries (\cref{ssec:smooth}). The problem of contextual bandits with adversarial contexts also naturally falls under our hybrid formulation (\cref{ssec:CBs}), and we provide novel results for this setting as well.

Our primary goal is to understand the complexity of learning problems at some level of generality, rather than specific examples. 
Still, as a concrete application, our framework provides a near-optimal $\sqrt{T}$-regret for linear contextual bandits with local privacy (without well-conditioned assumptions), settling the open problem of the optimal regret in this setting~\citep{zheng2020locally,han2021generalized,li2024optimal}.

In addition, we make the following connections to other previously studied notions: 
\begin{itemize}
    \item \textbf{SQ dimension}. The \emph{SQ dimension} proposed by \citet{feldman2017general} provides both lower and upper bounds for the optimal query complexity of SQ learning of \emph{distribution search problems}. Not surprisingly, we show that there is quantitative equivalence between the SQ dimension and our \SQDEC~(\cref{ssec:SQ-dim}). Therefore, our results extend the characterizations of \citet{feldman2017general} to general query-based learning problems.
    \item \textbf{Local-minimax optimality under LDP.} We show that our lower and upper bounds for LDP learning can be specialized to provide a tight characterization of the local-minimax complexity (\cref{ssec:local-minimax}). In particular, for functional estimation, our results recover (up to logarithmic factors) the characterization of \citet{duchi2024right} through the modulus of continuity.
    
    \item \textbf{LDP learnability.} We show that for any problem class, the \dct~\citep{chen2024beyond} characterizes the finite-time LDP learnability (\cref{sec:DC}). In \cref{ssec:JDP}, we also relate \dct~to the learnability under joint differential privacy (JDP) and the representation dimension~\citep{beimel2013characterizing}. 
\end{itemize}

\subsection{Related work}

\paragraph{Decision-Estimation Coefficient Framework}
Towards a unifying framework for interactive decision making, \citet{foster2021statistical} propose Decision-Estimation Coefficient (DEC), which provides both lower and upper bounds for any decision making problem.
An active line of research~\citep{foster2022complexity,chen2022unified,foster2023tight,foster2023complexity,glasgow2023tight,chen2024beyond} has extended the DEC framework to various more general learning goals, including adversarial decision making~\citep{foster2022complexity}, PAC decision making~\citep{chen2022unified,foster2023tight}, reward-free learning and preference-based learning~\citep{chen2022unified}, multi-agent decision making and partial monitoring~\citep{foster2023complexity}, and interactive estimation~\citep{chen2022unified,chen2024beyond}. The present work further extends the DEC framework to handle changing environments and constraints on the decision maker, and our results heavily draw on the techniques developed in these previous papers.

\paragraph{Exploration-by-Optimization}
The \emph{Exploration-by-Optimization} technique is powerful machinery developed in \citet{lattimore2020exploration,lattimore2021mirror} for partial monitoring in adversarial environments and later extended by \citet{foster2022complexity} to decision making in adversarial environments, achieving upper bounds in terms of the generalized Information Ratio~\citep{russo2014learning,russo2018learning,lattimore2021mirror} or the DEC~\citep{foster2021statistical,foster2022complexity}. In the present work, we further extend this technique by incorporating the notion of \emph{information sets}, allowing a more granular quantification of the information and model equivalences that the decision-maker can take advantage of. The idea of using information sets in the context of posterior sampling was proposed by Dylan Foster back in 2022, and was considered by the authors of \citep{foster2022note} as a way of improving DEC-based results for reinforcement learning.

\paragraph{Local differential privacy}
The notion of local differential privacy (LDP) was formalized by \citet{kasiviswanathan2011can,duchi2013local}, with some earlier work on this subject dating back to \citet{warner1965randomized}. 
A line of research has been investigating the statistical complexity of locally private learning for various statistical estimation problems~\citep{duchi2013local,duchi2018minimax,duchi2019lower}, including mean estimation~\citep{asi2022optimal,asi2024fast}, functional estimation~\citep{rohde2020geometrizing,butucea2021locally,butucea2023interactive,duchi2024right}, hypothesis testing~\citep{berrett2020locally,li2023robustness} and selection~\citep{gopi2020locally,pour2024sample}, and regression~\citep{wang2019sparse,berrett2021strongly}, to name a few. Beyond the setting of statistical estimation, recent research studies the complexity of interactive decision making with local privacy constraints, including contextual bandits~\citep{zheng2020locally,han2021generalized,li2024optimal} and %
episodic RL~\citep{garcelon2021local,liao2023locally}. Notably, these works mainly focus on specific problems and adopt problem-tailored approaches.

\paragraph{The role of interaction in LDP learning} 
It has long been known that there is a statistical separation between non-interactive private channels and sequential private channels \citep{kasiviswanathan2011can}. As is surveyed by \citet{butucea2023interactive}, the separation of sample complexity between non-interactive and interactive channels is identified for certain problems of testing \citep{berrett2020locally} and functional estimation \citep{butucea2021locally,butucea2023interactive}. Therefore, even for statistical problems (where samples are being generated i.i.d), interactive learning is generally necessary to achieve optimal sample complexity under LDP constraints. As the DEC framework characterizes the complexity of exploration of interactive decision making, it is suitable for quantifying the complexity of interactive LDP learning.

\paragraph{Statistical Queries}
The Statistical Query (SQ) model was introduced by \citet{kearns1998efficient} as a restricted PAC learning model, and it turns out to be a powerful tool for understanding the computational complexity of a wide range of algorithms and problems~\citep{feldman2015complexity,feldman2017statistical,diakonikolas2017statistical,brennan2020statistical}. Variants of SQ model have also been studied~\citep{bshouty2002using,feldman2017general,joshi2024complexity}. The connection between local DP and SQ learning has been identified by \citet{kasiviswanathan2011can}. 
For distributional search problems, \citet{feldman2017general} characterized the SQ query complexity in terms of the \emph{SQ dimension}, which turns out to be recovered by the \SQDEC~(when specialized to this case).

\paragraph{Robust statistics}
The \emph{robustness} of a statistical procedure refers to the ability to adapt to model mis-specification or perturbation. In robust statistics, the contamination model of \citet{huber1965robust} has been extensively studied, where the data are assumed to be sampled i.i.d from a distribution that is $\beta$-contaminated from the ground-truth distribution. A recent line of work~\citep{diakonikolas2019robust,diakonikolas2019recent,liu2021settling,diakonikolas2023algorithmic,canonne2023full}, among others, studied stronger contamination models, where the adversary is allowed to maliciously corrupt $\beta$-fraction of the whole dataset (detailed discussion in \cref{ssec:robust-more}). The connection between robustness and differential privacy is also studied by \citet{georgiev2022privacy,hopkins2023robustness,asi2023robustness}.

\section{Overview of Results}\label{sec:overview}

We start this section by formulating the \emph{\cDMSO} framework (\cref{ssec:cDMSO}), a generalization of the Decision Making with Structured Observation (DMSO) framework proposed by \citet{foster2021statistical}. We then show how this generalization encompasses query-based learning (\pref{ssec:query-demo}), locally differentially private learning (\pref{ssec:LDP-demo}), and robust decision making (\cref{ssec:robust}). For each setting, we formulate a corresponding variant of DMSO, the corresponding DEC, and the ensuing PAC guarantees. We also present regret guarantees for \cDMSO~(\cref{ssec:cDMSO-reg}), with application to smooth learning (\pref{ssec:smooth}).

\subsection{Hybrid DMSO}\label{ssec:cDMSO}

In the DMSO formulation, studied in  \citep{foster2021statistical},  the learner (or, the decision maker) interacts for $T$ rounds with the environment described by an underlying model $\Mstar$, unknown to the learner (detailed discussion in \cref{ssec:sto-DMSO}).
While the DMSO formulation is general enough to capture various learning tasks and problem classes, it is restricted to the \emph{stochastic} setting, where the underlying environment is stationary (specified by the model $\Mstar$). However, in many applications, the environment is best described as non-stationary and changing according to the previous history of interaction, while at the same time satisfying certain constraints. Inspired by \citet{foster2022complexity}, who consider an arbitrarily changing environment, we propose the following \emph{\cDMSO} formulation. We will reserve the term ``stochastic DMSO'' for the original DMSO setting of \citet{foster2021statistical}.

\newcommand{\cPs}{\cP^\star}
\newcommand{\bPi}{\mathbf{\Pi}}

In the \cDMSO~setting studied here, the environment is allowed to be (adaptively) adversarial with certain constraints, while the learner has to interact with the environment through a given class $\Phi$ of \emph{measurements}. Specifically, let $\bPi=\Pi\times \Phi$ be the joint decision space, and let $(\bPi\to\DO)$ be the set of all \emph{models}, with each model being a conditional distribution of observation given a (decision, measurement) pair. A \emph{constraint} for the adversary will be modeled by a subset $\cP\subseteq (\bPi\to\DO)$, and a collection of constraints—as a set $\MPow$ of such subsets. We consider the following $T$-round interaction protocol between the environment and the learner:

\begin{enumerate}
    \item Before the interaction, the environment (or, the adversary) selects a constraint $\cPs\subseteq (\bPi\to\DO)$, without revealing it to the learner. 
    \item For each $t=1,\cdots,T$:
\begin{itemize}
  \setlength{\parskip}{2pt}
    \item The environment selects $M^t\in \cPs$, and the learner selects a decision $\bpi\ind{t}=(\pi\ind{t},\phi\ind{t})\in \bPi$. 
    \item The learner observes $o\ind{t}\in\cO$, sampled according to $o\ind{t}\sim M^t(\pi\ind{t},\phi\ind{t})$.
\end{itemize}
\end{enumerate}
The set $\cPs$ restricts the power of the adversary, and we assume the learner has access to a collection $\MPow$ of constraints that contains $\cPs$. In other words, $\MPow$ reflects prior knowledge of the possible constraints on the adversary. We formalize this assumption as follows.
\begin{assumption}[Constraint realizability]
The given class $\MPow$ contains $\cPs$.
\end{assumption}

For some of the settings studied in this paper, the prior knowledge is additionally reflected in a more succinct \emph{model class} $\cM\subseteq (\bPi\to\DO)$, and the constraint class $\MPow$ will reflect this choice.

The general formulation of constraints interpolates between
\begin{itemize}
    \item \textit{stochastic DMSO} framework~\citep{foster2021statistical}, where the environment is stochastic, i.e., $M^1=\cdots=M^T=\Mstar\in\cM$, and it can be specified by constraint $\cPs=\set{\Mstar}$ and $\MPowiid=\set{\set{\Mstar}: \Mstar\in\cM}$, and
    \item \textit{adversarial DMSO} framework~\citep{foster2022complexity} (detailed in \cref{ssec:adv-DMSO}), where the environment is fully adversarial, i.e., the constraint is $\cPs=\cM$ and $\MPowadv=\set{\cM}$. 
\end{itemize}
Further examples of \cDMSO~include \emph{\qbDMSO} (\cref{ssec:query-demo}), where the environment is allowed to respond to queries with values that are $\tau$-correct with respect to a ground truth model $\Mstar$, and \emph{\rDMSO} (\cref{ssec:robust}), where the environment is allowed to perturb the observation generated by a ground truth model $\Mstar$ with a fixed probability $\beta$.

In addition to the constraints on the way the environment may change, the class $\Phi$ of measurements encodes constraints on the learner, affecting the information the learner observes. For instance, in the examples studied in this paper, the measurements will take the form of allowed queries (\cref{ssec:query-demo}) or differentially private channels (\cref{ssec:LDP-demo}). 
Of course, the case of $\cPs\subseteq (\Pi\to\DO)$ may be regarded as the trivial choice $\Phi=\set{\perp}$ of identity measurement.

\paragraph{Learning objective}
In PAC learning, the goal of the learner is to select an \emph{output decision} $\hpi\in\Pi$ after $T$ rounds of interaction, with the performance measured by
\begin{align}\label{eqn:def-risk-Ps}
    \Riskdm(T) \ldef \EE_{\hpi\sim \phat}\brac{ \LM[\cPs]{\hpi} },
\end{align}
where $\hpi\sim \phat$ is the randomized decision of the learner, $L:\MPow\times\Pi\to\R$ is a known loss function. 

To simplify the presentation in this section, we mainly focus on the PAC formulation, deferring the study of regret to \cref{ssec:cDMSO-reg}. Further, we present all the results in terms of a \emph{metric-based} loss function, which is specified by a certain pseudo-metric structure over the decision space $\Pi$.

\begin{definition}[Metric-based loss function]\label{def:metric-cP}
A loss function $L:\MPow\times\Pi\to\R$ is induced by a metric (or simply \emph{metric-based}) if the decision space $\Pi$ can be equipped with a pseudo-metric $\rho$ such that $L(\cP,\pi)=\rho(\pim[\cP],\pi)$, where $\cP\mapsto \pim[\cP]$ is a map from $\MPow$ to $\Pi$. 
\end{definition}

For many applications in statistics, the loss function is naturally metric-based, e.g., hypothesis testing and estimation \citep{casella2002statistical}.

\paragraph{PAC \gDEC~and guarantees}
For any \cDMSO~problem specified by the constraint class $\MPow$, we define the \gDEC~of $\MPow$ with respect to a reference model $\oM$ as
\begin{align}\label{def:p-dec-gen}
    \pdecg_{\eps}(\MPow,\oM)\defeq \infpqb \sup_{\cP\in\MPow} \constr{ \EE_{\pi\sim p} L(\cP,\pi) }{ \inf_{M\in\co(\cP)}\EE_{\bpi\sim q} \DH{ M(\bpi), \oM(\bpi) } \leq \eps^2 },
\end{align}
and $\pdecg_{\eps}(\MPow)=\sup_{\oM\in\bcM} \pdecg_{\eps}(\MPow,\oM)$, where the supremum is taken over the class of reference models $\bcM\defeq \co(\cup_{\cP\in\MPow}\cP)$. 

We now present the first result, which states that under the \cDMSO~framework, \gDEC~provides both lower and upper bounds for the minimax risk. The minimax risk quantifies the fundamental limit of learning, as it measures the best possible performance of an algorithm in the face of a worst-case environment constrained by $\MPow$ (see \pref{sec:cDMSO-main} for details).
\begin{theorem}[PAC lower and upper bounds; Informal]\label{demo:cDMSO-pac}
Let $T\geq 1$, and $L$ be metric-like. Under mild growth assumption, the following holds:
\begin{align*}
    \pdecg_{\ueps(T)}(\MPow)\leqsim \inf_{\alg}\sup_{\env} \EE\sups{\env,\alg}\brac{\riskdm(T)}\leqsim \pdecg_{\oeps(T)}(\MPow),
\end{align*}
where $\inf_{\alg}$ is taken over all $T$-round algorithms $\alg$, $\sup_{\env}$ is taken over all environments $\env$ constrained by $\MPow$, $\ueps(T)\asymp \frac{1}{\sqrt{T}}$, $\oeps(T)\asymp \sqrt{\frac{\log|\MPow|}{T}}$, and we omit poly-logarithmic factors.
\end{theorem}
We note that the lower bound applies to the \emph{stationary} adversaries, while the upper bound (achieved by \ExOp) applies to arbitrary (adaptive) adversarial environments.

Let us now discuss the qualitative behavior of $\pdecg_\eps(\MPow)$ with respect to the constraint class $\MPow$. To start, consider stochastic DMSO, where each constraint is given by a singleton $\cP=\{M\}$. In this case, the infimum over $M\in\co(\cP)$ disappears, recovering the definition of the original PAC DEC in \cite{foster2023tight} (see also \eqref{def:p-dec-c}). As constraints become less stringent (informally, $\cP$'s become larger), the value of the DEC increases as the Hellinger-based constraint becomes easier to satisfy. Similarly, constraints on the learner are also reflected in the Hellinger term through the amount of information the measurements provide, as will be evident in the forthcoming calculations. 

In the rest of this section, we detail how both types of  constraints result in the corresponding measures of complexity and the guarantees for the settings of query-based learning (\cref{ssec:query-demo}), locally differentially private learning (\cref{ssec:LDP-demo}), and robust decision making (\cref{ssec:robust}).

\subsection{Query-based learning}\label{ssec:query-demo}

\newcommand{\cMqb}{\cM^{\mathsf{d}}}
\newcommand{\cMDP}{\cM^{\scriptsize \mathsf{DP}}}
\newcommand{\cMtr}{\cM^{\mathsf{tr}}}

In query-based learning, the environment responds to the learner's \emph{measurements} (or, \emph{queries}) with answers that are close to the answer under the ground-truth model $\Mstar: \Pi\times\Phi\to\cV$, and we recall that we denote $\bPi\defeq \Pi\times\Phi$. 

We formulate the interaction protocol of ($\tau$-correct) \qbDMSO~as follows.
For each $t=1,\cdots,T$:
\begin{itemize}
    \item The learner selects a decision $\pi\ind{t}\in\Pi$ and a measurement $\phi\ind{t}\in \Phi$. %
    \item The environment selects (possibly adversarially) $v_t\in\cV$ such that $\nrm{v_t-\Mstar(\pi\ind{t},\phi\ind{t})}\leq \tau$ and reveals $v_t$ to the learner, where $\cV$ is a fixed normed vector space, and $\tau\geq 0$ is a known tolerance parameter.
\end{itemize}
In \qbDMSO, the underlying model $\Mstar$ is a \emph{deterministic} map $\Pi\times\Phi\to\cV$, and the learner is assumed to have access to a known model class $\cM\subseteq (\Pi\times\Phi\to\cV)$ that contains $\Mstar$.\footnote{The class of all stochastic models is given by $(\bPi\to\Delta(\cV))$, corresponding to noisy responses. We regard $\cM\subset (\bPi\to\Delta(\cV))$.} After $T$ rounds of interaction, the learner selects an output decision $\hpi\sim \phat$ and incurs the PAC risk
\begin{align}\label{eqn:def-risk-sq}
\Riskdm(T) \ldef \EE_{\phat}\brac{ \LM[\Ms]{\hpi} },
\end{align}
where $L:\cM\times\Pi\to \R$ is a given loss function. 
This formulation encompasses the commonly studied \emph{Statistical Query} (SQ) learning~\citep{kearns1998efficient} and its various variants~\citep[etc.]{bshouty2002using,feldman2017general}. Further examples are detailed in \cref{sec:SQ}.

The setting we just described combines constraints on both the learner and the adversary. Indeed, the class $\Phi$ represents constraints on the decision maker, limiting the information it receives. Since answers to the measurements may be imprecise (up to the tolerance level $\tau$), the interaction can be modeled as decision making with a constrained adversary. Before we discuss the details of specializing the \cDMSO~framework, we first present the definition of the DEC specific to query-based learning and its main guarantees.

\newcommand{\ov}{\Bar{v}}
\paragraph{\SQDEC}
For a given model class $\cM\subseteq (\bPi\to\cV)$ and a (randomized) reference model $\oM\in\bcM$, we define the \SQDEC~at $\oM$ as 
\begin{align}\label{def:sq-dec}
    \pdecltau_{\eps}(\cM,\oM)\defeq&~ \infpqb \sup_{M\in\cM}\constr{ \EE_{\pi\sim p}[\LM{\pi}] }{ \PP_{\bpi\sim q, \ov\sim \oM(\bpi)} \paren{ \nrm{M(\bpi)-\ov}>\tau }\leq \eps^2 }.
\end{align}
We further define the \SQDEC~of $\cM$ as $\pdecltau_{\eps}(\cM)=\sup_{\oM}\pdecltau_{\eps}(\cM,\oM)$,
where the supremum is taken over all randomized reference models $\oM:\Phi\to\Delta(\cV)$.

For query-based learning, our main result is given by the following theorem:
\begin{theorem}[\SQDEC~lower and upper bounds; Informal]\label{demo:SQ}
Let $T\geq 1$, $\cM$ be a given model class, and the loss function $L$ be metric-based. Then under certain growth conditions, it holds that
\begin{align*}
    \pdecltau_{\ueps(T)}(\cM)\leqsim \inf_{\alg}\sup_{\env} \EE\sups{\env,\alg}\brac{ \riskdm(T) } \leqsim \pdecltau_{\oeps(T)}(\cM),
\end{align*}
where $\sup_{\env}$ is taken over all  environments satisfying query correctness with tolerance $\tau$ for a model $\Mstar\in\cM$, $\ueps(T)\asymp \frac{1}{\sqrt{T}}$, $\oeps(T)\asymp \sqrt{\frac{\log|\cM|}{T}}$.
\end{theorem}

\paragraph{From \cDMSO~to \qbDMSO}
To frame the ($\tau$-correct) \qbDMSO~within \cDMSO, we can consider the constraint $\cP_{\Mstar}$ specified by a model $\Mstar\in\cM$:
\begin{align}\label{eqn:SQ-cP}
    \cP_{\Mstar}\defeq \set{ M\in\bcM : \forall \bpi\in\bPi, \forall v\in\supp(M(\bpi)), \nrm{v-\Mstar(\bpi)}\leq \tau },
\end{align}
and the constraint class corresponding to $\cM$ is given by $\MPowsq=\set{\cP_{\Mstar}: \Mstar\in\cM}$, with loss function $L(\cP_{\Mstar},\pi)\defeq L(\Mstar,\pi)$. 

While our characterization of query-based learning (\cref{demo:SQ}) is derived by a direct proof (cf. \cref{appdx:SQ}), we can also obtain it by applying \cref{demo:cDMSO-pac}. Specifically, under the above choice~\cref{eqn:SQ-cP}, for any model $M\in\cM$, we have
\begin{align*}
    \inf_{M'\in\co(\cP_M)}\EE_{\bpi\sim q} \DH{ M'(\bpi), \oM(\bpi) }
    \asymp \PP_{\bpi\sim q, v\sim \oM(\bpi)}\paren{ \nrm{M(\bpi)-v}>\tau } ,
\end{align*}
where $\asymp$ here means lower and upper bounds up to constant factors (cf. \cref{lem:Hellinger-to-indic}). %
Hence,
\begin{align*}
    \pdecltau_{\eps/2}(\cM)\leq \pdecg_{\eps}(\MPowsq)\leq \pdecltau_{\eps}(\cM).
\end{align*}
Therefore, under \qbDMSO, the \gDEC~is equivalent to the \SQDEC, and the general guarantees of \cref{demo:cDMSO-pac} apply. Details are postponed to \cref{appdx:gen-to-sq-dec}.

\subsection{Locally differentially private learning}\label{ssec:LDP-demo}

The second example of \cDMSO~is locally differentially private (LDP) learning. 
We first define the differentially private (DP) channels as follows.

\begin{definition}[Differentially private channels]
For the latent observation space $\cZ$ and the noisy observation space $\cO$, a channel $\pr$ is a (measurable) map from $\cZ\to\DO$. %
A channel $\pr$ is \pDP~if for $z, z'\in\cZ$ and any measurable set $E\subseteq \cO$,
\begin{align*}
    \pr(E|z)\leq e^\alpha\pr(E|z').
\end{align*}
\end{definition}
For a fixed pair $(\cZ,\cO)$ of spaces, we denote by $\Pcp$ the class of all \pDP~channels.
To simplify the presentation, we assume that $\alpha\leq \alpha_0$ for a pre-specified universal constant $\alpha_0$, and we will hide dependence on $\alpha_0$.
We also assume the observation space $\cO$ is non-trivial, i.e., $|\cO|\geq 2$.

\paragraph{DMSO with local privacy constraint (\PDMSO)}
We consider the following private variant of the DMSO framework, with the local privacy constraint formalized by a class of private channels $\Pc$. 
For each round $t=1,...,T$: 
\begin{itemize}
  \setlength{\parskip}{2pt}
    \item The learner selects a decision $\pit\in \Pi$ and a private channel $\prt\in\Pc$, where $\Pi$ is the decision space.
    \item The environment generates $\zt\in \cZ$ sampled via $\zt\sim \Mstar(\pit)$, where $\cZ$ is the observation space. %
    \item The learner receives a noisy observation $\ot\in\cO$ sampled via $\ot\sim \prt(\cdot|\zt)$.
\end{itemize}

In private DMSO, the environment is stationary and specified by an underlying model $\Mstar:\Pi\to \DZ$, and the learner is assumed to have access to a known model class $\cM\subseteq (\Pi\to\DZ)$ that contains $\Mstar$. As such, \pDMSO~is encompassed by the stochastic DMSO framework.

In this paper, we focus on $\Pc=\Pcp$, the class of \pDP~channels. We call a $T$-round algorithm as preserving \pLDP~(or simply \pLDP) if it is a learner in the above sense. This formulation is equivalent to the commonly studied model of \emph{sequential LDP channel}~\citep{duchi2018minimax}. Detailed discussion is deferred to \cref{appdx:sequential-channels}.

\paragraph{\PDEC}
Let $\Lc=(\cZ\to[0,1])$ be the class of functions from $\cZ$ to $[0,1]$. For any $\lf\in\Lc$, we define the \lfdiv~between distributions $P,Q\in\DZ$ as
\begin{align}\label{def:Dl}
    \Dl(P,Q)\defeq \abs{ \EE_{z\sim P} [\lf(z)] - \EE_{z\sim Q} [\lf(z)] }.
\end{align}
For a model class $\cM$ and a reference model $\oM\in\coM$, the convex hull of $\cM$, we define \pDEC~at $\oM$ as
\begin{align}
    \pdecl_{\eps}(\cM,\oM)\defeq&~ \inf_{\substack{p\in\DDD\\q\in\DPL}}\sup_{M\in\cM}\constr{ \EE_{\pi\sim p}[\LM{\pi}] }{ \EE_{(\pi,\lf)\sim q} \Dl^2( M(\pi), \oM(\pi) )\leq \eps^2 }, \label{def:p-dec-l}
\end{align}
and the \pDEC~of $\cM$ as $\pdecl_\eps(\cM)=\sup_{\oM\in\coM} \pdecl_{\eps}(\cM,\oM)$.
The \lfdiv is a measure of closeness of two distributions that is weaker than the Hellinger distance from the DEC framework for non-private learning (cf. \eqref{def:p-dec-c}). This divergence is closely connected to the notion of \emph{statistical queries} (SQ), but we postpone this discussion until \cref{ssec:SQ-vs-LDP}.

For learning with LDP constraints, the \pDEC~provides both lower and upper bounds for the expected risk, as stated in the following theorem.

\begin{theorem}[\PDEC~lower and upper bounds; Informal]\label{demo:LDP-pac}
Let $T\geq 1$. If the loss function $L$ is reward-based or metric-based, the following holds:
\begin{align*}
    \pdecl_{\ueps(T)}(\cM)\leqsim \inf_{\alg}\sup_{M\in\cM} \Emalg{\riskdm(T)}\leqsim \pdecl_{\oeps(T)}(\cM),
\end{align*}
where $\inf_{\alg}$ is taken over all $T$-round \pLDP~algorithms, $\ueps(T)\asymp \frac{1}{\sqrt{\alpha^2 T}}$, $\oeps(T)\asymp \sqrt{\frac{\log|\cM|}{\alpha^2 T}}$, and we omit poly-logarithmic factors.
\end{theorem}

\paragraph{Applications} By further specializing the above result, we provide concrete guarantees for various locally-private learning tasks, including  regression (\cref{ssec:regression}) and particularly linear regression (\cref{ssec:linear-regr}). Our lower and upper bounds also provide a tight characterization of the local-minimax complexity under LDP (\cref{ssec:local-minimax}), recovering the characterization in \citet{duchi2024right}. We also provide regret guarantees under LDP constraint, with applications to contextual bandits (\cref{ssec:CBs}), where the contexts can be chosen adversarially by the environment. In particular, we derive a near-optimal $\sqrt{T}$-regret for linear contextual bandits with local privacy through the private DEC theory, settling the open problem of the optimal regret in this setting~\citep{zheng2020locally,han2021generalized,li2024optimal}.

\paragraph{From \cDMSO~to \pDMSO}
For each model $M: \Pi\to\DZ$, $M$ induces a map $\tM:\Pi\times\cQ\to \DO$ given by $\tM(\pi,\pr)=\pr\circ M(\pi)$, where for any channel $\pr\in\Pc$ and any distribution $P\in\DZ$, we denote $\pr\circ P \in\DO$ to be the marginal distribution of $o$ under $z\sim P, o\sim \pr(\cdot|z)$. Therefore, %
the \pDMSO~is encompassed by the \cDMSO~with measurement class $\Phi=\cQ$ and constraint class $\MPowdp=\set{ \set{\tM}: M\in\cM }$ induced by $\cM$. 
Using the strong data-processing inequality (\cref{prop:pLDP}), for any distribution $q\in\Delta(\Pi\times\cQ)$, there exists a distribution $q'\in\DPL$, such that
\begin{align*}
    \EE_{\bpi\sim q} \DH{ \tM(\bpi), \toM(\bpi) } \asymp \alpha^2 \EE_{(\pi,\lf)\sim q'} \Dl^2(M(\pi),\oM(\pi)),
\end{align*}
where $\asymp$ denotes equivalence up to constant factors. Therefore, it holds that
\begin{align*}
    \pdecl_{c_0\alpha\eps}(\cM)\leq \pdecg_{\eps}(\MPowdp)\leq \pdecl_{c_1\alpha\eps}(\cM),
\end{align*}
where $c_0,c_1>0$ are absolute constants. Details are deferred to \cref{appdx:inst-LDP-pac-lower}.

\subsection{Robust decision making}\label{ssec:robust}

We now introduce the following formulation of decision making in the presence of adversarial contamination (or, \textit{robust decision making}). We mainly focus on Huber's contamination model~\citep{huber1965robust,huber2011robust}, as the application to other types of contamination (e.g. model mis-specifications) is analogous.

\paragraph{\RDMSO} Let $\beta\in[0,1]$ be a fixed rate of contamination.
In \rDMSO, the interaction protocol is as follows. For each round $t=1,2,\cdots,T$:
\begin{itemize}
  \setlength{\parskip}{2pt}
    \item The learner selects a decision $\bpi\ind{t}\in \bPi$ from the joint decision space.
    \item The environment generates $o\ind{t}^\star\in \cO$ sampled via $o\ind{t}^\star\sim \Mstar(\bpi\ind{t})$. %
    \item With probability $1-\beta$, the environment reveals $o\ind{t}=o\ind{t}^\star$ to the learner. Otherwise, the environment selects $o\ind{t}\in\cO$ arbitrarily (potentially depending on the interactions up to round $t$).
\end{itemize}

Similar to \pDMSO, we assume the ground truth model $\Mstar$ belongs to a given model class $\cM\subseteq (\bPi\to\DO)$.
In the formulation above, the environment is allowed to be adaptive, making the learning task harder than the Huber contamination model~\citep{huber1965robust,huber1992robust}, where the environment is \emph{stationary},
i.e., $M^1=\cdots=M^T=(1-\beta)\Mstar+\beta M'$ for an arbitrary but fixed contamination model $M'$ (cf. \cref{def:stationary}). Indeed, the environment under the Huber contamination model falls within the purview of the stochastic DMSO framework. Further discussion is deferred to \cref{ssec:robust-more}.

To frame the above setting within \cDMSO, we can consider the constraint specified by a model $\Mstar\in\cM$: 
\begin{align}\label{eqn:robust-cP}
    \cP_{\Mstar}\defeq \sset{ (1-\beta)\Mstar+\beta M': M'\in(\bPi\to\DO) },
\end{align}
and the constraint class (induced by $\cM$) as given by $\MPowrob\defeq \set{ \cP_{\Mstar}: \Mstar\in\cM }$, with loss function $L(\cP_{\Mstar},\pi)=L(\Mstar,\pi)$. Then, the \rDMSO~described above is exactly \cDMSO~with constraint class $\MPowrob$. %
By instantiating the general theory in \cref{ssec:cDMSO}, we arrive at the following (simpler) DEC formulation for robust decision making.
\newcommand{\Malt}{M_{\sf alt}}
\newcommand{\cMbeta}{\cM_{\beta}^+(M)}
\newcommandx{\Dr}[3][1=2,2=\beta]{D_{\beta\text{-}\mathsf{Huber}}^{#1}\paren{#3}}
\paragraph{Robust DEC}
For $\beta\in[0,1]$ and distributions $P, Q\in\DO$, we consider the $\beta$-perturbed Hellinger divergence
\begin{align}\label{def:beta-perturb-H}
    \Dr{P,Q}\defeq \inf_{P'\in\DO} \DH{ (1-\beta)P+\beta P', Q }.
\end{align}
For a model class $\cM$ and a reference model $\oM$, we define robust DEC at $\oM$ as 
\begin{align}
    \pdecr_{\eps}(\cM,\oM)\defeq&~ \infpqb \sup_{M\in\cM}\constr{ \EE_{\pi\sim p}[\LM{\pi}] }{ \EE_{\bpi\sim q} \Dr{ M(\bpi), \oM(\bpi) }\leq \eps^2 }, \label{def:p-dec-r}
\end{align}
and the robust DEC of $\cM$ is then defined as $\pdecr_\eps(\cM)=\sup_{\oM\in\bcM} \pdecr_{\eps}(\cM,\oM)$. 

In the definition of the robust DEC, we replace the Hellinger distance by the perturbed divergence~\cref{def:beta-perturb-H}, reflecting the fact that for a ground truth model $M$, the environment can vary $M$ by a probability mass $\beta$. By definition, we know $\pdecr_\eps(\cM)=\pdecg_\eps(\MPowrob)$ for $\eps\in[0,1]$ (detailed in \cref{ssec:robust-more}). Therefore, as a direct corollary of \cref{demo:cDMSO-pac}, the robust DEC provides both lower and upper bounds for robust PAC learning.
\begin{theorem}[Robust risk bounds; Informal]\label{demo:robust-pac}
Let $T\geq 1$, $\cM$ be a given model class, and the loss function $L$ be metric-based. Then under certain growth conditions, it holds that
\begin{align*}
    \pdecr_{\ueps(T)}(\cM)\leqsim \inf_{\alg}\sup_{\env} \EE\sups{\env,\alg}\brac{ \riskdm(T) } \leqsim \pdecr_{\oeps(T)}(\cM),
\end{align*}
where $\sup_{\env}$ is taken over all environments that are $\beta$-contaminated from a model $\Mstar\in\cM$, $\ueps(T)\asymp \frac{1}{\sqrt{T}}$, $\oeps(T)\asymp \sqrt{\frac{\log|\cM|}{T}}$.
\end{theorem}

The details and regret guarantees are presented in \cref{ssec:robust-more}.

\subsection{Regret guarantees for \cDMSO}\label{ssec:cDMSO-reg}

In this section, we study the no-regret learning goal under \cDMSO, and present the general regret guarantees and its application to smooth environments.

In the no-regret learning task, the performance of the learner is measured by the following notion of \emph{regret}:
\begin{align}\label{eqn:def-reg-cDMSO}
    \regdm(T)\defeq \max_{\pis\in\Pi} \sum_{t=1}^T \Vm[M^t](\pis)-\EE_{\pi\ind{t}\sim 
    \qt} \Vm[M^t](\pi\ind{t}),
\end{align}
where for each model $M$, $\Vm:\Pi\to\R$ is an associated value function. In words, $\regdm(T)$ measure the performance of the learner compared to the best decision in the hindsight.
Note that due to the adversarial nature of the environment, the PAC risk \pref{eqn:def-risk-Ps} cannot be directly reduced from the regret \pref{eqn:def-reg-cDMSO} by the \emph{online-to-batch} conversion.

We first extend the regret DEC~\citep{foster2023tight} to \cDMSO. For any model class $\cM\subseteq (\bPi\to\DO)$, reference model $\oM$, we define
\begin{align}\label{def:r-dec-c}
    \rdecc_{\eps}(\cM,\oM)\defeq&~ \infpb \sup_{M\in\cM}\constr{ \EE_{\pi\sim p}[\Vmm-\Vm(\pi)] }{ \EE_{\bpi\sim p} \DH{ M(\bpi), \oM(\bpi) }\leq \eps^2 },
\end{align} 
where for each model $M$, $\pim\defeq \argmax_{\pi\in\Pi} \Vm(\pi)$ is an optimal decision under $M$.
The regret DEC of $\cM$ is then defined as $\rdecc_{\eps}(\cM)=\sup_{\oM\in\coM} \rdecc_{\eps}(\cMp,\oM).$ 

Next, to define the regret DEC of a constraint class $\MPow$, we define 
\begin{align*}
    \cMPow\defeq \bigcup_{\cP\in\MPow} \co(\cP), \qquad
    \rdecg_\eps(\MPow)\defeq \rdecc_\eps(\cMPow).
\end{align*}

We show that the regret DEC of $\MPow$ provides both lower and upper bound for the minimax regret. 
\begin{theorem}[Regret lower and upper bounds; Informal]\label{demo:cDMSO-regret}
Let $T\geq 1$. Under assumptions on the value function and the growth of the DEC, the following holds:
\begin{align*}
    T\cdot \rdecc_{\ueps(T)}(\cMPow)\leqsim \inf_{\alg}\sup_{\env} \EE\sups{\env,\alg}\brac{\regdm(T)}\leqsim T\cdot \rdecc_{\oeps(T)}(\cMPow),
\end{align*}
where $\inf_{\alg}$ is taken over all $T$-round algorithms $\alg$, $\sup_{\env}$ is taken over all environments $\env$ constrained by $\MPow$, $\ueps(T)\asymp \frac{1}{\sqrt{T}}$, $\oeps(T)\asymp \sqrt{\frac{\log|\MPow|+\log|\Pi|}{T}}$, and we omit poly-logarithmic factors.
\end{theorem}
In particular, when the environment is fully adversarial, $\MPow=\set{\cM}$ is a singleton, $\cMPow=\coM$, and we recover the results of \citet{foster2022complexity}. Furthermore, the $\log|\Pi|$ factor in our upper bound can further be tightened by the \dct~\citep{chen2024beyond} (cf. \cref{ssec:cDMSO-reg-full}).

As an application, we consider no-regret learning against smooth adversaries. The results for robust no-regret learning are deferred to \cref{ssec:robust-more}.

\subsubsection{Example: Smooth adversaries}\label{ssec:smooth} 

\newcommand{\cMbase}{\cM_{\mathsf{base}}}
\newcommand{\Mbase}{{M_{\mathsf{base}}}}
\newcommand{\cMsm}{\cM_{\mathsf{sm}}}
\newcommand{\Dsm}[2]{D_{\infty}\ddiv{#1}{#2}}
\newcommand{\Cov}{{\sigma_{\mathsf{sm}}}}

Within the \cDMSO~framework, we can also consider decision making with a smooth adversary. In this setting, we focus on the case where $\Phi=\set{\perp}$, i.e., only the identity measurement is considered. 

For any two distributions $P, Q\in\DZ$, we define the density ratio between $P,Q$ as
\begin{align*}
    \Dsm{P}{Q}\defeq \mathrm{ess}\sup \frac{dP}{dQ}.
\end{align*}
We say $P$ is $\Cov$-smooth with respect to $Q$ if $\Dsm{P}{Q}\leq \frac1\Cov$.

For the setting of smooth adversary, we assume there is a known subclass $\cMbase\subseteq (\Pi\to\DZ)$, such that the adversary is constrained to fix a \emph{base model} $\Mbase$ ahead of the interaction and without revealing it to the learner, and then choose each $M^t$ that is $\Cov$-smooth with respect to $\Mbase$. Specifically, for each \emph{base model} $\Mbase\in\cMbase$, the constraint specified by $\Mbase$ is
\begin{align}
\cP_{\Mbase}\defeq \sset{ M\in\cM: \sup_{\pi\in\Pi} \Dsm{ M(\pi)}{\Mbase(\pi)} \leq \frac1\Cov },
\end{align} 
which is the class of all models that are $\Cov$-smooth with respect to the base model $\Mbase$. Specifying the \cDMSO~framework with $\MPow=\set{\cP_{\Mbase}: \Mbase\in\cMbase}$, we generalize the standard smooth online learning setting to interactive decision making.

Note that for each $\Mbase$, the class $\cP_\Mbase$ is convex, and hence we let
\begin{align*}
    \cMsm\defeq \cM_{\MPow} = \bigcup_{\Mbase\in\cMbase} \cP_\Mbase.
\end{align*}
It is a direct corollary of \cref{demo:cDMSO-regret} that the regret DEC of $\cMsm$ provides both lower and upper bounds for no-regret learning against smooth adversaries.
\begin{theorem}[Regret bounds against smooth adversaries; Informal]\label{demo:smooth}
Let $T\geq 1$. Under assumptions on the value function and the growth of the DEC, the following holds:
\begin{align*}
    T\cdot \rdecc_{\ueps(T)}(\cMsm)\leqsim \inf_{\alg}\sup_{\env} \EE\sups{\env,\alg}\brac{\regdm(T)}\leqsim T\cdot \rdecc_{\oeps(T)}(\cMsm),
\end{align*}
where $\sup_{\env}$ is taken over all environments $\env$ constrained to be $\Cov$-smooth with respect to a base model $\Mbase\in\cMbase$, $\ueps(T)\asymp \frac{1}{\sqrt{T}}$, $\oeps(T)\asymp \sqrt{\frac{\log|\cMbase|+\log|\Pi|}{T}}$.
\end{theorem}

\section{DEC Theory for Hybrid DMSO}\label{sec:cDMSO-main}
In this section, we present the details of the DEC theory for \cDMSO. Before proceeding to the main results, we rigorously formulate the notion of \emph{algorithms} and \emph{environments}.

\newcommandx{\mut}[1][1=t]{\mu^{#1}}

A $T$-round algorithm $\alg$ is specified by a sequence of mappings $\set{\qt}_{t\in [T]}\cup\set{\phat}$, where the $t$-th mapping $\qt(\cdot\mid{}\Hy\ind{t-1})$ specifies the distribution of $\bpi\ind{t}=(\pi\ind{t},\phi\ind{t})$ %
based on the history $\Hy\ind{t-1}=(\bpi\ind{s},o\ind{s})_{s\leq t-1}$, and the final map $\phat(\cdot \mid{} \Hy\ind{T})$ specifies the distribution of the \emph{output decision} $\pihat$ based on $\Hy\ind{T}$. Similarly, a $T$-round \emph{adaptive} environment $\env$ is specified by a sequence of mappings $\set{\mut}_{t\in [T]}$, where the $t$-th mapping $\mut(\cdot\mid{}\Hy\ind{t-1}')$ specifies the distribution of the model $M^t$ 
based on the full-information history $\Hy\ind{t-1}'=(M^s,\bpi\ind{s},o\ind{s})_{s\leq t-1}$. An environment is constrained by $\MPow$ if there exists $\cPs\in\MPow$ such that $\mut(\cdot\mid{}\Hy\ind{t-1}')$ is always supported on $\cPs$ for all $t\in[T]$. As already discussed, each model $M\in(\bPi\to\DO)$ corresponds to a stationary environment, which chooses $M^1=\cdots=M^T=M$ deterministically.

For any algorithm $\alg$ and environment $\env$, we let $\PP\sups{\env,\alg}\paren{\cdot}$ to be the distribution of $(\cH\ind{T},\hpi)$ generated by the algorithm $\alg$ under the environment $\env$, and let $\EE\sups{\env,\alg}[\cdot]$ to be the corresponding expectation.
In particular, for any model $M$, we let $\PP\sups{M,\alg}\paren{\cdot}$ to be the distribution of $(\cH\ind{T},\hpi)$ generated by the algorithm $\alg$ under the stationary environment that chooses $M^t=M$ for $t\in[T]$, and let $\EE\sups{M,\alg}[\cdot]$ to be the corresponding expectation.

\paragraph{Miscellaneous notation}
For a model class $\cM\subseteq (\bPi\to\DO)$, a finite subset $\cM_0\subseteq \cM$ is an $\eps$-covering of $\cM$ if for any model $M\in\cM$, there exists $M'\in\cM_0$ such that $\DHr{M(\bpi),M'(\bpi)}\leq \eps, \forall \bpi\in\Pi$. We define $\NM{\cM}$, the $\eps$-covering number of $\cM$, to be the minimal cardinality of the $\eps$-coverings of $\cM$. 

For the upper bounds in this section, we assume the model class $\cM=\cMPow$ admits finite $\eps$-covering for any $\eps>0$, ensuring that the Minimax theorem can be applied. 
\begin{assumption}[Compactness of the model class]\label{asmp:finite}
For any $\Delta>0$, the covering number $\NM{\Delta}$ is finite. 
\end{assumption}
Further, to simplify the presentation, we consider the following growth condition (following~\citet{foster2023tight,chen2024beyond}). 

\begin{definition}
[Moderate decay]\label{asmp:reg}
A function $\rmd: [0,1]\to\R_{\geq 0}$ is of \emph{moderate decay} if there exists a constant $c\geq 1$ such that $c\frac{\rmd(\eps)}{\eps}\geq \frac{\rmd(\eps')}{\eps'}$ for all $\eps'\geq \eps$.
\end{definition}
For many problems of interest, the DECs grow as $C\eps^{\rho}$ with $\rho\leq 1$ and is automatically of moderate decay (for details, see e.g. \citet{foster2023tight}).

\subsection{Guarantees for PAC learning}\label{ssec:cDMSO-pac-full}

\paragraph{PAC DEC lower bounds} 
To better illustrate the key observation for the \gDEC~lower bounds, we first introduce the notion of the \emph{stationary} adversary. 
\begin{definition}\label{def:stationary}
For an environment $\env$ constrained by $\MPow$, $\env$ is \emph{stationary} if there exists $\cPs\in\MPow$ and $\mu\in\Delta(\cPs)$ such that for each step $t\in[T]$, the model $M^t$ is chosen as $M^t\sim \mu$.
\end{definition}
In other words, in an stationary environment, the model $M^t\sim \mu$ is chosen independently of prior interactions. The key observation of \citet{foster2022complexity} is that lower bounds for adversarial DMSO can implied by the stochastic lower bounds~\citep{foster2021statistical,foster2023tight} by considering stationary environments, as stationary environments can be described by stochastic DMSO. This argument also applies to \cDMSO, implying the following lower bounds. The proof is deferred to \cref{appdx:DEC-general-lower}.

\begin{theorem}[Hybrid DEC lower bound for PAC risk]\label{thm:cDMSO-pac-lower}
Suppose that $L$ is metric-like. Then, for any $T$-round algorithm, 
\begin{align}\label{eq:GDEC-pac-lower}
    \sup_{\env}\EE\sups{\env,\alg}\brac{ \riskdm(T) }\geq \frac18\pdecg_{\ueps(T)}(\MPow), 
\end{align}
where $\ueps(T)=\frac{1}{20\sqrt{T}}$ and the supremum is taken over \emph{stationary} environments. 

Furthermore, for general loss function $L:\MPow\times\Pi\to [0,1]$, any $T$-round algorithm $\alg$, parameter $\delta\in(0,1]$, it holds that
\begin{align}\label{eq:GDEC-pac-lower-bad}
    \sup_{\env} \EE\sups{\env,\alg}\brac{\riskdm(T)}\geq \pdecg_{\ueps_\delta(T)}(\MPow) - \delta,
\end{align}
where $\ueps_\delta(T)\defeq \frac{1}{13}\sqrt{\frac{\delta}{T}}$.
\end{theorem}

We now briefly discuss the two lower bounds in \cref{thm:cDMSO-pac-lower}. \eqref{eq:GDEC-pac-lower} is stated for metric-based loss, and it nearly matches the upper bound provided in \cref{thm:cDMSO-pac-upper} (with a $\log|\MPow|$-gap). On the other hand, \eqref{eq:GDEC-pac-lower-bad} is stated for \emph{any} general loss function (without requiring metric structure) and it is looser.
It can be further re-written as
\begin{align}\label{eqn:GDEC-pac-lower-bad-2}
    \sup_{\env} \EE\sups{\env,\alg}\brac{\riskdm(T)}\geq \frac12\pdecg_{\ueps(T)}(\MPow), \qquad \text{where }
    T\ueps(T)^2\asymp \pdecg_{\ueps(T)}(\MPow).
\end{align}
For a problem with $\pdecg_\eps(\MPow)\asymp \eps$, \eqref{eqn:GDEC-pac-lower-bad-2} gives a lower bound of $\Om{\frac{1}{T}}$. While this is worse than the $\Om{\frac{1}{\sqrt{T}}}$ lower bound provided by \cref{eq:GDEC-pac-lower} under metric-based loss, such a worse lower bound can be tight for certain problems (as shown in~\citet{foster2023complexity}). We also note that under stochastic DMSO and \emph{reward-based} loss function, a tighter lower bound similar to \cref{eq:GDEC-pac-lower} can also be derived~\citep{foster2023tight} (see also \cref{appdx:DEC-general-lower}).

\paragraph{PAC DEC upper bounds}
Next, we present the upper bound provided by \ExOp~(\cref{alg:ExO}) as follows. The description of \ExOp~is deferred to \cref{appdx:ExO}. For the simplicity of presentation, we still assume that the loss function is metric-like. While \ExOp~is able to handle more general problems (and in particular reward-based loss function), we defer these details to \cref{appdx:ExO}.
\begin{theorem}[Hybrid DEC upper bound for PAC risk]\label{thm:cDMSO-pac-upper}
Let $T\geq 1$, $\delta\in(0,1)$, and $\MPow$ be given. Suppose that $L$ is metric-like, $\cMPow$ is compact (\cref{asmp:finite}), and the \gDEC~$\pdecg_\eps(\MPow)$ is of moderate decay (\cref{asmp:reg}). Then \ExOp~can be suitably instantiated (as detailed in \cref{appdx:proof-cDMSO-pac-upper}), such that in any environment constrained by $\MPow$, \ExOp~achieves \whp~that
\begin{align*}
    \riskdm(T)\leqsim  \pdecg_{\oeps(T)}(\MPow),
\end{align*}
where $\oeps(T)=\sqrt{\frac{\log(|\MPow|/\delta)}{T}}$.

Furthermore, for any $\Delta>0$, any general loss function $L$ bounded in $[0,1]$, \ExOp~can be suitably instantiated so that in any environment constrained by $\MPow$, \ExOp~achieves \whp~that
\begin{align}\label{eqn:GDEC-pac-upper-bad}
    \riskdm(T)\leqsim  \pdecg_{\oeps(\Delta T)}(\MPow)+\Delta.
\end{align}
\end{theorem}

\subsection{Guarantees for no-regret learning}\label{ssec:cDMSO-reg-full}

In this section, we consider no-regret learning in \cDMSO. To present the DEC theory in its simplest form, we make the following assumption, which essentially requires that the value of any decision can be estimated from observations.

\begin{assumption}[Observability]\label{asmp:lip-rew}
For any decision $\pi\in\Pi$, the map $M\mapsto \Vm(\pi)\in[0,\Vmax]$ is linear over $\cMPow$, and there exists a measurement $\phi_\pi\in\Phi$, such that for $\bpi=(\pi,\phi_\pi)\in\bPi$, 
\begin{align}\label{eqn:lip-rew-pr}
    \abs{ \Vm(\pi)-\Vm[\oM](\pi) }\leq \Lipr\dH\paren{M(\bpi),\oM(\bpi)},\qquad \forall M\in\cMPow, \oM\in\co(\cMPow).
\end{align}
\end{assumption}

\cref{asmp:lip-rew} can also be regarded as a continuity assumption on the value function. 

To better illustrate \cref{asmp:lip-rew}, we consider the example of identity measurement and reward-based value function. This setting is extensively studied in \citet[etc.]{foster2021statistical,foster2022complexity,foster2023tight}. 
\begin{example}[Reward-based learning]\label{example:original-reward-based}
Suppose that $\MPow$ is induced by a model class $\cM\subseteq (\Pi\to \DZ)$, where $\cZ\subseteq \cO$ and the measurement class is $\Phi=\set{\perp}$, i.e. we overload the notation and write $M(\pi,\perp)=M(\pi)$ for each model $M\in\cM$.
In this setting, the value function $V$ is \emph{reward-based}, if there is a \emph{known} reward function $R:\cZ\times\Pi\to [0,1]$ such that $\Vm(\pi)=\EE\sups{M,\pi}[R(z,\pi)]$.
\end{example}

This formulation encompasses many learning settings of interest, including bandits and contextual bandits, online control, reinforcement learning, etc. (for examples, see e.g. \citet{foster2021statistical}).
In this setting, \cref{asmp:lip-rew} holds with $\Lipr=\sqrt{2}$. We also note that for reward-based LDP learning, \cref{asmp:lip-rew} holds with $\Lipr=\bigO{\frac{1}{\alpha}}$ (as detailed in \cref{appdx:inst-LDP-pac-lower}).

\paragraph{Regret lower bound}
With \cref{asmp:lip-rew}, we now present the main regret lower bound.
\begin{theorem}[Hybrid DEC lower bound for regret]\label{thm:cDMSO-reg-lower}
Let $T\geq 1$, $\MPow$ be a given constraint class. Suppose that \cref{asmp:lip-rew} holds for the value function $V$. Then, for any $T$-round algorithm $\alg$, 
\begin{align}\label{eq:GDEC-reg-lower}
    \sup_{\env}\EE\sups{\env,\alg}\brac{\regdm(T)}\geq&~ \sup_{\env}\sup_{\pis\in\Pi}\EE\sups{\env,\alg}\brac{\sum_{t=1}^T \Vm[M^t](\pis)-\Vm[M^t](\pit)} \\
    \geq&~ \frac{T}{8}\paren{\rdecc_{\ueps(T)}(\cMPow)-6\Lipr\ueps(T)-\frac{\Vmax}{T}}
\end{align}
where the supremum is taken over \emph{stationary} environments $\env$ constrained by $\MPow$, and $\ueps(T)=\frac{1}{24\sqrt{T}}$. 
\end{theorem}
Similar to \cref{thm:cDMSO-pac-lower}, the above regret lower bound is also proven through a reduction to the stochastic setting by considering stationary environments (detailed in \cref{appdx:proof-cDMSO-reg-lower}).

\paragraph{Regret upper bound}
Before presenting the upper bound, we first introduce the notion of the \emph{\dct}~\citep{chen2024beyond}, which captures the complexity of the decision space $\Pi$ with respect to the class of models $\cM$. 
\begin{definition}[\Dct]\label{def:DC}
For a learning problem $(\cM,\Pi,L)$ and parameter $\Delta\geq 0$, we define the \dct~as
\begin{align}\label{def:DD}
    \DC{\LP}\defeq\inf_{p\in\DDD}\sup_{M\in\cM} ~~\frac1{p(\pi: \LM{\pi}\leq \Delta)}.
\end{align}
\end{definition}

We show that the regret of \ExOp~can be upper bounded in terms of the regret DEC, the \dct, and $\log|\MPow|$.
\begin{theorem}[Hybrid DEC upper bound for regret]\label{thm:cDMSO-reg-upper}
Let $T\geq 1$, $\delta\in(0,1)$, and $\MPow$ be given. Suppose that $\cMPow$ is compact, \cref{asmp:lip-rew} holds for the value function $V$, and the regret DEC $\rdecc_\eps(\cMPow)$ is of moderate decay. Then, in any environment constrained by $\MPow$, \ExOp~(instantiated as detailed in \cref{appdx:proof-cDMSO-reg-upper}) achieves \whp~that
\begin{align*}
    \frac1T\regdm(T)\leq \Delta+ \OsqrtT\cdot \brac{ \rdecc_{\oeps(T)}(\cMPow) + \Lipr\oeps(T) },
\end{align*}
where $\oeps(T)=\sqrt{\frac{\log(|\MPow|/\delta)+\log \DC{\cMPow}}{T}}$.
\end{theorem}

Finally, we remark that both our lower and upper bounds extend beyond \cref{asmp:lip-rew}, as detailed in \cref{appdx:proof-cDMSO-reg-lower} and \cref{appdx:ExO}.

\subsection{Implication: Tighter bounds for convex classes}\label{ssec:tigher}

\newcommand{\cMi}[1]{\cM^{(#1)}}

Our results for \cDMSO~also have interesting implications for stochastic DMSO. To begin with, we recall that for a model class $\cM\subseteq (\bPi\to \DO)$ under stochastic DMSO, the PAC DEC is defined as
\begin{align}\label{def:p-dec-c}
    \pdecc_{\eps}(\cM,\oM)\defeq&~ \infpqb \sup_{M\in\cM}\constr{ \EE_{\pi\sim p}[\LM{\pi}] }{ \EE_{\bpi\sim q} \DH{ M(\bpi), \oM(\bpi) }\leq \eps^2 },
\end{align}
and $\pdecc_\eps(\cM)\defeq\sup_{\oM\in\coM}\pdecc_\eps(\cM,\oM)$. 
DEC theory~\citep{foster2021statistical,foster2023tight} provides the following characterization (omitting logarithmic factors): 
\begin{align}
    \pdecc_{\ueps(T)}(\cM)\leqsim \inf_{\alg}\sup_{M\in\cM} \EE\sups{M,\alg}\brac{\riskdm(T)} \leqsim \pdecc_{\oeps(T)}(\cM),
\end{align}
under certain regularity assumptions on the loss function, where $\ueps(T)\asymp \frac{1}{\sqrt{T}}$, $\oeps(T)\asymp \sqrt{\frac{\log|\cM|}{T}}$. Therefore, a $\log|\cM|$-gap remains between the known DEC lower and upper bounds, corresponding to the complexity of \emph{estimation}, as noted by \citet{chen2024beyond}. The $\log|\cM|$ factor can be undesirable for many applications beyond model-based learning. 

Interestingly, it turns out the $\log|\cM|$ factor can be replaced by a smaller quantity, potentially at the price of degradation in the DEC term.
To illustrate this, we start with the \emph{hypothesis selection problem}, which is a generalization of the standard, non-interactive hypothesis testing problem. For example, the setting below encompasses LDP hypothesis selection, where $\Phi=\cQ$ is the class of \pDP~channels, and $\cM$ is induced by a class of distributions over $\DZ$.
\begin{example}[Interactive hypothesis selection]\label{example:multi-hypothesis}
Given a DMSO model class $\cM\subseteq (\Phi\to \DO)$, a hypothesis selection problem is described by a partition
\begin{align*}
    \cM=\bigsqcup_{i=1}^{m} \cMi{i},
\end{align*}
where $\cMi{1},\cdots,\cMi{m}$ are disjoint subclasses.
The decision space is $\Pi=[m]$,
and for each $M\in \cM$, $\pi\in\Pi$, the \losst~is given by $\LM{\pi}=\indic{ \pi\neq \pim }$, where $\pim$ is the unique index $i\in[m]$ such that $M\in\cMi{i}$.
\end{example}

While we can frame the hypothesis selection problem within stochastic DMSO (with $\MPowiid$ corresponding to $\cM$), the upper bound provided by DEC theory scales with $\log|\cM|$, the complexity of model class, which is undesirable. On the other hand, when the subclasses $\cMi{1},\cdots,\cMi{m}$ are convex, we can alternative frame this problem within \cDMSO, with $\MPow_m=\sset{ \cMi{1},\cdots,\cMi{m} }$ and loss function $L(\cMi{i},\pi)=\indic{\pi\neq i}$. With such specifications, we allow the environment to be adaptive (within a fixed underlying model class $\cMi{i}$), while
\begin{align*}
    \pdecg_\eps(\MPow_m)=\pdecc_\eps(\cM), \qquad \forall \eps>0.
\end{align*}
Therefore, \cref{thm:cDMSO-pac-upper} implies the following tighter upper bound for hypothesis selection.
\begin{proposition}[\ExOp~for convex hypothesis selection]\label{prop:multi-hypothesis}
Let $T\geq 1$, $\delta\in(0,1)$. In \cref{example:multi-hypothesis}, suppose that $\cM$ is compact, $\cMi{1},\cdots,\cMi{m}$ are convex, and
\begin{align*}
    \pdecc_{\oeps(T)}(\cM)\leq \frac{1}{3}, \qquad \oeps(T)=8\sqrt{\frac{\log(m/\delta)}{T}}.
\end{align*}
Then \ExOp~ can be suitably instantiated (on the constraint class $\MPow_m$, as detailed in \cref{appdx:proof-multi-hypothesis}), so that under any model $\Mstar\in\cM_{i^\star}$, the algorithm returns $\hpi=i^\star$ \whp.
\end{proposition}

In the above example, \cref{thm:cDMSO-pac-upper} naturally provides a tighter bound by considering \cDMSO~and replacing the $\log|\cM|$-factor by $\log m$. In general, such conversion will result in a degradation in the DEC term, if the model class is non-convex. In the following, we will make this trade-off precise.

\paragraph{Bounds for interactive estimation}
In the interactive estimation task, the decision space $\Pi$ is equipped with a pseudo-metric $\rho$, and a map $M\mapsto \pim$ is given such that $L(M,\pi)=\rho(\pim,\pi)$. To apply the idea described above, we fix a parameter $\Delta\geq 0$ and consider the constraint set specified by a $\pi\in\Pi$:
\begin{align*}
    \cM_\pi\defeq \set{M: \rho(\pim,\pi)\leq \Delta},
\end{align*}
and the corresponding constraint class is $\MPow=\set{ \cM_\pi: \pi\in\Pi_\Delta },$ with  $\Pi_\Delta$ being a $\Delta$-covering of the set $\Pi_{\cM}\defeq \set{ \pim: \pi\in\Pi }$. Then, $\cM=\bigcup_{\cP\in\MPow} \cP$, and hence we can apply \cref{thm:cDMSO-pac-upper}. Furthermore, assuming that $\cM$ is convex and $M\mapsto \rho(\pim,\pi)$ quasi-convex for any $\pi\in\Pi$, then we can show that
\begin{align*}
    \pdecg_\eps(\MPow)\leq \Delta+\pdecc_\eps(\cM).
\end{align*}
Therefore, \cref{thm:cDMSO-pac-upper} implies the following guarantee for interactive estimation (for affine functionals).
\begin{proposition}[\ExOp~for interactive estimation]\label{prop:interactive-est}
Let $T\geq 1$, $\delta\in(0,1)$, $\Delta\geq 0$. Suppose that $\cM$ is convex, $\Pi$ is a subset of a normed vector space, and an affine map $M\mapsto \pim$ is given such that $L(M,\pi)=\nrm{\pim-\pi}$. Further assume that the DEC $\pdecc_\eps(\cM)$ is of moderate decay. Then \ExOp~can be suitably instantiated so that \whp, it returns $\hpi$ with
\begin{align*}
    \riskdm(T)=\rho(\pim, \hpi)\leqsim \Delta+\pdecc_{\oeps(T)}(\cM),
\end{align*}
where $\oeps(T)=\sqrt{\frac{\log N_{\nrm{\cdot}}(\Pi,\Delta)+\log(1/\delta)}{T}}$, and $N_{\nrm{\cdot}}(\Pi,\Delta)$ is the $\Delta$-covering number of $\Pi$ under the norm $\nrm{\cdot}$. 
\end{proposition}
In particular, for bounded functional estimation, $\Pi=[0,1]$, we have $\log N_{\abs{\cdot}}(\Pi,\Delta)=\log(1/\Delta)+\bigO{1}$, and hence the minimax risk of interactive functional estimation is characterized by the DEC up to logarithmic factors. This upper bound generalizes the results of \citet{polyanskiy2019dualizing} for \emph{non-interactive} linear functional estimation with a convex model class.

\paragraph{Bounds for reward-based learning}
Generalizing the above idea, we consider the reward-based no-regret learning task (as per \cref{example:original-reward-based}) in stochastic DMSO and frame this task in \cDMSO. Fix a parameter $\Delta>0$ of sub-optimality, we can consider the following ``relaxed'' constraint for each $\pi\in\Pi$:
\begin{align}\label{def:MPow-pol-based}
    \cM_\pi\defeq \set{ M\in\cM: \Vmm-\Vm(\pi)\leq \Delta },
\end{align}
and the corresponding ``relaxed'' constraint class $\MPow\defeq \set{ \cM_\pi: \pi\in\Pi }$.
For clarity, we write $\cM_\Pi\defeq \bigcup_{\pi\in\Pi} \co(\cM_\pi)$. Then, \cref{thm:cDMSO-reg-upper} implies that \ExOp~can achieve an upper bound in terms of the regret DEC of $\cM_\Pi$. Following this idea and using a slightly more careful instantiation of \ExOp, we have the following upper bounds.

\begin{proposition}\label{prop:regret-via-cDMSO}
Let $T\geq 1, \delta\in(0,1)$, $\Delta\geq 0$, and we consider the reward-based no-regret learning task (\cref{example:original-reward-based}) with a model class $\cM$. Suppose that $\cM$ \finite, and the regret DEC $\rdecc_{\eps}(\cM_\Pi)$, as a function of $\eps$, is of moderate decay. Then \ExOp~%
can be suitably instantiated (as detailed in \cref{appdx:proof-regret-via-cDMSO}) to achieve \whp~that
\begin{align*}
    \frac{1}{T}\regdm(T)\leq \Delta+\OsqrtT \cdot \brac{  \rdecc_{\oeps(T)}(\cM_\Pi) + \oeps(T) },
\end{align*}
where $\oeps(T)=\sqrt{\frac{\log \DC{\cM}+\log(1/\delta)}{T}}$. 
\end{proposition}

We note that $\cM_\Pi\subseteq \coM$, and hence when the model class $\cM$ is convex, the above upper bound in fact scales with the regret-DEC and \dct~of $\cM$.
We also note that \cref{prop:regret-via-cDMSO} is not immediately implied by \cref{thm:cDMSO-reg-upper}, because the latter also involves a term $\log|\MPow|=\log|\Pi|$, which can be much larger than $\log \DC{\cM}$. However, only slight adaptions specific to stochastic DMSO are needed (as detailed in \cref{appdx:ExO}).

While guarantees of this form were first obtained by \citet{chen2024beyond}, their bounds are directly reduced from \citet{foster2022complexity} and scale with the DEC of $\coM$ (corresponding to the fully adversarial setting). In contrast, our framework provides finer upper bounds and has broader applicability, including convex hypothesis selection (\cref{prop:multi-hypothesis}), interactive estimation (\cref{prop:interactive-est}), and also private regression (\cref{prop:real-regression}).

\section{Query-Based Learning}\label{sec:SQ}

In this section, we employ our framework to provide characterization for any query-based learning problem (\cref{ssec:SQ-main}). In particular, for learning under the \emph{Statistical Queries} (SQ)~\citep{kearns1998efficient}, the corresponding DEC recovers the \emph{SQ dimension} of \citet{feldman2017general}, which is shown to provide both lower and upper bounds for the distributional search problems (\cref{ssec:SQ-dim}). We also discuss the connection between SQ learning and LDP learning through the lens of our DEC formulation.

\paragraph{Background on SQ learning}
The commonly studied setting of SQ learning is the \emph{distributional search problem} (see e.g. \citet{feldman2017general}), where a class $\cMqb\subseteq \DZ$ of distributions is given, and each $M\in\cMqb$ is associated with a set $\Pi_M\subseteq \Pi$ of \emph{solutions}, so that the loss function is specified as $\LM{\pi}=\indic{\pi\not\in\Pi_M}$. The goal of an SQ algorithm is to find a decision $\pi\in\Pi_M$ through adaptively querying the \emph{SQ oracle} for any model $M\in\cMqb$ (defined below).

\begin{definition}[SQ oracle]\label{def:SQ-oracle}
For a model $M\in\DZ$, tolerance parameter $\tau>0$, an Statistical Query (SQ) oracle $\SQ$ is an oracle that, given any input $\phi:\cZ\to[0,1]$, returns a value $v$ such that $\abs{v-\EE_{z\sim M}\phi(z)}\leq \tau$.
\end{definition}

To frame the problem of learning with SQ oracles, we consider the measurement class $\Phi=(\cZ\to[0,1])$, and we note that each distribution $M\in\cMqb$ induces a map $\bPi\to\R$ given by $M(\pi,\phi)=\EE_{z\sim M}[\phi(z)]$, i.e., the decision does not affect the response. Therefore, we may---with slight abuse of notation---write $\cMqb\subseteq (\bPi\to\R)$, and for any $M\in\cMqb$, an SQ oracle $\SQ$ corresponds to a constrained environment under the \qbDMSO. Conversely, under the specification above, any constrained environment under the \qbDMSO~corresponds to an (adaptive) SQ oracle. Therefore, our results for \qbDMSO~naturally imply guarantees for SQ learning, as we discuss in \cref{ssec:SQ-main} and \cref{ssec:SQ-dim}.

\subsection{General query oracles and DEC theory for query-based learning}\label{ssec:SQ-main}

Extending our discussion on SQ learning, we can formulate any \qbDMSO~problem as a learning problem under certain query oracles. Specifically, given a measurement class $\Phi$ and a model class $\cM\subseteq (\bPi\to\cV)$, we define \emph{general query} oracle as follows.
\begin{definition}[General Query]\label{def:GSQ}
For a model $M\in\cM$ and tolerance parameter $\tau>0$, a General Query (GQ) oracle $\GSQ$ is an oracle that, given any input decision $\pi\in\Pi$ and measurement $\phi\in\Phi$, returns a value $v\in\cV$ such that $\nrm{v-M(\pi,\phi)}\leq \tau$. 
\end{definition}

Clearly, there is an correspondence between the constrained environments under the \qbDMSO~and general query oracles.
Further, the formulation allows us to consider variants of SQ oracles, and, in particular, the standard SQ oracle and the VSTAT oracle. These are obtained below by suitably choosing the form of interaction between query and model.

\begin{example}[Symmetrized VSTAT oracle]
For a distributional search problem, we can also consider learning under the VSTAT oracles. For any distribution $M\in\DZ$, tolerance parameter $\tau\geq 0$, a symmetrized VSTAT oracle $\VSTAT$ is an oracle that, given any input $\phi:\cZ\to[0,1]$, returns a value $v$ such that $\abs{v-\sqrt{\EE_{z\sim M}\phi(z)}}\leq \tau$. As shown in \citet{feldman2017general}, the symmetrized VSTAT oracles are equivalent to the standard VSTAT oracles. Clearly, a symmetrized VSTAT oracle is a GQ oracle with measurement class $\Phi=(\cZ\to[0,1])$ and $M(\phi)=\sqrt{\EE_{z\sim M} \phi(z)}$.
\end{example}

\begin{example}[Interactive SQ learning]\label{example:interactive-SQ}
In interactive SQ learning, the measurement class is $\Phi=(\cZ\to[0,1])$, and each model $M\in(\Pi\to\DZ)$ induces a map $M: \Pi\times \Phi\to\R$ given by $M(\pi,\phi)=\EE_{z\sim M(\pi)}[\phi(z)]$. This is a natural generalization of SQ learning to interactive decision making.
\end{example}

More generally, our formulation also allows us to consider other query-based learning settings, e.g., Correlation Statistical Queries~\citep{bshouty2002using}, Differentiable Learning Queries~\citep{joshi2024complexity}, and the \emph{batch} SQ learning, where at each round the learner can select a batch of queries $\phi=(\phi^1,\cdots,\phi^n)\in\cL^n$.

\paragraph{\SQDEC~lower and upper bounds}
Now, we present the \SQDEC~lower and upper bounds implied by our framework. We begin with the lower bound for metric-based loss.
\begin{theorem}[Query-based lower bound]\label{thm:SQ-lower}
Let $T\geq 1$, model class $\cM\subseteq (\bPi\to \cV)$, and the loss function $L$ is metric-based. Suppose that $\alg$ is a $T$-round query-based algorithm. Then there exists a model $M\in\cM$ and a GQ oracle $\GSQ$ such that under this oracle, the expected risk of $\alg$ is lower bounded as
\begin{align*}
    \EE\sups{\alg}\brac{\riskdm(T)}\geq \frac18\pdecltau_{\ueps(T)}(\cM),
\end{align*}
where $\ueps(T)=\frac{1}{2\sqrt{T}}$.

Further, for general loss function $L:\cM\times\Pi\to[0,1]$ and $\delta\in(0,1)$, 
there exists a model $M\in\cM$ and a GQ oracle $\GSQ$ such that under this oracle, the expected risk of $\alg$ is lower bounded as
\begin{align*}
    \EE\sups{\alg}\brac{\riskdm(T)}\geq \pdecltau_{\ueps_\delta(T)}(\cM)-\delta,
\end{align*}
where $\ueps_\delta(T)=\sqrt\frac{\delta}{T}$.
\end{theorem}

Though \cref{thm:SQ-lower} is a direct corollary of \cref{thm:cDMSO-pac-lower}, we provide a more direct and simpler proof of \cref{thm:SQ-lower} in \cref{appdx:proof-SQ-lower} as an illustration.

For upper bound, we propose \sqetod, an adaption of the E2D algorithm~\citep{foster2023tight} for \qbDMSO, which achieves an upper bound in \SQDEC~with minimal assumptions (\cref{appdx:SQ-E2D}). 
By instantiating \cref{thm:cDMSO-pac-upper}, we also have the upper bound of \ExOp.

\begin{theorem}[Query-based upper bound]\label{thm:SQ-upper}
Let $T\geq 1, \delta\in(0,1)$, model class $\cM\subseteq (\bPi\to \cV)$. Then, for any model $M\in\cMqb$ and given access to any (possibly adaptive) GQ oracle $\GSQ$ of $M$, the \sqetod~(\cref{alg:SQE2D}) achieves \whp~that 
\begin{align*}
    \riskdm(T)\leqsim \pdecltau[2\tau]_{\oeps(T)}(\cM),
\end{align*}
where $\oeps(T)=\sqrt{\frac{\log(|\cM|/\delta)}{T}}$. 

Further, suppose that the loss function $L$ is metric-based, and the \SQDEC~$\pdecltau_\eps(\cM)$ is of moderate decay. Then, for any model $M\in\cMqb$ and given access to any (possibly adaptive) GQ oracle $\GSQ$ of $M$, \ExOp~(instantiated on $\MPowsq$, following \cref{thm:cDMSO-pac-upper}) achieves \whp
\begin{align*}
    \riskdm(T)\leqsim \pdecltau_{\oeps(T)}(\cM).
\end{align*}
\end{theorem}

Note that the upper bound of \sqetod~scales with the \SQDEC~at the correctness level $2\tau$. In contrast, the upper bound of \ExOp~eliminates this factor of 2 under additional assumptions.
We note that for \ExOp, the assumptions on the loss function and the regularity of the \SQDEC~can both be relaxed (similar to \eqref{eqn:GDEC-pac-upper-bad}).

\subsection{Connection to the SQ dimension}\label{ssec:SQ-dim}

\newcommand{\Qs}[1]{\mathrm{QC}_{\tau,\beta}(#1)}
\newcommandx{\SCSQ}[1][1=\Delta]{\SC[#1][\tau\SQtag]}
\newcommandx{\RISKSQ}[1][1=T]{\RISK[#1][\tau\SQtag]}
For a distributional search problem, \citet{feldman2017general} studies the optimal \emph{query complexity} to arbitrary SQ oracle with correctness $\tau$. Recall that in the distributional search problem, a class $\cMqb\subseteq \DZ$ of distributions is given, and each $M\in\cMqb$ is associated with a set $\Pi_M\subseteq \Pi$ of \emph{solutions}. Then, for success probability $\beta$ and correctness $\tau\geq 0$, the optimal query complexity is the minimum number of rounds required to return a solution $\pi\in\Pi_{\Mstar}$ with success probability at least $\beta$, given access to any SQ oracle $\SQ[\Mstar]$ for any $\Mstar\in\cMqb$.

More generally, for any query-based model class $\cM\subseteq (\bPi\to\Delta(\cV))$, we define the $T$-round minimax risk as
\begin{align*}
    \RISKSQ(\cM)=\inf_{\alg}\sup_{\env} \EE\sups{\env,\alg}\brac{\riskdm(T)},
\end{align*}
where the supremum is taken over all environments satisfying query correctness with tolerance $\tau$ for a model $\Mstar\in\cM$. Then, the minimax query complexity for achieving $\Delta$-risk is defined as
\begin{align*}
    \SCSQ(\cM)\defeq \inf\sset{T: \RISKSQ(\cM)\leq \Delta}.
\end{align*}
For distributional search problems, achieving success probability $\beta$ is equivalent to achieving $(1-\beta)$-risk. Hence, in the following, we state the results of \citet{feldman2017general} in terms of $\SCSQ[1-\beta](\cMqb)$.\footnote{Recall that we identify $\cMqb\subseteq (\bPi\to\R)$ by regarding each model $M\in\DZ$ as a map $(\pi,\phi)\mapsto \EE_{z\sim M}\phi(z)$.}

\paragraph{Characterization by SQ dimension}
In the following, we first discuss the notion of SQ dimension and the results of \citet{feldman2017general} in detail.

\newcommand{\SQdim}{\mathsf{SQDim}}
\newcommandx{\SQD}[3][1=\beta,2=\tau]{\SQdim^{#2}_{#1}(#3)}
\newcommandx{\Tdecsq}[3][1=\beta,2=\tau]{\uline{\SQdim}^{#2}_{#1}(#3)}
\newcommand{\cMper}{\cMqb_{p,\beta}}
\begin{definition}[SQ dimension]\label{def:SQ-dim}
In distributional search problems, given a model class $\cMqb\subseteq \DZ$, parameter $\tau>0$, success probability $\beta\in[0,1]$, the SQ dimension with the reference model $\oM\in\DZ$ is defined as
\begin{align*}
    \SQD{\cMqb,\oM}=\inf_{p\in\DPi}\sup_{\mu\in\Delta(\cMper)}\inf_{\phi:\cZ\to[0,1]} ~\frac{1}{\PP_{M\sim \mu}\paren{ \abs{M(\phi)-\oM(\phi)} > \tau }},
\end{align*}
where $\cMper\defeq\set{M\in\cMqb: p(\Pi_M)<\beta}$. The SQ dimension of $\cMqb$ is then defined as $\SQD{\cMqb}\defeq \sup_{\oM\in\DZ}\SQD{\cMqb,\oM}$. 
\end{definition}

In terms of the SQ dimension defined above, \citet{feldman2017general} provides the following lower and upper bounds on $\SCSQ(\cMqb)$ for any distribution search problem with a model class $\cM\subseteq \DZ$. 
\begin{proposition}[SQ dimension characterization of the query complexity, \citet{feldman2017general}]\label{prop:feldman}
For success probability $\beta\in[0,1]$, parameter $\delta\in(0,1-\beta]$, it holds that
\begin{align}\label{eqn:SQ-bounds-feldman}
    \delta\cdot \SQD[\beta-\delta]{\cMqb}\leq \SCSQ[1-\beta](\cMqb) \leq \tbO{ \SQD[\beta+\delta][3\tau]{\cMqb}\cdot \frac{\CKL(\cMqb)}{\tau^2}\log(1/\delta) },
\end{align}
where $\CKL(\cMqb)\defeq \inf_{\oM\in\DZ}\sup_{M\in\cMqb}\KLd{M}{\oM}$ is the KL radius of $\cMqb$.
\end{proposition}

\paragraph{Comparison to the \SQDEC~characterization}
To compare our results with the above characterization, we first show that the SQ dimension is quantitatively equivalent to the \SQDEC~of $\cM$, as long as the Minimax theorem applies.

\begin{proposition}\label{prop:SQ-dim-to-SQ-dec}
Suppose that $\cZ$ is finite, and $\cMqb\subseteq \DZ$ is a distribution class. Then for any success probability $\beta\in[0,1]$, reference model $\oM\in\DZ$, we have 
\begin{align*}
    \pdecltau_\eps(\cMqb,\oM)>1-\beta \quad\Leftrightarrow\quad
    \eps^{-2}\leq \SQD{\cMqb,\oM}.
\end{align*}
\end{proposition}
Proof can be found in \cref{appdx:proof-SQ-dim-to-SQ-dec}. Therefore, \SQDEC~can be viewed as a generalization of the SQ dimension to general query-based learning.

To have a clearer comparison, for any model class $\cM\subseteq (\bPi\to \cV)$, we define the DEC-induced SQ dimension as\footnote{This is slightly different from the original SQ dimension (cf. \cref{def:SQ-dim}), because in the definition~\cref{def:sq-dec} of \SQDEC, the supremum is taken over all \emph{randomized} reference models $\oM\in(\bPi\to\Delta(\cV))$.}
\begin{align*}
    \Tdecsq{\cM}\defeq \min\set{\eps^{-2}: \pdecltau_\eps(\cM)\leq 1-\beta}.
\end{align*}
Then, for any query-based learning problem with loss bounded in $[0,1]$, our results imply the following characterization
\begin{align}\label{eqn:SQ-bounds-ours}
    \delta\cdot \Tdecsq[\beta-\delta]{\cM}\leq \SCSQ[1-\beta](\cM) \leqsim \Tdecsq[\beta+\delta][2\tau]{\cM}\cdot \log(|\cM|/\delta),
\end{align}
for any success probability $\beta\in[0,1]$ and any parameter $\delta\in(0,1-\beta]$. %
We note that for metric-based loss, the $2\tau$-factor in the upper bound can be improved to $\tau$ under the assumption that the SQ DEC is of moderate decay (\cref{thm:SQ-upper}).

Compared to \eqref{eqn:SQ-bounds-feldman}, our characterization  (when specialized to SQ learning in distributional search problems) does not incur the $\tau^{-2}$-gap between lower and upper bounds, but its upper bound scales with $\log|\cMqb|$, the complexity of the class $\cMqb$. Although it can be replaced by the log-covering number of $\cMqb$, this dependence might still be much larger than the $\CKL$-factor in \eqref{eqn:SQ-bounds-feldman}. While the dependence on $\log|\cM|$ can be unavoidable beyond this setting, the upper bound of \cref{alg:SQE2D} for such problems can also be improved to take advantage of bounded $\CKL$ (see our discussion in \cref{appdx:SQ-E2D}).

\subsection{Relation between SQ learning and LDP learning}\label{ssec:SQ-vs-LDP}

It is well known that for PAC learning, there is a (polynomial) equivalence between LDP algorithms and SQ algorithms~\citep{kasiviswanathan2011can}. 
We show that such an equivalence also holds between LDP DEC and SQ DEC. This is expected, since the DECs capture the complexity of the corresponding learning task. In greater generality, we state this equivalence for \emph{interactive SQ} learning (\cref{example:interactive-SQ}), a generalization of SQ learning.

\begin{lemma}\label{lem:SQ-dec-to-LDP-dec}
Let $\cM\subseteq (\Pi\to\DO)$. Then, for interactive SQ learning (\cref{example:interactive-SQ}), the \SQDEC~can be bounded as
\begin{align}\label{eqn:SQ-dec-to-LDP-dec}
    \pdecl_{\tau+\eps}(\cM,\oM)
    \leq \pdecltau_{\eps}(\cM,\oM)
    \leq \pdecl_{\eps/\tau}(\cM,\oM), \qquad \forall \oM.
\end{align}
\end{lemma}
Proof is presented in \cref{appdx:proof-SQ-dec-to-LDP-dec}. From \eqref{eqn:SQ-dec-to-LDP-dec}, it is clear that a comparison between the DECs would typically lead to \emph{loose} rates. This can be explained by the difference between SQ learning (where the response can be perturbed adversarially) and LDP learning (where the observations are stochastic).

In view of the relationship between LDP algorithms and SQ algorithms, \citet{kasiviswanathan2011can} established a lower bound for LDP \emph{learning parity} by reduction. In \cref{appdx:LDP-SQ-parity}, we show that DEC theory provides a more direct LDP lower bound for learning parity through lower bounding the \pDEC.

\section{Locally Private Learning}\label{sec:LDP}

In this section, we employ the DEC formulation to analyze \pDMSO~and characterize the complexity of LDP learning.

\paragraph{Problems encompassed by \pDMSO} 
Before diving into details, we first discuss several common settings of private learning that are encompassed by \pDMSO~(page \pageref{ssec:LDP-demo}). Recall that in this setting, the learner selects, on round $t$, a decision $\pit\in \Pi$ and a private channel $\prt\in\Pc$, the environment generates latent observation $z\ind{t}\sim \Mstar(\pit)$, and the learner observes $o\ind{t}\sim \prt(\cdot|z\ind{t})$. The ground truth model $\Mstar$ is known to belong to a given model class $\cM\subseteq (\Pi\to\DZ)$.

In this section, one of our primary foci is the setting of \emph{reward-based} learning~\citep{foster2021statistical,foster2023tight,chen2024beyond}, where the goal of the learner is to maximize the expected reward of the decision, or equivalently, minimize its sub-optimality. 

\begin{definition}[Reward-based value and loss function]\label{rew-max}
Given a model class $\cM\subseteq (\Pi\to\DZ)$, we call the value function $V$ \emph{reward-based}, if there is a \emph{known} reward function $R:\cZ\times\Pi\to [0,1]$ such that $\Vm(\pi)=\EE\sups{M,\pi}[R(z,\pi)]$ is the expected cumulative reward of $\pi$ under $M$. We also denote $\pim\defeq \argmax_{\pi\in\Pi} \Vm(\pi)$ to be the optimal decision for $M$ (under the value function). A loss function $L:\cM\times\Pi\to\R$ is \emph{reward-based} if it is specified by a reward-based value function $V$ as
\begin{align}\label{eqn:def-subopt}
\LM{\pi}= \Vmm-\Vm(\pi).
\end{align}
\end{definition}

Loss functions of the above form appear in many LDP learning problems of interest, including classification and regression, online learning, bandits and contextual bandits, and Reinforcement Learning (RL).

We also consider examples of  \emph{statistical tasks}, where $z\ind{1},\cdots,z\ind{T}\sim \Mstar$ are independent and identically distributed, i.e., the latent observation is independent of the decision. Nonetheless, here the learner is actively choosing channels $\prt$, affecting the amount of information received, and the performance is assessed by the final decision $\hpi$.

\begin{definition}[Statistical task]\label{def:sest}
We call the model class $\cM$ a \textit{statistical model class} if for each model $M\in\cM$, $M(\pi)=M\in\DZ$ is independent of $\pi\in\Pi$, i.e., we may regard $\cM\subseteq \DZ$.
\end{definition}
Examples of statistical tasks include hypothesis testing, hypothesis selection, classification and regression, functional estimation, and density estimation, among others. For statistical tasks, our definition of \pLDP~algorithms agrees with the notion of sequential private channels~\citep{duchi2013local,duchi2018minimax} (as detailed in \cref{appdx:sequential-channels}).

\subsection{DEC theory for private PAC learning}
\label{sec:LDPDEC}

We start with the \pDEC~lower bounds for reward-based loss and metric-based loss.
\begin{theorem}[\PDEC~lower bound]\label{thm:pdec-lin-lower}
Let $T\geq 1$, $\alg$ be a $T$-round \pLDP~algorithm.

(1) Suppose that the loss function $L$ is metric-based. Then it holds that
\begin{align*}
    \sup_{M\in\cM}\Emalg{\riskdm(T)}\geq \frac{1}{8}\pdecl_{\ueps(T)}(\cM),
\end{align*}
where $\ueps(T)=\frac{c}{\sqrt{\alpha^2 T}}$, and $c$ is a universal constant. 

(2) Suppose that the loss function $L$ is reward-based. Then
\begin{align*}
    \sup_{M\in\cM}\Emalg{\riskdm(T)}\geq \frac{1}{4}\paren{ \pdecl_{\ueps(T)}(\cM)-6\ueps(T)}.
\end{align*}
\end{theorem}

The proof of \cref{thm:pdec-lin-lower} is deferred to \cref{appdx:p-dec-lin-q} and is based on the strong data-processing inequality stated below (\cref{prop:pLDP}). We note that \cref{thm:pdec-lin-lower} (1) can also be proven directly by combining the \gDEC~lower bound (\cref{thm:cDMSO-pac-lower}) with \cref{prop:pLDP}. 
Finally, we also note that \dct~also provides a lower bound (\cref{thm:DC-lower}), which is complementary to the  \pDEC~lower bounds above. 

\paragraph{Key ingredients for the lower bound} 
As we have discussed in \cref{ssec:LDP-demo}, \pDEC~can be viewed as a special case of the \gDEC, based on the following characterization of the data-processing under DP channels. We recall that for any channel $\pr\in\Pc$ and any distribution $P\in\DZ$, we denote $\pr\circ P$ to be the marginal distribution of $o$ under $z\sim P, o\sim \pr(\cdot|z)$.
The proof of the following result is presented in \cref{appdx:proof-strong-DP}.

\begin{proposition}[Strong data-processing inequality]\label{prop:pLDP}
Suppose that $\pr$ is an \pLDP~channel. Then there exists a distribution $q:=q_{\pr}\in\DL$, such that for any two distributions $P_1, P_2\in\DZ$ over $\cZ$, it holds that
\begin{align}\label{eqn:DH-LDP}
    \frac{(\ea-1)^2}{8e^{2\alpha}}\EE_{\lf\sim q} \Dl^2(P_1,P_2) \leq \DH{\pr\circ P_1,\pr\circ P_2}\leq \frac{(\ea-1)^2}{8}\EE_{\lf\sim q} \Dl^2(P_1,P_2).
\end{align}
Furthermore,
\begin{align}\label{eqn:KL-chi-LDP}
    \KLd{\pr\circ P_1}{\pr\circ P_2}\leq \chis{\pr\circ P_1}{\pr\circ P_2}\leq (\ea-1)^2 \EE_{\lf\sim q} \Dl^2(P_1,P_2).
\end{align}
\end{proposition}

In particular, \eqref{eqn:KL-chi-LDP} recovers the \emph{strong data-processing inequality} of \citet{duchi2018minimax}, as the \lfdivs~are always upper bounded by TV distance.

An interpretation of the characterization in \cref{prop:pLDP} is that, in terms of divergences, any private channel can be expressed in terms of a distribution over the \emph{binary channels}.
\begin{example}[Binary channel]\label{example:binary-pr}
Perhaps the simplest nontrivial channel is the \textit{binary channel}, defined as follows. For any map $\lf:\cZ\to [0,1]$, the binary channel $\bpr$  associated with $\lf$ is given by
\begin{align*}
    \bpr(+1|z)=\frac{1+\ca\lf(z)}{2}, \qquad \bpr(-1|z)=\frac{1-\ca\lf(z)}{2},
\end{align*}
where $\ca=1-\eai$ and $\cO=\set{-1,1}$. It can be verified that this channel is indeed \pDP. We define $\Pcbin$ to be the class of all binary channels described above, i.e., $\Pcbin\defeq \set{ \bpr: \lf\in\Lc }.$
\end{example}

It is clear that for any map $\lf:\cZ\to[0,1]$, we have $\DH{\pr\circ P_1,\pr\circ P_2}\asymp \alpha^2 \Dl^2(P_1,P_2)$ (up to absolute constants).

\paragraph{\PDEC~upper bounds}
We propose \LDPetod, an extension of the \etod~algorithm of \citet{foster2023tight} to the LDP setting, %
providing the following upper bound for PAC learning with any problem class $\cM$.

\begin{theorem}[\PDEC~upper bound via E2D]\label{thm:pdec-lin-upper}
For any model class $\cM$, the \LDPetod~algorithm (\cref{alg:E2D}) preserves \pLDP~and achieves with probability at least $1-\delta$ that
\begin{align*}
    \riskdm(T)\leq \pdecl_{\oeps(T)}(\cM),
\end{align*}
where $\oeps(T)=C\sqrt{\frac{\log(|\cM|/\delta)\log(1/\delta)}{\alpha^2 T}}$. 
\end{theorem}

We note that under certain assumptions, \ExOp~can also be instantiated to achieve a similar upper bound, and we call the obtained algorithm \LDPexo~(detailed in \cref{appdx:ExO-LDP}).
In the next result, we derive an upper bound of \LDPexo~scaling with the \dct~of $\cM$, following \cref{prop:regret-via-cDMSO}. 
\begin{theorem}\label{thm:p-dec-convex-upper}
Let $T\geq 1$, $\delta\in(0,1)$, 
model class $\cM\subseteq (\Pi\to\DO)$, and the loss function $L$ be reward-based. Suppose that $\cM$ \finite, and the \pDEC~$\pdecl_{\eps}(\cM_\Pi)$, as a function of $\eps$, is of moderate decay. Then \LDPexo~(instantiated as in \cref{appdx:ExO-policy-based}) preserves \pLDP~and achieves \whp~that
\begin{align*}
    \riskdm(T)\leq \Delta+\bigO{1} \cdot \brac{ \pdecl_{\oeps(T)}(\coM) + \oeps(T) },
\end{align*}
where $\oeps(T)=\sqrt{\frac{\log \DC{\cM}+\log(1/\delta)}{\alpha^2 T}}$. 
\end{theorem}

\subsection{Application: private regression}\label{ssec:regression}

\newcommand{\Lsq}{\loss_{\rm sq}}
\newcommand{\Labs}{\loss_{\rm abs}}
\newcommand{\Lcl}{\loss_{\rm cl}}

In this section, we consider the task of proper regression under LDP. 
\begin{example}[Regression]\label{example:regression}
In the regression task, $\cX$ is a given covariate space, $\cF\subseteq (\cX\to[-1,1])$ is a given function class, and $\loss(\cdot,\cdot): [-1,1]^2\to [0,1]$ is a given loss. The observation space is $\cZ=\cX\times[-1,1]$, and the loss function $\LOSS$ is then given by
\begin{align*}
    \LM{f}=\EE_{(x,y)\sim M} \loss(y, f(x))-\min_{\fs\in\cF} \EE_{(x,y)\sim M} \loss(y, \fs(x)).
\end{align*}
\end{example}

Regression is a statistical task, in the sense of \cref{def:sest}, as the model class $\cM$ is a subset of $\DZ$. The loss function for this task is reward-based, in the sense of \cref{rew-max}, if we set the reward function as $R((x,y),f)=1-\loss(y,f(x))$.

The choices of loss function $\loss$ of interest include (1) squared loss: $\Lsq(y,y')=(y-y')^2$, and (2) absolute loss: $\Labs(y,y')=|y-y'|$.
We also note that the classification task is a special case of the regression problem described above, by specializing $\cF\subseteq (\cX\to\set{0,1})$, $\cM\subseteq \Delta(\cX\times\set{0,1})$ and $\Lcl(y,y')=\indic{y\neq y'}$.

\newcommand{\wellspec}{well-specified}
\newcommand{\Wellspec}{Well-specified}

In the literature, both \emph{agnostic} regression and \emph{\wellspec} regression are studied, where the model class $\cM$ is specified as follows:
\begin{itemize}
    \item Agnostic regression: the model class is $\cMagn=\DZ$, i.e., there is no prior knowledge of the underlying environment.
    \item \Wellspec~regression: the model class $\cMwell$ consists of all models $M\in\DZ$ such that there exists $\fm\in\cF$, such that $y|x\sim \Rad{\fm(x)}$ under $M$.\footnote{For simplicity, we assume that $y\in\set{-1,1}$ in this case without loss of generality.} %
\end{itemize}

Notice that for agnostic regression, the model class $\cMagn=\DZ$ is convex, and hence \cref{thm:p-dec-convex-upper} applies immediately. %
In the following, we state the guarantees for agnostic regression and realizable regression. To avoid measure-theoretic issues, we assume that $\cX$ is finite.

\begin{proposition}[Agnostic regression]\label{cor:agn-regression}
Let $T\geq 1, \delta\in(0,1), \Delta>0$.
Suppose that the \pDEC~$\pdecl_{\eps}(\cMagn)$ is of moderate decay as a function of $\eps$. Then, \LDPexo~can be instantiated (following \cref{thm:p-dec-convex-upper}) to achieve \whp
\begin{align*}
    \riskdm(T)\leq \Delta+\bigO{1} \cdot \brac{ \pdecl_{\oeps(T)}(\cMagn) + \oeps(T) },
\end{align*}
where $\oeps(T)=\sqrt{\frac{\log\DC{\cMagn}+\log(1/\delta)}{\alpha^2 T}}$.
\end{proposition}

For \wellspec~regression, a similar guarantee also applies. 
\begin{proposition}[\Wellspec~regression]\label{prop:real-regression}
Let $T\geq 1, \delta\in(0,1), \Delta\geq 0$. Suppose that %
the \pDEC~$\pdecl_{\eps}(\cMwell)$ is of moderate decay as a function of $\eps$. Then, \LDPexo~can be instantiated (as detailed in \cref{appdx:ExO-val-based}) to achieve \whp
\begin{align*}
    \riskdm(T) \leqsim  \pdecl_{\oeps(T)}(\cMwell)+ \oeps(T) ,
\end{align*}
where $\oeps(T)=\Delta+\sqrt{\frac{\log \DC{\cF}+\log(1/\delta)}{\alpha^2 T}}$, and the \dct~of $\cF$ is defined as
\begin{align}\label{def:DC-cF-proper}
    \DC{\cF}\defeq \inf_{p\in\DF} \sup_{\mu\in\DX, \fs\in\cF} ~\frac{1}{p\paren{ f: \EE_{x\sim \mu}|f(x)-\fs(x)| \leq \Delta } }.
\end{align}
\end{proposition}
A detailed discussion of the \dct~$\DC{\cF}$ is deferred to \cref{ssec:regression-DC}. 
In \cref{ssec:LDP-examples}, we also consider the online regression task (where the $(x,y)$ pair is chosen adversarially by the environment).

\newcommand{\cFlip}{\cF_{\mathsf{Lip}}}
\newcommand{\cMlip}{\cM_{\mathsf{Lip}}}

\newcommand{\tTho}{\Tilde{\Theta}}

\subsection{Application: private linear regression}\label{ssec:linear-regr}

In this section, we investigate LDP regression in \emph{linear models}.

\newcommand{\cFlin}{\cF_{\sf Lin}}
\newcommand{\cMlin}{\cM_{\sf Lin}}
\newcommand{\Lonel}[2]{L_1(#1,#2)}
\newcommand{\Ltwol}[2]{L_2(#1,#2)}
\newcommand{\Lone}{$L_1$}
\newcommand{\Ltwo}{$L_2$}
\newcommandx{\thM}[1][1=M]{\theta\sups{#1}}

\begin{example}[Linear models]
Suppose that $\cX\subseteq \Bone$, the linear function class $\cFlin$ is given by
\begin{align*}
    \cFlin\defeq \sset{ f_\theta(x)=\lr \theta, x\rr }_{\theta\in\Bone},
\end{align*}
and let $\cMlin$ be the induced class of well-specified models, i.e., each model $M\in\cM$ is associated with a covariate distribution $\muM$ and a parameter $\thM$, such that $(x,y)\sim M$ is generated as $x\sim \muM, y\sim\Rad{\lr x, \thM\rr}$. 

In linear models, we consider decision space $\Pi=\Bone$ (the space of estimators).
For an estimator $\theta\in\Bone$, we consider the following loss functions that measure the \Lone (\Ltwo) \emph{estimation error}:
\begin{align*}
    \Lonel{M}{\theta}=\EE_{x\sim M} |\la x,\theta-\thM\ra|, \qquad
     \Ltwol{M}{\theta}=\EE_{x\sim M} \la x,\theta-\ths\ra^2.
\end{align*}
\end{example}
Note that the \Ltwo~error agrees with the squared loss of the function $f_\theta(x)=\lr \theta,x\rr$ considered in \cref{ssec:regression}. However,
we note that the \Lone~loss here measures the error of the estimator $\theta$ with respect to the ground-truth parameter $\thM$, which is different from the absolute-loss regression considered in \cref{ssec:regression}.

\paragraph{Rates for \Ltwo~regression}
For LDP linear regression, to achieve the standard $T^{-1}$-rate under \Ltwo~risk, it is necessary to require the covariance matrix to be well-conditioned~\citep{duchi2018minimax,duchi2024right}. Otherwise, the convergence rate can degrade to $T^{-1/2}$ in the worst case, as indicated by the following folklore lower bound~\citep{duchi2024right,li2024optimal}.
\begin{lemma}\label{lem:linear-l2}
Suppose that $d=1$, and $\nu$ is a given distribution over $[-1,1]$.
Then for any $T$-round \pLDP~algorithm $\alg$ with output estimator $\hth$, there exists a model $\Mstar$ with covariate distribution $\nu$ and parameter $\ths\in[-1,1]$, such that
\begin{align*}
    \EE\sups{\Mstar,\alg} \Ltwol{\Mstar}{\hth}\geqsim \EE_\nu|x|^2\cdot \min\sset{ \frac{1}{\alpha^2 T(\EE_\nu|x|)^2 }, 1 }.
\end{align*}
In particular, for any $T\geq 1$, there exists a ``worst-case'' covariate distribution $\nu_T$ with $\nu_T(0)=1-\frac{1}{\sqrt{\alpha^2 T}}$ and $\nu_T(1)=\frac{1}{\sqrt{\alpha^2 T}}$, such that any \pLDP~algorithm incurs an \Ltwo~loss of $\Om{\frac{1}{\sqrt{\alpha^2 T}}}$.
\end{lemma}

\paragraph{Rates for \Lone~regression}
In contrast, we show that a $T^{-1/2}$-rate under $\ell_1$-loss can still be achieved. Note that in the upper bound below, we do \emph{not} assume the covariate distribution is known. Details are deferred to \cref{appdx:proof-linear-upper}.
\begin{theorem}\label{thm:linear-upper}
Let the loss function $\LOSS=L_1$ be given by the \Lone~error. Then it holds that
\begin{align*}
    \pdecl_{\eps}(\cMlin)\leq \bigO{\sqrt{d}\eps}.
\end{align*}
Further, \LDPexo~can be instantiated to output $\hth\in\Bone$ so that \whp,
\begin{align*}
    \Lonel{\Mstar}{\hth}\leq \tbO{\sqrt{\frac{d^2\log(1/\delta)}{\alpha^2 T}}},
\end{align*}
which is minimax-optimal up to logarithmic factors (cf. the minimax lower bound in \cref{cor:linear-lower}).
\end{theorem}
To the best of our knowledge, such a \emph{assumption-free} $T^{-1/2}$-rate is new for LDP linear regression under \Lone~error. More specifically, previous works mostly focus on \Ltwo~loss regression, and hence when converted to \Lone~loss, the results either have a $T^{-1/4}$-rate or need extra assumptions, e.g. a bounded condition number of the covariance matrix $\Sigma=\EE[xx^\top]$~\citep[etc.]{duchi2018minimax,wang2019sparse}. We note that $L_1$ error, while less well-studied, can be of interest for a broad range of applications, including offline policy evaluation with linear function approximation.

In \cref{ssec:CB-lin}, we apply a similar technique to provide a near-optimal regret for learning linear contextual bandits.

\subsection{DEC theory for private no-regret learning}\label{ssec:LDP-noreg}

In this section, we present the \rDEC~and the guarantees for private no-regret learning. We focus on the reward-based setting.

\paragraph{\RDEC}
For a model class $\cM\subseteq (\Pi\to\DZ)$ and a value function $V$, we define the \rDEC~of $\cM$ with respect to a reference model $\oM\in\coM$ as
\begin{align}
    \rdecl_{\eps}(\cM,\oM)\defeq&~ \inf_{\substack{p\in\DPL}}\sup_{M\in\cM}\constr{ \EE_{\pi\sim p}[\Vmm-\Vm(\pi)] }{ \EE_{(\pi,\lf)\sim p} \Dl^2( M(\pi), \oM(\pi) )\leq \eps^2 }, \label{eqn:def-r-dec-lin}
\end{align}
and we define the \rDEC~of $\cM$ as 
\begin{align}
    \rdecl_{\eps}(\cM)\defeq \sup_{\oM\in\coM} \rdecl_{\eps}(\cMp,\oM).
\end{align}
Similar to the \pDEC, the \rDEC~can also be viewed as a specification of the \gDEC. By instantiating \cref{thm:cDMSO-reg-lower} and \cref{thm:cDMSO-reg-upper}, we have the following regret bounds.

\begin{theorem}[\RDEC~lower bound]\label{thm:rdec-lin-lower}
Let $T\geq 1$. Suppose that the value function $V$ is reward-based (\cref{rew-max}). Then, for any $T$-round \pLDP~algorithm $\alg$, it holds that
\begin{align*}
    \sup_{M\in\cM}\EE\sups{M,\alg}\brac{\regdm(T)}\geq \frac{T}{4}\paren{ \rdecl_{\ueps(T)}(\cM)-C\ueps(T) }-1,
\end{align*}
where $\ueps(T)=\frac{c}{\sqrt{\alpha^2 T}}$, and $c, C$ are universal constants.
\end{theorem}

\begin{theorem}[\RDEC~upper bounds]\label{thm:rdec-lin-upper}
Let $T\geq 1, \delta\in(0,1)$. Suppose that the model class $\cM$ is compact, the value function $V$ is reward-based, and the \rDEC~$\rdecl_\eps(\cM)$ is of moderate decay as a function of $\eps$. Then, a suitable instantiation of \LDPexo~(as detailed in \cref{appdx:ExO-model-based}) achieves \whp~that
\begin{align}\label{eqn:rdec-lin-upper}
    \frac1T \regdm(T)\leq \OsqrtT \cdot \brac{ \rdecl_{\oeps(T)}(\cM) + \oeps(T) },
\end{align}
where $\oeps(T)=\sqrt{\frac{\log(|\cM|/\delta)}{\alpha^2 T}}$. 

Further, suppose that the \rDEC~$\rdecl_\eps(\coM)$ is of moderate decay. Then an alternative instantiation of \LDPexo~(as detailed in \cref{appdx:ExO-policy-based}) achieves \whp
\begin{align}\label{eqn:rdec-lin-upper-co}
    \frac1T \regdm(T)\leq \Delta+\OsqrtT \cdot \brac{ \rdecl_{\oeps'(T)}(\coM) + \oeps'(T) },
\end{align}
where $\oeps'(T)=\sqrt{\frac{\log\DC{\cM}+\log(1/\delta)}{\alpha^2 T}}$.
\end{theorem}
We note that under reward-based value function, the algorithms of \citet{foster2023tight,glasgow2023tight} may also be adapted to achieve a regret bound similar to \eqref{eqn:rdec-lin-upper}, under a weaker regularity assumption on the \rDEC~$\rdecl_\eps(\cM)$. We state the upper bound \eqref{eqn:rdec-lin-upper} with \LDPexo~as it is more flexible.

\paragraph{Applications}
As a main application of the \rDEC~theory, in \cref{ssec:CBs}, we present the DEC theory for LDP learning in contextual bandits. We do not present the implications for bandits (which our framework subsumes easily) because it is already encompassed by non-private DEC framework for bandits~\citep{foster2021statistical,foster2023tight,chen2024beyond}: it is well-known that LDP bandits learning can be directly reduced to the standard bandits learning by adding additive noises (Laplace noise or Gaussian noise) to the random rewards.

\subsection{Application: Contextual bandits}\label{ssec:CBs}

\newcommand{\cMFcb}{\cM_{\cF,\mathsf{CB}}}
\newcommand{\xt}{x_t}

In this section, we focus on no-regret learning in contextual bandits, where the contexts can be adversarially chosen.
Specifically, we introduce the (private) contextual DMSO framework: For each $t=1,\cdots,T$:
\newcommand{\tz}{\Tilde{z}}
\begin{itemize}
  \setlength{\parskip}{2pt}
    \item The learner selects a decision $\pit:\cX\to\cA$ and a private channel $\prt\in\Pc$.
    \item The environment selects context $\xt\in\cX$ and receives 
    $(\pit,\prt)$.
    \item The environment selects the action $\act=\pit(\xt)$ according to $\pit$, receives the reward $\rt\sim \Rad{ \fs(\xt,\act) }$,\footnote{For simplicity, we assume the reward is a binary random variable without loss of any generality.} 
    generates a noisy observation $\ot\in\cO$ via $\ot\sim \prt(\cdot|\xt,\act,\rt)$ and reveals it to the learner. 
\end{itemize}
Here, we go beyond the \pDMSO~in that we do not assume the context of each user is stochastic; Instead, we allow $\xt$ to depend on the history prior to step $t$, i.e., the context $\xt$ can be chosen in an adversarial manner.
The underlying reward function $\fs:\cX\times\cA\to [-1,1]$ encodes the mean reward value of the underlying environment, and we assume that the learner has access to a known reward function class $\cF\subseteq (\cX\times\cA\to [-1,1])$ containing $\fs$. The decision space $\Pi=(\cX\to\cA)$ consists of all maps (policies) from the context space to the action space. 

In contextual bandits, the regret of the learner is measured by
\begin{align*}
    \regdm(T)=\sum_{t=1}^T \fs(\xt, \pis(\xt))-\EE_{\pit\sim \qt}\fs(\xt,\pit(\xt)),
\end{align*}
where $\pis$ is an optimal policy under the reward function $\fs$, i.e., $\pis(x)=\argmax_{a\in\cA} \fs(x,a)$ for $x\in\cX$, and the expectation is with respect to $\pit\sim \qt$, the randomness of the choice of $\pit$ at the $t$-th step.

\paragraph{Formulation in \cDMSO}
We first briefly discuss how to frame this problem within \cDMSO. For $\nu\in\DX$ and $f\in\cF$, we define the contextual bandit model $M_{\nu,f}: \Pi\to \DO$ as
\begin{align*}
    (x,a,r)\sim M_{\nu,f}(\pi): \qquad x\sim \nu, a=\pi(x), r\sim \Rad{f(x,a)}.
\end{align*}
We then consider the model class $\cMFcb=\set{ M_{\nu,f}: \nu\in\DX, f\in\cF }$, which is the model class of contextual bandits with stochastic context and mean reward function in $\cF$. For each $f\in\cF$, $f$ specifies a constraint $\cP_f$ as
\begin{align}\label{def:MPow-cxt}
    \cP_f\defeq \set{ \tM_{\nu,f}: \nu\in\DX },
\end{align}
i.e., $\cP_f$ consists of all private (that is,  $\tM_{\nu,f}$ includes the private channel choice) contextual bandit instances with mean reward function $f$, and we let $\MPowcxt\defeq \set{\cP_f: f\in\cF}$. Then, the contextual bandits problem with function class $\cF$ can be framed within \cDMSO~with constraint class $\MPowcxt$. 

\paragraph{Regret guarantees}
We show that \LDPexo~achieves a regret bound scaling with the \rDEC~of $\cMFcb$. Similar to \cref{ssec:regression}, we assume that $\cX$ and $\cA$ are both finite throughout this section, mainly to avoid measure theoretic issues (our results do not have any dependence on $|\cX|$).

\begin{proposition}\label{thm:CB-adv}
Let $T\geq 1, \delta\in(0,1)$. Suppose that $\cX$ and $\cA$ are finite, and the \rDEC~$\rdecl_\eps(\cMFcb)$ is of moderate decay as a function of $\eps$. Then, \LDPexo~(instantiated as in \cref{appdx:ExO-CB}) achieves \whp:
\begin{align*}
    \frac1T\regdm(T)\leq \OsqrtT \cdot \brac{ \rdecl_{\oeps(T)}(\cMFcb) + \oeps(T) },
\end{align*}
where $\oeps(T)=\inf_{\Delta\geq 0}\paren{ \Delta+\sqrt{\frac{\log \Ncov[\infty]{\cF}+\log(1/\delta)}{\alpha^2 T}} }$, and $\Ncov[\infty]{\cF}$ is the $\Delta$-covering number of $\cF$ under $L_{\infty}$-norm (cf. \cref{def:covering-cF}). %
\end{proposition}

Therefore, up to a gap of the log-covering number of $\cF$, the complexity of no-regret learning is characterized by the \rDEC~of $\cMFcb$. It is worth noting that our upper bound scales with the DEC of the \emph{stochastic} contextual bandits, while it applies to any environment that generates contexts adversarially. Therefore, within the DEC framework, contextual decision making with (potentially) adversarial contexts is no more difficult than stochastic contexts. 

This result is somewhat surprising, because with the LDP constraint, the learner can never directly observe the contexts. Indeed, this makes it challenging to estimate the ground truth mean reward function $\fs$, and previous works typically had to adopt problem-specific estimation methods. In contrast, \cref{thm:CB-adv} allows us to derive regret bounds by directly studying the DEC.

In the following, we apply our frameworks to derive near-optimal regret guarantees for linear contextual bandits and Lipschitz contextual bandits.

\subsubsection{Linear contextual bandits}\label{ssec:CB-lin}

\newcommand{\CBn}[1]{\mathsf{#1}\!\operatorname{-}\!\mathsf{CB}}
\newcommand{\cMlincb}{\cM_{\CBn{Lin}}}
\newcommand{\Bd}{\mathbf{B}^d(1)}
In the linear contextual bandits setting, we are given a bounded feature map $\phi: \cX\times\cA \to \Bd$. The 
linear value function class $\cFlin$ is given by
\begin{align*}
    \cFlin=\set{ f_{\theta}:  f_{\theta}(x,a)=\<\theta,\phi(x,a)\>}_{\theta\in\Bd},
\end{align*}
Let $\cMlincb$ be the corresponding contextual bandits model class.
In the following, we bound the \rDEC~of $\cMlincb$ and provide a near-optimal guarantee for learning linear contextual bandits. Proof is presented in \cref{appdx:proof-linear-CB-upper}.

\begin{theorem}[Near-optimal regret for linear contextual bandits]\label{thm:LCB-dec-upper}
For the model class $\cMlincb$, it holds that
\begin{align*}
    \rdecl_\eps(\cMlincb)\leqsim d\eps.
\end{align*}
Therefore, \LDPexo~achieves the following regret bound in linear contextual bandits \whp: 
\begin{align*}
    \regdm(T)\leq \bigO{ \frac{\sqrt{d^3 T\log(T/\delta)}}{\alpha} }.
\end{align*}
\end{theorem}
The above regret bound of \LDPexo~is only a $\tO(\sqrt{d})$ factor larger than the regret lower bound of $\Om{\sqrt{d^2 T}/\alpha}$ for linear contextual bandits (detailed in \cref{appdx:proof-lin-cb-lower}).

Our upper bound nearly settles the optimal regret for linear contextual bandits with LDP constraints. Previous works either suffer a $T^{3/4}$ rate~\citep{zheng2020locally}, a $\log^d(T)\cdot \sqrt{T}$ rate~\citep{li2024optimal}, or require a strong assumption that the covariance matrix under \emph{any} linear policy is well-conditioned~\citep{han2021generalized}. The benefit of our DEC framework is that it provides a systematic approach to obtain regret bounds, which reduces the problem to studying the \rDEC. We expect our techniques can be applied to a broader setting, e.g., RL with linear function approximation. %

\subsubsection{Lipschitz contextual bandits with finite arms}\label{ssec:CB-lip}

\newcommand{\cblip}{\cM_{\CBn{Lip}}}
As the next example, we consider a standard non-parametric contextual bandit problem: Lipschitz contextual bandits, with $\cX$ equipped with a metric $\rho$. The reward function class is 
\begin{align*}
    \cFlip=\set{ f: \text{for any $a\in\cA$, $f(\cdot,a)$ is a 1-Lipschitz function w.r.t. }\rho },
\end{align*}
and let $\cblip$ be the corresponding contextual bandits model class. In the following proposition, we provide both upper and lower bounds for learning contextual bandits with $\cFlip$. We define $N_\rho(\cX,\Delta)$ to be the $\Delta$-covering number of $\cX$ under $\rho$. Details are deferred to \cref{appdx:proof-Lip-CBs}.
\begin{proposition}\label{prop:Lip-CB}
For the model class $\cblip$, it holds that
\begin{align*}
    \rdecl_\eps(\cblip)\leqsim \inf_{\Delta>0}\paren{ \Delta+\sqrt{N_\rho(\cX,\Delta)|\cA|} \eps }.
\end{align*}
For contextual bandits with mean reward function $\fs\in\cFlip$, \LDPexo~(suitably instantiated as in \cref{appdx:proof-Lip-CBs}) achieves \whp
\begin{align*}
    \regdm(T)\leqsim \inf_{\Delta>0}\paren{ T\Delta+N_\rho(\cX,\Delta)\sqrt{\alpha^{-2}|\cA|T\log(|\cA|/\delta)} }.
\end{align*}
On the other hand, for any $\Delta\in(0,1]$, to learn an $\Delta$-optimal policy for $\cblip$, and \pLDP~algorithm must require $T$-round of interactions with $T\geqsim \frac{N_\rho(\cX,8\Delta)^2}{\alpha^2\Delta^2}$ (cf. \cref{appdx:proof-lin-cb-lower}).
\end{proposition}

In particular, when $N_\rho(\cX,\Delta)\asymp \Delta^{-d}$ (e.g. $\cX$ is a bounded domain in $\R^d$), the minimax-optimal regret of privately learning $\cFlip$ is $\tTho(\alpha^{-\frac{1}{d+1}}T^{\frac{2d+1}{2d+2}})$, up to a polynomial factor of $|\cA|$.

\subsubsection{Concave-Lipschitz contextual bandits}\label{ssec:CB-lip-con}

\newcommand{\cFlipc}{\cF_{\mathsf{LC}}}
\newcommand{\cMlipc}{\cM_{\CBn{LC}}}
Our final example is a generalization of the Lipschitz contextual bandits to continuously many arms. Assume that $\cX$ is equipped with a metric $\rho$, $\cA\subset \R^K$ is a bounded convex domain, and
\begin{align*}
    \cFlipc=(f: \text{1-Lipschitz function in }(x,a)\in\cX\times\cA, \text{concave in }a\in\cA ),
\end{align*}
Let $\cMlipc$ be the corresponding contextual bandits model class. Similar to the Lipschitz contextual bandits, we have the following upper bound. 
\begin{proposition}\label{prop:L-C-CB}
For the model class $\cMlipc$, it holds that
\begin{align*}
    \rdecl_\eps(\cMlipc)\leq \inf_{\Delta>0}\paren{ \Delta+\tbO{1}\sqrt{N_\rho(\cX,\Delta)K^4} \eps },
\end{align*}
where we hide poly-logarithmic factors of the diameter of $\cA$.
For contextual bandits with mean reward function $\fs\in\cFlipc$, \LDPexo~(suitably instantiated as in \cref{appdx:proof-Lip-CBs}) achieves \whp,
\begin{align*}
    \regdm(T)\leq \inf_{\Delta>0}\paren{ T\Delta+\tbO{N_\rho(\cX,\Delta)\sqrt{\alpha^{-2}K^5T}} }.
\end{align*}
\end{proposition}

The upper bound above is derived by (1) reducing the contextual concave bandits to the concave bandits (without contexts) by bounding the corresponding DECs, and then (2) applying the results of \citet{lattimore2020improved}. This streamlined approach demonstrates again the advantage of the DEC framework, without which the reduction may not be easy, and we may instead need to repeat the analysis of \citet{lattimore2020improved}. %

Note that the lower bound of \cref{prop:Lip-CB} also applies here (cf. \cref{appdx:proof-lin-cb-lower}). Therefore, when $N_\rho(\cX,\Delta)\asymp \Delta^{-d}$, the minimax-optimal regret of privately learning $\cFlipc$ is also $\tTho(\alpha^{-\frac{1}{d+1}}T^{\frac{2d+1}{2d+2}})$, up to a polynomial factor of $K$.

\section{Local Minimaxity, Learnability, and Joint Privacy}\label{sec:connection}
In this section, we still focus on locally private learning, and discuss how our framework relates various other notions, including local-minimax complexity, learnability, and joint differential privacy.
\subsection{Local-minimax optimality}\label{ssec:local-minimax}

In this section, we demonstrate that the \pDEC~framework also applies to local-minimax statistical estimation under LDP, recovering the existing results in \citet{duchi2024right} and also providing new insights.

\newcommand{\RISKloc}{\RISK[T][\rm loc]}
\newcommand{\pdecloc}{{\normalfont \textsf{p-dec}}^{\rm loc}}

\paragraph{Local-minimax risk}
For any learning problem given by $\cM$ and a model $M_0\in\cM$, we define the \pLDP~\emph{local-minimax} risk at $M_0$ as
\begin{align}\label{def:loc-minimax-risk}
    \RISKloc(\cM,M_0)\defeq \sup_{M_1\in\cM}\inf_{\alg} \sup_{M\in\set{M_0,M_1}} \EE\sups{M,\alg} \brac{ \Riskdm(T) },
\end{align}
where the $\inf_{\alg}$ is taken over all possible $T$-round \pLDP~algorithms. 
In words, the local minimax risk measures the best performance the algorithm can achieve when it is given the knowledge two possible models.
This risk is called local because it measures the difficulty of a particular model $M_0$ against a \emph{single} worst-case alternative $M_1\in\cM$.

Modulus of continuity is a commonly studied complexity measure in statistical estimation and is shown to capture the complexity of various problem classes \citep{donoho1991geometrizingii,juditsky09,polyanskiy2019dualizing}.
Under local privacy constraints, \citet{duchi2024right} show that the following TV modulus of continuity captures the difficulty of  \emph{local} minimax-optimal statistical estimation:
They show that, for functional estimation, 
the minimax risk is characterized by the following TV variant of modulus of continuity: 
\begin{align}\label{def:mod-cont}
    w_{\eps}(\cM,M_0) \ldef \sup_{M_1\in \cM} \constr{\abs{\pim[M_1]-\pim[M_0]} }{ \DTV{M_1,M_0} \leq \eps  }.
\end{align}
We note that under LDP, the TV modulus of continuity also characterizes the complexity of linear functional estimation with a convex model class, as shown in \citet{rohde2020geometrizing}.

In the following, we study the local-minimax complexity of any LDP PAC learning problem (not necessarily limited to statistical tasks as per \cref{def:sest}). 

\paragraph{Local DEC theory}
We show that the local-minimax risk of any LDP PAC learning problem is tightly captured by the following \emph{local} DEC:
\begin{align}\label{def:p-dec-loc}
    \pdecloc_\eps(\cM,M_0)=\sup_{M\in\cM}\constr{ \inf_{\pi\in\DD} \LM[M_1]{\pi}+\LM[M_0]{\pi}  }{ \sup_{\pi\in\Pi} \DTV{ M_1(\pi), M_0(\pi) } \leq \eps }.
\end{align}
In particular, for functional estimation problems (where $\DD=\R$, and $\LM{\pi}=\abs{\pim-\pi}$), the definition above \emph{exactly} recovers the modulus of continuity \cref{def:mod-cont}. Moreover, for stochastic convex optimization, local DEC also agrees with the modulus of continuity considered in \citet{duchi2016local}. Therefore, local DEC can be regarded as the natural generalization of the modulus of continuity to any local-minimax PAC learning problem.

As an corollary of the \pDEC~lower and upper bounds (\cref{appdx:p-dec-lin-q} and \cref{thm:pdec-lin-upper}), local DEC provides the following nearly-optimal characterization of the local-minimax risk. Details are presented in \cref{appdx:proof-local-minimax}. 

\newcommand{\Lmax}{L_{\max}}
\begin{theorem}\label{thm:local-minimax}
Let $T\geq 1$, model class $\cM$ be given. Suppose that the loss function $L$ is bounded in $[0,\Lmax]$, and for any model $M\in\cM$, we have $\min_{\pi} \LM{\pi}=0$.
Then, the local-minimax risk at a model $M_0\in\cM$ is bounded as
\begin{align*}
    \frac18\pdecloc_{\ueps(T)}(\cM,M_0)\leq \RISKloc(\cM,M_0)
    \leq \inf_{\delta>0}\paren{ \pdecloc_{\oeps_\delta(T)}(\cM,M_0)+\delta \Lmax },
\end{align*}
where $\ueps(T)=\frac{c_0}{\sqrt{T}}$ and $\oeps_\delta(T)=\frac{c_1\log(1/\delta)}{\sqrt{T}}$.
\end{theorem}

Therefore, the local-minimax risk of interactive learning under LDP is tightly captured by the local DEC. For the particular case of functional estimation, local DEC is equivalent to the TV modulus of continuity. Hence, up to logarithmic factors, we recover the characterization of the LDP local-minimax risk of \citet{duchi2024right}, assuming certain growth conditions. The fact that such a characterization extends to statistical estimation tasks with interaction and general loss function is a testament to the unifying power of the DEC framework. 

Furthermore, from the definition of local DEC~\cref{def:p-dec-loc}, we can gain some quantitative insights into how locality reduces the difficulty of learning. More specifically, with locality, the algorithm only needs to distinguish between two models $\set{M_1, M_0}$, and hence avoids (1) the complexity of estimation, e.g. the log-cardinality of the model class or the function class (cf. \cref{thm:pdec-lin-upper}), and (2) the complexity of exploration, because it suffices to pick the best distinguishing decision $\pi$ that maximizes $\DTV{M_1(\pi),M_0(\pi)}$. 
Hence, even though the local-minimax formulation avoids the undesirable worst-case behavior of the global-minimax LDP learning, it may be too restrictive as it trivializes the difficulty of both interaction (exploration) and estimation.

\subsection{Finite-time learnability under LDP}\label{sec:DC}
\newcommand{\DFp}{\Delta(\cFp)}
\newcommandx{\DCF}[3][1=\Delta,2={\cF},3={\cFp}]{\DC[#1]{#2,#3}}
\newcommandx{\SCDP}[1][1=\Delta]{\SC[#1][\LDPtag]}

In learning theory, a central task is to investigate complexity measures that characterize the \emph{finite-time learnability} of certain problem classes, e.g., VC dimension for binary classification, Littlestone dimension~\citep{littlestone1988learning} for online classification~\citep{bendavid2009agnostic}, and their real-valued analogues for regression and online learning (see e.g. \citet{rakhlin2014nonparametric}). Further, \citet{bun2020equivalence,alon2022private} show that \emph{jointly private} classification is possible if and only if the Littlestone dimension is finite. Recently, the notion of \dct~\cref{def:DC} was proposed by \citet{chen2024beyond} and shown to characterize the non-private learnability of any stochastic bandits problems.

Following this line of work, in this section, we characterize the LDP learnability of any learning problem with reward-based loss through its \dct, generalizing the results of \citet{chen2024beyond}.  
To rigorously formulate the notion of learnability, we introduce the following \emph{minimax sample complexity} under LDP: For a model class $\cM\subset (\Pi\to\DZ)$, risk level $\Delta>0$, we define\footnote{We note that both the minimax sample complexity $\SCDP(\cM)$ and the \dct~$\DC{\cM}$ depend on the loss function $L$ implicitly. }
\begin{align}\label{def:minimax-SC}
    \SCDP(\cM)\defeq \min\sset{ T: \exists \text{$T$-round \pLDP~algorithm $\alg$ s.t. }\sup_{M\in\cM}\Emalg{\riskdm(T)}\leq \Delta }.
\end{align}
A model class $\cM$ is \pLDP~learnable if for all risk levels $\Delta>0$, $\SCDP(\cM)<+\infty$, i.e., there is an \pLDP~algorithm that achieves $\Delta$-risk in finite number of rounds.

We first show that \dct~provides a lower bound for any LDP learning problem, following the approach of \citet{chen2024beyond}.
\begin{theorem}\label{thm:DC-lower}
Let $T\geq 1$, $\cM\subseteq (\Pi\to\DZ)$ be a model class. Suppose that there is a $T$-round \pLDP~algorithm $\alg$ that achieves that for all $M\in\cM$, $\riskdm(T)\leq \Delta$ with probability at least $\frac12$ under $\PP\sups{M,\alg}$. Then it holds that
\begin{align*}
    T\geq \frac{\log\DC{\LP}-2}{2(\ea-1)^2}.
\end{align*}
\end{theorem}
This result differs from the \dct~lower bound for non-private learning~\citep{chen2024beyond}, which additionally involves the KL radius of $\cM$: 
\begin{align*}
    \CKL(\cM)=\inf_{\oM}\sup_{M\in\cM,\pi\in\Pi} \KLd{M(\pi)}{\oM(\pi)}.
\end{align*}
In non-private learning, the dependence on $\CKL^{-1}$ in the lower bound can be unavoidable (e.g., for binary classification, see also our discussion in \cref{ssec:connection-dim}). By contrast, \cref{thm:DC-lower} applies to LDP learning for \emph{any} problem class, even when $\CKL=+\infty$.

\paragraph{\Dct~upper bound} 
When the loss function is reward-based, we show that \dct~also provides a ``brute-force'' upper bound.

\begin{proposition}\label{thm:DC-upper}
Let $T\geq 1$, $\delta\in(0,1)$, and $\cM$ be a model class. Suppose that the loss function is reward-based, then there is a ``brute-force'' algorithm (\cref{alg:DC-upper}) such that \whp,
\begin{align*}
    \Riskdm(T)\leq \Delta+\bigO{\log(T/\delta)} \sqrt{\frac{\DC{\cM}}{\alpha^2T}}.
\end{align*}
\end{proposition}

Combining the above upper bound with the lower bound of $\DC{\cM}$, we have shown that $\DC{\cM}$ characterizes the sample complexity of LDP learning the model class, up to an exponential gap: 
\begin{align}\label{eqn:SC-DC}
    \frac{\log\DC[2\Delta]{\LP}}{\alpha^2}\leqsim \SCDP(\LP)\leqsim \frac{\DC[\Delta/2]{\LP} }{\alpha^2\Delta^2},
\end{align}
where we omit poly-logarithmic factors.
We remark that the gap between the lower and upper bounds cannot be improved in terms of \dct~alone:
\begin{itemize}
\item For classification with the parity class $\cFpar$, a lower bound scaling linearly with $\DC[\Delta/2]{\cFpar}=|\cFpar|$ can be obtained (\cref{thm:LDP-exp-lower}), meaning the upper bound can be tight even for the \emph{statistical} tasks (as per \cref{def:sest}).
\item For the problem of Multi-Armed Bandits, we also have $\DC[1/2]{\cM}=|\cA|$, while $\Om{\frac{|\cA|}{\alpha^2\eps^2}}$ samples are necessary to learn an $\eps$-optimal policy.
\item For linear bandits, $\log\DC[1/2]{\cM}=\Omega(d)$, and it is known that $\tbO{\frac{d^2}{\alpha^2\eps^2}}$ samples are sufficient to learn an $\eps$-optimal policy, meaning that the lower bound can also be (nearly) tight. 
\end{itemize}
While the exponential gap in \eqref{eqn:SC-DC} is unavoidable solely with \dct, we have shown that the upper bound can be improved with DEC (at least for convex model classes, cf. \cref{thm:p-dec-convex-upper}). 

A direct implication of \eqref{eqn:SC-DC} is that the finiteness of \dct~characterizes the finite-time learnability under LDP, as long as the loss function is reward-based. 
\begin{theorem}[LDP learnability]\label{thm:LDP-learnability}
Under reward-based loss, the problem class is LDP learnable if and only if $\DC{\cM}<\infty$ for all $\Delta>0$.
\end{theorem}
The learnability characterization above is similar to the bandit learnability characterization in \citet{chen2024beyond}. However, we do show that \dct~characterizes the learnability under LDP for \emph{any} model class $\cM$, while for non-private learning \dct~only characterizes the learnability of model class with a bounded $\CKL$.

As an application of \cref{thm:LDP-learnability}, in \cref{ssec:regression-DC} we discuss how the \dct~provides insights into the LDP learnability of \emph{regression}.

\subsection{Learnability under joint differential privacy}\label{ssec:JDP}

Parallel to the concept of local differential privacy (LDP), there is a notion of \emph{joint differential privacy} (JDP)~\citep{dwork2006calibrating}.\footnote{This notion is often referred to simply as “differential privacy.” To distinguish it from local differential privacy, we use the term “joint differential privacy,” as it preserves the privacy of the data points in a dataset jointly.}  
For simplicity, in the following discussion, we focus on the notion of \emph{pure} JDP for \statp. Detailed discussion for interactive decision making is deferred to \cref{appdx:JDP}. 

In this setting, the learner (algorithm) is given a dataset $\Hy\ind{T}=(z_1,\cdots,z_T)$ consisting of i.i.d observations, i.e., $z_1,\cdots,z_T\sim \Mstar$ for a model $\Mstar\in\DZ$. As always, we assume the learner is given a model class $\cM\subseteq \DZ$ that contains $\Mstar$. 

For this setting, an algorithm (learner) is simply a map $\alg: \cZ^T \to \DPi$. In the following, we define \pJDP~algorithms.

\begin{definition}[Pure JDP for \statp]\label{def:JDP}
For two sequence of observations $\Hy\ind{T}=(z_1,\cdots,z_T)$, $\Hy\ind{T}'=(z_1',\cdots,z_T')\in \cZ^T$, they are neighbored if there is at most one index $t\in[T]$ such that $z_t\neq z_t'$.
An algorithm%
$\alg$ preserves \pJDP~if for any neighbored dataset $\Hy\ind{T}, \Hy\ind{T}'$ and any measurable set $E\subseteq \DD$,
\begin{align*}
    \PP\sups{\alg}(\hpi\in E|\Hy\ind{T})\leq \ea\cdot \PP\sups{\alg}(\hpi\in E|\Hy\ind{T}').
\end{align*}
\end{definition}

Similar to \cref{thm:DC-lower}, we show that the \dct~also provides a lower bound for JDP learning. 

\begin{proposition}[\Dct~lower bound for JDP learning]\label{prop:JDP-lower}
Let $T\geq 1$, model class $\cM\subseteq \DZ$ be given. Suppose that $\alg$ is a $T$-round \pJDP~algorithm, such that it achieves $\riskdm(T)\leq\Delta$ with probability at least $\frac12$ under $\PP\sups{M,\alg}$ for any $M\in\cM$.
Then it holds that %
\begin{align*}
    T\geq \frac{\log\DC{\cM}-\log2}{\alpha}.
\end{align*}
\end{proposition}

For binary classification under pure JDP, \citet{beimel2013characterizing} provide both lower and upper bounds of the sample complexity in terms of the \emph{representation dimension}. As we discuss in \cref{ssec:connection-dim}, for binary classification, \dct~is equivalent to the representation dimension (up to an additive constant, \cref{prop:RDim-to-DC}).

\paragraph{Pure JDP learnability $\equiv$ LDP learnability}
It is clear that if an algorithm preserves \pLDP, then it also preserves \pJDP. Therefore, when the loss function is reward-based, as the finiteness of \dct~characterizes the LDP learnability, it also characterizes the JDP learnability.\footnote{We note that for JDP learning in statistical problems, the exponential mechanism achieves a better upper bound scaling with $\log\DC{\cM}$ (see e.g. \citet{beimel2013private}). However, for interactive learning (with or without JDP), an upper bound scaling linearly with $\DC{\cM}$ can be necessary in general~\citep{chen2024beyond}. }

\begin{theorem}\label{thm:JDP-equiv}
Let privacy parameter $\alpha>0$, model class $\cM\subseteq \DZ$, and the reward-based loss function $L$ be given. Then the following statements are equivalent: 

(1) $\cM$ is \pLDP~learnable, 

(2) $\cM$ is \pJDP~learnable, and

(3) $\DC{\cM}<+\infty$ for all $\Delta>0$.
\end{theorem}

We note that a similar argument also applies to interactive decision making problems, as the \dct~also provides a lower bound for interactive learning under JDP (\cref{appdx:JDP}).

\subsubsection{Connection to representation dimension and Littlestone dimension}\label{ssec:connection-dim}

In this section, we discuss the connection between \dct~and two well-studied complexity measures for binary classification: representation dimension~\citep{beimel2013characterizing} and Littlestone's dimension~\citep{littlestone1988learning}.

\paragraph{Representation dimension}
It has been known that for JDP binary classification with a function class $\cF\subseteq (\cX\to\set{0,1})$, the sample complexity of (proper or improper) learning is tightly characterized by the following \emph{representation dimension}~\citep{beimel2013characterizing}. For the simplicity of presentation, we focus on proper learning.

\newcommand{\scH}{\mathscr{H}}
\newcommand{\sz}{\mathrm{size}}
\begin{definition}
A distribution $\scH$ over finite subsets of $\cF$ is an $\eps$-probabilistic representation of $\cF$ if for any distribution $\nu\in\DX$ and $f\in\cF$, with probability at least $\frac34$ over $\cH\sim \scH$, there exists $h\in\cH$ such that
\begin{align*}
    \PP_{x\sim \nu}\paren{ h(x)\neq f(x) }\leq \eps.
\end{align*}
The \emph{size} of $\scH$ is defined as $\sz(\scH)=\sup_{\cH\in\supp(\scH)} \log|\cH|$. The \emph{representation dimension} of $\cF$ is then defined as
\begin{align*}
    \rdim_{\eps}(\cF)\defeq \inf_{\scH}~ \sz(\scH),
\end{align*}
where $\inf_{\scH}$ is taken over all $\eps$-probabilistic representations of $\cF$. 
\end{definition}

We show that for binary classification, the \dct~is equivalent to the representation dimension. Recall that for binary classification, the loss function (implicit in the definition of the \dct, cf. \cref{ssec:regression}) is given by
\begin{align*}
    L(M,f)\defeq \PP_{(x,y)\sim M}\paren{ f(x)\neq y } - \inf_{\fs\in\cF} \PP_{(x,y)\sim M}\paren{ \fs(x)\neq y }.
\end{align*}

\begin{proposition}\label{prop:RDim-to-DC}
For any $\eps\in[0,1]$, it holds that
\begin{align*}
    \abs{ \rdim_{\eps}(\cF)-\log\DC{\cF} } \leq 2.
\end{align*}
\end{proposition}
The details are postponed to \cref{appdx:proof-RDim-to-DC}.
This equivalence also agrees with the fact that both representation dimension and \dct~characterizes the JDP learnability of classification.

\paragraph{Littlestone dimension} It is known that for binary class, $\rdim(\cF)\geq \Om{\ldim(\cF)}$~\citep{feldman2014sample}, and there exists classes with $\ldim(\cF)=2$ while $\rdim(\cF)$  arbitrary large. Hence, LDP learnability is a stronger notion of complexity of a class than online learnability. 

It is also well-known that for binary classification, there is an equivalence between learnability under \emph{approximate} JDP and online learnability \citep{bun2020equivalence,alon2022private}. For regression, joint DP learnability can be achieved under a certain growth condition on the sequential fat-shattering dimension~\citep{golowich2021differentially}.
However, to learn a binary class $\cF$ under approximate JDP, it is only known that $\log^\star (\ldim(\cF))$ samples are necessary~\citep{bun2020equivalence}.

\section{Conclusion}

We presented a systematic approach to analyzing problems of decision making with a changing environment and constraints on the amount of information received by the learner. While this approach yields upper and lower bounds on minimax performance, the question of efficient algorithms is entirely open.

\section*{Acknowledgments} 
We acknowledge support from ARO through award W911NF-21-1-0328, as well as Simons Foundation and the NSF through awards DMS-2031883 and PHY-2019786. 

\bibliographystyle{abbrvnat}
\bibliography{ref.bib}

\appendix
\newpage

\tableofcontents

\counterwithin{theorem}{section}
\counterwithin{assumption}{section}

\section{Additional Discussions and Results from \cref{sec:overview}}
\subsection{Stochastic DMSO}
\label{ssec:sto-DMSO}

In this section, we briefly review the original DMSO formulation of  \citep{foster2021statistical}, which we call ``stochastic DMSO'' for clarity. In this setting, the learner (or, the decision maker) interacts for $T$ rounds with the environment described by an underlying model $\Mstar$, unknown to the learner. On each round $t=1,...,T$:
\begin{itemize}
  \setlength{\parskip}{2pt}
    \item The learner selects a decision $\pit\in \Pi$, where $\Pi$ is the decision space.
    \item The learner observes $\ot\in \cO$ sampled via $\ot\sim \Mstar(\pit)$, where $\cO$ is the observation space.
\end{itemize}

Formally speaking, the underlying model $\Mstar$ is a conditional distribution, and the learner is given a model class $\cM\subseteq (\Pi\to\Delta(\cO))$ that contains $\Mstar$. To frame stochastic DMSO in our \cDMSO~framework, we can consider the constraint $\cPs=\set{\Mstar}$ and the constraint class $\MPowiid=\sset{\set{\Mstar}: \Mstar\in\cM}$.

Stochastic DMSO captures a number of decision making tasks, including reward-based learning~\citep{foster2021statistical,foster2023tight}, interactive estimation and preference-based learning~\citep{chen2022unified}, multi-agent decision making and partial monitoring~\citep{foster2023complexity}. 

\paragraph{Constrained DEC and \gDEC}
Extending \citet{foster2021statistical}, \citet{foster2023tight} propose the constrained PAC-DEC (regret-DEC)  and derive lower and upper bounds for reward-based PAC learning (no-regret learning). Recall that constrained PAC-DEC is defined in \eqref{def:p-dec-c} and the constrained regret-DEC is defined in \eqref{def:r-dec-c}. For stochastic DMSO, (with the constraint class being $\MPowiid=\sset{\set{M}:M\in\cM}$), and clearly
\begin{align}
    \pdecg_\eps(\MPowiid)=\pdecc_\eps(\cM), \qquad
    \rdecg_\eps(\MPowiid)=\rdecc_\eps(\cM), \qquad \forall \eps\geq 0.
\end{align}
Therefore, the \gDEC~can be regarded as a generalization of the constrained DECs.

\subsection{Adversarial DMSO}\label{ssec:adv-DMSO}

In this section, we consider decision making against an adaptive adversary and instantiate the \gDEC~theory developed in \cref{ssec:cDMSO-reg}. 
For simplicity, we focus on the setting of \cref{example:original-reward-based}, where $\Phi=\set{\perp}$ and the value function is reward-based. In particular, our results tighten \citet{foster2022complexity}. 

\newcommand{\aDMSO}{adversarial DMSO}
\newcommand{\ADMSO}{Adversarial DMSO}
\paragraph{\ADMSO} In the \aDMSO~framework~\citep{foster2022complexity}, we consider the following protocol for $T$ rounds. For each $t=1,\cdots,T$:
\begin{itemize}
  \setlength{\parskip}{2pt}
    \item The environment selects a model $M^t\in\cM$ (potentially depends on the interactions up to step $t$), and the learner selects a decision $\pi\ind{t}\in \bPi$.
    \item The learner observes a noisy observation $o\ind{t}$ via $o\ind{t}\sim M^t(\pi\ind{t})$.
\end{itemize}
In the protocol above, the model $M^t\in\cM$ at step $t$ can adaptively selected, i.e., it may depend on the history $\cH\ind{t-1}$ prior to step $t$. The regret of the learner is measured against the best decision in hindsight:
\begin{align}\label{def:adv-regret}
    \regdm(T):=\max_{\pis\in\Pi} \sum_{t=1}^T \Vm[M^t](\pis)-\EE_{\pi\ind{t}\sim q\ind{t}}\Vm[M^t](\pi\ind{t}),
\end{align}
where the expectation of $\pi\ind{t}\sim q\ind{t}$ is taken over the randomness of the learner at step $t$, and $\Vm(\pi)=\EE\sups{M,\pi} R(o,\pi)$ is specified by a known reward function $R:\cO\times\Pi\to[0,1]$.

It is clear that \aDMSO~can be framed within \cDMSO~framework with the constraint class $\MPowadv=\set{\cM}$, i.e., the constraint is always $\cPs=\cM$. Therefore, we can directly apply \cref{thm:cDMSO-reg-upper}, as follows.

\begin{theorem}[No-regret learning against an adversary]\label{thm:adv-no-reg}
Let $T\geq 1, \delta\in(0,1)$, model class $\cM$, and a reward function $R\in[0,1]$ be given.
Suppose that $\cM$ \finite, and the regret DEC~$\rdecc_\eps(\coM)$ is of moderate decay as a function of $\eps$. Then,
\ExOp~(instantiated on $\MPowadv$, following \cref{thm:cDMSO-reg-upper}) achieves \whp~that
\begin{align*}
    \frac1T \regdm(T)\leq \Delta+\OsqrtT\cdot \brac{\rdecc_{\oeps(T)}(\coM) +\oeps(T)},
\end{align*}
where $\oeps(T)=\sqrt{\frac{\log\DC{\coM}+\log(1/\delta)}{T}}$.
\end{theorem}
The above upper bound scales with the regret DEC of $\coM$ and the \dct~of $\coM$, which is tighter than \citet{foster2022complexity}: the latter involves a $\log|\Pi|$ factor, whereas it always holds that  $\log\DC{\coM}\leq \log|\Pi|$.

\paragraph{Lower bounds}
A direct instantiation of \cref{thm:cDMSO-reg-lower} recovers the lower bound of \citet{foster2022complexity}.
\begin{proposition}[Regret lower bound with stationary adversary]\label{thm:adv-reg-lower}
Let $T\geq 1$, $\cM$ be a given model class. Then, for any $T$-round algorithm $\alg$, 
\begin{align}
    \sup_{\env}\EE\sups{\env,\alg}\brac{\regdm(T)}\geq&~\frac{T}{8}\paren{\rdecc_{\ueps(T)}(\coM)-8\ueps(T)}-1,
\end{align}
where the supremum is taken over \emph{stationary} environments $\env$ specified by a distribution $\mu\in\Delta(\cM)$, and $\ueps(T)=\frac{1}{24\sqrt{T}}$. 
\end{proposition}

In addition to the regret DEC lower bound, we can show that \dct~of $\coM$ also provides a lower bound. \cref{prop:adv-reg-lower-dct} below is a direct corollary of the \dct~lower bound of \citet{chen2024beyond} (see also \cref{appdx:proof-DC-lower}).
Thus, we omit its proof for succinctness.
\begin{proposition}[\Dct~lower bound]\label{prop:adv-reg-lower-dct}
Let $T\geq 1, \Delta\geq 0$. Suppose that $\cM$ is a given model class, and $\alg$ is a $T$-round algorithm that achieves $\EE\sups{\env,\alg}\brac{\regdm(T)}\leq T\Delta$ for any stationary environment $\env$ specified by a distribution $\mu\in\Delta(\cM)$. Then it holds that
\begin{align*}
    T\geq \frac{\log\DC[\Delta/2]{\coM}-2}{2\CKL(\coM)}.
\end{align*}
\end{proposition}

\paragraph{A nearly ``complete'' characterization of the minimax regret}
For no-regret learning in \cDMSO, the minimax regret is defined as
\begin{align*}
    \RISK(\MPow)\defeq \inf_{\alg}\sup_{\env}\EE\sups{\env,\alg} \brac{\regdm(T)},
\end{align*}
where the supremum is taken over all environments constrained by $\MPow.$
We also consider the following notion of minimax regret and sample complexity: 
\begin{align*}
    \SC(\MPow)\defeq \min\sset{T: \RISK(\MPow) \leq T\Delta},
\end{align*}
i.e., $\SC(\MPow)$ is the minimum of $T$ such that an $T$-round algorithm may achieve $T\Delta$-regret.

Under the above notation, we can translate the lower and upper bounds in this section into the following characterization of $\SC(\MPowadv)$ (with $\MPowadv=\set{\cM}$):
\begin{align}\label{eq:lower-upper-coM}
    \max\sset{ \Tdec{\coM}, \frac{\log \DC[2\Delta]{\coM}}{\CKL(\coM)} }\leqsim \SC(\MPowadv)
    \leqsim \Tdec{\coM}\cdot\log \DC[\Delta/\Delta]{\coM},
\end{align}
where $\Tdec{\coM}\defeq \min\sset{ \eps^{-2}: \rdecc_\eps(\coM)\leq \Delta }$, and we omit logarithmic factors and assume suitable growth conditions on the regret DEC of $\coM$. Note that the lower and upper bounds of \eqref{eq:lower-upper-coM} match up to squaring and a factor of $\CKL(\coM)$. In particular, for a model class $\cM$ with $\CKL(\coM)=\bigO{1}$, the DEC and \dct~together characterize the minimax sample complexity $\SC(\MPowadv)$ (polynomially).

\subsection{Robust DMSO}\label{ssec:robust-more}
In this section, we discuss the relationship between our formulation of \rDMSO~and other contamination models, and present the PAC and no-regret guarantees for robust decision making. 

Recall that in \rDMSO~(\cref{ssec:robust}), the constraint set is
\begin{align*}
\cP_{\Mstar}\defeq \sset{ (1-\beta)\Mstar+\beta M': M'\in(\bPi\to\DO) },
\end{align*}
and the constraint class (induced by $\cM$) as given by $\MPowrob\defeq \set{ \cP_{\Mstar}: \Mstar\in\cM }$.
To ease the notational burden, we define
\begin{align*}
    \cMrob\defeq \cMPow[\MPowrob]=\set{ (1-\beta)\Mstar+\beta M': \Mstar\in\cM, M'\in(\Pi\to\DO) },
\end{align*}
consisting of all stationary environments that are $\beta$-contaminated from a ground-truth model $\Mstar\in\cM$.

\paragraph{Contamination models in robust statistics}
In Huber contamination model~\citep{huber1965robust,huber2011robust}, the environment is stationary and specified by $(1-\beta)\Mstar+\beta M'$, where $\Mstar\in\cM$ is the ``true model'', and $M'$ is an arbitrary contamination model. Clearly, Huber's contamination model is encompassed by stochastic DMSO (with model class $\cMrob$). Recently, for statistical estimation, the \emph{adaptive} and \emph{oblivious} contamination models were studied by \citep{diakonikolas2019robust,diakonikolas2019recent,liu2021settling,diakonikolas2023algorithmic,canonne2023full}, among others. In these contamination models, after the i.i.d. samples $z\ind{1},\cdots,z\ind{T}\sim \Mstar$ is generated, the adversary may arbitrarily corrupt $\beta T$ many samples. The adversary is adaptive if it can choose the $\beta T$ corrupted samples based on the whole sequence. Otherwise, the adversary is called \emph{oblivious}. For statistical tasks, the adaptive adversary (in the above sense) can be stronger than the constrained environment in \cDMSO, as it is allowed to inspect the whole sequence of samples before contaminating it. On the other hand, the oblivious adversary can be much weaker. Finally, we note that both definitions of the adaptive and oblivious adversary are specialized to the statistical estimation (where the samples $z\ind{1},\cdots,z\ind{T}$ are i.i.d). For general interactive decision making tasks, we believe the \rDMSO~is a natural choice of contamination model.

\paragraph{PAC lower and upper bounds}
To apply the results of \cDMSO, we only need to show that $\pdecg_\eps(\MPowrob)=\pdecr_\eps(\cM)$. By definition, for any $\Mstar\in\cM$, reference model $\oM$, we have
\begin{align*}
    \inf_{M\in\co(\cP_{\Mstar})} \EE_{\bpi\sim q}\DH{ M(\bpi), \oM(\bpi) }
    =&~ \inf_{M':\bPi\to\DO} \EE_{\bpi\sim q}\DH{ (1-\beta)\Mstar(\bpi)+\beta M'(\bpi), \oM(\bpi) } \\
    =&~ \EE_{\bpi\sim q} \inf_{P'\in\DO} \DH{ (1-\beta)\Mstar(\bpi)+\beta P', \oM(\bpi) } \\
    =&~ \EE_{\bpi\sim q} \Dr{ \Mstar(\bpi), \oM(\bpi) }.
\end{align*}
Therefore, for any reference model $\oM$ and $\eps\in[0,1]$, it holds that
\begin{align*}
    \pdecg_\eps(\MPowrob,\oM)=&~\infpqb \sup_{\cP_{\Mstar}\in\MPowrob} \constr{ \EE_{\pi\sim p} L(\cP_{\Mstar},\pi) }{ \inf_{M\in\co(\cP_{\Mstar})}\EE_{\bpi\sim q} \DH{ M(\bpi), \oM(\bpi) } \leq \eps^2 } \\
    =&~\infpqb \sup_{\Mstar\in\cM} \constr{ \EE_{\pi\sim p} L(\Mstar,\pi) }{ \EE_{\bpi\sim q} \Dr{ \Mstar(\bpi), \oM(\bpi) } \leq \eps^2 }\\
    =&~ \pdecr_\eps(\cM,\oM).
\end{align*}
Therefore, we have proven $\pdecg_\eps(\MPowrob)=\pdecr_\eps(\cM)$ for $\eps\in[0,1]$. By instantiating \cref{thm:cDMSO-pac-lower} and \cref{thm:cDMSO-pac-upper}, we have the following bounds.
\begin{theorem}[PAC bounds for robust decision making]\label{thm:robust-pac}
Let $T\geq 1, \beta\in[0,1]$, model class $\cM\subseteq (\bPi\to \DO)$ be given, and the loss function $L$ is metric-based. 

(1) Lower bound: Let $\alg$ be a $T$-round algorithm. Then there exists $\Mstar\in\cM$ and a stationary environment $\env$ that is specified by $M=(1-\beta)\Mstar+\beta M'$, such that the expected risk of $\alg$ under $\env$ is lower bounded as
\begin{align*}
    \EE\sups{\env,\alg}\brac{\riskdm(T)}\geq \frac18\pdecr_{\ueps(T)}(\cM),
\end{align*}
where $\ueps(T)=\frac{1}{2\sqrt{T}}$.

(2) Upper bound: Suppose the robust DEC $\pdecr_\eps(\cM)$ is of moderate decay. Then \ExOp~(instantiated on $\MPowrob$, following \cref{thm:cDMSO-pac-upper}) achieves, in any $\beta$-contaminated environment, that \whp
\begin{align*}
    \riskdm(T)\leqsim \pdecr_{\oeps(T)}(\cM),
\end{align*}
where $\oeps(T)=\sqrt{\frac{\log(|\cM|/\delta)}{T}}$. 
\end{theorem}

\paragraph{Regret lower and upper bounds}
In \rDMSO, we may also consider the no-regret learning goal (specified by \eqref{eqn:def-reg-cDMSO}). 
For simplicity, we present the regret bounds in the setting of \cref{example:original-reward-based}, i.e., the measurement class $\Phi=\set{\perp}$ consists of the identity measurement, and the value function is reward-based. Then, by instantiating \cref{thm:cDMSO-reg-lower} and \cref{thm:cDMSO-reg-upper}, we have the following bounds in terms of the regret DEC of $\cMrob=\cMPow[\MPowrob]$.

\begin{theorem}[Regret bounds for robust decision making]\label{thm:robust-regret}
Let $T\geq 1, \beta\in[0,1]$, $\Phi=\set{\perp}$, model class $\cM\subseteq (\Pi\to \DO)$, and the value function $V$ is reward-based (\cref{example:original-reward-based}). 

(1) Lower bound: Let $\alg$ be a $T$-round algorithm. Then there exists $\Mstar\in\cM$ and a stationary environment $\env$ that is specified by $M=(1-\beta)\Mstar+\beta M'\in\cMrob$, such that the expected regret of $\alg$ under $\env$ is lower bounded as
\begin{align*}
    \EE\sups{\env,\alg}\brac{\regdm(T)}\geq \frac{T}8\paren{ \rdecc_{\ueps(T)}(\cMrob)-8\ueps(T) }-1,
\end{align*}
where $\ueps(T)=\frac{1}{10\sqrt{T}}$.

(2) Upper bound: Suppose that $\cMrob$ is compact, and the robust DEC $\rdecc_\eps(\cMrob)$ is of moderate decay. Then \ExOp~(instantiated on $\MPowrob$, following \cref{thm:cDMSO-reg-upper}) achieves, in any $\beta$-contaminated environment, that \whp
\begin{align*}
    \frac{1}{T}\regdm(T)\leq \Delta+\OsqrtT \cdot\brac{ \rdecc_{\oeps(T)}(\cMrob) +\oeps(T)},
\end{align*}
where $\oeps(T)=\sqrt{\frac{\log(|\cM|/\delta)+\log\DC{\cMrob}}{T}}$. 
\end{theorem}

\section{Additional Discussions and Results from \cref{sec:LDP}}\label{appdx:LDP-more}
\subsection{Sequential private channel}\label{appdx:sequential-channels}
\newcommand{\oo}{\Bar{o}}
\newcommand{\oz}{\Bar{z}}

The work of \citet{duchi2013local,duchi2018minimax} formalizes the problem of sequential private channel selection for statistical tasks (cf. \cref{def:sest}). We rephrase its definition as follows.
\begin{definition}
A sequential channel $Q$ from the data space $\cZ$ to the privatized data space $\cO$ is specified by a class of conditional distributions 
\begin{align*}
    \sset{ Q(o\ind{t}=\cdot|z\ind{t}=\cdot,o\ind{1}=\cdot,\cdots,o\ind{t-1}=\cdot) }_{t\in [T]}.
\end{align*}
A sequential channel $Q$ is $\alpha$-private if for any $t\in[T]$, any $\oz\ind{t},\tz\ind{t}\in\cZ$, any $\oo\ind{1},\cdots,\oo\ind{t-1}\in\cO$, we have
\begin{align*}
    \frac{ Q(o\ind{t}\in E|z\ind{t}=\oz\ind{t},o\ind{1}=\oo\ind{1},\cdots,o\ind{t-1}=\oo\ind{t-1}) }{ Q(o\ind{t}\in E|z\ind{t}=\tz\ind{t},o\ind{1}=\oo\ind{1},\cdots,o\ind{t-1}=\oo\ind{t-1}) } \leq \ea, \qquad \forall \text{ measurable }E\subseteq \cO.
\end{align*}
\end{definition}

Clearly, in statistical tasks, any $\alpha$-private sequential channel $Q$ induces an \pLDP~algorithm, which at each step $t\in[T]$ selects the \pLDP~channel $\pr^t$ given by
\begin{align*}
    \pr^t(o|z)=Q(o\ind{t}=o|z\ind{t}=z, o\ind{1},\cdots,o\ind{t-1}),
\end{align*}
based on the history $\cH^{(t-1)}=(o\ind{1},\cdots,o\ind{t-1})$. Conversely, an \pLDP~algorithm also induces a sequential $\alpha$-private channel. Therefore, sequential $\alpha$-private channels are equivalent to the \pLDP~algorithms in \pDMSO.

A similar argument also shows that for interactive decision making, our formulation in \cref{ssec:sto-DMSO} recovers the commonly studied interactive private channels (see e.g., ~\citet{zheng2020locally,garcelon2021local}). 

\subsection{Approximate DP channels}\label{appdx:approx-to-pure}

\newcommand{\prpr}{\pr_{\mathsf{pure}}}
\newcommandx{\prprt}[1][1=t]{\pr_{\mathsf{pure},#1}}
\newcommand{\tpr}{\widetilde{\pr}}
\newcommandx{\tprpr}[1][1=t]{\tpr_{\mathsf{pure},#1}}
\newcommand{\algpr}{\alg_{\sf pure}}

We first recall the definition of approximate DP channels.
\begin{definition}[Approximate DP channels] 
A channel $\pr$ (from latent observation space $\cZ$ to observation space $\cO$) is \aDP~if for $z, z'\in\cZ$ and any measurable set $E\subseteq \cO$,
\begin{align*}
    \pr(E|z)\leq e^\alpha\pr(E|z')+\beta.
\end{align*}
\end{definition}

The equivalence between approximate DP and pure DP under local privacy model is known \citep{duchi2019lower,duchi2024right}. In this section, we formalize such an equivalence in the general context of interactive decision making. 

In the following, we assume $\cO$ is countable.
The following lemma from \citet[Lemma 25]{duchi2019lower} shows that any \aLDP~channel is close to an \pLDP~channel.
\begin{lemma}
For any \aLDP~channel $\pr$, there exists an \pLDP~channel $\prpr$ such that
\begin{align*}
    \sup_{z\in\cZ} \DTV{ \pr(\cdot|z), \prpr(\cdot|z) }\leq \frac{\beta}{1+\ea-\beta}.
\end{align*}
\end{lemma}

As a corollary, we can show that any algorithm that preserves \aLDP~is close to an algorithm that preserves \pLDP. Proof is presented in \cref{appdx:proof-alg-approx-to-pure}.
\begin{proposition}\label{cor:alg-approx-to-pure}
Suppose that $\beta\leq \frac12$ and $\alg$ is a $T$-round algorithm that preserves \aLDP. Then there is a $T$-round algorithm $\algpr$ that preserves \pLDP, such that for any model $M$,
\begin{align*}
    \DTV{ \PP\sups{M,\alg}(\cH_\pi=\cdot), \PP\sups{M,\algpr}(\cH_\pi=\cdot) }\leq 2T\beta,
\end{align*}
where the TV distance is taken between the distribution of the trajectory of the decisions $\cH_\pi=(\pi\ind{1},\cdots,\pi\ind{T},\hpi)$. In particular, when the loss function is bounded in $[0,1]$, it holds that for any model $M$,
\begin{align*}
    \EE\sups{M,\alg}\brac{\riskdm(T)}\geq \EE\sups{M,\algpr}\brac{\riskdm(T)} - 2T\beta.
\end{align*}
\end{proposition}

Hence, as long as $\beta=\frac{1}{\poly(T)}$, there is \emph{essentially no gain} of allowing the algorithms to be \aLDP.

\subsection{Additional examples}\label{ssec:LDP-examples}

Recall that in \cref{thm:p-dec-convex-upper}, we show that \LDPexo~provides an upper bound scaling with the \pDEC~of $\coM$ and the \dct~of $\cM$. To draw a clearer comparison between this upper bound and the lower bounds, we re-state our lower and upper bounds in terms of the minimax sample complexity \cref{def:minimax-SC}. 
Define
\begin{align*}
    \Tdec{\cM}\defeq \min\set{\eps^{-2}: \pdecl_\eps(\cM)\leq \Delta}.
\end{align*}
Then, under the assumption of \cref{thm:p-dec-convex-upper}, we have the following characterization of $\SC(\cM)$:
\begin{align}\label{eqn:pdec-convex-SC}
     \max\set{ \Tdec{\cM} , \log\DC[2\Delta]{\cM} } \leqsim \alpha^2\cdot \SC(\cM) \leqsim  \Tdec{\coM} \cdot \log \DC[\Delta/2]{\cM}.
\end{align}
In particular, for a convex model class $\cM$, under mild assumption on the growth of the \pDEC~and \dct, the lower and upper bounds match up to squaring. We note that \eqref{eqn:pdec-convex-SC} is analogous to the observations of \citet{chen2024beyond} for non-private learning.

In the following, we discuss similar characterizations for convex hypothesis selection and online regression.

\paragraph{Convex hypothesis selection}
As an application of \cref{prop:multi-hypothesis}, we consider the LDP hypothesis selection problem, which is a statistical task (\cref{def:sest}).
\begin{example}\label{example:multi-hypothesis-LDP}
Given a model class $\cM\subseteq \DZ$, a hypothesis selection problem is described by a partition
\begin{align*}
    \cM=\bigsqcup_{i=1}^{m} \cMi{i} ,
\end{align*}
where $\cMi{1},\cdots,\cMi{m}$ are disjoint subclasses.
The decision space is $\Pi=[m]$,
and for each $M\in \cM$, $\pi\in\Pi$, the \losst~is given by $\LM{\pi}=\indic{ \pi\neq \pim }$, where $\pim$ is the unique index $i\in[m]$ such that $M\in\cMi{i}$.
\end{example}
Note that the LDP hypothesis selection problem can be regarded as a special case of \cref{example:multi-hypothesis} (with the measurement class $\Phi=\Pcp$ the class of all \pDP~channels). Therefore, we summarize the lower and upper bounds for this problem, as follows.
\begin{proposition}[Private hypothesis selection]\label{prop:multi-hypothesis-LDP}
Let $T\geq 1$, $\delta\in(0,1)$.

(1) Lower bound: For any \pLDP~algorithm $\alg$, it holds that
\begin{align*}
    \sup_{\Mstar\in\cM} \PP\sups{\Mstar,\alg}\paren{ \hpi\neq \pim[\Mstar] }\geq \frac{1}{8}\pdecl_{\ueps(T)}(\cM),
\end{align*}
where $\ueps(T)=\frac{c}{\sqrt{\alpha^2T}}$.

(2) Upper bound: Suppose that $\cM$ is compact, $\cMi{1},\cdots,\cMi{m}$ are convex, and
\begin{align*}
    \pdecl_{\oeps(T)}(\cM)\leq \frac{1}{3}, \qquad \oeps(T)=C\sqrt{\frac{\log(m/\delta)}{\alpha^2 T}}.
\end{align*}
Then \LDPexo~can be suitably instantiated to preserve \pLDP, so that under any model $\Mstar\in\cM_{i^\star}$, the algorithm returns $\hpi=i^\star$ \whp.
\end{proposition}

In terms of the sample complexity, assuming that $\cMi{1},\cdots,\cMi{m}$ are convex, we have
\begin{align*}
    \Tdec[1/3]{\cM} \leqsim \alpha^2\cdot \SC(\cM) \leqsim  \Tdec[1/3]{\cM} \cdot \log (m/\Delta),
\end{align*}
for all $\Delta\in[0,0.05]$. Therefore, up to the factor of $\log(m/\Delta)$, the sample complexity of private convex hypothesis selection is completely characterized by the \pDEC.

\paragraph{Online regression} We consider the online variant of the regression task (\cref{ssec:regression}). In the setting of \emph{online regression}, for every step $t\in[T]$, the environment selects a pair $(x\ind{t},y\ind{t})$ (potentially depends on the history prior to step $t$), and the learner has to pick a (randomized) prediction function $f\ind{t}\in\cF$. The regret of the learner is measured by
\begin{align*}
    \regdm(T):=\sum_{t=1}^T \loss(y\ind{t},f\ind{t}(x\ind{t}))-\inf_{\fs\in\cF} \sum_{t=1}^T \loss(y\ind{t},\fs(x\ind{t})),
\end{align*}
where $\loss: [-1,1]^2\to[0,1]$ is a given loss. 

Clearly, online regression is encompassed by adversarial DMSO (\cref{ssec:adv-DMSO}), with the constraint being $\cPs=\set{\cMagn}$.
As a corollary of \cref{thm:adv-no-reg}, we have the following regret bound for online regression.\footnote{For regression (a statistical task), we have $\rdecl_\eps(\cMagn)=\pdecl_\eps(\cMagn)$ because the decision $f\in\cF$ does not affect the distribution of the observation.}
\begin{proposition}\label{cor:online-reg}
Let $T\geq 1$, $\delta\in(0,1)$. Suppose that $\cX$ is finite, and $\pdecl_\eps(\cMagn)$ is of moderate decay as a function of $\eps$. For online regression, \LDPexo~achieves the following regret bound \whp:
\begin{align*}
    \frac1T \regdm(T)\leq \Delta+\OsqrtT\cdot \pdecl_{\oeps(T)}(\cMagn),
\end{align*}
where $\oeps(T)=\sqrt{\frac{\log\DC{\cMagn}+\log(1/\delta)}{\alpha^2 T}}$.
\end{proposition}
This also recovers the risk bound of \cref{cor:agn-regression} when the data are drawn i.i.d from a $\Ms\in\cMagn$. Therefore, in this sense, online private regression is \emph{no} more difficult than the agnostic private regression (with potential degradation of the rate of the regret), because the \pDEC~$\pdecl_\eps(\cMagn)$ and the \dct~$\DC{\cMagn}$ also provide lower bounds (similar to \eqref{eqn:pdec-convex-SC}). This is in sharp contrast to the non-private setting, where there is a separation between the complexity of regression and online regression.

\subsection{LDP lower bounds via SQ lower bounds}\label{appdx:LDP-SQ-parity}

For a more general demonstration of the power of \pDEC, we consider the following variant of the commonly used SQ lower bound methods~\citep[etc.]{blum1994weakly,feldman2017statistical,brennan2020statistical}. We focus on the statistical tasks (\cref{def:sest}, where $\cM\subseteq \DZ$).

\newcommand{\crdim}{\mathsf{cor}}

\begin{definition}[Minimum correlation]\label{def:SQ-cor}
For distributions $\cD_1, \cD_2, \cD\in\DZ$, we define the pairwise correlation as 
\begin{align*}
    \rho_{\cD}(\cD_1,\cD_2)=\EE_{z\sim \cD} \paren{\frac{d\cD_1(z)}{d\cD(z)}-1}\paren{\frac{d\cD_2(z)}{d\cD(z)}-1}.
\end{align*}
We say a set of $m$ distributions $\set{\cD_1,\cdots,\cD_m}$ is $\eps$-correlated relative to $\cD$ if
\begin{align*}
    \forall i,j, \qquad \abs{\rho_\cD(\cD_i,\cD_j)}\leq\begin{cases}
        \eps^2, & i\neq j,\\
        m\eps^2, & i=j.
    \end{cases}
\end{align*}
Suppose $\cM\subseteq \DZ$. For any $\Delta$, we define the minimum correlation $\crdim(\cM,\Delta)$ to be the minimum of $\eps$ such that there exists a reference model $\oM$ and a set of models $\set{M_1,\cdots,M_m}\subseteq \cM$, such that (1) $\set{M_1,\cdots,M_m}$ is $\eps$-correlated relative to $\oM$; (2) for any $\pi\in\DD$, there is at most $m/2$ indices $i\in[m]$ such that $\LM[M_i]{\pi}\leq \Delta$.
\end{definition}

In the following, we show that $\crdim(\cM,\Delta)$ provides a lower bound of \pDEC~of $\cM$, and hence it also provides a lower bound for learning $\cM$ under LDP.
\begin{proposition}\label{prop:SQ-lower-simple}
For any $\Delta>0$, it holds that
\begin{align*}
    \pdecl_{\eps}(\cM)\geq \frac{\Delta}{4}, \qquad \forall \eps\leq \crdim(\cM,\Delta).
\end{align*}
In terms of the sample complexity, any \pLDP~algorithm requires $\Om{ \frac{1}{\alpha^2\cdot \crdim(\cM,4\Delta)^2}}$ samples to learn a $\Delta$-optimal decision in $\cM$.
\end{proposition}
Proof can be found in \cref{appdx:proof-SQ-lower-simple}.

\paragraph{Hardness of LDP learning parity}
It has been shown that learning parity under LDP is hard~\citep{kasiviswanathan2011can}, in the sense that there is a $2^{\Om{d}}$ lower bound on the sample complexity (where $d$ is the dimension). In the following, we apply \cref{prop:SQ-lower-simple} to recover the exponential lower bound and discuss its implication. Proof in \cref{appdx:proof-par-lower}.

\begin{proposition}[Learning parity]\label{thm:LDP-exp-lower}
Let $d\geq 2$, $\eps\in[0,1]$, and $\cX=\set{0,1}^d$, and $\cFpar=\set{f_S}_{S\subseteq [d]}$, where for each subset $S\subseteq [d]$, the function $f_S:\cX\to \set{-1,1}$ is defined as
\begin{align*}
    f_S(x)=(-1)^{\sum_{i\in S} x_i}, \qquad \forall x\in \set{0,1}^d.
\end{align*}

Then, there exists a distribution $\mu\in\DX$, such that for $\cMpar$ the class of all realizable models with the covariate distribution $\mu$, it holds that
\begin{align*}
    \pdecl_{\eps}(\cMpar)\geq \Omega(\sqrt{2^d}\eps).
\end{align*}
This implies a lower bound of $\sup_{M\in\cMpar} \Emalg{\riskdm(T)}\geq \Om{\sqrt{\frac{2^d}{T}}}$ for any $T$-round algorithm $\alg$.
\end{proposition}

Notice that for the parity function class, we have $|\cFpar|=2^d$, and hence the lower bound above is in sharp contrast to the non-private setting, where the ERM can achieve a risk bound of $\sqrt{\frac{\log|\cFpar|}{T}}$.

\section{Additional Discussions and Results from \cref{sec:connection}}\label{appdx:connection-more}
\subsection{Joint DP in interactive learning}\label{appdx:JDP}

Generalizing the notion of JDP for non-interactive learning, \citet{shariff2018differentially} propose a definition of JDP for contextual bandits, which is later extended to reinforcement learning by \citet{vietri2020private}. In the following, we formalize the notion of JDP for general interactive decision problems.

Recall that a $T$-round algorithm $\alg$ (without LDP constraints) is specified by a sequence of mappings $\set{\qt}_{t\in [T]}\cup\set{\phat}$ , where the $t$-th mapping $\qt(\cdot\mid{}\Hy\ind{t-1})$ specifies the distribution of $\pit$ 
based on the history $\Hy\ind{t-1}=(\pi\ind{s},z\ind{s})_{s\leq t-1}$, and the final map $\phat(\cdot \mid{} \Hy\ind{T})$ specifies the distribution of the \emph{output policy} $\pihat$ based on $\Hy\ind{T}$.

\newcommand{\Hyz}{\Hy_{z,T}}
\newcommand{\Hypi}{\Hy^{\pi}_T}

\begin{definition}[Interactive JDP]\label{def:JDP-interactive}
For sequences of observations $\Hyz=(z\ind{1},\cdots,z\ind{T})$ and $\Hyz'=(z\ind{1}',\cdots,z\ind{T}')$, we say $\Hyz$ and $\Hyz'$ are neighbored if there is at most one index $t\in[T]$ such that $z\ind{t}\neq z\ind{t}'$. 

The algorithm $\alg$ preserves \pJDP~if for any two neighbored sets of observations $\Hyz=(z\ind{1},\cdots,z\ind{T})$ and $\Hyz'=(z\ind{1}',\cdots,z\ind{T}')$, it holds that
\begin{align*}
    \PP\sups{\alg}\paren{ (\pi\ind{1},\cdots,\pi\ind{T},\pi\ind{T+1})\in E | \Hyz}\leq \ea \PP\sups{\alg}\paren{ (\pi\ind{1},\cdots,\pi\ind{T},\pi\ind{T+1})\in E | \Hyz'},
\end{align*}
for any measurable set $E\subseteq \Pi$, where $\PP\sups{\alg}$ is taken over the randomness of the algorithm, i.e.,
\begin{align*}
    \PP\sups{\alg}\paren{ (\pi\ind{1},\cdots,\pi\ind{T},\pi\ind{T+1})=\cdot| z\ind{1},\cdots,z\ind{T}}
    =\prod_{t=1}^{T+1} \qt(\pi\ind{t}=\cdot|\pi\ind{1:t-1},z\ind{1:t-1}),
\end{align*}
where we regard $q\ind{T+1}:=\phat$.
\end{definition}

For statistical estimation problems, the definition above clearly recovers \cref{def:JDP}. It also recovers the definition of interactive JDP considered by \citet{shariff2018differentially,vietri2020private,he2022reduction}.

Similar to \cref{prop:JDP-lower}, we show that \dct~provides a lower bound for interactive learning under JDP.
\begin{proposition}[\Dct~lower bound for JDP learning]\label{prop:JDP-lower-interactive}
Let $T\geq 1$, and $\alg$ is a weak \pJDP~algorithm. Suppose that with $T$-round of interactions, $\alg$ achieves $\Riskdm(T)\leq\Delta$ with probability at least $\frac12$ under $\PP\sups{M,\alg}$ for any $M\in\cM$.
Then it holds that %
\begin{align*}
    T\geq \frac{\log\DC{\cM}-\log2}{\alpha}.
\end{align*}
\end{proposition}

Note that any \pLDP~algorithm preserves \pJDP. Hence, \cref{thm:JDP-equiv} naturally extends to interactive learning.

\subsection{Learnability of regression}\label{ssec:regression-DC}

In this section, we consider the learnability of the regression task, continuing \cref{sec:DC}. 
Recall that in \cref{ssec:regression}, we study \emph{proper} regression. More generally, in this section, we also consider the problem of \emph{improper} regression with a function class $\cFp$ not necessarily equal to $\cF$. 

In improper regression, the decision space is $\Pi=\cFp$, and the \losst~is defined as
\begin{align*}
    \LM{f}=\EE_{(x,y)\sim M} \loss(y, f(x))-\min_{\fs\in\cF} \EE_{(x,y)\sim M} \loss(y, \fs(x)), \qquad \forall f\in\cFp.
\end{align*}
Define the \dct~of the pair $(\cF,\cFp)$ as
\begin{align}\label{def:DC-Fp}
    \DCF\defeq \inf_{p\in\DFp} \sup_{\mu\in\DX, \fs\in\cF} ~\frac{1}{p\paren{ f: \EE_{x\sim \mu}|f(x)-\fs(x)| \leq \Delta } }.
\end{align}
When $\cFp=\cF$, this definition recovers the definition \cref{def:DC-cF-proper} of the \dct~of $\cF$.

We first relate $\DCF$ to the \dct~of $\cMagn$ under the absolute loss $\Labs(y,y')=\abs{y-y'}$.
\begin{lemma}\label{lem:DC-agn-regr}
Recall that $\cMagn=\DZ$ is the class of all agnostic models. Then, under the absolute loss $\Labs$ and decision space $\Pi=\cFp$, it holds that %
\begin{align*}
    \DC{\cMagn}=\DCF, \qquad \forall \Delta>0.
\end{align*}
More generally, for any 1-Lipschitz loss, we have $\DC{\cMagn}\leq \DCF$.
\end{lemma}
In particular, under absolute loss, the agnostic learnability with $(\cF, \cFp)$ is characterized by the finiteness of the complexity measure $\DCF$.

\paragraph{Realizable regression} We consider the ``easier'' task of realizable regression. Given the function class $\cF$, a model $M\in\DZ$ is realizable if there exists $\fm\in\cF$ such that for $(x,y)\sim M$, $y=\fm(x)$ with probability 1. Let $\cMreal$ be the class of all realizable models.
\begin{lemma}\label{lem:DC-real-regr}
Under the absolute loss $\Labs$, it holds that $\DC{\cMreal}=\DCF$ for $\Delta>0$.
\end{lemma}
Therefore, under absolute loss, the learnability of realizable regression is also characterized by the finiteness of the \dct~$\DCF$. In particular, the agnostic learnability is equivalent to the realizable learnability. A similar argument also applies to the squared loss, where we can show that $\DCF$ simultaneously characterizes the learnability of agnostic regression, well-specified regression (\cref{ssec:regression}), and realizable regression.

\paragraph{Separation between proper learning and improper learning} 
We show that, for high-dimensional linear model, there is a separation between proper and improper learning under LDP. More specifically, we consider $\cX\defeq \set{x\in\R^d: \nrm{x}\leq 1}$, and the function class $\cF$ given by
\begin{align*}
    \cFlin\defeq \set{ f_\theta(x)=\lr \theta, x\rr }_{\theta:\nrm{\theta}\leq 1}.
\end{align*}
\begin{proposition}\label{prop:impr-lin}
Let $\cFp\defeq \set{ f_\theta(x)=\lr \theta, x\rr }_{\theta\in\R^d}$ be the class of unbounded linear functions. Then it holds that
\begin{align*}
    \log\DC{\cF}\geq \Om{d}, \qquad 
    \log\DCF\leq \tbO{ \frac{1}{\Delta^2} }.
\end{align*}
\end{proposition}
Therefore, $d$-dimensional proper linear regression is infeasible when $d$ is unbounded, while improper learning is still tractable as $d\to \infty$. Proof appears in \cref{appdx:proof-impr-lin}.

\section{Technical Tools}\label{appdx:tools}
\newcommand{\filt}{\mathfrak{F}}
The following lemma can be regarded as a ``chain rule'' of Hellinger distance~\citep{jayram2009hellinger} (see also \citet[Lemma 11.5.3]{duchi-notes} or \citet[Lemma D.2]{foster2024online}).
\begin{lemma}[Sub-additivity for squared Hellinger distance]
\label{lem:Hellinger-chain}
  Let $(\cX^1,\filt^1),\ldots,(\cX^T,\filt^T)$ be a sequence of measurable spaces, and let $\cX\ind{t}=\prod_{i=1}^{t}\cX^i$ and $\filt\ind{t}=\bigotimes_{i=1}^{t}\filt^i$. For each $t$, let $\PP\ind{t}(\cdot\mid\cdot)$ and $\QQ\ind{t}(\cdot\mid\cdot)$ be probability kernels from $(\cX\ind{t-1},\filt\ind{t-1})$ to $(\cX^t,\filt^t)$. 
  
  Let $\PP$ and $\QQ$ be the laws of $X_1,\ldots,X_T$ under $X_t\sim\PP\ind{t}(\cdot\mid X_{1:t-1})$ and $X_t\sim\QQ\ind{t}(\cdot\mid X_{1:t-1})$ respectively. Then it holds that
\begin{align*}
  \Dhels{\PP}{\QQ}\leq 
7~\En_{\PP}\brk*{\sum_{t=1}^{T}\Dhels{\PP\ind{t}(\cdot\mid X_{1:t-1})}{\QQ\ind{t}(\cdot\mid X_{1:t-1})}}.
\end{align*}
\end{lemma}

We also invoke the Minimax theorem.
\begin{theorem}[{Ky Fan's minimax theorem, \citet{fan1953minimax}}]\label{thm:minimax}
Let $X$ be a compact Hausdorff space and Y an arbitrary
set (not topologized). Let $f$ be a real-valued function on $X\times Y$ such that,
for every $y\in Y$, $f(\cdot,y)$ is continuous over $X$. 

Then, if $f$ is convex-like on $X$ and concave-like on $Y$, then
\begin{align*}
    \min_{x\in X} \sup_{y\in Y} f(x,y) = \sup_{y\in Y} \min_{x\in X} f(x,y).
\end{align*}
\end{theorem}

Therefore, if $f$ is instead concave-like on $X$ and convex-like on $Y$,
then we can apply \cref{thm:minimax} to $-f$ to obtain
\begin{align*}
    \max_{x\in X} \inf_{y\in Y} f(x,y) = \inf_{y\in Y} \max_{x\in X} f(x,y).
\end{align*}

\begin{theorem}[{Kakutani’s fixed point theorem, \citet[Lemma 20.1]{osborne1994course}}]\label{thm:K-fixed-point}
Let $X$ be a compact convex subset of $\R^n$, and let $F: X \to \Power{X}$ be a set-valued function for which
\begin{enumerate}
    \item for all $x \in X$, the set $F(x)$ is nonempty and convex, and
    \item $F$ is upper hemicontinuous (i.e. for all sequences $x_n$ and $y_n$ such that $y_n\in F(x_n)$ for all $n$, $x_n\to x$, $y_n\to y$, then we have $y\in F(x)$).
\end{enumerate}
Then, there exists $x\in X$ such that $x\in F(x)$.
\end{theorem}

\section{Proofs for Lower Bounds}\label{appdx:lower}

\subsection{Proof of \cref{thm:cDMSO-pac-lower}}\label{appdx:DEC-general-lower}

In this section, we prove a more general version of \cref{thm:cDMSO-pac-lower} through the approach developed in \citet{chen2024beyond}, which applies to \emph{any} loss function $L$.

Given model class $\cM$, for each $\eps>0$ and $\delta\in[0,1]$, we define the quantile-based PAC DEC as
\begin{align}\label{def:p-dec-q-gen}
    \pdecqg_{\eps,\delta}(\MPow,\oM)\defeq \infpqb \sup_{\cP\in\MPow}\constr{ \hgm[\cP]{p} }{ \inf_{M\in\co(\cP)}\EE_{\bpi\sim q}\DH{ M(\bpi),\oM(\bpi) } \leq \eps^2 },
\end{align}
where $\hgm[\cP]{p}$ is the \emph{$\delta$-quantile loss} of $p$, defined as
\begin{align*}
    \hgm[\cP]{p}=\sup_{\Delta\geq 0}\set{\Delta: \PP_{\pi\sim p}(\LM[\cP]{\pi}\geq \Delta)\geq \delta }.
\end{align*}
We also denote $\pdecqg_{\eps,\delta}(\MPow)\defeq \sup_{\oM\in\coM} \pdecqg_{\eps,\delta}(\MPow,\oM)$.
By definition, the quantile-based PAC DEC is always bounded by the original hybrid PAC DEC:
\begin{align}\label{eqn:p-dec-q-to-c-trivial}
    \pdecg_{\eps}(\MPow,\oM)-\delta\leq \pdecqg_{\eps,\delta}(\MPow,\oM)\leq \delta^{-1}\pdecg_{\eps}(\MPow,\oM),
\end{align}
as long as the loss function is bounded in $[0,1]$.
However, such a conversion can be loose in general.

The advantage of considering the quantile \pDEC is that it provides the following unified lower bound for PAC learning under \cDMSO. Proof is presented in \cref{appdx:proof-cDMSO-pac-lower-quantile}.
\begin{proposition}[Quantile-based \gDEC~lower bound]\label{prop:p-dec-q-gen-lower}
For any $T\geq 1$ and constant $\delta\in[0,1)$, we denote $\ueps_\delta(T)\defeq \frac{1}{13}\sqrt{\frac{\delta}{T}}$. Then, under \cDMSO, for any $T$-round algorithm $\alg$, there exists $\cPs\in\MPow$ and a distribution $\mus\in\Delta(\cPs)$, such that for the stationary environment $\env$ specified by $\mus$, %
\begin{align*}
    \LM[\cPs]{\hpi}\geq \sup_{\oM}\pdecqg_{\ueps_\delta(T),\delta}(\MPow,\oM),\qquad\text{with probability at least $\delta/2$ under }\PP\sups{\env,\alg},
\end{align*}
where the supremum $\sup_{\oM}$ is taken over all reference models $\oM\in(\bPi\to\DO)$.
\end{proposition}

Subsequently, we specify the above lower bound to metric-based loss and any general loss function.

\paragraph{Application: metric-based loss function}
When the loss function is metric-based (\cref{def:metric-cP}), we can show that the quantile-based \gDEC~can be lower bounded by the original \gDEC. More specifically, we prove the following lemma.
\begin{lemma}\label{lem:p-dec-q-gen-metric}
Suppose that for some constant $C_1, C_2$, it holds that for any models $\cP, \cP'\in\MPow$, any decision $\pi\in\Pi$,
\begin{align}\label{asmp:gen-tri}
    \LM[\cP']{\pi}\leq C_1 \LM[\cP]{\pi}+C_2\inf_{\pi'}\paren{ \LM[\cP]{\pi'}+\LM[\cP']{\pi'} }.
\end{align}
Then for any $\delta\in[0,\frac12)$ and any reference model $\oM$, it holds that
\begin{align*}
    \pdecqg_{\eps,\delta}(\MPow,\oM)\geq \frac1{2C_2}\pdecg_{\eps}(\MPow,\oM).
\end{align*}
\end{lemma}
For example, when $\LM[\cP]{\pi}=\rho( \pip, \pi )$ for certain pseudo-metric $\rho$ over $\Pi$, \eqref{asmp:gen-tri} holds with $C_1=C_2=1$. Therefore, \eqref{asmp:gen-tri} can be viewed as a \emph{generalized} metric structure on the loss function $L$. In particular, \eqref{eq:GDEC-pac-lower} of \cref{thm:cDMSO-pac-lower} follows immediately from \cref{prop:p-dec-q-gen-lower} and \cref{lem:p-dec-q-gen-metric}.

\paragraph{Proof of \cref{thm:cDMSO-pac-lower}: \eqref{eq:GDEC-pac-lower}}
Suppose that the loss function $L$ is metric-based. Then, \cref{lem:p-dec-q-gen-metric} implies that $\pdecqg_{\eps,\delta}(\MPow)\geq \frac12\pdecg_{\eps}(\MPow)$ for any $\delta<\frac12$. Thus, applying \cref{prop:p-dec-q-gen-lower} yields
\begin{align*}
    \sup_{\env} \EE\sups{\env,\alg}\brac{ \riskdm(T) }\geq \frac{\delta}{2}
    \sup_{\oM}\pdecqg_{\ueps_\delta(T),\delta}(\MPow)\geq \frac{\delta}{4} \sup_{\oM}\pdecg_{\ueps_\delta(T)}(\MPow,\oM).
\end{align*}
Letting $\delta\to\frac12$ gives the desired lower bound:
\begin{align}\label{eq:GDEC-lower-metric-oM}
    \sup_{\env} \EE\sups{\env,\alg}\brac{ \riskdm(T) }\geq \frac18 \sup_{\oM}\pdecg_{\ueps(T)}(\MPow,\oM)\geq \frac18 \pdecg_{\ueps(T)}(\MPow).
\end{align}
\qed

Similarly, we can apply \cref{prop:p-dec-q-gen-lower} to general loss function.

\paragraph{Proof of \cref{thm:cDMSO-pac-lower}: \eqref{eq:GDEC-pac-lower-bad}}
By \eqref{eqn:p-dec-q-to-c-trivial}, we have
\begin{align*}
    \pdecqg_\eps(\MPow)\geq \pdecg_\eps(\MPow)-\delta.
\end{align*}
Hence, \eqref{eq:GDEC-pac-lower-bad} is a direct corollary of \cref{prop:p-dec-q-gen-lower}.
\qed

As a final remark, we note that under stochastic DMSO, if the loss function is reward-based (\cref{example:original-reward-based}), the quantile DEC can also be lower bounded by the constrained DEC (see~\citet{chen2024beyond} and also \cref{appdx:p-dec-lin-q}).

\subsection{Proof of \cref{thm:cDMSO-reg-lower}}\label{appdx:proof-cDMSO-reg-lower}

In this section, we prove \cref{thm:cDMSO-reg-lower} by first reducing to stochastic DMSO, and then apply the lower bound for stochastic DMSO (\cref{thm:decc-lower-general}).

\paragraph{Reduction from \cDMSO~to stochastic DMSO}
We first argue that for any problem under \cDMSO, the minimax regret can always be lower bounded by a corresponding stochastic DMSO problem. The idea follows from the observation of \citet{foster2022complexity}.

For any stationary environment $\env$ constrained by $\MPow$, $\env$ is specified by a constraint $\cP\in\MPow$ and $\mu\in\Delta(\cP)$. Then, for each round $t\in[T]$, the model $M^t\sim \mu$ independently, and hence conditional on $(\cH\ind{t-1},\bpi\ind{t})$, the observation $o\ind{t}\sim M_\mu(\bpi\ind{t})$, where $M_\mu=\EE_{M'\sim \mu}[M']\in\co(\cP)$. Therefore, for any $T$-round algorithm $\alg$, the marginal distribution of $\cH\ind{T}$ generated by $\alg$ under $\env$ agrees with the distribution of $\cH\ind{T}$ generated by $\alg$ under the model $M_\mu$, i.e.,
\begin{align*}
    \PP\sups{\env,\alg}\paren{\cH\ind{T}=\cdot}=\PP\sups{M_\mu,\alg}\paren{\cH\ind{T}=\cdot}.
\end{align*}
In particular, using the linearity of the value function, we have
\begin{align*}
    \EE\sups{\env,\alg}\brac{\sum_{t=1}^T \Vm[M^t](\pi\ind{t})}=\EE\sups{M_\mu,\alg}\brac{\sum_{t=1}^T \Vm[M_\mu](\pi\ind{t})},
\end{align*}
and hence
\begin{align*}
    \EE\sups{\env,\alg}\brac{\regdm(T)}
    \geq&~\max_{\pis\in\Pi}\EE\sups{\env,\alg}\brac{\sum_{t=1}^T \Vm[M^t](\pis)-\Vm[M^t](\pi\ind{t})} \\
    =&~\max_{\pis\in\Pi}\EE\sups{M_\mu,\alg}\brac{\sum_{t=1}^T \Vm[M_\mu](\pis)-\Vm[M_\mu](\pi\ind{t})}
    =\EE\sups{M_\mu,\alg}\brac{\regdm(T)}.
\end{align*}
Note that for any $M\in\cMPow$, there exists $\cP\in\MPow$ and $\mu\in\Delta(\cP)$ such that $M=\EE_{M'\sim \mu}[M']$, and hence there exists a corresponding stationary environment.
Therefore, for any algorithm $\alg$, it holds that
\begin{align}\label{eq:reduce-stationary-reg}
    \sup_{\textrm{stationary}~\env} \EE\sups{\env,\alg}\brac{\regdm(T)}\geq \sup_{M\in\cMPow}\Emalg{\regdm(T)},
\end{align}
where $\sup_{\env}$ is taken over all stationary environments $\env$ constrained by $\MPow$.

\paragraph{Reduction to the regret DEC lower bound}
Then, we invoke the following lower bound, which is strengthened from \citet{foster2023tight,glasgow2023tight,chen2024beyond}. The proof is deferred to \cref{appdx:proof-decc-reg-lower}.

\begin{theorem}[Constrained DEC lower bounds for stochastic DMSO]\label{thm:decc-lower-general}
Let $T\geq 1$, and $\cM\subseteq (\bPi\to \DO)$ be a given model class. Suppose that $V$ is a value function such that $\Vm(\pi)\in[0,\Vmax]$, and for any $\pi\in\Pi$, there exists $\phi_\pi\in\Phi$, such that
\begin{align}\label{eqn:lip-rew-reduction}
    \abs{ \Vm(\pi)-\Vm[\oM](\pi) }\leq \Lipr\dH\paren{M(\pi,\phi_\pi),\oM(\pi,\phi_\pi)},\qquad \forall M\in\cM, \oM\in\coM.
\end{align}
Then for any $T$-round algorithm $\alg$, it holds that
\begin{align*}
    \sup_{M\in\cM} \EE\sups{M,\alg}\brac{\regdm(T)}\geq  \frac{T}{8}\paren{\rdecc_{\ueps(T)}(\cM)-6\Lipr\ueps(T)-\frac{\Vmax}{T}}
\end{align*}
where $\ueps(T)=\frac{1}{24\sqrt{T}}$.
\end{theorem}

\cref{thm:cDMSO-reg-lower} is then proven by combining \eqref{eq:reduce-stationary-reg} and \cref{thm:decc-lower-general}.
\qed

\subsection{Instantiations}

In the following, we extend the discussion in \cref{sec:overview} and apply \cref{thm:cDMSO-pac-lower} and \cref{thm:cDMSO-reg-lower} to prove the lower bounds for query-based learning and LDP learning.

\subsubsection{Query-based learning: Proof of \cref{thm:SQ-lower}}\label{appdx:gen-to-sq-dec}

In this section, we formalize the discussion in \cref{ssec:query-demo} and prove that the \SQDEC~can be derived from the \gDEC~with $\MPow=\MPowsq$. In particular, we derive \cref{thm:SQ-lower} from \cref{thm:cDMSO-pac-lower}. Alternatively, a direct proof of \cref{thm:cDMSO-pac-lower} is presented in \cref{appdx:proof-SQ-lower}.

\paragraph{From \gDEC~to \SQDEC}
The key observation is the following lemma, which relates the squared Hellinger distance to the ``error probability''-style quantity in the definition of \SQDEC~\cref{def:sq-dec}.
\begin{lemma}\label{lem:Hellinger-to-indic}
Suppose that $P\in\DO$, and $\cO_0\subseteq \cO$ is a measurable subset of $\cO$. Then it holds that
\begin{align*}
    \frac{1}{2}P(\cO_0^c)\leq \inf_{P': \supp(P')\subseteq \cO_0} \DH{P', P} \leq P(\cO_0^c).
\end{align*}
\end{lemma}
Note that $\cP_M$ consists of all models $M'$ such that $\supp(M'(\bpi))\subseteq \set{ v: \nrm{M(\bpi)-v}\leq \tau }$ for all $\pi\in\Pi$, and particularly, $\cP_M$ is convex. Therefore, we can bound the quantity
\begin{align*}
    \inf_{M'\in\cP_M}\EE_{\bpi\sim q} \DH{ M'(\bpi), \oM(\bpi) }=\EE_{\bpi\sim q} \inf_{M'\in\cP_M}\DH{ M'(\bpi), \oM(\bpi) }
\end{align*}
using \cref{lem:Hellinger-to-indic}:
\begin{align*}
    \frac12\PP_{\bpi\sim q, v\sim \oM(\bpi)}\paren{ \nrm{M(\bpi)-v}>\tau }
    \leq \inf_{M'\in\cP_M}\EE_{\bpi\sim q} \DH{ M'(\bpi), \oM(\bpi) }\leq \PP_{\bpi\sim q, v\sim \oM(\bpi)}\paren{ \nrm{M(\bpi)-v}>\tau }.
\end{align*}
Therefore, we have proven the following lemma.
\begin{lemma}\label{lem:gen-to-sq-dec}
Suppose that $\tau\geq0$, $\MPowsq$ is specified by the model class $\cM\subseteq (\bPi\to\cV)$. Then, for any reference model $\oM:\bPi\to\Delta(\cV)$, it holds that
\begin{align*}
    \pdecltau_{\eps/2}(\cM,\oM)\leq \pdecg_{\eps}(\MPowsq,\oM)\leq \pdecltau_{\eps}(\cM,\oM), \qquad \forall \eps\geq 0.
\end{align*}
In particular, we have $\pdecg_{\eps}(\MPowsq)\leq \pdecltau_{\eps}(\cM)$.\footnote{The converse might not hold, because in our definition \cref{def:sq-dec} of \SQDEC, the supremum is taken over \emph{all} reference models $\oM:\bPi\to\Delta(\cV)$.}
\end{lemma}

\paragraph{Proof of \cref{thm:SQ-lower}}
For metric-based loss $L$, we can apply \eqref{eq:GDEC-lower-metric-oM} with $\MPowsq$:
\begin{align*}
    \sup_{\env} \EE\sups{\env,\alg}\brac{ \riskdm(T) }\geq \frac18 \sup_{\oM}\pdecg_{\ueps(T)}(\MPow,\oM)
    \geq \frac18\pdecltau_{\ueps(T)/2}(\cM),
\end{align*}
where the supremum is taken over all environments specified by a GQ oracle $\GSQ$ with respect to a model $M\in\cM$, and the second inequality follows from \cref{lem:gen-to-sq-dec}. Similarly, for more general loss $L$, a lower bound in terms of $\pdecltau_{\eps}(\cM)$ also follows from \eqref{eq:GDEC-pac-lower-bad} of \cref{thm:cDMSO-pac-lower}.
\qed

\paragraph{Proof of \cref{lem:Hellinger-to-indic}}
We first consider the distribution $P_0=P(\cdot|o\in\cO_0)$. Clearly, $\supp(P_0)\subseteq \cO_0$, and
\begin{align*}
    \DH{P_0,P}=\frac{1}{2}\brac{ P(\cO_0^c)+\paren{1-\sqrt{P(\cO_0)}}^2 }
    \leq \DTV{P_0,P}=P(\cO_0^c).
\end{align*}
Hence, the upper bound is proven.

Next, we proceed to prove the lower bound. For any $P'\in\DO$ such that $\supp(P')\subseteq \cO_0$, we fix a base measure $\mu$, and then
\begin{align*}
    \DH{P_0,P}=&~1-\int_{\cO} \sqrt{\frac{dP}{d\mu}\cdot \frac{dP'}{d\mu}} \mu(do) \\
    =&~ 1- \sqrt{P(\cO_0)}\int_{\cO_0} \sqrt{\frac{dP_0}{d\mu}\cdot \frac{dP'}{d\mu}} \mu(do) \\
    =&~ 1-\sqrt{P(\cO_0)}+\sqrt{P(\cO_0)}\DH{P',P_0} \\
    \geq&~ 1-\sqrt{P(\cO_0)} \geq \frac{1}{2}P(\cO_0^c).
\end{align*}
This gives the desired lower bound.
\qed

\subsubsection{LDP learning: Proof of \cref{thm:pdec-lin-lower} (1) and \cref{thm:rdec-lin-lower}}\label{appdx:inst-LDP-pac-lower}

We first recall the discussion in \cref{ssec:LDP-demo}: 
Given a model class $\cM \subseteq (\Pi\to \DZ)$ and the class $\Pc=\Pcp$ of all \pDP~channels (from $\cZ$ to $\cO$), each model $M\in\cM$ induces a map $\tM:\Pi\times\Pc\to\DO$ by $\tM(\pi,\pr)=\pr\circ M(\pi)$ for all $\pi\in\Pi$, $\pr\in\Pc$. Therefore, $\cM$ induces a model class $\tcM$ under \cDMSO:
\begin{align}\label{def:tcM}
    \tcM\defeq \sset{\tM: M\in\cM}\subseteq (\Pi\times\cQ\to\DO)
\end{align}
Then, a direct application of \cref{prop:pLDP} yields the following lemma.

\begin{lemma}\label{lem:decg-to-decl}
Let the model class $\cM\subseteq (\Pi\to\DZ)$ be given. For the corresponding constraint class $\MPowdp=\set{{\tM}:M\in\cM}$, it holds that
\begin{align*}
    \pdecl_{c_0\eps/\alpha}(\cM)\leq \pdecg_\eps(\MPowdp)=\pdecc_\eps(\tcM)\leq \pdecl_{c_1\eps/\alpha}(\cM), \quad\forall \eps>0,
\end{align*}
where $c_0, c_1>0$ are universal constants. 
Similarly, we also have
\begin{align*}
    \rdecl_{c_0\eps/\alpha}(\cM)\leq \rdecg_\eps(\MPowdp)=\rdecc_\eps(\tcM)\leq \rdecl_{c_1\eps/\alpha}(\cM), \quad\forall \eps>0.
\end{align*}
\end{lemma}

Therefore, there is an equivalence between the \gDEC s and the private DECs. Based on such an equivalence, we apply the \gDEC~lower bounds (\cref{thm:cDMSO-pac-upper} and \cref{thm:cDMSO-reg-lower}) to prove \cref{thm:pdec-lin-lower} (1) and \cref{thm:rdec-lin-lower}. The proof of \cref{thm:pdec-lin-lower} (2) is deferred to \cref{appdx:p-dec-lin-q}, as it involves the specific properties of reward-based loss.

\paragraph{Proof of \cref{thm:pdec-lin-lower} (1)}
Fix a $T$-round \pLDP~algorithm $\alg$. Then, by \cref{thm:cDMSO-pac-lower}, it holds that
\begin{align*}
    \sup_{M\in\cM} \Emalg{\riskdm(T)}\geq \frac{1}{8}\pdecg_{\ueps(T)}(\MPowdp)\geq \frac{1}{8}\pdecl_{c_0\ueps(T)/\alpha}(\cM),
\end{align*}
where $\ueps(T)=\frac{1}{20\sqrt{T}}$, and the second inequality follows from \cref{lem:decg-to-decl}. 
\qed

\paragraph{Proof of \cref{thm:rdec-lin-lower}}
We only need to verify \cref{asmp:lip-rew}.
For any decision $\pi\in\Pi$, we consider the binary channel $\pr_\pi\in\Pcp$ given by
\begin{align*}
    \pr_\pi(+1|z)=\frac{1+\ca R(z,\pi)}{2}, \qquad
    \pr_\pi(-1|z)=\frac{1-\ca R(z,\pi)}{2},
\end{align*}
where $\ca=1-\eai$ ensures that $\pr_\pi$ is \pDP~(cf. \cref{example:binary-pr}), and we assume without loss of generality that $\set{-1,1}\subseteq \cO$. Then, by definition, it holds that
\begin{align*}
    \ca \abs{\Vm(\pi)-\Vm[\oM](\pi)}\leq \DTV{ M(\pi,\pr_\pi), \oM(\pi,\pr_\pi) }
    \leq \sqrt{2}\DHr{M(\pi,\pr_\pi), \oM(\pi,\pr_\pi)}.
\end{align*}
Therefore, \cref{asmp:lip-rew} holds with $\Lipr=\frac{\sqrt{2}}{\ca}=\bigO{\frac1\alpha}$. Hence, for any \pLDP~algorithm $\alg$, \cref{thm:cDMSO-reg-lower} yields
\begin{align*}
    \sup_{M\in\cM} \Emalg{\regdm(T)}\geq&~ \frac{T}{8}\paren{ \pdecg_{\ueps(T)}(\MPowdp) -6\Lipr\ueps(T) } -1 \\
    \geq&~ \frac{T}{8}\paren{ \rdecl_{c_0\ueps(T)/\alpha}(\cM) - \frac{6\sqrt{2}\ueps(T)}{\ca}}-1,
\end{align*}
where $\ueps(T)=\frac{1}{24\sqrt{T}}$, and the second inequality follows from \cref{lem:decg-to-decl}. This gives the desired lower bound.
\qed

\paragraph{Proof of \cref{lem:decg-to-decl}}
We begin with the first inequality for \pDEC. By \cref{prop:pLDP}, for any \pDP~channel $\pr\in\Pc$, there exists a distribution $\tq_{\pr}\in\DL$, such that
\begin{align*}
    \DH{ \tM(\pi,\pr), \toM(\pi,\pr) } \leq \frac{(\ea-1)^2}{8} \EE_{\lf\sim \tq_\pr} \Dl^2(M(\pi),\oM(\pi)), \qquad \forall \pi\in\Pi, M\in\cM, \oM\in\coM.
\end{align*}
Therefore, for any $q\in\Delta(\Pi\times\Pc)$, there exists $\tq\in\DPL$ such that
\begin{align*}
    \sset{M\in\cM: \EE_{(\pi,\lf)\sim \tq} \Dl^2(M(\pi),\oM(\pi)) \leq \paren{c_0\eps/\alpha}^2}\subseteq \sset{ M\in\cM: \EE_{\bpi\sim q} \DH{\tM(\bpi),\toM(\bpi)}\leq \eps^2 },
\end{align*}
where $c_0>0$ is a lower bound of $\frac{4\alpha}{\ea-1}$ that only depends on $\alpha_0$. Then, by the definition of \pDEC, we know
\begin{align*}
    \pdecc_\eps(\tcM,\toM)\geq \pdecl_{c_0\eps/\alpha}(\cM,\oM),\qquad \forall \oM\in\coM.
\end{align*}
Note that $\co(\tcM)=\ext{\coM}$, and hence we have $\pdecc_\eps(\tcM)\geq \pdecl_{c_0\eps/\alpha}(\cM)$.

Next, we prove the second inequality for the \pDEC. Recall that for any $\lf\in\cL$, there is a corresponding binary channel $\pr_\lf$, such that
\begin{align*}
    \DH{ \pr_\lf \circ P_1, \pr_\lf \circ P_2 }
    \geq&~ \frac{1}{2}\DTV{ \pr_\lf \circ P_1, \pr_\lf \circ P_2 }^2 \\
    =&~\frac{1}{2}\DTV{ \Bern{\frac{1+\ca\EE_{P_1}\lf(z)}{2}}, \Bern{\frac{1+\ca\EE_{P_2}\lf(z)}{2}} }^2 \\
    =&~ \frac{\ca^2}{8}\abs{\EE_{P_1}\lf(z)-\EE_{P_2}\lf(z)}^2
    \geq (\alpha/c_1)^2\Dl^2(P_1,P_2),
\end{align*}
where $c_1>0$ is a upper bound of $\frac{4\alpha}{\ca}$ that only depend on $\alpha_0$.
Therefore, for any $q\in\DPL$, there exists $q'\in\Delta(\Pi\times\Pc)$, such that
\begin{align*}
    \sset{ M\in\cM: \EE_{\bpi\sim q'} \DH{\tM(\bpi),\toM(\bpi)}\leq \eps^2 }\subseteq \sset{M\in\cM: \EE_{(\pi,\lf)\sim q} \Dl^2(M(\pi),\oM(\pi)) \leq \paren{c_1\eps/\alpha}^2}.
\end{align*}
Then, by the definition of \pDEC, we know
\begin{align*}
    \pdecc_\eps(\tcM,\toM)\leq \pdecl_{c_1\eps/\alpha}(\cM,\oM),\qquad \forall \oM\in\coM,
\end{align*}
and hence $\pdecc_\eps(\tcM)\leq \pdecl_{c_1\eps/\alpha}(\cM)$.

The bounds for \rDEC~can be proven analogously, and we omit the proof for succinctness.
\qed

\subsection{Proof of \cref{prop:p-dec-q-gen-lower}}\label{appdx:proof-cDMSO-pac-lower-quantile}

Following \citet{foster2021statistical}, we first introduce some notations.

Recall that an algorithm $\alg = \set{q\ind{t}}_{t\in [T]}\cup\set{p}$ in \cDMSO~is specified by a sequence of mappings, where the $t$-th mapping $q\ind{t}(\cdot\mid{}\Hy\ind{t-1})$ specifies the distribution of $\bpi\ind{t}=(\pi\ind{t},\phi\ind{t})$ based on the history $\Hy\ind{t-1}$, and the final map $p(\cdot \mid{} \Hy\ind{T})$ specifies the distribution of $\pihat$ based on $\Hy\ind{T}$. Therefore, for any model $M:\bPi\to\DO$, we define
\begin{align}\label{def:pM-qM}
q_{M,\alg}=\Emalg{ \frac{1}{T} \sum_{t=1}^{T} q\ind{t}(\cdot|\cH\ind{t-1}) }\in \DbPi, \quad 
p_{M,\alg}=\Emalg{p(\cH\ind{T})} \in \DPi,
\end{align}
The distribution $q_{M,\alg}$ is the expected distribution of the average profile $(\bpi\ind{1},\cdots,\bpi\ind{T})$, and $p_{M,\alg}$ is the expected distribution of the output decision $\hpi$. 

Using the sub-additivity of the squared Hellinger divergence (by \cref{lem:Hellinger-chain}, see e.g., \citet[Section 3.2]{chen2024beyond}), for any model $M, \oM$, it holds that 
\begin{align}\label{eqn:chain-Hellinger}
\DH{\PP\sups{M,\alg}, \PP\sups{\oM,\alg}}  
\leq &~ 7 \EE\sups{\oM,\alg}\brac{ \sum_{t=1}^T \DH{ M(\bpi\ind{t}), \oM(\bpi\ind{t}) } } \\
=&~ 7T\cdot \EE_{\bpi\sim q_{\oM,\alg}} \DH{ M(\bpi), \oM(\bpi) }.
\end{align}

With \eqref{eqn:chain-Hellinger}, we now present the proof of \cref{prop:p-dec-q-gen-lower}~(which is essentially following the analysis in \citet{chen2024beyond}).

\paragraph{Proof of \cref{prop:p-dec-q-gen-lower}}
We abbreviate $\eps=\ueps(T)$. Fix a $\Delta<\sup_{\oM}\pdecqg_{\eps,\delta}(\MPow,\oM)$, and then there exists $\oM$ such that $\Delta<\pdecqg_{\eps,\delta}(\MPow,\oM)$. Hence, by the definition \cref{def:p-dec-q-gen}, we know that
\begin{align*}
    \Delta<\sup_{\cP\in\MPow}\constr{ \hgm[\cP]{p_{\oM,\alg}} }{ \inf_{M\in\co(\cP)}\EE_{\bpi\sim q_{\oM,\alg}}\DH{ M(\bpi) , \oM(\bpi) }\leq \eps^2 }.
\end{align*}
Therefore, there exists $\cPs\in\MPow$  and $\Mstar\in\co(\cPs)$ such that
\begin{align*}
    \EE_{\bpi\sim q_{\oM,\alg}}\DH{ \Mstar(\bpi) , \oM(\bpi) }\leq \eps^2, \qquad
    p_{\oM,\alg}(\pi: \LM[\cPs]{\pi}\geq \Delta)\geq \delta.
\end{align*}
By \eqref{eqn:chain-Hellinger}, we know
\begin{align*}
    \DH{\PP\sups{\Mstar,\alg}, \PP\sups{\oM,\alg}} \leq 7T\eps^2.
\end{align*}
Because $\Mstar\in\co(\cPs)$, there exists a distribution $\mus\in\Delta(\cPs)$ such that $\Mstar=\EE_{M\sim \mus}[M]$. Then, for the stationary environment $\env$ specified by $\mus$ (i.e., it selects $M^t\sim \mus$ independently), it holds that $\PP\sups{\env,\alg}(\cH\ind{T}=\cdot)=\PP\sups{\Mstar,\alg}(\cH\ind{T}=\cdot)$.
Therefore, by data-processing inequality, we have
\begin{align*}
    \frac12\paren{ \sqrt{\PP\sups{\env,\alg}(L(\cPs,\hpi)\geq\Delta)}-\sqrt{\PP\sups{\oM,\alg}(L(\cPs,\hpi)\geq\Delta)} }^2 \leq \DH{ \PP\sups{\env,\alg}, \PP\sups{\oM,\alg} } \leq 7T\eps^2.
\end{align*}
Therefore, combining the inequalities above, we have
\begin{align*}
    \PP\sups{\env,\alg}(L(\cPs,\hpi)\geq\Delta)\geq  \paren{ \sqrt{p_{\oM,\alg}(\pi: \LM[\cPs]{\pi}\geq  \Delta )}-\sqrt{14T\eps^2} }^2 \geq \frac{\delta}{2},
\end{align*}
where we use $p_{\oM,\alg}(\pi: \LM[\cPs]{\pi} \geq \Delta )\geq \delta$ and $\sqrt{14T\eps^2}\leq (1-\frac{1}{\sqrt{2}})\sqrt{\delta}$.

Letting $\Delta\to \pdecqg_{\eps,\delta}(\MPow)$ completes the proof.
\qed

\subsection{Proof of \cref{lem:p-dec-q-gen-metric}}

Fix a reference model $\oM$ and let $\Delta_0>0 \vee \pdecq_{\eps,\delta}(\MPow,\oM)$. Then there exists $p\in\DPi, q\in\DbPi$ such that
\begin{align*}
    \sup_{\cP\in\MPow}\constr{ \hgm[\cP]{p} }{ \inf_{M\in\co(\cP)}\EE_{\bpi\sim q}\DH{ M(\bpi) , \oM(\bpi) }\leq \eps^2 }< \Delta_0.
\end{align*}
Therefore, we denote
\begin{align*}
    \MPow_{q,\eps}(\oM)\defeq \set{ \cP\in\MPow: \inf_{M\in\co(\cP)}\EE_{\bpi\sim q}\DH{ M(\bpi) , \oM(\bpi) }\leq \eps^2 },
\end{align*}
and it holds that
\begin{align*}
    \PP_{\pi\sim q}(\LM[\cP]{\pi}\geq \Delta_0)<\delta, \qquad \forall \cP\in\MPow_{q,\eps}(\oM).
\end{align*}
If the constrained set $\MPow_{q,\eps}(\oM)$ is empty, then we immediately have $\pdecc_{\eps}(\MPow,\oM)=-\infty<\Delta_0$, and the proof is completed. Therefore, in the following we may assume $\MPow_{q,\eps}(\oM)$ is non-empty, and fix a model $\cP_0\in \MPow_{q,\eps}(\oM)$.

Notice that for any model $\cP\in \MPow_{q,\eps}(\oM)$, we have
\begin{align*}
    \PP_{\pi\sim q}(\LM[\cP]{\pi}< \Delta_0, \LM[\cP_0]{\pi}< \Delta_0)\geq 1-2\delta>0,
\end{align*}
and hence
\begin{align*}
    \inf_\pi\paren{ \LM[\cP]{\pi}+\LM[\cP_0]{\pi} }\leq 2\Delta_0.
\end{align*}
Therefore, \eqref{asmp:gen-tri} implies that
\begin{align*}
    \LM[\cP]{\pi}\leq C_1\LM[\cP_0]{\pi}+2C_2\Delta_0, \qquad \forall \cP\in \MPow_{q,\eps}(\oM).
\end{align*}
Hence, we can take any $\pis$ such that $\LM[\cP_0]{\pis}=0$, and let $p\in\DPi$ be supported on $\pis$. Then, $(p,q)$ certifies that
\begin{align*}
    \pdecg_{\eps}(\MPow,\oM)\leq \sup_{\cP\in\MPow}\constr{ \LM[\cP]{\pis} }{ \inf_{M\in\co(\cP)}\EE_{\bpi\sim q}\DH{ M(\bpi) , \oM(\bpi) }\leq \eps^2 } \leq 2C_2\Delta_0.
\end{align*}
Letting $\Delta_0\to\pdecqg_{\eps,\delta}(\MPow,\oM)$ yields $\pdecl_{\eps}(\MPow,\oM)\leq 2C_2 \cdot \pdecqg_{\eps,\delta}(\MPow,\oM)$, which is the desired result.
\qed

\subsection{Proof of \cref{thm:decc-lower-general}}\label{appdx:proof-decc-reg-lower}

Fix a $T$-round algorithm $\alg$ and a reference model $\oM\in\coM$. Denote $\eps\defeq \ueps(T)$ and $\Delta\defeq \rdecc_{\eps}(\cMp,\oM)$. It remains to prove the following claim.

\textbf{Claim.} It holds that
\begin{align}
    \sup_{M\in\cM} \EE\sups{M,\alg}\brac{\regdm(T)}\geq  \frac{T}{8}\paren{\Delta-6\Lipr\eps-\frac{\Vmax}{T}}.
\end{align}

\paragraph{Proof of the claim}
We set $\Delta_0\defeq \frac{1}{2}\paren{\Delta-\sqrt{2}\Lipr \eps-\frac{\Vmax}{T}}$. If $\Delta_0\leq 0$, then the claim is vacuous. In the following, we focus on the case $\Delta_0>0$.

Fix an arbitrary $\phi_0\in\Phi$. For each decision $\pi\in\Pi$, we let $\phi_\pi\in\Pi$ be an associated measurement such that \eqref{eqn:lip-rew-pr} holds. 

Consider a modified algorithm $\alg':$ for $t=1, \cdots, T$, and history $\cH\ind{t-1}$, we set $\qt'(\cdot|\cH\ind{t-1})=\qt(\cdot|\cH\ind{t-1})$ if the quantity $G_{t-1}\defeq \sum_{s=1}^{t-1} \brac{ \Vmm[\oM]-\EE_{\pi\sim q\ind{s}}\Vm[\oM](\pi) }<T\Delta_0$, and set $\qt'(\cdot|\cH\ind{t-1})$ be supported on $(\pim[\oM],\phi_0)$ if otherwise. By our construction, it holds that under $\alg'$, 
\begin{align*}
    G_T=\sum_{s=1}^{T} \brac{ \Vmm[\oM]-\EE_{\pi\sim q\ind{s}}\Vm[\oM](\pi) }<T\Delta_0+\Vmax.
\end{align*}
Furthermore, we can define the stopping time
\begin{align*}
    \tau=\max\sset{ t\in[T]: G_{t-1}< T\Delta_0} .
\end{align*}
If $\tau<T$, then it holds that $G_T=G_\tau\geq T\Delta_0$.

Now, we consider $p_0:=\EE\sups{\oM,\alg'}\brac{ \frac{1}{T} \sum_{t=1}^{T} q\ind{t}(\cdot|\cH\ind{t-1}) }\in \DbPi$ (following \cref{appdx:proof-cDMSO-pac-lower-quantile}). We let $p_0'\in\DPi$ be the marginal distribution of $\pi$ under $(\pi,\phi)\sim p_0$, and $p_1$ be the distribution of $(\pi,\phi_\pi)$ with $\pi\sim p_0'$. We set $p=\frac{1}{2}\paren{p_0+p_1}$. 

Note that $p_0'$ is the marginal distribution of $\pi\sim p$. Thus,
\begin{align*}
    \EE_{\pi\sim p}\brac{ \Vmm[\oM]-\Vm[\oM](\pi) }
    =&~ \EE_{\pi\sim p_0'}\brac{ \Vmm[\oM]-\Vm[\oM](\pi) } \\
    =&~ \frac1T \EE\sups{\oM,\alg'}\brac{ \sum_{t=1}^T \Vmm[\oM]-\Vm[\oM](\pi\ind{t}) } \\
    =&~ \frac{1}{T}\EE\sups{\oM,\alg'}\brac{G_T}<\Delta_0+\frac{\Vmax}{T}\leq \Delta.
\end{align*}
Therefore, by the definition of $\rdecc_\eps(\cMp,\oM)$, there exists $M\in\cM$ such that
\begin{align*}
    \EE_{\pi\sim p}\brac{ \Vmm-\Vm(\pi) }\geq \Delta, \qquad 
    \EE_{\bpi\sim p} \DH{ M(\bpi), \oM(\bpi) }\leq \eps^2.
\end{align*}
 We also have
\begin{align*}
    \EE_{\pi\sim p_0'} \abs{ \Vm(\pi)-\Vm[\oM](\pi) }^2
    \leq&~ \Lipr^2 \EE_{\pi\sim p_0'} \DH{ M(\pi,\phi_\pi), \oM(\pi,\phi_\pi) } \\
    =&~ \Lipr^2 \EE_{\bpi\sim p_1} \DHr{ M(\bpi), \oM(\bpi) } \\
    \leq&~ 2\Lipr^2\eps^2.
\end{align*}
Therefore, we have
\begin{align*}
    \Vmm-\Vmm[\oM]
    =&~ \brac{\Vmm-\EE_{\pi\sim p} \Vm(\pi)}+\EE_{\pi\sim p} \brac{ \Vm(\pi)-\Vm[\oM](\pi) }- \brac{\Vmm[\oM]-\EE_{\pi\sim p} \Vm[\oM](\pi)} \\
    \geq&~ \Delta-\sqrt{2}\Lipr\eps-\paren{ \Delta_0+ \frac{\Vmax}{T}}\geq \Delta_0.
\end{align*}

In the following, we proceed to lower bound $\regdm(\tau)$ under model $M$ and algorithm $\alg'$. Consider the random variable
\begin{align*}
    X=\sum_{t=1}^T \EE_{\pi\sim \qt} \abs{ \Vm(\pi)-\Vm[\oM](\pi) }.
\end{align*}
We then bound
\begin{align*}
    \regdm(\tau)=&~\sum_{t=1}^\tau \brac{ \Vmm-\EE_{\pi\sim \qt}\Vm(\pi) } \\
    =&~ \tau (\Vmm-\Vmm[\oM])+\sum_{t=1}^\tau \brac{ \Vmm[\oM]-\EE_{\pi\sim \qt}\Vm[\oM](\pi) }+\sum_{t=1}^\tau \brac{ \EE_{\pi\sim \qt}\Vm[\oM](\pi)-\EE_{\pi\sim \qt}\Vm(\pi) } \\
    \geq&~ \tau(\Vmm-\Vmm[\oM])+G_\tau - X,
\end{align*}
where the last line follows from the definition of $G_\tau$ and $X$. Note that if $\tau<T$, we have $G_\tau\geq T\Delta_0$. Otherwise, we have $\tau=T$ and $G_\tau\geq 0$. Therefore, under model $M$, it holds that (almost surely)
\begin{align*}
    \regdm(\tau)
    \geq&~ \min\sset{ T(\Vmm-\Vmm[\oM]), T\Delta_0 } - X
    \geq T\Delta_0-X.
\end{align*}

Consider the event $\cE\defeq \set{ X>TC\Lipr \eps }$. By Markov's inequality,
\begin{align*}
    \PP\sups{\oM,\alg'} \paren{ \cE }\leq&~  \frac{\EE\sups{\oM,\alg'}X^2}{(TC\Lipr\eps)^2} \\
    \leq&~ \frac{1}{T(C\Lipr\eps)^2}\EE\sups{\oM,\alg'}\brac{ \sum_{t=1}^T \EE_{\pi\sim \qt}\abs{ \Vm(\pi)-\Vm[\oM](\pi) }^2 } \\
    =&~\frac{1}{(C\Lipr\eps)^2}\EE_{\pi\sim p_0'}\brac{ \abs{ \Vm(\pi)-\Vm[\oM](\pi) }^2 }
    \leq \frac{2}{C^2}.
\end{align*}
Further, by \eqref{eqn:chain-Hellinger}, we have
\begin{align*}
\DH{\PP\sups{M,\alg'}, \PP\sups{\oM,\alg'}} 
\leq &~ 7T\cdot \EE_{\bpi\sim p_0} \DH{ M(\bpi), \oM(\bpi) }\leq 14T\eps^2.
\end{align*}
Therefore, by data-processing inequality, it holds that
\begin{align*}
\abs{\PP\sups{M,\alg'}(\cE)-\PP\sups{\oM,\alg'}(\cE)} \leq \DTV{\PP\sups{M,\alg'}, \PP\sups{\oM,\alg'}}\leq \sqrt{28T\eps^2},
\end{align*}
which gives $\PP\sups{M,\alg'}(\cE)\leq \frac{2}{C^2}+\sqrt{28T\eps^2}$. 

Note that under the event $\cE^c$, we have $X\leq TC\Lipr\eps$. Therefore, we can lower bound
\begin{align*}
    \EE\sups{M,\alg'}\brac{ \regdm(\tau) }
    \geq&~  \EE\sups{M,\alg'}\brac{ \indic{\cE^c} \regdm(\tau) } \\
    \geq&~ \EE\sups{M,\alg'}\brac{ \indic{\cE^c}\paren{ T\Delta_0-X } } \\
    \geq&~ \PP\sups{M,\alg'}(\cE^c)\cdot T\paren{\Delta_0-C\Lipr\eps}  \\
    \geq&~ \paren{1-\frac{2}{C^2}-\sqrt{28T\eps^2}}\cdot T\paren{\Delta_0-C\Lipr\eps}.
\end{align*}
In particular, we can choose $C=2$, and by the choice $\eps=\frac{1}{24\sqrt{T}}$, we have $\EE\sups{M,\alg'}\brac{ \regdm(\tau) }\geq \frac{T}{4}\paren{\Delta_0-2\Lipr\eps}$. Then, we can conclude that
\begin{align*}
    \EE\sups{M,\alg}\brac{ \regdm(T) }\geq \EE\sups{M,\alg}\brac{ \regdm(\tau) }
    =\EE\sups{M,\alg'}\brac{ \regdm(\tau) }
    \geq \frac{T}{4}\paren{\Delta_0-2\Lipr\eps}.
\end{align*}
This gives the desired lower bound.
\qed

\section{Exploration-by-Optimization Algorithm and Guarantees}
\label{appdx:ExO}

\newcommand{\infosets}{information set structure}
\newcommand{\Infosets}{Information set structure}

\newcommand{\DPs}{\Delta(\Psi)}
\newcommand{\wf}{\xi}
\newcommand{\wF}{\Xi}

\newcommand{\optpac}{{\normalfont\textsf{pac}}\xspace}
\newcommand{\optreg}{{\normalfont\textsf{reg}}\xspace}

In this section, we present a generalization of the \ExO~Algorithm (\ExOp) developed by~\citet{foster2022complexity}, which is built upon~\citet{lattimore2020exploration,lattimore2021mirror} and is later extended by \citet{chen2024beyond}. The \ExOp~algorithm of \citet{foster2022complexity} has an \emph{adversarial} regret guarantee for any model class $\cM$, scaling with the \emph{offset DEC} of the convexified model class $\coM$ and $\log|\Pi|$. 
For our purpose, we adapt it by incorporating certain measurement class $\Phc$ and \emph{\infosets} $\Psi$, so that it (1) handles any \cDMSO~problem, and also (2) adapts to the structure of the decision space (e.g. capable of achieving an upper bound that scales with \dct, \cref{appdx:proof-regret-via-cDMSO}).

We organize this section as follows:
\begin{itemize}
\item In \cref{appdx:info-sets}, we introduce the notion of \emph{\infosets}.
\item In \cref{appdx:exo-full}, we present the detailed description of \ExOp~algorithm based on a given \infosets~$\Psi$.
\item In \cref{appdx:exo-dec-upper}, we bound the risk (regret) of \ExOp~in terms of the offset DEC and the \dct~associated with $(\MPow,\Psi)$.
\item In \cref{appdx:exo-inst}, we instantiate \ExOp~to prove the upper bounds of \cref{sec:cDMSO-main}.
\item In \cref{appdx:ExO-LDP}, we apply \ExOp~to \pDMSO~to obtain the \LDPexo~algorithm the upper bounds of \cref{sec:LDP}. Specifically, we instantiate \LDPexo~with the following \infosets:
\begin{itemize}
\item Model-based information sets (\cref{appdx:ExO-model-based}), where we prove \cref{thm:rdec-lin-upper} (1).
\item Policy-based information sets (\cref{appdx:ExO-policy-based}), where we prove \cref{thm:p-dec-convex-upper} and \cref{thm:rdec-lin-upper} (2).
\item Value-based information sets (\cref{appdx:ExO-val-based}), where we prove \cref{prop:real-regression}.
\item Contextual bandits (\cref{appdx:ExO-CB}), where we prove \cref{thm:CB-adv}.
\end{itemize}
\item The remaining subsections contain the proofs of the results of this section.
\end{itemize}

\paragraph{Offset DECs} For a model class $\cM\subseteq (\Pi\times\Phi\to\DO)$ under \cDMSO, we define the offset DECs~\citep{foster2021statistical} for each $\gamma>0$ as
\begin{align}\label{def:p-dec-o}
    \pdeco_{\gamma}(\cM,\oM)\defeq \infpqb \sup_{M\in\cM}\sset{ \EE_{\pi\sim p}[\LM{\pi}] -\gamma \EE_{\bpi\sim q} \DH{ M(\bpi), \oM(\bpi) } },
\end{align}
\begin{align}\label{def:r-dec-o}
    \rdeco_{\gamma}(\cM,\oM)\defeq \infpb \sup_{M\in\cM}\sset{ \EE_{\pi\sim p}[\Vmm-\Vm(\pi)] -\gamma \EE_{\bpi\sim p} \DH{ M(\bpi), \oM(\bpi) } },
\end{align}
and we let
\begin{align}
    \pdeco_{\gamma}(\cM)=\sup_{\oM\in\coM} \pdeco_{\gamma}(\cM,\oM),
    \quad
    \rdeco_{\gamma}(\cM)=\sup_{\oM\in\coM} \rdeco_{\gamma}(\cM,\oM).
\end{align}
More generally, for any constraint set $\MPow$ under \cDMSO, we define the offset \gDEC~as
\begin{align}\label{def:p-dec-o-gen}
    \pdecog_{\gamma}(\MPow,\oM)\defeq \infpqb \sup_{\substack{\cP\in\MPow\\ M\in\co(\cP)}}\sset{ \EE_{\pi\sim p}[\LM[\cP]{\pi}] -\gamma \EE_{\bpi\sim q} \DH{ M(\bpi), \oM(\bpi) } },
\end{align}
and let $\pdecog_{\gamma}(\MPow)\defeq\sup_{\oM\in\coM}\pdecog_{\gamma}(\MPow,\oM)$.

As a remark, we note that when the loss function is bounded in $[0,1]$, it holds that
\begin{align}\label{eqn:deco-and-decc}
    \pdeco_{\eps^{-2}}(\cM,\oM)\leq \pdecc_\eps(\cM,\oM)\leq \pdeco_\gamma(\cM,\oM)+\gamma\eps^2,
\end{align}
and analogous conversions also hold for the regret-DECs and the \gDEC s. The first inequality in \eqref{eqn:deco-and-decc} can be loose in general, and a tighter conversion is possible under reward-based loss function (\cref{prop:offset-to-c}).

\subsection{\Infosets}\label{appdx:info-sets}

\newcommandx{\cMps}[1][1=\psi]{\cM_{#1}}
\newcommand{\pss}{\psi^\star}
\newcommand{\cMpss}{\cMps[\pss]}
\newcommandx{\cMPs}[1][1=\Psi]{\cM_{#1}}
\newcommand{\pips}{\pi_\psi}
\newcommandx{\Lps}[2][1=M]{L_{\psi}(#1,#2)}
\newcommand{\DMPs}{\Delta(\cMPs)}

\newcommand{\epsps}{\eps_{\Psi}}

\newcommand{\tpa}{Type 2}
\newcommand{\tpb}{Type 1}
\newcommand{\subloss}{sub-loss}

Recall that in \cref{ssec:cDMSO}, we consider both PAC risk~\cref{eqn:def-risk-Ps} (in terms of the loss function $L$) and the regret~\cref{eqn:def-reg-cDMSO} (in terms of the value function $V$). 

To present the \ExOp~algorithm in a unified form, we first introduce the notion of \emph{\infosets}. We consider two types of \infosets: \tpb~\infosets~is introduced to handle ``value-based'' learning (cf. discussion below), generalizing \citet{foster2022complexity}; \tpa~\infosets~is for general PAC learning under \cDMSO.

\paragraph{\tpb~\infosets}
We introduce the \tpb~\infosets~primarily for no-regret learning in \cDMSO.

\begin{definition}[\tpb~\infosets]
Given a constraint class $\MPow$ under \cDMSO~and a value function $V$, a \tpb~\infosets~is a class $\Psi$, where each $\psi\in\Psi$ is associated with a model class $\cMps\subseteq (\bPi\to\DO)$ and a decision $\pips\in\Pi$, such that the following holds:

(1) For each $\cP\in\MPow$, there exists $\psi\in\Psi$ such that $\cP\subseteq \cMps$. 

(2) The value $M\mapsto \Vm(\pi)$ is linear over $\cMPs\defeq \bigcup_{\psi\in\Psi} \co(\cMps)$ for any $\pi\in\Pi$. We also denote $\Lps[M]{\pi}=\Vm(\pips)-\Vm(\pi)$ for each $\psi\in\Psi, M\in\cMps, \pi\in\Pi$.
\end{definition}

For no-regret learning in \cDMSO, the simplest \tpb~\infosets~is given by $\Psi=\MPow\times \Pi$, and for each $\psi=(\cP^\psi,\pi^\psi)\in\Psi$, we assign $\cMps=\cP^\psi, \pips=\pi^\psi$. Then, the loss $\Lps[M]{\pi}$ measures the sub-optimality of a decision $\pi$ compared to the decision $\pips$ (for the information set $\psi$) under the model $M$. 

Another example of \infosets~is the ``policy-based'' one (cf. \cref{ssec:tigher}): $\Psi=\Pi$, where for each $\pi\in\Psi$, $\cM_\pi=\set{M: \Vmm-\Vm(\pi)\leq \Delta}$. In this example, $\pi$ is a near-optimal decision for models in $\cM_\pi$. With such an \infosets, we can derive an upper bound scaling with the \dct~of $\cM$ and the DEC of $\cM_\Pi$ (see \cref{appdx:proof-regret-via-cDMSO}). 

The notion of \tpb~\infosets~can be viewed as an abstraction of the ideas of \citet{foster2022complexity}. The idea of using information sets in the context of posterior sampling (and then AIR) was conveyed to the authors by Dylan
Foster back in 2022. 

In addition to no-regret learning in \cDMSO, \tpb~\infosets~can also be applied to the ``value-based'' PAC learning under stochastic DMSO, as long as the loss function $L$ is specified by the value function $V$ as $L(M,\pi)=\Vmm-\Vm(\pi)$, where $\pim=\argmax_{\pi\in\Pi} \Vm(\pi)$.

\paragraph{\tpa~\infosets}
For PAC learning under \cDMSO, we consider $\MPow$ itself as an \infosets.
\begin{definition}[\tpa~\infosets~for PAC learning in \cDMSO]
Given a problem class $(\cM,\MPow)$ under \cDMSO, we say that  $\Psi=\MPow$ is a \tpa~\infosets. To be consistent with \tpb~\infosets, we write $\cMps=\psi$ and $\Lps[M]{\pi}\defeq L(\psi,\pi)$ (i.e., the loss of a decision $\pi$ only depends on the information set $\psi\in\MPow$).
\end{definition}

\newcommand{\Sp}{\mathbb{S}}
\newcommand{\Spnr}{\mathbb{S}^{\sf reg}}
\newcommand{\Sppac}{\mathbb{S}^{\sf pac}}
\newcommand{\qd}{w}
\newcommand{\pw}{{p\in\DPi,\wf\in\wF}}
\newcommand{\Gamnr}{\Gamma^{\sf reg}}
\newcommand{\Gampac}{\Gamma^{\sf pac}}
\newcommand{\pqw}{{(p,q)\in \Sp,\wf\in\wF}}
\newcommand{\pqwA}{{(p,q)\in \Sp,\wf\in\wF_A}}

\newcommand{\DtPi}{\Delta(\tPi)}

\subsection{\ExO~algorithm}\label{appdx:exo-full}
The algorithm, \ExOp, is stated in \cref{alg:ExO}. It has two options: \optpac for PAC learning and \optreg for no-regret learning. For these two tasks, we specify different spaces $\Sp$ of distributions for exploration-exploitation:
\begin{align*}
    \Sppac\defeq \DPi\times \DbPi, \qquad 
    \Spnr\defeq \set{ (q|_{\Pi},q): q\in\DbPi } \subset \DPi\times \DbPi,
\end{align*}
where we recall that $\bPi\defeq \Pi\times \Phc$, and for any distribution $q\in\DbPi$, $q|_{\Pi}\in\DPi$ is the marginal distribution of $\pi$ under $(\pi,\phi)\sim q$. We note that for \tpa~\infosets, only the option \optpac~applies.

At each round $t$, the algorithm maintains a reference distribution $\qd\ind{t}\in\DPs$, and uses it to obtain a joint \emph{exploration-exploitation} distribution $(p\ind{t},q\ind{t})\in\Sp$ and a weight function $\wf\ind{t}\in\wF\defeq (\Psi\times\bPi\times\cO\to\R)$,\footnote{Formally, for infinite $\Pi$ or $\cO$, $\wF$ is the class of measurable, uniformly bounded functions over $\Psi\times\Pi\times \cO$.} by solving a joint minimax optimization problem based on the \emph{exploration-by-optimization} objective: Defining
\begin{align}\label{eqn:Gamma-f}
\begin{aligned}
    \Gamma_{\qd,\gamma}(p,q,\wf;M,\psi)
    =&~ \EE_{\pi\sim p}\brac{\Lps[M]{\pi}} \\
    &~\qquad-\gamma \EE_{\tpi\sim q}\EE_{o\sim M(\tpi)}\EE_{\psi'\sim \qd}\brac{ 1-\exp\paren{ \wf(\psi';\tpi,o)-\wf(\psi;\tpi,o) } },
\end{aligned}
\end{align}
and
\begin{align}
    \Gamma_{\qd,\gamma}(p,q,\wf)
    =\sup_{(M,\psi):M\in\cMps} \Gamma_{\qd,\gamma}(p,q,\wf;M,\psi),
\end{align}
the algorithm solves
\begin{align*}
    (p\ind{t},q\ind{t},\wf\ind{t}) \leftarrow \argmin_\pqw \Gamma_{\qd\ind{t},\gamma}(p,q,\wf).
\end{align*}
The algorithm then samples $\tpi\ind{t}=(\pi\ind{t},\phi\ind{t})\sim q\ind{t}$ from the exploration distribution, executes $\tpi\ind{t}$ and observes $o\ind{t}$ from the environment. Finally, the algorithm updates the reference distribution by performing the exponential weight update \cref{eqn:ExO-exp-upd} with weight function $\wf\ind{t}(\cdot;\tpi\ind{t},o\ind{t})$.

At the end of the interactions, the algorithm may also output $\phat=\frac{1}{T}\sumt p\ind{t} \in\DDD$ as the distribution of $\hpi$, which is the mixture of the per-step exploitation distributions.

\newcommand{\OPT}[1]{{\color{blue} #1}}
\newcommand{\hp}{\hat{p}}
\begin{algorithm}[h]
\caption{\ExO~with \infosets~(\ExOp)}\label{alg:ExO}
\begin{algorithmic}[1]
\REQUIRE Decision space $\Pi$, measurement class $\Phc$, \infosets~$\Psi$, prior $\qd\ind{1}\in\DPs$, parameter $T\geq 1$, $\gamma>0$.
\STATE For option \optpac, set $\Sp=\Sppac$; for option \optreg, set $\Sp=\Spnr$. 
\FOR{$t=1,\cdots,T$}
    \STATE Solve the \emph{exploration-by-optimization} objective:
    \begin{align}\label{eqn:ExO-opt}
        (p\ind{t},q\ind{t},\wf\ind{t}) \leftarrow \argmin_\pqw \Gamma_{\qd\ind{t},\gamma}(p,q,\wf)
    \end{align}
    \STATE Sample $\tpi\ind{t}=(\pi\ind{t},\phi\ind{t})\sim q\ind{t}$ and observe $o\ind{t}\sim M^t(\bpi\ind{t})$
    \STATE Perform exponential-weight update:
    \begin{align}\label{eqn:ExO-exp-upd}
        \qd\ind{t+1}(\psi)~\propto_\psi~ \qd\ind{t}(\psi)\exp(\wf\ind{t}(\psi;\tpi\ind{t},o\ind{t}))
    \end{align}
\ENDFOR
\ENSURE $\phat=\frac{1}{T}\sumt p\ind{t} \in\DDD$
\end{algorithmic}
\end{algorithm}

\newcommand{\exo}[1]{\normalfont{\textsf{exo}}_{#1}}
\newcommand{\rexo}[1]{\normalfont{\textsf{r-exo}}_{#1}}
\newcommand{\pexo}[1]{\normalfont{\textsf{p-exo}}_{#1}}
Following \citet{foster2022complexity}, we define 
\begin{align}
    \exo{\gamma}(\Psi,\qd)\defeq \inf_\pqw\Gamma_{\qd,\gamma}(p,q,\wf),
\end{align}
and $\exo{\gamma}(\Psi)=\sup_{q\in\DPs} \exo{\gamma}(\Psi,\qd)$. Note that $\exo{\gamma}(\Psi)$ implicitly depends on the space $\Sp$. 

\newcommand{\cEs}{\cE^\star_\Delta}
\newcommand{\psis}{\psi^\star}
\newcommand{\oMs}{\oM^\star}

Now, we present the primary guarantees of \ExOp. 

\paragraph{Bounds for \tpb~\infosets}
Suppose that the algorithm \ExOp~is instantiated with a \tpb~\infosets~$\Psi$ (with respect to the constraint class $\MPow$), and the environment is constrained by $\cPs\in\MPow$.
Define $\oMs=\frac1T\sumt M^t\in\co(\cPs)$ and consider the set
\begin{align*}
    \cEs\defeq \set{\psi: \cPs\subseteq \cMps, \Vmm[\oMs]-\Vm[\oMs](\pips)\leq \Delta }.
\end{align*}
Note that $\cEs\subseteq \Psi$ depends on $M^1,\cdots,M^T$, i.e. $\cEs$ depends on the $T$-round interactions between the environment and the \ExOp~algorithm. We present an upper bound scaling with $\log(1/\qd\ind{1}(\cEs))$.

\begin{theorem}[\ExOp~upper bound; \tpb]\label{thm:ExO-upper}
Let $T\geq 1$, the constraint class $\MPow$ and the value function $V$ be given, and $\Psi$ be a \tpb~\infosets. Suppose that the environment is constrained by $\MPow$. Then the algorithm \ExOp~achieves that \whp,
\begin{align*}
    \max_{\pi\in\Pi}\sumt \VM[M^t](\pi)-\EE_{\pi\ind{t}\sim p\ind{t}}\brac{\VM[M^t](\pi\ind{t})} \leq T\cdot \brac{ \Delta+\exo{\gamma}(\Psi) } +2\gamma\cdot \brac{ \log(1/\qd\ind{1}(\cEs))+\log(1/\delta) }.
\end{align*}
\end{theorem}

The proof of \cref{thm:ExO-upper} is deferred to \cref{appdx:proof-ExO-upper}. It is based on bounding the performance of the exponential weight update \cref{eqn:ExO-exp-upd}, and then relating it to the performance of \ExOp~using the definition of $\exo{\gamma}(\Psi)$. Different from the analysis in \citet{foster2022complexity,chen2024beyond}, the proof here has to carefully deal with $\cEs$, which is an event that depends on the $T$-round interactions.

\paragraph{Bounds for \tpa~\infosets}
Similarly, for \tpa~\infosets~$\Psi=\MPow$, we have the following guarantee of \ExOp.
\begin{theorem}[\ExOp~upper bound for PAC learning; \tpa]\label{thm:ExO-upper-ii}
For PAC learning under \cDMSO, suppose that the algorithm \ExOp~is instantiated with the \tpa~\infosets~$\Psi=\MPow$, and $w^1=\Unif(\MPow)$.
Then for any environment constrained by $\MPow$,  
\ExOp~achieves \whp
\begin{align*}
    \riskdm(T)=\EE_{\hpi\sim \phat}\LM[\cPs]{\hpi} \leq \exo{\gamma}(\Psi) +\frac{2\gamma}{T}\cdot \brac{ \log|\MPow|+2\log(1/\delta) }.
\end{align*}
\end{theorem}

The proof is postponed to \cref{appdx:proof-ExO-upper-ii}.

\subsection{Guarantees of the \ExOp~algorithm}\label{appdx:exo-dec-upper}

In this section, we simplify the upper bound of \cref{thm:ExO-upper} and \cref{thm:ExO-upper-ii}.
In the following, we bound the term $\exo{\gamma}(\Psi)$ and $\log(1/\qd\ind{1}(\cEs))$ separately.

\paragraph{Bounding ExO coefficient} We relate $\exo{\gamma}(\Psi)$ to the offset DECs, following \citet[Theorem 3.1 and 3.2]{foster2022complexity}. 

\begin{theorem}\label{thm:exo-to-dec}
Suppose that the model class $\cMPs$ \finite. Then, the following holds:

(1) Suppose $\Psi$ is a \tpb~\infosets~and the value function $V$ is uniformly continuous over $\cMPs$. Then, for PAC learning (option \optpac, $\Sp=\Sppac$), we have
\begin{align*}
    \exo{\gamma}(\Psi)
    \leq \pdeco_{\gamma/4}(\cMPs), \qquad \forall \gamma>0.
\end{align*}
Analogously, for no-regret learning (option \optreg, $\Sp=\Spnr$), we have
\begin{align*}
    \exo{\gamma}(\Psi)
    \leq \rdeco_{\gamma/4}(\cMPs), \qquad \forall \gamma>0.
\end{align*}

(2) If $\Psi=\MPow$ is a \tpa~\infosets, then 
\begin{align*}
    \exo{\gamma}(\Psi)
    \leq \pdecog_{\gamma/4}(\MPow), \qquad \forall \gamma>0.
\end{align*}
\end{theorem}
The proof of \cref{thm:exo-to-dec} is a generalization of the analysis in \citet{foster2022complexity} and is deferred to \cref{appdx:proof-exo-to-dec}.

\paragraph{Bounding $\qd\ind{1}(\cEs)$}
For \tpb~\infosets~$\Psi$, we also need to provide a uniform upper bound on the quantity $\qd\ind{1}(\cEs)$.
Following \cref{def:DC}, we consider the \dct~of $\MPow$ under an \infosets~$\Psi$:
\newcommandx{\IDC}[2][1=\Delta]{N_{\mathsf{frac}}(#2;#1)}
\newcommand{\LMV}[2][1=M]{\Vmm[#1]-\Vm[#1](#2)}
\newcommand{\qds}{\qd^\star}
\begin{align}\label{def:IDC}
    \IDC{\MPow,\Psi}\defeq\inf_{w\in\DPs}\sup_{(\cP,\oM)} ~~\frac{1}{\PP_{\psi\sim w}\paren{ \psi: \cP\subseteq \cMps, \LMV[\oM]{\pi_\psi}\leq \Delta)} },
\end{align} 
where the supremum $\sup_{(\cP,\oM)}$ is taken over all possible pair $(\cP,\oM)$ with $\cP\in\MPow$ and $\oM\in\co(\cP)$. %
Then, the optimal $\qds\ind{1}$ is given by
\begin{align}\label{def:IDC-w}
    \qds\ind{1}\defeq \argmin_{w\in\DPs}\sup_{(\cP,\oM)} ~~\frac{1}{\PP_{\psi\sim w}\paren{ \psi: \cP\subseteq \cMps, \LMV[\oM]{\pi_\psi}\leq \Delta)} }.
\end{align}
By definition, it holds that $\qds\ind{1}(\cEs)\geq \frac{1}{\IDC{\MPow,\Psi}}$ for any possible $\cEs.$

Putting these pieces together, we derive the following guarantees of \ExOp~for PAC learning and no-regret learning under \cDMSO.

\begin{theorem}[Guarantees of \ExOp; \tpb]\label{thm:ExO-full-offset}
Let $T\geq 1$, parameter $\gamma, \Delta>0, \delta\in(0,1)$, constraint class $\MPow$, value function $V$ be given. Suppose that $\Psi$ is a \tpb~\infosets, and $\cMPs$ \finite. We instantiate \ExOp~on $\Psi$ and choose $\qd\ind{1}\in\DPs$ according to \eqref{def:IDC-w}.

(1) With the option \optreg, \ExOp~achieves \whp
\begin{align*}
    \regdm(T)=&~ \max_{\pi\in\Pi}\sumt \VM[M^t](\pi)-\EE_{\pi\ind{t}\sim q\ind{t}}\brac{\VM[M^t](\pi\ind{t})} \\
    \leq&~ T\cdot\brac{\Delta+\rdeco_{\gamma/4}(\cMPs)}+2\gamma\brac{ \log\IDC{\MPow,\Psi}+\log(1/\delta) }.
\end{align*}

(2) When $\MPow=\MPowiid$ (stochastic DMSO), \ExOp~with option \optpac achieves \whp~that
\begin{align*}
    \EE_{\hpi\sim \hp}\brac{\LMV[\Ms]{\hpi}}\leq \Delta+\pdeco_{\gamma/4}(\cMPs)+\frac{2\gamma}{T}\brac{ \log\IDC{\MPowiid,\Psi}+\log(1/\delta) }.
\end{align*}
\end{theorem}

\paragraph{Guarantees for \tpa~\infosets} Similarly, when \ExOp~is instantiated with \tpa~\infosets, we have a similar upper bound by simply choosing $\qd\ind{1}=\Unif(\Psi)$.

\begin{theorem}[Guarantees of \ExOp; \tpa]\label{thm:ExO-full-offset-ii}
Let $T\geq 1, \gamma>0, \delta\in(0,1)$, constraint class $\MPow$ be given. Suppose that $\cM$ \finite, and \ExOp~is instantiated with the \tpa~\infosets~$\Psi=\MPow$, $\qd\ind{1}=\Unif(\MPow)$, and option \optpac. Then \whp,
\begin{align*}
    \EE_{\hpi\sim \hp}\LM[\cPs]{\hpi}\leq \pdecog_{\gamma/4}(\MPow)+\frac{2\gamma\log(|\MPow|/\delta) }{T}.
\end{align*}
\end{theorem}

\begin{remark}
We assume that $\cM$ admits finite covering to ensure the Minimax theorem can be applied in \cref{thm:exo-to-dec}. Alternatively, we can assume that (1) the decision space $\Pi$ is finite, (2) the latent observation space $\cZ$ is a compact metric space under a certain metric $\rho$, and (3) the value function is given by a reward function $R$ (cf. \cref{rew-max}) with $R(z,\pi)$ being Lipschitz with respect to $z$. This is indeed the case for agnostic regression task (\cref{ssec:regression}).

In these assumptions, we can consider a finite $\eps$-covering $\cZ_\eps$ of $\cZ$, and take $\cM_\eps\subseteq (\Pi\to \Delta(\cZ_\eps))$ to be the model class induced by $\cM$. Apply \cref{thm:ExO-full-offset} to the model class $\cM_\eps$ with a sufficiently small $\eps$ yields the same bound on $\riskdm(T)$ (or $\regdm(T)$, respectively).
\end{remark}

\subsection{Proofs for upper bounds in \cref{sec:cDMSO-main}}\label{appdx:exo-inst}

In the following, we instantiate \cref{thm:ExO-full-offset} and \cref{thm:ExO-full-offset-ii} to prove the upper bounds in \cref{sec:cDMSO-main}.

\subsubsection{Proof of \cref{thm:cDMSO-pac-upper}}\label{appdx:proof-cDMSO-pac-upper}

For \cref{thm:cDMSO-pac-upper}, we instantiate \ExOp~as in \cref{thm:ExO-full-offset-ii}, taking the \tpa~\infosets~$\Psi=\MPow$. It remains to upper bound the offset \gDEC~of $\MPow$ by the \gDEC, and we invoke the following lemma. Its proof largely mimics \citet{foster2023tight} and is postponed to \cref{appdx:proof-offset-to-c-metric}.
\begin{lemma}\label{lem:offset-to-c-metric}
Suppose that the loss function $L$ is metric-based. Then it holds that
\begin{align*}
    \inf_{\gamma>0}\paren{ \pdecog_{\gamma}(\MPow) + \gamma\eps^2 } \leq 8\eps\cdot \sup_{\eps'\in[ \eps, 1]}\frac{ \pdecg_{\eps'}(\MPow) }{\eps'}. 
\end{align*}
\end{lemma}

In particular, under the assumption that the \gDEC~of $\MPow$ is of moderate decay with constant $\creg$ (\cref{asmp:reg}), we have
\begin{align*}
    \inf_{\gamma>0}\paren{ \pdecog_{\gamma}(\MPow) + \gamma\eps^2 } \leq 10\creg \pdecg_{\eps}(\MPow).
\end{align*}
Hence, with an optimally tune parameter $\gamma$, \ExOp~(as instantiated in \cref{thm:ExO-full-offset-ii}) achieves
\begin{align*}
    \riskdm(T)\leq&~ \Delta+\inf_{\gamma>0}\sset{ \pdecog_{\gamma/4}(\MPow)+\frac{2\gamma\log(|\MPow|/\delta) }{T} } \\
    \leq&~ \Delta+80\creg\cdot \pdecg_{\oeps(T)}(\MPow),
\end{align*}
where $\oeps(T)=\sqrt{\frac{\log(|\MPow|/\delta)}{T}}$.
\qed

\subsubsection{Proof of \cref{thm:cDMSO-reg-lower}}\label{appdx:proof-cDMSO-reg-upper}

For no-regret learning, the most natural \infosets~is given by $\Psi\defeq \Psi_{\MPow,\Pi}=\MPow\times \Pi$, such that for each $\psi=(\cP,\pi)\in\Psi$, we may specify $M_\psi=\cP$ and $\pips=\pi$. With such a construction, it is direct to verify that $\cMPs=\cMPow$ and
\begin{align*}
    \IDC{\MPow,\Psi}\leq |\MPow|\cdot \DC{\cMPow}.
\end{align*}

Therefore, it remains to upper bound the offset DEC of $\cMPow$ with the constrained DEC of $\cMPow$. We invoke the following conversion result, which follows from \citet[Theorem G.5]{chen2024beyond}.
\begin{proposition}\label{prop:offset-to-c-gen}
Suppose that \eqref{eqn:lip-rew-reduction} holds for the value function $\Vm$ over the model class $\cM$. Then it holds that
\begin{align*}
    \inf_{\gamma>0}\paren{ \rdeco_\gamma(\cM)+\gamma\eps^2 }\leq C\sqrt{\log(1/\eps)}\cdot \paren{ \sup_{\eps'\in[\eps,1]} \frac{\rdecc_{\eps'}(\cM)}{\eps'} + \Lipr } \cdot \eps,
\end{align*}
where $C$ is a universal constant. 

Similarly, for PAC learning with loss function $\LM{\pi}=\Vmm-\Vm(\pi)$, it holds that
\begin{align*}
    \inf_{\gamma>0}\paren{ \pdeco_\gamma(\cM)+\gamma\eps^2 }\leq C'\paren{ \sup_{\eps'\in[\eps,1]} \frac{\pdecc_{\eps'}(\cM)}{\eps'} + \Lipr } \cdot \eps,
\end{align*}
where $C'$ is a universal constant.
\end{proposition}
Note that when the value function is reward-based (\cref{example:original-reward-based}), \eqref{eqn:lip-rew-reduction} holds with $\Lipr=\sqrt{2}$. 

Hence, under the assumptions of \cref{thm:cDMSO-reg-upper}, if we instantiate \ExOp~with the \infosets~$\Psi=\Psi_{\MPow}$, option \optreg, and choose $\gamma>0$ optimally, then \ExOp~achieves \whp~that
\begin{align*}
    \frac{1}{T}\regdm(T)\leq&~ \Delta+\inf_{\gamma>0}\sset{ \rdeco_{\gamma/4}(\cMPs)+\frac{2\gamma}{T}\brac{ \log\IDC{\MPow,\Psi}+\log(1/\delta) } } \\
    \leq&~ \Delta+\bigO{\sqrt{\log T}}\cdot \brac{\rdecc_{\oeps(T)}(\cMPow)+\oeps(T)},
\end{align*}
where $\oeps(T)=C\sqrt{\frac{\log(|\MPow|/\delta)+\log\DC{\cMPow}}{T}}$.
This gives \cref{thm:cDMSO-reg-upper} immediately.
\qed

\subsubsection{Proof of \cref{prop:multi-hypothesis}}\label{appdx:proof-multi-hypothesis}

As we have discussed in \cref{ssec:tigher}, for the convex hypothesis selection problem, we can consider the ``relaxed'' constraint class $\MPow_m=\set{\cM_1,\cdots,\cM_m}$. Then, by \cref{thm:ExO-full-offset-ii}, under any environment specified by a model $\Mstar\in\cM$, \ExOp~(when instantiated on $\MPow_m$ and $\gamma=\frac{4}{\oeps(T)^2}$) achieves \whp~that
\begin{align*}
    \EE_{\hpi\sim \hp}\LM[\Mstar]{\hpi}\leq \pdecog_{\gamma/4}(\MPow_m)+\frac{2\gamma\log(m/\delta) }{T}.
\end{align*}
Because $\cM_1,\cdots,\cM_m$ are convex, we have
\begin{align*}
    \pdecog_{\gamma/4}(\MPow_m)=\pdeco_{\gamma/4}(\cM)\leq \pdecc_{2/\sqrt{\gamma}}(\cM)=\pdecc_{\oeps(T)}(\cM),
\end{align*}
where the inequality follows from \eqref{eqn:deco-and-decc}. Thus, \whp, we have
\begin{align*}
    \EE_{\hpi\sim \hp}\LM[\Mstar]{\hpi}\leq \pdecc_{\oeps(T)}(\cM)+\frac{8\log(m/\delta) }{\oeps(T)^2T}\leq \frac{1}{3}+\frac{1}{8}<\frac{1}{2}.
\end{align*}
Note that for $\Mstar\in\cM_{i^\star}$, we have $\LM[\Mstar]{\hpi}=\indic{\hpi\neq i^\star}$, and hence
\begin{align*}
    \EE_{\hpi\sim \hp}\LM[\Mstar]{\hpi}=\PP_{\pi\sim \hp}\paren{ \pi\neq i^\star }.
\end{align*}
Therefore, we may modify \ExOp~so that it outputs $\hpi=\argmin_{i\in[m]} \PP_{\pi\sim \hp}\paren{ \pi\neq i }$. Then, \whp, we have $\hpi=i^\star$.
\qed

\subsubsection{Proof of \cref{prop:regret-via-cDMSO}}\label{appdx:proof-regret-via-cDMSO}

Recall that the constraint class for stochastic DMSO is $\MPowiid$. Therefore, fix the parameter $\Delta\geq 0$, we may consider the ``policy-based'' \infosets~$\Psi=\Pi$ (specified as in \eqref{def:MPow-pol-based}):
\begin{align}
    \cMps=\set{M\in\cM: \Vmm-\Vm(\pi)\leq \Delta}, \qquad
    \pi_\psi=\psi, \qquad \forall\psi\in\Psi.
\end{align}
With such an \infosets, it is clear that $\cMPs=\cup_{\pi}\co(\cM_\pi)=\cM_\Pi$, and
\begin{align*}
    \IDC{\MPowiid, \Psi}=\DC{\cM}.
\end{align*}
Therefore, under the assumptions of \cref{prop:regret-via-cDMSO}, if we instantiate \ExOp~with the \infosets~$\Psi=\Pi$ specified above, option \optreg, and choose $\gamma>0$ optimally, then \ExOp~achieves \whp~that
\begin{align*}
    \frac{1}{T}\regdm(T)\leq&~ \Delta+\inf_{\gamma>0}\sset{ \rdeco_{\gamma/4}(\cM_\Pi)+\frac{2\gamma}{T}\brac{ \log\DC{\cM}+\log(1/\delta) } } \\
    \leq&~ \Delta+\bigO{\sqrt{\log T}}\cdot \brac{ \rdecc_{\oeps(T)}(\cM_\Pi)+\oeps(T)},
\end{align*}
where $\oeps(T)=\sqrt{\frac{\log\DC{\cM}+\log(1/\delta)}{T}}$ and the second inequality follows from \cref{prop:offset-to-c-gen}.
This gives the desired regret bound. 
\qed

As a remark, for reward-based PAC learning, we can similarly obtain an upper bound scaling with the PAC-DEC of $\cM_\Pi$ and the \dct.

\subsection{Instantiations of \ExOp~to LDP learning}\label{appdx:ExO-LDP}

In this section, we turn our focus to \pDMSO~(\cref{ssec:LDP-demo}), where the learner is given a model class $\cM\subseteq (\Pi\to\DZ)$. 

Let observation space be $\cO=\sset{-1,1}$ (i.e., only binary channels are considered), and let $\Phi=\Pcp$ be the class of all \pDP~channels from $\cZ$ to $\sset{-1,1}$. When \ExOp~is instantiated with the measurement class $\Phi=\Pcp$, we will call the obtained algorithm \LDPexo, because it naturally preserves \pLDP. 

Recall that in \pDMSO, the corresponding constraint class is $\MPowdp=\sset{\tcM:M\in\cM}$, where for each model $M\in\cM$, the model $\tM:\Pi\times\Phi\to\DO$ is specified by $\tM(\pi,\pr)=\pr\circ M(\pi)$ and $\Vm[\tM](\pi)=\Vm(\pi)$ for all $\pi\in\Pi$, $\pr\in\Phi$. 
For simplicity, we focus on the setting of reward-based learning (\cref{rew-max}), where there is a reward function $R$ such that the value function is given by $\Vm(\pi)=\EE\sups{M,\pi} R(z,\pi)$, and the loss function $L(M,\pi)=\Vmm-\Vm(\pi)$, where $\pim=\argmax_{\pi\in\Pi} \Vm(\pi)$. 

For \pDMSO~with a model class $\cM\subseteq (\Pi\to\DZ)$, we restate the definition of \infosets~structure as follows. Here, we focus on \tpb~\infosets, and (with slight abuse of notation) we regard $\cM$ as a subset of $(\Pi\times\Phi\to\DO)$ by identifying each model $M\in\cM$ with $\tM$.

\begin{definition}[\Infosets~for \pDMSO]
Given a model class $\cM\subseteq (\Pi\to\DZ)$, an \infosets~is a class $\Psi$, where each $\psi\in\Psi$ is associated with a model class $\cMps\subseteq (\Pi\to\DO)$ and a decision $\pips\in\Pi$, such that for each $M\in\cM$, there exists $\psi\in\Psi$ such that $M\in \cMps$. We denote $\cMPs\defeq \bigcup_{\psi\in\Psi} \co(\cMps)$.
\end{definition}

\paragraph{Private offset-DECs}
To state the upper bounds of \LDPexo~with minimal assumptions, we introduce the offset \pDEC/\rDEC~as follows. For any model class $\cM\subseteq (\Pi\to\DZ)$ and a reference model $\oM\in\coM$, we let
\begin{align}\label{def:p-dec-o-sq}
    \pdecol_{\gamma}(\cM,\oM)\defeq \infpql \sup_{M\in\cM}\sset{ \EE_{\pi\sim p}[\LM{\pi}] -\gamma \EE_{(\pi,\lf)\sim q} \Dl^2( M(\pi), \oM(\pi) ) } ,
\end{align}
\begin{align}\label{def:r-dec-o-sq}
    \rdecol_{\gamma}(\cM,\oM)\defeq \infpl \sup_{M\in\cM}\sset{ \EE_{\pi\sim p}[\Vmm-\Vm(\pi)] -\gamma \EE_{(\pi,\lf)\sim q} \Dl^2( M(\pi), \oM(\pi) ) } ,
\end{align}
and we define
\begin{align}
    \pdecol_{\gamma}(\cM)=\sup_{\oM\in\coM} \pdecol_{\gamma}(\cM,\oM),
    \quad
    \rdecol_{\gamma}(\cM)=\sup_{\oM\in\coM} \rdecol_{\gamma}(\cM,\oM).
\end{align}
By the data-processing inequality (\cref{prop:pLDP}), we can relate the offset private PAC-DEC (regret-DEC) of $\cM$ to the offset PAC-DEC (regret-DEC) of the induced model class $\tcM$:
\begin{align}
    \pdecol_{c_0\alpha^2\gamma}(\cM)
    \leq \pdeco_{\gamma}(\tcM)
    \leq \pdecol_{c_1\alpha^2\gamma}(\cM).
\end{align}
\begin{align}
    \rdecol_{c_0\alpha^2\gamma}(\cM)
    \leq \rdeco_{\gamma}(\tcM)
    \leq \rdecol_{c_1\alpha^2\gamma}(\cM).
\end{align}
The proof is essentially the same as \cref{appdx:inst-LDP-pac-lower} and hence omitted.

\paragraph{Guarantees of \LDPexo}
For simplicity, we denote (cf. \eqref{def:IDC})
\begin{align}\label{def:IDC-cM}
    \IDC{\cM,\Psi}\defeq\IDC{\MPowdp, \Psi}=\inf_{w\in\DPs}\sup_{M\in\cM} ~~\frac{1}{\PP_{\psi\sim w}\paren{ \psi: M\in\cMps, \LMV[M]{\pi_\psi}\leq \Delta)} }.
\end{align} 
With above notation, we state the guarantee of \LDPexo~as follows.

\begin{theorem}[\LDPexo~for \pDMSO]\label{thm:ExO-LDP}
Let $T\geq 1$, parameter $\gamma,\Delta>0, \delta\in(0,1)$, model class $\cM\subseteq (\Pi\to\DZ)$ and value function $V$ be given. Suppose that $\Psi$ is an \infosets~with respect to the model class $\cM$, and $\cMPs$ is compact. We instantiate \ExOp~on $\Psi$ and choose $\qd\ind{1}\in\DPs$ according to \eqref{def:IDC-w}.

(1) With the option \optreg, \ExOp~achieves \whp
\begin{align*}
    \regdm(T)=&~ \sumt \Vmm[\Mstar]-\EE_{\pi\ind{t}\sim q\ind{t}}\brac{\VM[\Mstar](\pi\ind{t})} \\
    \leq&~ T\cdot\brac{\Delta+\rdecol_{c\alpha^2\gamma}(\cMPs)}+2\gamma\brac{ \log\IDC{\cM,\Psi}+\log(1/\delta) }.
\end{align*}

(2) With the option \optpac, \ExOp~achieves \whp~that 
\begin{align*}
    \EE_{\hpi\sim \hp}\brac{\LMV[\Ms]{\hpi}}\leq \Delta+\pdeco_{c\alpha^2\gamma}(\cMPs)+\frac{2\gamma}{T}\brac{ \log\IDC{\cM,\Psi}+\log(1/\delta) }.
\end{align*}
\end{theorem}

In the following, we provide detailed specifications of the \infosets~and guarantees for various settings.
To obtain upper bounds in private PAC-DEC (regret-DEC), we will frequently invoke the following conversion lemma.
\begin{proposition}\label{prop:offset-to-c}
Let $\cM\subseteq (\Pi\to\DZ)$ be a given model class. Then the following holds.

(1) No-regret learning: If the value function $V$ is reward-based (\cref{rew-max}), then
\begin{align*}
    \inf_{\gamma>0}\paren{ \rdecol_\gamma(\cM)+\gamma\eps^2 }\leq C\sqrt{\log(1/\eps)}\cdot \paren{ \sup_{\eps'\in[\eps,1]}\frac{\rdecl_{\eps'}(\cM)}{\eps'} + 1 } \cdot \eps,
\end{align*}
where $C$ is a universal constant.

(2) PAC learning: If the loss function is reward-based, then
\begin{align*}
    \inf_{\gamma>0}\paren{ \pdecol_\gamma(\cM)+\gamma\eps^2 }\leq C\paren{ \sup_{\eps'\in[\eps,1]}\frac{\pdecl_{\eps'}(\cM)}{\eps'} + 1 } \cdot \eps,
\end{align*}
where $C'$ is a universal constant.
\end{proposition}
\cref{prop:offset-to-c} follows immediately from \citet[Theorem E.7]{chen2024beyond} (see also \citet[Proposition 4.2]{foster2023tight}).

\subsubsection{Model-based learning}\label{appdx:ExO-model-based}

\newcommand{\Psmod}{\Psi_{\mathsf{mod}}}
\newcommand{\Pspol}{\Psi_{\mathsf{pol}}}
\newcommand{\Psval}{\Psi_{\mathsf{val}}}
\newcommand{\Psadv}{\Psi_{\mathsf{adv}}}
\newcommand{\Pscxt}{\Psi_{\mathsf{cxt}}}

Perhaps the most natural \infosets~is the \emph{model-based} \infosets~$\Psmod$, given by
\begin{align}\label{def:info-mod-based}
    \Psmod=\cM, \qquad \cM_{\psi}=\set{\psi},\qquad \forall \psi\in\Psmod,
\end{align}
i.e., each information set $\psi\in\Psmod$ corresponds to a model $M\in\cM$.

By definition, we know that $\cMPs[\Psmod]=\cM$ and $\log\IDC[0]{\cM,\Psmod}=\log|\cM|$, achieving at the prior $\qd\ind{1}=\Unif(\cM)$.\footnote{When $\cM$ is infinite, we may instead take $\Psi$ to be a covering of $\cM$, and our results still hold with $\log|\cM|$ replace by the logarithmic covering number.} We instantiate \LDPexo~on $\Psmod$ to obtain the upper bound \eqref{eqn:rdec-lin-upper} in \cref{thm:rdec-lin-upper}.

\paragraph{Proof of \cref{thm:rdec-lin-upper} (1)}
Let $\Delta=0$ and \LDPexo~be instantiated on the \infosets~$\Psmod$. Then, by \cref{thm:ExO-LDP}, \LDPexo~with an optimally-chosen parameter $\gamma>0$ achieves \whp
\begin{align*}
    \frac{1}{T}\regdm(T)
    \leq&~ \inf_{\gamma>0}\paren{\rdecol_{c\alpha^2\gamma}(\cM)+\frac{2\gamma\log(|\cM|/\delta)}{T} } 
    \leq \OsqrtT\cdot\brac{ \rdecl_{\oeps(T)}(\cM) +\oeps(T)},
\end{align*}
where the second inequality uses \cref{prop:offset-to-c} and the assumption that $\rdecl_\eps(\cM)$ is of moderate decay.
\qed

As a remark, we note that for reward-based PAC learning, the upper bound of \cref{thm:pdec-lin-upper} can also be obtained in this way.

\subsubsection{Policy-based learning}\label{appdx:ExO-policy-based}

\newcommand{\cMPi}{\cM_{\Pi}}
Following \citet{chen2024beyond} (see also \cref{ssec:tigher}), we consider the decision-based (or, ``\emph{policy-based}'') \infosets~$\Psi=\Pspol$ given by %
\begin{align}\label{def:info-pol-based}
    \Pspol=\Pi, \qquad
    \cMps[\pi]=\set{ M: \Vmm-\Vm(\pi)\leq \Delta }, \qquad\forall \pi\in\Pi.
\end{align}
By definition,
\begin{align*}
    \cMPs[\Psmod]=\bigcup_{\pi\in\Pi} \co(\cM_\pi) =\cMPi, \qquad
    \IDC{\cM,\Pspol}=\DC{\cM}. %
\end{align*}
Therefore, we may instantiate \LDPexo~with $\Pspol$ to obtain the following upper bounds, which are direct implied by \cref{thm:ExO-LDP}.
\begin{proposition}[Policy-based \LDPexo~for private PAC learning]\label{thm:p-dec-convex-upper-full}
Let $T\geq 1$, $\gamma>0, \Delta>0$, model class $\cM\subseteq (\Pi\to\DZ)$ be given. Suppose that $\cM$ is compact, and \LDPexo~is instantiated with the \infosets~$\Pspol$. Then, the following holds.

(1) With option \optpac, it holds that \whp
\begin{align*}
    \riskdm(T)=\EE_{\hpi\sim \hp}\LM[\Ms]{\hpi}\leq \Delta+\pdecol_{c\alpha^2\gamma}(\cMPi)+\frac{2\gamma}{T}\brac{ \log\DC{\cM}+\log(1/\delta) }.
\end{align*}

(2) With option \optreg, it holds that \whp
\begin{align*}
    \frac1T\regdm(T)\leq \Delta+\rdecol_{c\alpha^2\gamma}(\cMPi)+\frac{2\gamma}{T}\brac{ \log\DC{\cM}+\log(1/\delta) }.
\end{align*}
\end{proposition}

\paragraph{Proof of \cref{thm:p-dec-convex-upper}}
Note that $\cMPi\subseteq \coM$.
Thus, 
\cref{thm:p-dec-convex-upper} follows immediately by choosing the optimal parameter $\gamma>0$ in the upper bound of \cref{thm:p-dec-convex-upper-full} (1) and then applying \cref{prop:offset-to-c}.
\qed

\paragraph{Proof of \cref{thm:rdec-lin-upper} (2)}
Similarly, \eqref{eqn:rdec-lin-upper-co} of \cref{thm:rdec-lin-upper} follows immediately from \cref{thm:p-dec-convex-upper-full} (2) and \cref{prop:offset-to-c}.
\qed

\subsubsection{Value-based learning}\label{appdx:ExO-val-based}

For the \wellspec~regression task with a function class $\cF\subseteq (\cX\to[-1,1])$ (\cref{ssec:regression}), we can employ the \emph{value-based} \infosets~$\Psval=\cF$ (for a fixed parameter $\Delta\geq 0$): %
\begin{align}\label{def:info-val-based}
    \cMps[f]=\sset{ M: \EE_{x\sim M} \abs{ \fm(x)-f(x) }\leq \Delta}, \qquad 
    \pi_f=f, \qquad \forall f\in\Psval.
\end{align}
This clearly gives a valid \infosets~$\Psval$, and we have
\begin{align*}
    \IDC{\cMF,\Psval}=\DC[\Delta]{\cF}. %
\end{align*}
Therefore, we may instantiate \LDPexo~with such an \infosets, and it remains to upper bound the offset DEC of $\cMPs$ as follows.
\begin{lemma}\label{lem:realizable-dec-o}
Suppose that $\cM=\cMwell$ is the class of \wellspec~models, $\Psi=\Psval$.
Then it holds that
\begin{align*}
    \pdecol_\gamma(\cMPs)\leq \pdecol_{\gamma/2}(\cM)+\gamma\Delta^2+2\Delta.
\end{align*}
\end{lemma}

\paragraph{Proof of \cref{prop:real-regression}}
Let \LDPexo~be instantiated on the \infosets~$\Psi=\Psval$. Then, by \cref{thm:ExO-LDP}, \LDPexo~with an optimally-chosen parameter $\gamma>0$ achieves \whp
\begin{align*}
    \riskdm(T)
    \leq&~ \inf_{\gamma>0}\paren{\Delta+\pdecol_{c\alpha^2\gamma}(\cMPs)+\frac{2\gamma\paren{ \log\DC{\cF}+\log(1/\delta) }}{T} } \\
    \leq&~ \inf_{\gamma>0}\paren{3\Delta+\pdecol_{c\alpha^2\gamma/2}(\cM)+\gamma\Delta^2+\frac{2\gamma\paren{ \log\DC{\cF}+\log(1/\delta) }}{T} } \\
    \leq&~  \bigO{1}\cdot\paren{ \pdecl_{\oeps(T)}(\cM)+\oeps(T) }.
\end{align*}
where the last inequality uses \cref{prop:offset-to-c} and the assumption that $\pdecl_\eps(\cM)$ is of moderate decay.
\qed

\newcommand{\fbar}{\Bar{f}}
\paragraph{Proof of \cref{lem:realizable-dec-o}}
By definition, $\cMPs=\bigcup_{f\in\cF} \co(\cM_f)$. We first prove the following claim.

\textbf{Claim.} For any $f\in\cF$ and $\oM\in\co(\cMps[f])$, there exists a model $M'\in\cM$ with $\fm[M']=f$ and $\DTV{M',\oM}\leq \Delta$. 

Suppose $\oM\in\co(\cMps[f])$ is given by $\oM=\EE_{M\sim \mu}[M]$ with $\mu\in\Delta(\cMps[f])$. For each $M\in\cM$, we denote $\nu_M$ to be the distribution of $x$ under $M$, and we denote $\onu=\EE_{M\sim \mu}\nu_M$ to be the distribution of $x$ under $\oM$. Further, we know that under $(x,y)\sim \oM$,
\begin{align*}
    y|x\sim \Rad{ \fbar(x) }, \quad\text{where }\fbar(x)=\EE_{M|x} \fm(x),
\end{align*}
where the conditional expectation is taken over $M\sim \mu, x\sim \nu_M$. Therefore, we have
\begin{align*}
    \EE_{x\sim \onu}\abs{ \fbar(x)-f(x) }
    =&~ \EE_{x\sim \onu}\abs{ \EE_{M|x} \fm(x)-f(x) } \\
    \leq&~ \EE_{x\sim \onu}\EE_{M|x}\abs{  \fm(x)-f(x) } \\
    =&~ \EE_{M\sim \mu} \EE_{x\sim \nu_M} \abs{  \fm(x)-f(x) }\leq \Delta.
\end{align*}
Therefore, we can take $M'$ to be the model with covariate distribution $\onu$ and $f\sups{M'}=f$, and we have
\begin{align*}
    \DTV{ M', \oM }=\EE_{x\sim \onu} \abs{\fbar(x)-f(x)}\leq \Delta.
\end{align*}
The proof of the claim is hence completed. 

Now, with the above claim, for any reference model $\oM\in\coM$, we can bound
\begin{align*}
&~\pdecol_\gamma(\cMPs,\oM)\\
=&~ \infpql \sup_{M\in\cMPs}\sset{ \EE_{\pio\sim p}[\LM{\pi}] -\gamma \EE_{\lf\sim q} \Dl^2( M, \oM) }  \\
\leq&~ \infpl \sup_{M\in\cMPs} \inf_{M'\in\cM} \EE_{\pio\sim p}[\LM[M']{\pi}] + 2\DTV{M',M} -\gamma \EE_{\lf\sim q} \brac{ \frac12\Dl^2( M', \oM)-\Dl^2(M',M) } \\
\leq&~ \infpl \sup_{M'\in\cM} \EE_{\pio\sim p}[\LM[M']{\pio}] + 2\Delta -\frac{\gamma}{2} \EE_{\lf\sim q} \brac{ \Dl^2( M', \oM) } +\gamma\Delta^2 \\
=&~ \pdecol_{\gamma/2}(\cM)+2\Delta+\gamma\Delta^2,
\end{align*}
where the second line follows from the fact that $\abs{ \LM{\pi}-\LM[M']{\pi} }\leq 2\DTV{M',M}$ (because $L$ is reward-based) and
\begin{align*}
    \Dl^2( M', \oM)
    \leq 2\Dl^2( M, \oM)+2\Dl^2( M', M).
\end{align*}
Taking supremum over $\oM\in\coM$ gives the desired result.
\qed

\subsubsection{Contextual Bandits}\label{appdx:ExO-CB}

\newcommand{\cFd}{\cF_\Delta}
In this section, we work with contextual DMSO (introduced in \cref{ssec:CBs}). Note that contextual DMSO is not encompassed by \pDMSO, because the distribution of contexts can be changing throughout $T$ rounds of interactions. However, the idea of \cref{appdx:ExO-val-based} can still be applied, and we frame it through the notation of \tpb~\infosets~(with respect to the constraint class $\MPowcxt$, defined in \eqref{def:MPow-cxt}). 

We first recall the definition of the $L_\infty$-covering number.
\begin{definition}\label{def:covering-cF}
For a function class $\cF\subseteq (\cX\times\cA\to [-1,1])$ and parameter $\Delta\geq 0$, a $\Delta$-covering of $\cF$ is a subset $\cF'\subseteq \cF$ such that for any $f\in\cF$, there exists $f'\in\cF'$ with $\sup_{x,a}\abs{f(x,a)-f'(x,a)}\leq \Delta$.

We define the $\Delta$-covering number of $\cF$ as $N_{\infty}(\cF,\Delta)\defeq \inf\sset{|\cF'|: \cF'\text{ is a $\Delta$-covering of }\cF}$.
\end{definition}

Now, we define an \tpb~\infosets~$\Psi=\Pscxt$ for the constraint class $\MPowcxt$ by taking a minimal $\Delta$-covering $\cFd$ of $\cF$, and let
\begin{align}\label{def:info-cxt}
    \Pscxt=\cFd, \qquad
    \cMps=\sset{ M_{\nu,f}: \sup_{x,a}|f(x,a)-\psi(x,a)|\leq \Delta },\qquad\forall \psi\in\Pscxt,
\end{align}
and we set $\pi_\psi\in\Pi$ be $\pi_\psi(x)=\argmax_{a\in\cA} \psi(x,a)$.
Then by definition, $\log\IDC[2\Delta]{\MPowcxt,\Pscxt}\leq \log|\Pscxt|=\log N_{\infty}(\cF,\Delta)$. 

\paragraph{Proof of \cref{thm:CB-adv}}
Similar to \cref{lem:realizable-dec-o}, we can show that with $\cM=\cMFcb$ and $\Psi=\Pscxt$,
\begin{align*}
    \rdecol_{\gamma}(\cMPs)
    \leq \rdecol_{\gamma/2}(\cM)+\gamma\Delta^2+2\Delta.
\end{align*}
Therefore, when \LDPexo~is instantiated with $\Psi=\Pscxt$, with probability at least $1-\delta$,
\begin{align*}
    \frac1T\regdm(T)\leq 4\Delta+\rdecol_{c\alpha^2\gamma}(\cM)+2\gamma\brac{ \Delta^2+\frac{\log N_{\infty}(\cF,\Delta)+\log(1/\delta)}{T} }.
\end{align*}
Taking a suitable $\Delta\geq 0$ and $\gamma>0$ according to \cref{prop:offset-to-c} completes the proof of \cref{thm:CB-adv}.
\qed

\subsection{Proof of \cref{thm:ExO-upper}}\label{appdx:proof-ExO-upper}

We first invoke the following lemma, which requires careful analysis due to the adversarial nature (in particular, $\cEs$ may depend on the full history). The proof of \cref{lem:ExO} is deferred to the end of this section.
\begin{lemma}\label{lem:ExO}
Denote
\begin{align*}
    \Err(p,q,\wf;\qd,\Ms,\psi)\defeq -\log\EE_{\tpi\sim q}\EE_{o\sim \Ms(\tpi)}\EE_{\psi'\sim \qd}\brac{ \exp\paren{\wf(\psi';\tpi,o)-\wf(\psi;\tpi,o)} }.
\end{align*}
Then \whp, it holds that
\begin{align*}
    \min_{\psis\in\cEs} \sumt \Err(p\ind{t},q\ind{t},\wf\ind{t};\qd\ind{t},M^t,\psis) 
    \leq 2\log(1/\qd\ind{1}(\cEs))+2\log(1/\delta).
\end{align*}
\end{lemma}

Under the success event of \cref{lem:ExO}, there exists a $\psis\in\cEs$ such that
\begin{align*}
    \sumt \Err(p\ind{t},q\ind{t},\wf\ind{t};\qd\ind{t},M^t,\psis) 
    \leq 2\log(1/\qd\ind{1}(\cEs))+2\log(1/\delta).
\end{align*}
Notice that $\psis\in\cEs$ implies $M^1,\cdots,M^T\in \cMps[\psis]$, and $\VMs[\oMs]-\VM[\oMs](\pi_{\psis})\leq \Delta$, and in particular,
\begin{align*}
    \max_{\pi} \frac1T\sumt\paren{ \VM[M^t](\pi)-\VM[M^t](\pi_{\psis}) }\leq \Delta.
\end{align*}
Hence,
\begin{align*}
    &~ \max_{\pi\in\Pi}\sumt \VM[M^t](\pi)-\EE_{\pi\ind{t}\sim p\ind{t}}\brac{\VM[M^t](\pi\ind{t})} \\
    \leq&~ T\Delta+ \sumt \EE_{\pi\ind{t}\sim p\ind{t}}\brac{ \VM[M^t](\pi_{\psis})-\VM[M^t](\pi\ind{t}) } \\
    =&~ T\Delta
    + \gamma \sumt \Err(p\ind{t},q\ind{t},\wf\ind{t};\qd\ind{t},M^t,\psis) \\
    &~ +\sumt  \underbrace{ \brac{ \EE_{\pi\ind{t}\sim p\ind{t}}\brac{ \VM[M^t](\pi_{\psis})-\VM[M^t](\pi\ind{t}) } - \gamma\Err(p\ind{t},q\ind{t},\wf\ind{t};\qd\ind{t},M^t,\psis) } }_{\leq \Gamma_{\qd\ind{t},\gamma}(p\ind{t},q\ind{t},\wf\ind{t};M^t,\psis)} \\
    \leq&~ T\Delta+2\gamma\brac{\log(1/\qd\ind{1}(\cEs))+\log(1/\delta)} + \sumt \Gamma_{\qd\ind{t},\gamma}(p\ind{t},q\ind{t},\wf\ind{t}) \\
    \leq&~ T\paren{ \Delta+\exo{\gamma}(\Psi) }+2\gamma\brac{\log(1/\qd\ind{1}(\cEs))+\log(1/\delta)},
\end{align*}
where the second inequality uses $M^t\in\cMps[\psis]$. This is the desired upper bound.
\qed

\subsubsection{Proof of \cref{lem:ExO}}
For simplicity of presentation, we only consider the case where $\Psi$ is countable.
By definition,
\begin{align*}
    \qd\ind{t}(\psi)=\frac{\qd\ind{1}(\psi)\exp\paren{\sum_{s=1}^t  \wf^s(\psi;\tpi^s,o^s)}}{\sum_{\psi'\in\Psi} \qd\ind{1}(\psi')\exp\paren{\sum_{s=1}^{t-1}  \wf^s(\psi';\tpi^s,o^s)}},
\end{align*}
and hence
\begin{align*}
    \log \EE_{\psi\sim \qd\ind{t}}\brac{ \exp\paren{\wf\ind{t}(\psi;\tpi\ind{t},o\ind{t})} }=&~\log \EE_{\psi\sim \qd\ind{1}} \exp\paren{\sum_{s=1}^t  \wf^s(\psi;\tpi^s,o^s)} \\
    &~-\log \EE_{\psi\sim \qd\ind{1}} \exp\paren{\sum_{s=1}^{t-1}  \wf^s(\psi;\tpi^s,o^s)}.
\end{align*}
Therefore, taking summation over $t=1,\cdots,T$, we have
\begin{align}\label{eqn:proof-FTRL}
    -\sum_{t=1}^T \log \EE_{\psi\sim \qd\ind{t}}\brac{ \exp\paren{\wf\ind{t}(\psi;\tpi\ind{t},o\ind{t})} } = -\log\EE_{\psi\sim \qd\ind{1}}\brac{ \exp\paren{\sum_{t=1}^T \wf\ind{t}(\psi;\tpi\ind{t},o\ind{t})} }.
\end{align}
Thus, we define
\begin{align*}
    X_t(\psi; \tpi\ind{t},o\ind{t})\defeq  -\wf\ind{t}(\psi;\tpi\ind{t},o\ind{t}) + \log \EE_{\psi\sim \qd\ind{t}}\brac{ \exp\paren{\wf\ind{t}(\psi;\tpi\ind{t},o\ind{t})} },
\end{align*}
and \eqref{eqn:proof-FTRL} implies (deterministically)
\begin{align*}
    \EE_{\psi\sim \qd\ind{1}} \exp\paren{-\sum_{t=1}^T X_t(\psi; \tpi\ind{t},o\ind{t})}=1.
\end{align*}
Notice that for any $\psi\in\Psi$, we also have
\begin{align*}
    \EE^{\ExOp}\exp\paren{\sum_{t=1}^T X_t(\psi; \tpi\ind{t},o\ind{t})-\log \EE_{t-1}\brac{ \exp\paren{X_t(\psi; \tpi\ind{t},o\ind{t}) } }}=1, 
\end{align*}
where the expectation $\EE^{\ExOp}$ is taken over the randomness of the interaction between \ExOp~algorithm and the environment. 

Further, by the definition of $X_t$ and $\Err$, it holds that for any fixed $\psi$,
\begin{align*}
    -\log \EE_{t-1}\brac{ \exp\paren{X_t(\psi; \tpi\ind{t},o\ind{t}) }}=\Err(p\ind{t},q\ind{t},\wf\ind{t};\qd\ind{t},M^t,\psi).
\end{align*}
Combining the equations above and applying Cauchy inequality, we now have
\begin{align*}
    \EE_{\psi\sim \qd\ind{1}}\EE^{\ExOp} \exp\paren{\frac{1}{2} \sum_{t=1}^T \Err(p\ind{t},q\ind{t},\wf\ind{t};\qd\ind{t},M^t,\psi) } \leq 1.
\end{align*}

Notice that $\psi\sim \qd\ind{1}$ is independent of the randomness of the $T$-round interactions under \ExOp. Therefore, we know
\begin{align*}
    \EE^{\ExOp}\brac{ \qd\ind{1}(\cEs)\exp\paren{\frac{1}{2} \min_{\psis\in\cEs}\sum_{t=1}^T \Err(p\ind{t},q\ind{t},\wf\ind{t};\qd\ind{t},M^t,\psis) } }\leq 1.
\end{align*}
Applying Markov's inequality completes the proof.
\qed

\subsection{Proof of \cref{thm:ExO-upper-ii}}\label{appdx:proof-ExO-upper-ii}

We follow the notations of \cref{appdx:proof-ExO-upper}, and the proof is essentially analogous. 

For \tpa~\infosets~$\Psi=\MPow$, there exists $\psis\in\Psi$ such that $\cPs=\cMps[\psis]$.
Then, a direct adaption of \cref{lem:ExO} yields
\begin{align*}
    \sumt \Err(p\ind{t},q\ind{t},\wf\ind{t};\qd\ind{t},M^t,\psis) 
    \leq 2\log|\MPow|+2\log(1/\delta),
\end{align*}
as $\qd^1=\Unif(\MPow)$.

Therefore,
\begin{align*}
\riskdm(T)=&~\EE_{\hpi\sim \phat} \LM[\cPs]{\pi} \\
=&~\frac1T \sumt \EE_{\pi\ind{t}\sim p\ind{t}}\brac{ \LM[\cPs]{\pi\ind{t}} } \\
=&~\frac1T \sumt \EE_{\pi\ind{t}\sim p\ind{t}}\brac{ L_{\psis}(M^t,\pi\ind{t}) } \\
=&~ \frac{\gamma}{T} \sumt \Err(p\ind{t},q\ind{t},\wf\ind{t};\qd\ind{t},M^t,\psis) \\
    &~ +\frac1T \sumt  \underbrace{ \brac{ \EE_{\pi\ind{t}\sim p\ind{t}}\brac{ L_{\psis}(M^t,\pi\ind{t}) } - \gamma\Err(p\ind{t},q\ind{t},\wf\ind{t};\qd\ind{t},M^t,\psis) } }_{\leq \Gamma_{\qd\ind{t},\gamma}(p\ind{t},q\ind{t},\wf\ind{t};M^t,\psis)} \\
    \leq&~ \frac{2\gamma}{T}\brac{\log|\MPow|+\log(1/\delta)} + \frac1T \sumt \Gamma_{\qd\ind{t},\gamma}(p\ind{t},q\ind{t},\wf\ind{t}) \\
    \leq&~ \exo{\gamma}(\Psi) +\frac{2\gamma}{T}\brac{\log|\MPow|+\log(1/\delta)},
\end{align*}
where the first inequality uses the fact that $M^t\in\cPs=\cMpss$. This is the desired result.
\qed

\subsection{Proof of \cref{thm:exo-to-dec}}\label{appdx:proof-exo-to-dec}

The analysis below essentially follows the ideas of \citet{foster2022complexity}.

Let $\cM_0\defeq \bigcup_{\psi\in\Psi} \cMps$.
Below, we prove \cref{thm:exo-to-dec} for finite $\cM_0$ (and $\Psi$ is then automatically finite). The result for the general case then follows immediately by a covering argument (for details, see \cref{rmk:ExO-finite}).

\newcommand{\DI}{\Delta(\cI)}
To proceed, we fix $\gamma\geq 0, \qd\in\DPs$ and denote
\begin{align*}
    \cI\defeq \set{ (M,\psi): \psi\in\Psi, M\in\cMps }.
\end{align*}
We also fix a parameter $A>0$, and we define
\begin{align*}
    \wF_A\defeq \set{ \wf\in\wF: \linf{\wf}\leq A }.
\end{align*}
Then,
\begin{align*}
    \exo{\gamma}(\Psi,\qd)\leq &~ \inf_\pqwA \sup_{(M,\psi)\in\cI} \Gamma_{\qd,\gamma}(p,q,\wf;M,\psi) \\
    =&~ \inf_\pqwA \sup_{\mu\in\DI} \EE_{(M,\psi)\sim \mu} \Gamma_{\qd,\gamma}(p,q,\wf;M,\psi).
\end{align*}
To proceed, we apply Ky Fan's minimax theorem~\citep{fan1953minimax} (\cref{thm:minimax}). Note that $\DI$ is a compact and convex subset of the Euclidean space $\R^{\cI}$ (because $\cI$ is finite), and $\wF$ is a vector space. %
Thus, we consider the following function
\begin{align*}
    F_{p,q}(\wf, \mu)\defeq \EE_{(M,\psi)\sim \mu} \Gamma_{\qd,\gamma}(p,q,\wf;M,\psi),
\end{align*}
and by definition, $F_{p,q}$ is a bilinear function, and %
for any fixed $\wf\in\wF_A$, $F_{p,q}(\wf, \cdot)$ is a concave, continuous function of $\mu\in\DI$ (the continuity follows from the fact that $\wf$ is uniformly bounded by $A$). Therefore, Ky Fan's minimax theorem (\cref{thm:minimax}) gives
\begin{align*}
    \inf_{\wf\in\wF_A} \max_{\mu\in\DI} F_{p,q}(\wf,\mu) = \max_{\mu\in\DI}  \inf_{\wf\in\wF_A} F_{p,q}(\wf,\mu).
\end{align*}
Next, we compute $G(p,q;\mu)\defeq \inf_{\wf\in\wF_A} F_{p,q}(\wf,\mu)$. It is equivalent to compute
\begin{align}
    G_0(p,q;\mu)\defeq \inf_{\wf\in\wF_A} \EE_{\tpi\sim q}\EE_{o\sim M(\tpi)}\EE_{\psi'\sim \qd}\brac{ \exp\paren{ \wf(\psi';\tpi,o)-\wf(\psi;\tpi,o) } }.
\end{align}
For any $\mu\in\DI$ and $\tpi\in\bPi$, we define $\PP_{\mu,\tpi}$ to be the distribution of $(\pi,o,\psi)$ generated by $(M,\psi)\sim \mu, \tpi\sim q, o\sim M(\tpi)$. Then, by \cref{lem:DH-var}, 
\begin{align*}
    G_0(p,q;\mu)=&~ \inf_{\wf\in\wF_A} \EE_{\tpi\sim q, o\sim \PP_{\mu,\tpi}}\sset{ \EE_{\psi\sim \PP_{\mu,\tpi}(\cdot|o)}\brac{\exp\paren{-\wf(\psi;\tpi,o)}}\cdot \EE_{\psi'\sim \qd}\brac{ \exp\paren{ \wf(\psi';\tpi,o) } } } \\
    \leq &~ \EE_{\tpi\sim q, o\sim \PP_{\mu,\tpi}}\paren{1-\DH{\PP_{\mu,\tpi}(\cdot|o), \qd}}^2+3e^{-A},
\end{align*}
and hence
\begin{align*}
    G(p,q;\mu)\leq \EE_{(M,\psi)\sim \mu, \pi\sim p}\brac{ \Lps[M]{\pi} } + \gamma \EE_{\tpi\sim q, o\sim \PP_{\mu,\tpi}}\paren{1-\DH{\PP_{\mu,\tpi}(\cdot|o), \qd}}^2 - \gamma+3e^{-A}.
\end{align*}
Notice that for any fixed $(p,q)\in\Sp$, $G(p,q;\mu)$ is a convex, continuous function of $\mu$ (by definition). %
For any fixed $\mu\in\DI$, $G(p,q;\mu)$ is a linear function of $(p,q)\in\Sp$ and hence convex-like. Therefore, applying Ky Fan's minimax theorem (\cref{thm:minimax}) again gives
\begin{align*}
    \inf_{(p,q)\in\Sp} \max_{\mu\in\DI} G(p,q;\mu)= \max_{\mu\in\DI} \inf_{(p,q)\in\Sp} G(p,q;\mu).
\end{align*}
Finally, we proceed to bound $G(p,q;\mu)$. Using the fact that $1-(1-x)^2\geq x$ for $x\in[0,1]$, we have
\begin{align*}
    G(p,q;\mu)\leq \EE_{(M,\psi)\sim \mu, \pi\sim p}\brac{ \Lps[M]{\pi} } - \gamma \EE_{\tpi\sim q, o\sim \PP_{\mu,\tpi}}\DH{\PP_{\mu,\tpi}(\cdot|o), \qd}.
\end{align*}
We then invoke the following lemma:
\newcommand{\Mpsmu}{M_{\psi|\mu}}
\begin{lemma}\label{lem:PXY}
For any $\mu\in\DI$ and $\psi\in\Psi$, we denote
\begin{align*}
    \Mpsmu=\EE_{M\sim \mu(\cdot|\psi)}[M]\in\co(\cMps), \qquad \oM_\mu=\EE_{M\sim \mu}[M]\in\co(\cM_0).
\end{align*}
Then it holds that for any $\tpi\in\bPi$, $w\in\Delta(\Psi)$, %
\begin{align*}
    4\cdot \EE_{o\sim \PP_{\mu,\tpi}}\DH{\PP_{\mu,\tpi}(\cdot|o), \qd} \geq \EE_{\psi\sim \mu}\DH{ \Mpsmu(\tpi), \oM_\mu(\tpi) }.
\end{align*}
\end{lemma}
Therefore, using \cref{lem:PXY} and the fact that $\Lps[M]{\pi}$ is affine over $M$, it holds that
\begin{align*}
    G(p,q;\mu)\leq \EE_{\psi\sim \mu, \pi\sim p}\brac{ \Lps[\Mpsmu]{\pi} }+3e^{-A} - \frac\gamma4 \EE_{\psi\sim \mu, \tpi\sim q}\DH{\Mpsmu(\tpi), \oM_\mu(\tpi) }.
\end{align*}
Hence, we have
\begin{align*}
    &~\max_{\mu\in\DI} \inf_{(p,q)\in\Sp} G(p,q;\mu) -3e^{-A} \\
    \leq&~ \max_{\mu\in\DI} \inf_{(p,q)\in\Sp} \EE_{\psi\sim \mu, \pi\sim p}\brac{ \Lps[\Mpsmu]{\pi}} - \frac\gamma4\EE_{\psi\sim \mu, \tpi\sim q} \DH{\Mpsmu(\tpi), \oM_\mu(\tpi) }  \\
    \leq&~ \max_{\oM\in\co(\cM_0)} \max_{\mu'\in\DMPs}  \inf_{(p,q)\in\Sp} \EE_{M'\sim \mu', \pi\sim p}\brac{ \Lps[M']{\pi} }  - \frac\gamma4\EE_{M'\sim \mu', \tpi\sim q}\DH{M'(\tpi), \oM(\tpi) } \\
    \leq&~ \max_{\oM\in\co(\cM_0)}  \inf_{(p,q)\in\Sp} \max_{(M,\psi): M\in\cMps} \EE_{\pi\sim p}\brac{ \Lps[M]{\pi} }  - \frac\gamma4\EE_{\tpi\sim q}\DH{M(\tpi), \oM(\tpi) },
\end{align*}
where the last line follows again from the weak duality. 

To finalize the proof, we notice that by the arbitrariness of $w\in\DPs$, we have already proven
\begin{align}\label{eqn:ExO-coef-upper-full}
    \exo{\gamma}(\Psi)\leq 3e^{-A}+ \max_{\oM\in\co(\cM_0)}  \inf_{(p,q)\in\Sp} \max_{(M,\psi): M\in\cMps} \EE_{\pi\sim p}\brac{ \Lps[M]{\pi} }  - \frac{\gamma}{4} \EE_{\tpi\sim q}\DH{M(\tpi), \oM(\tpi) }.
\end{align}
Then, taking $A\to +\infty$, we obtain the following results (note that $\co(\cM_0)=\co(\cMPow)=\cM^+$): 

(1) If $\Psi$ is \tpb~\infosets, we have $\Lps[M]{\pi}\leq \LM{\pi}$. Hence, with option \optpac, $\Sp=\Sppac=\DPi\times\DbPi$, and hence in this case
\begin{align*}
    \exo{\gamma}(\Psi)\leq \pdeco_{\gamma/4}(\cMPs).
\end{align*}
Similarly, with option \optreg, $\Sp=\Spnr=\set{(p|_\Pi,p): p\in\DbPi}$, and hence
\begin{align*}
    \exo{\gamma}(\Psi)\leq \rdeco_{\gamma/4}(\cMPs).
\end{align*}

(2) If $\Psi=\MPow$ is the \tpa~\infosets, we have $\Lps[M]{\pi}=\LM[\psi]{\pi}$, and hence
\begin{align*}
    \exo{\gamma}(\Psi)\leq \pdecog_{\gamma/4}(\MPow).
\end{align*}
\qed

\paragraph{Proof of \cref{lem:PXY}} Our proof essentially follows \citet[Appendix C.2]{foster2022complexity}. For simplicity of presentation, we abbreviate $\PP=\PP_{\mu,\tpi}$. By the convexity of the squared Hellinger distance, we have
\begin{align*}
    \EE_{o\sim \PP}\DH{\PP(\psi=\cdot|o), \qd}\geq \DH{ \PP(\psi=\cdot), \qd }.
\end{align*}
Therefore, using the triangle inequality,
\begin{align*}
    4\EE_{o\sim \PP} \DH{\PP(\cdot|o), \qd}
    \geq&~ \EE_{o\sim \PP}\brac{ 2 \DH{\PP(\psi=\cdot|o), \qd} + 2\DH{ \PP(\psi=\cdot), \qd }}\\
    \geq&~ \EE_{o\sim \PP} \DH{\PP(\psi=\cdot|o), \PP(\psi=\cdot) } \\
    =&~ \EE_{\psi\sim \PP} \DH{ \PP(o=\cdot|\psi), \PP(o=\cdot) },
\end{align*}
where the last equality is because squared Hellinger distance is a $f$-divergence. 

Recall that $\PP=\PP_{\mu,\tpi}$ generated $(\psi, o)$ as $(M,\psi)\sim \mu, o\sim M(\tpi)$. Therefore, for any $\psi$, $\PP(o=\cdot|\psi)$ is the distribution of $o$ generated as $M\sim \mu(\cdot|\psi), o\sim M(\tpi)$, i.e., $o\sim M_{\psi|\mu}(\tpi)$. Hence,
\begin{align*}
    \DH{ \PP(o=\cdot|\psi), \PP(o=\cdot) } = \EE_{\tpi\sim q}\DH{ \Mpsmu(\tpi), \oM_\mu(\tpi) }.
\end{align*}
Combining the equations above completes the proof.
\qed

\newcommand{\IB}[1]{I_{\mathrm{B}}(#1)}
\begin{lemma}\label{lem:DH-var}
For any distribution $\PP,\QQ\in \Delta(\Psi)$, we denote
\begin{align}
    \IB{\PP,\QQ}=1-\DH{\PP,\QQ}=\sum_{\psi\in\Psi} \sqrt{\PP(\psi)\QQ(\psi)}.
\end{align}
Then for $A>0$, it holds that
\begin{align*}
    \IB{\PP,\QQ}^2\leq \inf_{f\in (\Psi\to \R): \linf{f}\leq A} \EE_{\PP} [e^{f(x)}] \EE_{\QQ} [e^{-f(x)}] \leq \IB{\PP,\QQ}^2+3e^{-A} .
\end{align*}
\end{lemma}
\begin{proof}
The lower bound follows immediately from Cauchy inequality. In the following, we proceed to prove the upper bound.

\newcommand{\clip}{\mathsf{clip}}
Consider the function $f=f_{\PP;\QQ}$ given by
\begin{align*}
    f_{\PP;\QQ}(\psi)=\frac12 \clip_{[-A,A]}\paren{ \log\frac{\QQ(\psi)}{\PP(\psi)} }.
\end{align*}
Then, by definition, %
\begin{align*}
    \EE_{\PP} [e^{f(x)}]
    \leq \EE_{\PP} \exp\paren{\frac12\log\frac{\QQ(\psi)}{\PP(\psi)}}+e^{-A}
    = \sum_{\psi\in\Psi} \sqrt{\PP(\psi)\QQ(\psi)}+e^{-A},
\end{align*}
and similarly,
\begin{align*}
    \EE_{\QQ} [e^{f(x)}]
    \leq \EE_{\QQ} \exp\paren{\frac12\log\frac{\PP(\psi)}{\QQ(\psi)}}+e^{-A}
    = \sum_{\psi\in\Psi} \sqrt{\PP(\psi)\QQ(\psi)}+e^{-A}.
\end{align*}
Therefore, for such a choice of $f$ ensures
\begin{align*}
    \EE_{\PP} [e^{f(x)}] \EE_{\QQ} [e^{-f(x)}]\leq \paren{ \IB{\PP,\QQ}+e^{-A} }^2\leq \IB{\PP,\QQ}^2+3e^{-A},
\end{align*}
where the last inequality uses $\IB{\PP,\QQ}\leq 1$.
\end{proof}

\begin{remark}[Covering argument]\label{rmk:ExO-finite}
In the following, we briefly discuss how our analysis applies to an infinite $\cM_0$ with a covering argument. It is easy to deal with \tpb~\infosets, so we focus on \tpb~\infosets.

Fix a parameter $\eps\in(0,1]$.  We take a finite subset $\cM'\subseteq \cM_0$, so that for any $M\in\cM_0$, there exists $M'\in\cM'$, such that
\begin{align*}
    \sup_{\bpi}\DTV{ M(\bpi), M'(\bpi) }\leq \eps,\qquad
    \sup_{\pi}\abs{\Vm(\pi)-\Vm[M'](\pi)}\leq \eps.
\end{align*}
Then, $\Psi$ induces an \infosets~over $\cM'$, given by
\begin{align*}
    \cMps'\defeq \set{ M\in\cM': M\in\cMps }, \qquad \forall \psi\in\Psi.
\end{align*}
Because $\cM'$ is finite, the set $\set{ \cMps': \psi\in\Psi }$ is also finite. 
Therefore, there exists a finite subset $\Psi'\subseteq \Psi$, such that for any $\psi\in\Psi, M\in\cMps$, there exists $[\psi]\in\Psi', M'\in\cMps[{[\psi]}]'$, so that
\begin{align*}
    \sup_{\bpi}\DTV{ M(\bpi), M'(\bpi) }\leq \eps,\qquad
    \sup_{\pi}\abs{\Vm(\pi)-\Vm[M'](\pi)}\leq \eps, \qquad \Vm(\pi_\psi)\leq \Vm[M'](\pi_{[\psi]})+2\eps.
\end{align*}
Then, we can bound %
\begin{align*}
    \exo{\gamma}(\Psi,\qd)
    \leq&~ \inf_\pqwA \sup_{\psi\in\Psi,M\in\cMps} \Gamma_{\qd,\gamma}(p,q,\wf;M,\psi) \\
    \leq&~ \inf_{(p,q)\in\Sp, \wf'\in\wF_A'} \sup_{\psi\in\Psi',M\in\cMps'} \Gamma_{\qd',\gamma}(p,q,\wf';M,\psi) +(2+2\gamma +e^{2A})\eps,
\end{align*}
where we let $\qd'\in\Delta(\Psi')$ to be given by $\qd'(\psi')=\qd(\psi: [\psi]=\psi')$ for all $\psi'\in\Psi'$, and the second inequality because for any map $\wf'\in\wF_A'\defeq (\Psi\times\bPi\times\cO\to [-A,A])$, we can consider the induced map $\wf\in\wF_A$ given by $\wf(\psi;\tpi,o)=\wf'([\psi];\tpi,o)$ for any $\psi\in\Psi$. %

Using \eqref{eqn:ExO-coef-upper-full}, we have for PAC learning,
\begin{align*}
    \exo{\gamma}(\Psi)\leq \pdeco_{\gamma/4}(\cM'_{\Psi'})+3e^{-A}+(2+2\gamma+e^{2A})\eps.
\end{align*}
Note that $\pdeco_{\gamma/4}(\cM'_{\Psi'})\leq \pdeco_{\gamma/4}(\cM_{\Psi})$, and hence first letting $\eps\to 0$ and then letting $A\to \infty$ gives the desired result. A similar argument also applies to no-regret learning. 
\end{remark}

\subsection{Proof of \cref{lem:offset-to-c-metric}}\label{appdx:proof-offset-to-c-metric}

Denote $D\defeq \sup_{\eps'\in[ \eps, 1]}\frac{ \pdecg_{\eps'}(\MPow) }{\eps'}$. We consider $\gamma=\frac{6D}{\eps}$.

We fix an arbitrary reference model $\oM$.
For each $j\geq 0$, we define $\eps_j=2^{-j}$, and let $d_j\defeq \pdecg_{\eps_j}(\MPow,\oM)$,
\begin{align*}
    (p_j,q_j)\defeq \argmin_\pqb \sup_{\cP\in\MPow} \constr{ \EE_{\pi\sim p} L(\cP,\pi) }{ \inf_{M\in\co(\cP)}\EE_{\bpi\sim q} \DH{ M(\bpi), \oM(\bpi) } \leq \eps_j^2 }.
\end{align*}
We define 
\begin{align*}
    \MPow_j\defeq \sset{ \cP: \inf_{M\in\co(\cP)}\EE_{\bpi\sim q_j} \DH{ M(\bpi), \oM(\bpi) } \leq \eps_j^2 }.
\end{align*}

We first claim that if for $j<k$, $\MPow_j \cap \MPow_k$ is empty, then $\MPow_{k+1}$ is empty. This is because if $\MPow_j \cap \MPow_k=\emptyset$, then for $q=\frac{q_j+q_k}{2}$, the set
\begin{align*}
    \sset{ \cP: \inf_{M\in\co(\cP)}\EE_{\bpi\sim q} \DH{ M(\bpi), \oM(\bpi) } \leq \eps_{k+1}^2 }
\end{align*}
must be empty,
which certifies $\pdecg_{\eps_{k+1}}(\MPow,\oM)=-\infty$, and hence by the optimality of $(p_{k+1},q_{k+1})$, $\MPow_{k+1}$ must be empty.

Therefore, we define $K_0$ be the minimum integer $k$ such that $\pdecg_{\eps_{k}}(\MPow,\oM)=-\infty$ (if such $k$ does not exist, we write $K_0=\infty$). We further define $K=\min\set{ \floor{\log_2(1/\eps)}, K_0-2 }$.

For every $j\leq K$, by definition, for any $\cP\in\MPow_j$, we have $\EE_{\pi\sim p_j} \rho(\pip, \pi) \leq d_j$. Thus, we can take $\pi_j=\argmin_{\pi'\in\Pi}\EE_{\pi\sim p_j} \rho(\pi', \pi)$, and then for any $\cP\in\MPow_j$, we have $\rho(\pip,\pi_j)\leq 2d_j$. Further, because $\MPow_{K+1}$ is not empty, $\MPow_j \cap \MPow_K$ is also not empty, and hence $\rho( \pi_j, \pi_K )\leq 2d_j+2d_K$. 

In the following, we choose $\lambda_{j}=2^{j-K-1}$ for $j=1,2,\cdots,K$ and $\lambda_0=2^{-K}$, and we set $q=\sum_{j=0}^K \lambda_j q_j$. For any $\cP\in\MPow$, we proceed to bound the quantity
\begin{align*}
    F(P)\defeq \rho(\pip, \pi_K) - \inf_{M\in\co(\cP)}\EE_{\bpi\sim q} \DH{ M(\bpi), \oM(\bpi) }.
\end{align*}
We let $j\geq 0$ to be the largest integer such that $\cP\in\MPow_j$ (note that $\cP\in\MPow_0$ always). If $j=K$, then we have
\begin{align*}
    F(P)\leq \rho(\pip, \pi_K) \leq 2d_K.
\end{align*}
If $j<K$, then we have $\cP\not\in\MPow_{j+1}$, and hence
\begin{align*}
    F(P)\leq&~ \rho(\pip, \pi_K) - \gamma \cdot \lambda_{j+1} \inf_{M\in\co(\cP)}\EE_{\bpi\sim q_{j+1}} \DH{ M(\bpi), \oM(\bpi) } \\
    \leq&~ \rho(\pip, \pi_{j})+\rho(\pi_{j},\pi_K) - \gamma \lambda_{j+1} \eps_{j+1}^2 \\
    \leq&~ 3d_j+2d_K-\gamma \lambda_{j+1} \eps_{j+1}^2
    \leq 2d_K,
\end{align*}
where the last line uses the fact that $d_j\leq D\eps_j$ and $\lambda_{j+1} \eps_{j+1}^2=2^{-K-1}\eps_j\geq \frac{\eps\eps_j}{2}$. Therefore, we can conclude that
\begin{align*}
    F(\cP)\leq 2d_K\leq 2D\eps_K\leq 2D\eps, \qquad \forall \cP\in\MPow.
\end{align*}
This immediately implies $\pdecog_\gamma(\MPow,\oM)\leq 2D\eps$, and the desired upper bound follows by taking supremum over all $\oM$.
\qed

\section{Estimation-to-Decision Algorithm and Guarantees}\label{appdx:E2D}
In this section, we present the extensions of the PAC E2D algorithm~\citep{foster2021statistical,foster2023tight} to LDP learning and query-based learning. 

\subsection{LDP-E2D Algorithm}\label{appdx:LDP-E2D}

\newcommand{\AlgEst}{\mathbf{Alg}_{\bf Est}}
\newcommand{\To}{N}
\newcommand{\Estl}{\mathbf{Est}_{\bf sq}}
\newcommandx{\Estlb}[2][1=T,2=\delta]{\Estl\paren{#1,#2}}

\newcommand{\EstBar}{\overline{\mathbf{Est}}_{\bf sq}}
\renewcommand{\veps}{\Bar{\eps}(T)}
\newcommand{\til}{\widetilde}
\newcommand{\ql}{q^{(k)}}
\newcommand{\hl}{\widehat{k}}
\newcommandx{\hist}[2][1=1,2={t-1}]{(\pi\ind{i},
\lf\ind{i},o\ind{i})_{i=#1}^{#2}}
\newcommandx{\hists}[2][1=1,2={t-1}]{(
\lf\ind{i},o\ind{i})_{i=#1}^{#2}}
\newcommandx{\Esthist}[2][1=1,2={t-1}]{\AlgEst\paren{ (\pi\ind{i},
\lf\ind{i},o\ind{i})_{i=#1}^{#2} }}

In the following, we present \LDPetod, the LDP extension of the Estimation-to-Decision algorithm~\citep{foster2021statistical,foster2023tight}, for PAC learning in \pDMSO. In the following, we assume without loss of generality that $\cO=\set{-1,1}$.

The \LDPetod~algorithm is based on the \emph{binary} channels (\cref{example:binary-pr}). Specifically, \LDPetod~adopts the following protocol: For $t=1,\cdots,T$:
\begin{itemize}
    \item The algorithm selects a distribution $q\ind{t}\in\DPL$ (based on the history), sample $(\pi\ind{t},\lf\ind{t})\sim{}q\ind{t}$.
    \item The environment generates a noisy observation $o\ind{t}\sim \bpr[\lf\ind{t}]\circ \Mstar(\pi\ind{t})$, and reveals $o\ind{t}$ to the algorithm.
\end{itemize}
Note that this protocol automatically ensures the algorithm preserves \pLDP. Furthermore, conditional on $(\cH\ind{t-1},\pi\ind{t},\lf\ind{t})$, the noisy observation is generated as
\begin{align*}
    o\ind{t}\sim \Rad{ \ca \cdot \EE_{z\sim \Mstar(\pi\ind{t})}[\lf\ind{t}(z)] }.
\end{align*}
For simplicity of presentation, we denote $M(\pi)[\lf]\defeq \EE_{z\sim M(\pi)}[\lf(z)]$ in the following.

\subsubsection{Online estimation oracle}
The general DEC framework~\citep{foster2021statistical,foster2023tight} uses the primitive of an \emph{online estimation oracle}, denoted by
$\AlgEst$, which is an algorithm that produce estimates of the underlying model $\Mstar$ at each step based on the prior
observations. For \LDPetod, an estimation oracle at each round $t$, given the history
$\cH\ind{t-1}=\hist$, returns an estimator
\[
\hM\ind{t}=\AlgEst\prn*{ \cH\ind{t-1} }
\]
for the true model $\Mstar$. 
Here, the oracle's estimation performance is measured by cumulative squared error under each functional $\lf\ind{t}$, which is different from the non-private setting~\citep{foster2021statistical,foster2023tight} where 
the performance is measured in terms of the squared Hellinger error.
\begin{assumption}[Estimation oracle for $\cM$]
\label{asmp:EST}
At each time $t\in[T]$, an online estimation oracle
$\AlgEst$ for $\cM$ returns,
given  $$\cH\ind{t-1}=(\pi\ind{1},\lf\ind{1},o\ind{1}),\ldots,(\pi\ind{t-1},\lf\ind{t-1},o\ind{t-1})$$
with $(\pi\ind{i},\lf\ind{i})\sim p\ind{i}$ and $o\ind{i}\sim \Rad{\Mstar(\pi\ind{i})[\lf\ind{i}]}$, an estimator
$\hM^t\in(\Pi\to\DZ)$ such that whenever $\Mstar\in\cM$,
\begin{align}
\Estl \ldef{}
\sum_{t=1}^{T}\En_{(\pi\ind{t},\lf\ind{t})\sim{}q\ind{t}} \Dl[\lf\ind{t}]^2\paren{ \Mstar(\pi\ind{t}),\hM^t(\pi\ind{t}) }
\leq \Estlb,
\end{align}
with probability at least $1-\delta$, where $\Estlb$ is a
known upper bound that we assumed to be a non-decreasing function in $(T,\delta^{-1})$.
\end{assumption}

Oracles satisfying \pref{asmp:EST} can be obtained via online linear regression algorithms, the
estimation rate $\Estlb$ will typically reflect the statistical complexity of the class
$\cM$. Standard examples include Vovk's
Aggregation (\cref{prop:Vovk}) and Online Mirror Descent (\cref{prop:OMD}). For further background, see e.g.
\citet[Section 4]{foster2021statistical}.

\begin{proposition}[Vovk's Aggregation]\label{prop:Vovk}
Suppose that $\cM$ is finite. Then the \emph{Vovk's aggregation} algorithm achieves
\begin{align*}
    \Estlb\leqsim \frac{1}{\alpha^2}\cdot \log(|\cM|/\delta).
\end{align*}
Furthermore, for each round $t\in[T]$, $\hM^t\in\co(\cM)$.
\end{proposition}

\begin{proposition}[Online Mirror Descent]\label{prop:OMD}
Suppose that $\cM\subseteq \DZ$. Then the \emph{Online Mirror Descent} (\cref{alg:OMD}) achieves
\begin{align*}
    \Estlb\leqsim \frac{1}{\alpha}\sqrt{\CKL(\cM)\cdot T}+\frac{\log(1/\delta)}{\alpha^2},
\end{align*}
where $\CKL(\cM)$ is defined in \cref{prop:feldman}.
\end{proposition}

\subsubsection{LDP-E2D Algorithm and its guarantees}\label{appdx:LDP-E2D-details}

With an online estimation oracle $\AlgEst$, we present the \LDPetod~algorithm (\cref{appdx:E2D}), which generalizes the \etod~algorithm of \citet{foster2023tight} to LDP learning. \LDPetod~algorithm consists of two phases: the exploration phase and the refining phase.

\begin{algorithm}[H]
\setstretch{1.5}
\begin{algorithmic}[1]
\REQUIRE Round $T\geq 1$, error probability $\delta>0$, model class $\cM$, estimation oracle $\AlgEst$.
\STATE Define $K \ldef \lceil \log 2/\delta\rceil$, $\To \ldef
\frac{T}{K+1}$, and $\EstBar\ldef\Estlb[\To][
\frac{\delta}{4K}]$.
\STATE Set $\veps \ldef 8\sqrt{\frac{1}{N} \cdot
\EstBar}$.
\STATE \algcommentbig{Exploration phase}
\FOR{$t=1, 2, \cdots, \To$}
\STATE Compute estimator $\hM\ind{t} = \Esthist$.  \label{line:pac-compute-mhat}
\STATE Compute
\label{line:ptqt-pac}
\begin{align*}
(p\ind{t},q\ind{t}) := \argmin_{\substack{p\in\DDD\\q\in\DPL}}\sup_{M\in\cM}\constr{ \EE_{\pi\sim p}[\LM{\pi}] }{ \EE_{(\pi,\lf)\sim q} \Dl^2( M(\pi), \hM^t(\pi) )\leq \veps^2 }
\end{align*}
\STATE Sample decision $(\pi\ind{t},\lf\ind{t})\sim{}q\ind{t}$.
\STATE Receive $o\ind{t}\sim \bpr[\lf\ind{t}]\circ \Mstar(\pi\ind{t})$ from the environment.
\ENDFOR
\STATE \algcommentbig{Refining phase}
\STATE Sample $K$ indices $t_1, \ldots, t_K \sim\Unif(
[\To])$ independently. 
\FOR{$k=1,2,\cdots,K$}
\STATE Set $\ql\defeq q\ind{t_k}$.
\FOR{$t=k \To+1,\cdots,(k+1)\To$}
\STATE Compute estimator $\hM\ind{t} = \Esthist[k \To+1]$.
\STATE Sample $(\pi\ind{t},\lf\ind{t})\sim{}\ql$, and receive $o\ind{t}\sim \bpr[\lf\ind{t}]\circ \Mstar(\pi\ind{t})$ from the environment.
\ENDFOR
\STATE Compute $M^{(k)} := \frac{1}{\To} \sum_{t=k \To+1}^{(k+1)\To} \hM^t$.
\ENDFOR
\STATE Set $\hl\defeq \argmin_{k\in [K]} \EE_{(\pi,\lf)\sim q^{(k)}} \Dl^2\paren{ \hM^{t_k}, M^{(k)} }$ 
\ENSURE $\phat\ldef{} p^{(k)}$ and
$\pihat\sim{}\phat$
\end{algorithmic}
\caption{LDP Estimation-to-Decision Algorithm for PAC learning (\LDPetod)}
\label{alg:E2D}
\end{algorithm}

\paragraph{Exploration phase}
At each round $t\in[N]$ in this phase, the algorithm uses $\AlgEst$ to compute an estimator $\hM^t=\AlgEst(\cH\ind{t-1})$ based on the history $\cH\ind{t-1}=\hist$. Then, based on $\hM^t$, the algorithm computes a joint \emph{exploration-exploitation} distribution $(p\ind{t},q\ind{t})$ by solving the following Estimation-to-Decision objective:
\begin{align}\label{def:E2D-obj}
    (p\ind{t},q\ind{t}) := \argmin_{\substack{p\in\DDD\\q\in\DPL}}\sup_{M\in\cM}\constr{ \EE_{\pi\sim p}[\LM{\pi}] }{ \EE_{(\pi,\lf)\sim q} \Dl^2( M(\pi), \hM^t(\pi) )\leq \veps^2 }.
\end{align}
Note that the value of this minimax optimization problem is always bounded by $\pdecl_{\veps}(\cM,\hM^t)$. The algorithm then samples $(\pi\ind{t},\lf\ind{t})\sim q\ind{t}$ from the exploitation distribution, sends it to the $t$-th user, and receives the noisy observation $o\ind{t}\sim \Rad{\Mstar(\pi\ind{t})[\lf\ind{t}]}$ according to the interaction protocol.

After the exploration phase, the goal of the algorithm is to select an index $t\in[N]$ such that the distribution $p\ind{t}$ achieves low risk. Note that in general, the risk of $p\ind{t}$ may not be estimated from samples. However, if we can certify that $\EE_{(\pi,\lf)\sim q\ind{t}} \Dl^2( \Mstar(\pi), \hM^t(\pi) )\leq \veps^2$, then the risk $\EE_{\pi\sim p^t} \LM[\Mstar]{\pi}\leq \pdecl_{\veps}(\cM)$ is bounded automatically. Notice that by our assumption on $\AlgEst$ (\cref{asmp:EST}), \whp[\frac{\delta}{4K}],
\begin{align*}
    \sum_{t=1}^{N}\En_{(\pi\ind{t},\lf\ind{t})\sim{}q\ind{t}} \Dl[\lf\ind{t}]^2\paren{ \Mstar(\pi\ind{t}),\hM^t(\pi\ind{t}) }\leq \EstBar,
\end{align*}
and hence there are at least $N/2$ indices $t$ such that $\EE_{(\pi,\lf)\sim q\ind{t}} \Dl^2( \Mstar(\pi), \hM^t(\pi) )\leq \frac1{16}\veps^2$.
Therefore, in the refining phase, the algorithm proceeds as follows to identify an index $t$ such that $\hM^t$ achieves a small estimation error.

\paragraph{Refining phase}
At the start of this phase, the algorithm randomly samples $t_1,\cdots,t_K\sim \Unif([N])$. Then, \whp[\frac{3}{4}\delta],
\begin{align}\label{eqn:e2d-exist-k}
    \text{ there exists $k\in[K]$ such that }\EE_{(\pi,\lf)\sim q\ind{t_k}} \Dl^2( \Mstar(\pi), \hM^{t_k}(\pi) )\leq \frac{1}{16}\veps^2,
\end{align}
as we have argued above. Thus, for each batch $k\in[K]$, the algorithm uses $N$ rounds to obtain an estimator $M^{(k)}$ of the ground-truth model $\Mstar$ under the distribution $q^{(k)}\defeq q_{t_k}$: 

For each round $t\in[kN+1,(k+1)N]$ in the $k$-th batch, the algorithm samples $(\pi\ind{t},\lf\ind{t})\sim q^{(k)}$ and sends the pair to the learner. By running an instance of $\AlgEst$ within the batch, it is guaranteed that with probability at least $1-\frac{\delta}{4K}$
\begin{align*}
    \sum_{t=kN+1}^{(k+1)N}\En_{(\pi,\lf)\sim{}q^{(k)}} \Dl^2\paren{ \Mstar(\pi),\hM^t(\pi) }\leq \EstBar.
\end{align*}
Hence, by the convexity of the divergence $\Dl^2$, we have
\begin{align}\label{eqn:e2d-err}
    \En_{(\pi,\lf)\sim{}q^{(k)}} \Dl^2\paren{ \Mstar(\pi),M^{(k)}(\pi) }\leq \frac{\EstBar}{N}=\frac{1}{16}\veps^2.
\end{align}
Therefore, taking the union bound, Eq. \cref{eqn:e2d-exist-k} and Eq. \cref{eqn:e2d-err} (for each $k\in[K]$) hold simultaneously \whp. Therefore, under this success event, we know
\begin{align*}
    \min_{k\in[K]} \EE_{(\pi,\lf)\sim q^{(k)}} \Dl^2\paren{ \hM^{t_k}, M^{(k)} }\leq \frac{1}{4}\veps^2,
\end{align*}
and hence by triangle inequality,
\begin{align*}
    &~\EE_{(\pi,\lf)\sim q^{(\hl)}} \Dl^2\paren{ \Mstar(\pi),\hM^{t_{\hl}}(\pi) } \\
    \leq&~ \EE_{(\pi,\lf)\sim q^{(\hl)}} \Dl^2\paren{ \Mstar(\pi),M^{(\hl)}(\pi) }+\EE_{(\pi,\lf)\sim q^{(\hl)}} \Dl^2\paren{ M^{(\hl)}(\pi),\hM^{t_{\hl}}(\pi) }\leq \veps^2.
\end{align*}
Therefore, for $\phat=p\ind{t_{\hl}}$, we have
\begin{align*}
    \EE_{\hpi\sim \phat}\LM[\Mstar]{\hpi}\leq \pdecl_{\veps}(\cM,\hM^{t_{\hl}}).
\end{align*}

The argument above immediately yields the following guarantee of \LDPetod~(\cref{alg:E2D}).
\begin{theorem}\label{thm:PAC-E2D}
\LDPetod~(\cref{alg:E2D}) preserves \pLDP, and \whp, it holds that
\begin{align*}
    \riskdm(T)=\EE_{\phat}\LM[\Mstar]{\hpi}\leq \max_{t\in[T]}\pdecl_{\veps}(\cM,\hM^t).
\end{align*}
\end{theorem}
\cref{thm:pdec-lin-upper} is then a direct corollary by instantiating $\AlgEst$ with Vovk's aggregation (\cref{prop:Vovk}, where $\hM^t\in\coM$ for all $t\in[T]$). For \statp, we may also instantiate $\AlgEst$ with Online Mirror Descent (\cref{prop:OMD}) which gives \whp,
\begin{align*}
    \riskdm(T)\leq \sup_{\oM\in\DZ}\pdecl_{\veps}(\cM,\oM), \qquad
    \text{where }\veps\asymp \sqrt{\frac{\CKL(\cM)}{\alpha^2 T}}+\frac{\log(1/\delta)}{\alpha^2 T}.
\end{align*}

\subsubsection{Proof of \cref{prop:OMD}}

We present the specifications of Online Mirror Descent for online estimation in \cref{alg:OMD}, which is inspired by \citet{feldman2017general}.

\begin{algorithm}[H]
\begin{algorithmic}[1]
\REQUIRE History $\cH\ind{t-1}=\hists$, number of total rounds $N$
\STATE Parameters: Initial reference $\oM$ and stepsize $\eta=\sqrt{\frac{\CKL}{16N}}$.
\STATE Compute
\begin{align*}
    \hM^t[z]\propto_{z} \oM[z]\cdot \exp\paren{ -\eta\sum_{s=1}^{t-1} \paren{ \ca\lr \lf^s, \hM^{s}\rr-o^s }\lf^s(z) }
\end{align*}
\ENSURE Output $\hM^t\in\DZ$.
\end{algorithmic}
\caption{Online Mirror Descent}
\label{alg:OMD}
\end{algorithm}

\paragraph{Proof of \cref{prop:OMD}} Consider the loss function sequence
\begin{align*}
    L^t(M)=\frac{1}{2\ca}\paren{\ca\lr \lf\ind{t}, M\rr-o\ind{t}}^2, \qquad t\in[N].
\end{align*}
Then, \cref{alg:OMD} implements the online mirror descent with regularizer $R(M)=\KLd{M}{\oM}$ and stepsize $\eta$.
Using the well-known guarantee of mirror descent (see e.g. \citet{hazan2016online}), we have
\begin{align*}
    \sum_{t=1}^N \lr \nabla L^t(\hM^t), \hM^t-M \rr\leq \eta\sum_{t=1}^N\linf{ \nabla L^t(\hM^t) }^2+\frac{\KLd{M}{\oM}}{\eta}, \qquad \forall M\in\DZ.
\end{align*}
Notice that $\linf{ \nabla L^t(M) }\leq 2$ for any $M\in\DZ$.
Therefore, using the upper bound $\KLd{\Mstar}{\oM}\leq \CKL$ and our choice of $\eta$, we know
\begin{align*}
    \sum_{t=1}^N \lr \nabla L^t(\hM^t), \hM^t-M \rr\leq 16\eta N+\frac{\CKL}{\eta}=4\sqrt{N\CKL}.
\end{align*}
Notice that $\nabla L^t(\hM^t)=\paren{ \ca\lr \lf\ind{t}, \hM^t \rr - o\ind{t}}\cdot \lf\ind{t} $, and hence
\begin{align*}
    \lr \nabla L^t(\hM^t), \hM^t-M \rr
    =\paren{\Mstar[\lf\ind{t}]-o\ind{t}}\cdot \paren{ \hM^t[\lf\ind{t}]-\Mstar[\lf\ind{t}] }+\ca\paren{ \hM^t[\lf\ind{t}]-\Mstar[\lf\ind{t}] }^2.
\end{align*}
Therefore, we denote $X_t:=\hM^t[\lf\ind{t}]-\Mstar[\lf\ind{t}]$ and $Z_t:=o\ind{t}-\Mstar[\lf\ind{t}]$, and it holds that %
\begin{align*}
    \EE[Z_t|\cH\ind{t-1},\lf\ind{t}]=0,  \quad
    \EE[X_t^2|\cH\ind{t-1}]=\EE_{\lf\sim q\ind{t}} \Dl^2\paren{\hM^t,\Mstar},
\end{align*}
where we recall that $o\ind{t}\sim \Rad{\Mstar[\lf\ind{t}]}$. In particular, by Hoeffding's inequality, for any fixed parameter $\lambda>0$, \whp,
\begin{align*}
    \lambda \sum_{t=1}^N Z_tX_t - \frac{\lambda^2}{2} \sum_{t=1}^N X_t^2 \leq \log(1/\delta).
\end{align*}
Further, by Freedman's inequality and the fact that $X_t^2\in[0,1]$, \whp, 
\begin{align*}
    \sum_{t=1}^N \EE[X_t^2|\cH\ind{t-1}]\leq \frac32 \sum_{t=1}^N X_t^2 + 5\log(1/\delta).
\end{align*}
Therefore, we may choose $\lambda=\frac{\ca}{2}$, and then \whp[2\delta],
\begin{align*}
    4\sqrt{N\CKL}
    \geq &~ \sum_{t=1}^N \ca X_t^2 - X_tZ_t \\
    \geq &~ \frac{3\ca}{4} \sum_{t=1}^N  X_t^2 - \frac{2\log(1/\delta)}{\ca} \\
    \geq&~ \frac{\ca}{2} \sum_{t=1}^N \EE[X_t^2|\cH\ind{t-1}] - 3\ca\log(1/\delta)-\frac{2\log(1/\delta)}{\ca}.
\end{align*}

Using the fact $\EE[X_t^2|\cH\ind{t-1}]=\EE_{\lf\sim q\ind{t}} \Dl^2\paren{\hM^t,\Mstar}$ gives the desired upper bound.
\qed

\subsection{Query-based E2D algorithm}\label{appdx:SQ-E2D}

In the following, we present the E2D algorithm (\sqetod, \cref{alg:SQE2D}) for \qbDMSO. 

\newcommand{\hcM}{\widehat{\cM}}
\newcommand{\Oracle}{\mathsf{O}}
\newcommand{\vdel}{\Bar{\gamma}}
\newcommand{\herr}{\hat{e}}
\newcommand{\ks}{k^\star}

\begin{algorithm}[H]
\setstretch{1.5}
\begin{algorithmic}[1]
\REQUIRE Round $T\geq 1$, error probability $\delta>0$, model class $\cM\subseteq (\bPi\to \cV)$, GQ oracle $\Oracle$.
\STATE Define $K \ldef \lceil \log 2/\delta\rceil$, $T_0=\frac{T}{2}$, $N \ldef
\frac{T}{2K}$.
\STATE Set $\vdel \ldef C_0\max\sset{ \frac{\log|\cM|}{T}, \frac{\log(1/\delta)}{N} }$ for a large absolute constant $C_0$.
\STATE \algcommentbig{Exploration phase}
\FOR{$t=1, 2, \cdots, T_0$}
\STATE Compute $\mu^t=\Unif(\hcM^t)$, where
\begin{align}\label{eq:SQ-elimin}
    \hcM^t\defeq \set{ M\in\cM: \nrm{ \phq[\bpi^s]{M}-v^s }\leq \tau, \forall s<t },
\end{align}
\STATE Compute
\begin{align*}
(p\ind{t},q\ind{t}) := \argmin_\pqb\sup_{M\in\cM}\constr{ \EE_{\pi\sim p}[\LM{\pi}] }{ \PP_{\bpi\sim q, \oM\sim \mu^t} \paren{ \nrm{M(\bpi)-\oM(\bpi)}>2\tau }\leq \vdel }.
\end{align*}
\STATE Sample $\bpi\ind{t}\sim{}q\ind{t}$, query $\bpi\ind{t}$, and receive $v\ind{t}$ from the oracle $\Oracle$.
\ENDFOR
\STATE \algcommentbig{Refining phase}
\STATE Sample $K$ indices $t_1, \ldots, t_K \sim\Unif(
[T_0])$ independently. 
\STATE Set $\ks=1$.
\FOR{$k=1,2,\cdots,K$}
\STATE Set $\ql\defeq q^{t_k}$ and batch $\cT_k\defeq \set{T_0+(k-1)N+1,\cdots,T_0+kN}$.
\FOR{$t\in\cT_k$}
\STATE Sample $\bpi\ind{t}\sim{}q\ind{t}$, query $\bpi\ind{t}$, and receive $v\ind{t}$ from the oracle $\Oracle$.
\ENDFOR
\STATE Compute $\herr^{(k)} := \frac{1}{N} \sum_{t\in\cT_k} \PP_{M\sim \mu\ind{t}}\paren{ \nrm{\phq[\bpi\ind{t}]{M}-v\ind{t}}>\tau }$.
\IF{$\herr^{(k)}<\vdel$}
\STATE Set $\ks=k$ and \textbf{break}.
\ENDIF
\ENDFOR
\ENSURE $\phat\ldef{} p^{(\ks)}$ and
$\pihat\sim{}\phat$
\end{algorithmic}
\caption{Query-based Estimation-to-Decisions (SQ-\etod)}
\label{alg:SQE2D}
\end{algorithm}

We state the following guarantee of \sqetod.
\begin{theorem}\label{thm:SQ-upper-E2D}
For any model class $\cM$, \sqetod~(\cref{alg:SQE2D}) achieves that given access to any GSQ oracle $\GSQ$, with probability at least $1-\delta$,
\begin{align*}
    \riskdm(T)\leq \pdecltau[2\tau]_{\oeps(T)}(\cM),
\end{align*}
where $\oeps(T)=C\sqrt{\frac{\log|\cM|+\log^2(1/\delta)}{T}}$. 
\end{theorem}

\begin{remark}
We note that the $\log|\cM|$-factor above can be necessary for more general setting (e.g., interactive SQ learning). However, under SQ setting and distributional search problem (i.e., $\cM\subseteq \DZ$), \citet{feldman2017general} derives an upper bound~\cref{eqn:SQ-bounds-feldman} scaling with the SQ dimension (cf. \cref{ssec:SQ-dim}) and $\CKL$. When specialized to this setting, our upper bound above does not involve extra $\tau^{-1}$-factors, but the $\log|\cM|$-factor can be much larger than $\CKL$. However, if we replace the model elimination subroutine \cref{eq:SQ-elimin} with the Online Mirror Descent subroutine (\cref{alg:OMD}), then the obtained algorithm is essentially an analog of the one of \citet{feldman2017general} and achieves an upper bound scaling with \SQDEC~and $\CKL$.
\end{remark}

\subsubsection{Proof of \cref{thm:SQ-upper}}
The proof is analogous to the analysis in \cref{appdx:LDP-E2D-details}.
We first invoke the following lemma.
\newcommand{\Tz}{T_0}
\newcommand{\ts}{t^\star}
\begin{lemma}\label{lem:proof-SQ-E2D-telescope}
It holds that
\begin{align*}
    \sum_{t=1}^{\Tz} \PP_{M\sim \mu^t}\paren{ \nrm{ \phq[\bpi\ind{t}]{M}-v\ind{t} } > \tau }
    \leq \log|\cM|.
\end{align*}
\end{lemma}
Then, by Freedman's inequality, \whp[\frac{\delta}{4}], it holds that
\begin{align*}
    \sum_{t=1}^{\Tz} \PP_{\bpi\sim q\ind{t}, M\sim \mu^t}\paren{ \nrm{ \phq{M}-\phq{\Mstar} } > 2\tau }
    \leq&~ 2\sum_{t=1}^{\Tz} \PP_{M\sim \mu^t}\paren{ \nrm{ \phq[\bpi\ind{t}]{M}-\phq[\bpi\ind{t}]{\Mstar} } > 2\tau } + 4\log(4/\delta) \\
    \leq&~ 2\sum_{t=1}^{\Tz} \PP_{M\sim \mu^t}\paren{ \nrm{ \phq[\bpi\ind{t}]{M}-v\ind{t} } > \tau } + 4\log(4/\delta) \\
    \leq&~ 2\log|\cM|+4\log(4/\delta).
\end{align*}
In the following, we denote
\begin{align*}
    e(t)\defeq \PP_{\bpi\sim q\ind{t}, M\sim \mu^t}\paren{ \nrm{ \phq{M}-\phq{\Mstar} } > 2\tau }.
\end{align*}
Therefore, conditional on this success event, for at least $\frac{\Tz}{2}$ many $t\in[\Tz]$, $t$ belongs to the set
\begin{align*}
    \cB\defeq \sset{ t\in[\Tz]:  \PP_{\bpi\sim q\ind{t}, M\sim \mu^t}\paren{ \nrm{ \phq{M}-\phq{\Mstar} } > 2\tau }
    \leq \frac{2(2\log|\cM|+4\log(2/\delta))}{\Tz}\leq \frac{\vdel}{16} }.
\end{align*}
In particular, \whp[\frac{\delta}{2}], there exists $k\in[K]$ such that $t_k\in\cB$.

In the following, we denote $e^{(k)}\defeq e(t_k)$.
Then, by Freedman's inequality, \whp[\frac{\delta}{2}], the following holds for all $k\in[K]$:
\begin{align*}
    \herr^{(k)}\leq 2e^{(k)}+\frac{4\log(4K/\delta)}{N}, \qquad 
    e^{(k)}\leq 2\herr^{(k)}+\frac{4\log(4K/\delta)}{N}.
\end{align*}
Therefore, conditional on the all the success events, we know that there exists $k\in[K]$ such that $t_k\in\cB$, which implies $\herr^{(k)}\leq \frac{\vdel}{4}$. Hence, it is ensured that $\herr^{(\ks)}\leq \frac{\vdel}{4}$, which in terms implies $e^{(\ks)}\leq \vdel$. Therefore, for $\ts=t_{\ks}$, it holds that
\begin{align*}
    \PP_{\bpi\sim q^{\ts}, \oM\sim \mu^{\ts}}\paren{ \nrm{ \phq{\oM}-\phq{\Mstar} } > 2\tau }\leq \vdel.
\end{align*}
Thus,
\begin{align*}
    \EE_{\pi\sim p^{\ts}} \LM[\Mstar]{\pi} \leq&~  \inf_\pqb \sup_{M\in\cM} \constr{ \EE_{\pi\sim p}[\LM{\pi}] }{ \PP_{\bpi\sim q, \oM\sim \mu^{\ts}} \paren{ \nrm{\phq{M}-\phq{\oM}}>2\tau }\leq \vdel } \\
    \leq&~ \pdecltau[2\tau]_{2\vdel}(\cM).
\end{align*}
The proof of \cref{thm:SQ-upper} is hence completed.
\qed

\paragraph{Proof of \cref{lem:proof-SQ-E2D-telescope}} 
Let $U_t=|\hcM^t|$, and then
\begin{align*}
    \PP_{M\sim \mu^t}\paren{ \nrm{ \phq[\bpi\ind{t}]{M}-v\ind{t} } > \tau }
    =\frac{|\hcM^t\backslash \hcM^{t+1}|}{|\hcM^t|}=\frac{U_t-U_{t+1}}{U_t}
    \leq \log U_t - \log U_{t+1}.
\end{align*}
Taking summation and using $U_1=|\cM|$ completes the proof.
\qed

\section{Proofs from \cref{sec:SQ}}\label{appdx:SQ}
\subsection{Proof of \cref{thm:SQ-lower}}\label{appdx:proof-SQ-lower}

Fix $T\geq 1$, $\tau\geq 0$, reference model $\oM$. We first consider the case $L$ is metric-based, i.e., it is given by $L(M,\pi)=\rho(\pim,\pi)$ for a pseudo-metric $\rho$ over $\Pi$. 
We denote $\ueps\defeq \frac{1}{2\sqrt{T}}$ and $\Delta\defeq \pdecltau_{\ueps}(\cM,\oM)$.

We first describe any $T$-round query-based algorithm in the following way (cf. \cref{sec:cDMSO-main}).
A $T$-round algorithm $\alg = \set{q\ind{t}}_{t\in [T]}\cup\set{p}$ is specified by a sequence of mappings, where the $t$-th mapping $q\ind{t}(\cdot\mid{}\Hy\ind{t-1})$ specifies the distribution of $\bpi\ind{t}$ based on the history $\Hy\ind{t-1}=(\bpi^s,v^s)_{s\leq t-1}$, and the final map $p(\cdot \mid{} \Hy\ind{T})$ specifies the distribution of the \emph{output policy} $\pihat$ based on $\Hy\ind{T}$.

Next, we fix an arbitrary, randomized reference model $\oM:\bPi\to \Delta(\cV)$, and we construct a GQ oracle for each model $M\in\cM$ as follows. For $M\in\cM$, we let $\GSQ$ be an oracle that response to any decision $\bpi$ as
\begin{align*}
    \GSQ(\bpi)=\begin{cases}
        v, & \text{if } \nrm{\phq{M}-v}\leq \tau, \\
        \phq{M}, & \text{otherwise}.
    \end{cases} \qquad \text{where }v\sim \oM(\bpi).
\end{align*}
For any model $M$,  we let $\Pmalg{\cdot}$ to be the distribution of $(\cH^T,\hpi)$ generated by the algorithm $\alg$ under the oracle $\GSQ$, and let $\Emalg{\cdot}$ to be the corresponding expectation. We also define $\bPP$ to be the distribution of $(\cH^T,\hpi)$ by the algorithm $\alg$ under the oracle $\GSQ[\oM]$ that always return $v\sim \phq{\oM}$, and let $\bEE$ be the corresponding expectation.

Following \eqref{def:pM-qM}, we define
\begin{align}\label{eqn:proof-SQ-lower-q}
q=\frac{1}{T} \sum_{t=1}^{T} \bPP(\bpi\ind{t}=\cdot)\in \DbPi,
\end{align}
and
\begin{align}
    \cM_{q,\ueps}(\oM)\defeq \set{ M\in\cM: \PP_{\bpi\sim q, v\sim \oM(\bpi)} \paren{ \nrm{\phq{M}-\phq{\oM}}>\tau }\leq \ueps^2 }.
\end{align}
By definition, for any distribution $p\in\DPi$, there exists $M\in\cM_{q,\ueps}(\oM)$ such that
\begin{align*}
     \EE_{\pi\sim p}[\LM{\pi}]\geq \Delta.
\end{align*}
In particular, $\cM_{q,\ueps}(\oM)$ is non-empty, and we fix a model $M_0\in\cM_{q,\ueps}(\oM)$ and let $\pi_0\defeq \pim[M_0]$. Then there exists $M_1$ such that $\LM[M_1]{\pi_0}\geq \Delta$, i.e., $\rho(\pim[M_1],\pi_0)\geq \Delta$. We denote $\pi_1\defeq \pim[M_1]$

Now, using the chain rule of TV distance, it holds that for any $M\in\cM_{q,\ueps}(\oM)$
\begin{align*}
    \DTV{ \PP\sups{M,\alg}, \bPP }
    \leq&~ \sum_{t=1}^T \bEE\brac{ \DTV{ \Pmalg{v\ind{t}=\cdot|\cH\ind{t-1},\bpi\ind{t}}, \bPP\paren{v\ind{t}=\cdot|\cH\ind{t-1},\bpi\ind{t}} } } \\
    =&~\sum_{t=1}^T \bEE\brac{ \DTV{ \GSQ(\bpi\ind{t}), \GSQ[\oM](\bpi\ind{t}) } } \\
    =&~\sum_{t=1}^T \bEE\brac{ \PP_{v\sim \oM(\bpi\ind{t})}\paren{\nrm{\phq[\bpi\ind{t}]{M}-v}>\tau}} \\
    =&~ T\cdot \EE_{\bpi\sim q} \brac{ \PP_{v\sim \oM(\bpi)}\paren{\nrm{\phq[\bpi]{M}-v}>\tau} }\\
    =&~ T\cdot \PP_{\bpi\sim q, v\sim \oM(\bpi)} \paren{\nrm{\phq[\bpi]{M}-v}>\tau}
    \leq T\ueps^2,
\end{align*}
where the third line follows from the definition of $\GSQ$. Therefore, by triangle inequality, we have
\begin{align*}
    \DTV{\PP^{M_0,\alg}, \PP^{M_1,\alg}}\leq 2T\ueps^2.
\end{align*}
Hence, it holds that
\begin{align*}
    &~\EE^{M_0,\alg}\brac{ \rho(\hpi,\pi_0) }+\EE^{M_1,\alg}\brac{ \rho(\hpi,\pi_1) } \\
    \geq&~ \frac{\Delta}{2}\brac{ \PP^{M_0,\alg}\paren{ \rho(\hpi,\pi_0)>\frac{\Delta}{2} }+ \PP^{M_1,\alg}\paren{ \rho(\hpi,\pi_1)>\frac{\Delta}{2} }} \\
    \geq&~ \frac{\Delta}{2}\brac{ 1- \DTV{\PP^{M_0,\alg}, \PP^{M_1,\alg}}}\geq \frac{\Delta}{4},
\end{align*}
where the second inequality follows from $\rho(\pi_0,\pi_1)\geq \Delta$. Therefore,
\begin{align*}
    \max_{M\in\set{M_0,M_1}}\Emalg{\riskdm(T)}\geq \frac{\Delta}{8}.
\end{align*}
Taking the supremum over all reference models $\oM$ completes the proof for metric-based $L$.

For a general loss $L$, we may choose $\ueps:=\sqrt{\frac{\delta}{T}}$, $\Delta:=\pdecltau_{\ueps}(\cM,\oM)$, and we let $q$ as in \eqref{eqn:proof-SQ-lower-q}, and $p=\bPP(\hpi=\cdot)\in\DPi$. Then, we can pick a model $M\in\cM_{q,\ueps}(\oM)$ such that $\EE_{\pi\sim p}[\LM{\pi}]\geq \Delta.$ Then, using the fact that $\DTV{ \PP\sups{M,\alg}, \bPP }\leq T\ueps^2=\delta$, we can lower bound
\begin{align*}
    \Emalg{\riskdm(T)}\geq \bEE\brac{\riskdm(T)}-\DTV{ \PP\sups{M,\alg}, \bPP }
    \geq \Delta-\delta.
\end{align*}
\qed

\subsection{Proof of \cref{thm:SQ-upper}}

The first upper bound of \cref{thm:SQ-upper} is proven in \cref{appdx:SQ-E2D} (cf. \cref{thm:SQ-upper-E2D}), and the second upper bound follows immediately from combining \cref{thm:cDMSO-pac-upper} and \cref{lem:gen-to-sq-dec}.
\qed

\subsection{Proof of \cref{prop:SQ-dim-to-SQ-dec}}\label{appdx:proof-SQ-dim-to-SQ-dec}

We recall that for any model $M\in\DZ$, it induces a map $\phi\mapsto M(\phi)=\EE_{z\sim M}[\phi(z)]$, and hence we can regard $M:\Phi\to\R$.

By definition,
\begin{align*}
    &~
     \pdecltau_\eps(\cMqb,\oM)>1-\beta \\
    \Leftrightarrow\quad &~
     \forall p\in\DPi, \forall q\in\DL, \exists M\in\cMqb, \text{ such that } 1-p(\Pi_M)>1-\beta, \PP_{\lf\sim q}\paren{ \Dl(M,\oM)>\tau }\leq \eps^2 \\
    \Leftrightarrow\quad &~
     \forall p\in\DPi, \forall q\in\DL, \exists M\in\cMper, \text{ such that } \PP_{\lf\sim q}\paren{ \Dl(M,\oM)>\tau }\leq \eps^2\\
    \Leftrightarrow\quad &~
     \forall p\in\DPi, \sup_{q\in\DL}\inf_{M\in\cMper} \PP_{\lf\sim q}\paren{ \Dl(M,\oM)>\tau }\leq \eps^2.
\end{align*}
Further, using the Minimax theorem, we know
\begin{align*}
    \sup_{q\in\DL}\inf_{M\in\cMper}\PP_{\lf\sim q}\paren{ \Dl(M,\oM)>\tau }
    =\inf_{\mu\in\Delta(\cMper)}\sup_{\lf\in\Lc}\PP_{M\sim \mu}\paren{ \Dl(M,\oM)>\tau }.
\end{align*}
The Minimax theorem can be applied here because as long as $\tau>0$ and $|\cZ|$ is finite, the function class $\set{ \lf\mapsto \indic{ \Dl(M,\oM)>\tau } }_{M\in\cM}$ admits finite eluder dimension\footnote{By regarding $\cL,\cM\subseteq \R^{\cZ}$, we can write $\Dl(M,\oM)=\abs{\lr \lf, M-\oM\rr}$ and apply the standard elliptical potential argument.} and hence finite threshold dimension~\citep{li2022understanding}, and hence the Minimax theorem holds true~\citep{hanneke2021online}.

Therefore, we have
\begin{align*}
    &~
     \pdecltau_\eps(\cMqb,\oM)>1-\beta \\
    \Leftrightarrow\quad &~
     \forall p\in\DPi, \inf_{\mu\in\Delta(\cMper)}\sup_{\lf\in\Lc} \PP_{M\sim \mu}\paren{ \Dl(M,\oM)>\tau }\leq \eps^2 \\
    \Leftrightarrow\quad &~
     \forall p\in\DPi, \sup_{\mu\in\Delta(\cMper)}\inf_{\lf\in\Lc}~\frac{1}{ \PP_{M\sim \mu}\paren{ \Dl(M,\oM)>\tau } }\geq \eps^{-2} \\
    \Leftrightarrow\quad &~
    \inf_{p\in\DPi}\sup_{\mu\in\Delta(\cMper)}\inf_{\lf\in\Lc}~\frac{1}{ \PP_{M\sim \mu}\paren{ \Dl(M,\oM)>\tau } }\geq \eps^{-2} \\
    \Leftrightarrow\quad &~
    \SQD{\cMqb,\oM}\geq \eps^{-2}.
\end{align*}
This is the desired result.
\qed

\subsection{Proof of \cref{lem:SQ-dec-to-LDP-dec}}\label{appdx:proof-SQ-dec-to-LDP-dec}

By definition, in interactive SQ learning, the measurement class is $\Phi=(\cZ\to [0,1])$, and for model class $\cM\subseteq (\Pi\to\DZ)$, we regard $\cM\subseteq (\Pi\times\Phi\to \R)$ by  $M(\pi,\phi)=\EE_{z\sim M(\pi)} \phi(z)$. Thus, the \SQDEC~can be written as
\begin{align*}
    \pdecltau_{\eps}(\cM,\oM)\defeq&~ \infpql \sup_{M\in\cM}\constr{ \EE_{\pi\sim p}[\LM{\pi}] }{ \PP_{(\pi,\lf)\sim q} \paren{ \Dl(M(\pi),\oM(\pi)) >\tau }\leq \eps^2 }.
\end{align*}
For any $q\in\DPL$, using Markov's inequality, we have
\begin{align*}
    \PP_{ (\pi,\lf)\sim q }\paren{ \Dl(M(\pi),\oM(\pi)) > \tau }\leq \frac{1}{\tau^2}\EE_{(\pi,\lf)\sim q}\Dl^2(M(\pi),\oM(\pi)).
\end{align*}
Conversely, we also have
\begin{align*}
    \EE_{(\pi,\lf)\sim q}\Dl^2(M(\pi),\oM(\pi))
    \leq \tau^2+\PP_{ (\pi,\lf)\sim q }\paren{ \Dl(M(\pi),\oM(\pi)) > \tau }.
\end{align*}
Combining the inequalities above completes the proof of \eqref{eqn:SQ-dec-to-LDP-dec}.
\qed

\section{Remaining Proofs from \cref{sec:LDP} and \cref{appdx:LDP-more}}\label{appdx:LDP}

We note that we have presented the proof of the following results in the previous sections: %
\begin{itemize}
\item Proof of \cref{thm:pdec-lin-upper}: \cref{appdx:LDP-E2D}, and see also \cref{appdx:ExO-model-based}.
\item Proof of \cref{thm:p-dec-convex-upper} and \cref{cor:agn-regression}: \cref{appdx:ExO-policy-based}.
\item Proof of \cref{prop:real-regression}: \cref{appdx:ExO-val-based}.
    \item Proof of \cref{thm:rdec-lin-lower}: \cref{appdx:inst-LDP-pac-lower}, where we also provide a proof of \cref{thm:pdec-lin-lower} (1).
    \item Proof of \cref{thm:rdec-lin-upper}: \cref{appdx:ExO-model-based} and \cref{appdx:ExO-policy-based}.
    \item Proof of \cref{thm:CB-adv}: \cref{appdx:ExO-CB}.
\end{itemize}
In the subsequent subsections, we present the remaining proofs from \cref{sec:LDP}. 

\subsection{Proof of \cref{prop:pLDP}}\label{appdx:proof-strong-DP}

Fix an \pDP~channel $\pr$. By definition, the class of distributions $\set{\pr(\cdot|z)}$ admits a common base measure $\mu$, and hence in the following we slightly abuse notations and write a distribution $P$ and its density $dP/d\mu$ interchangeably. We also denote $\PP'=\pr\circ\PP, \QQ'=\pr\circ\QQ$.

Define $p(o)=\inf_{z\in\cZ}\pr(o=\cdot|z)$ for any $o\in\cO$. Then, by definition, for any $z\in \cZ$, 
\begin{align*}
    p(o)\leq \pr(o|z)\leq \ea p(o).
\end{align*}
Therefore, we define
\begin{align*}
    \lf_o(z)=\frac{1}{\ea-1}\paren{\frac{\pr(o|z)}{p(o)}-1}\in[0,1].
\end{align*}
Then, for each $o\in\cO$, it holds that
\begin{align*}
    \PP'(o)=\EE_{z\sim \PP} \pr(o|z)
    =\EE_{z\sim \PP}\brac{ (\ea-1)p(o)\lf_o(z)+p(o) },
\end{align*}
and hence we know $\PP'(o)\in[p(o),\ea p(o)]$, and similarly $\QQ'(o)\in[p(o),\ea p(o)]$. Further, we also have
\begin{align*}
    \abs{ \PP'(o)-\QQ'(o) }= (\ea-1) p(o)\abs{ \PP[\lf_o]-\QQ[\lf_o] },
\end{align*}
Now, by definition,
\begin{align*}
    \DH{\PP',\QQ'}=\frac{1}{2}\int \paren{\sqrt{\PP'(o)}-\sqrt{\QQ'(o)}}^2do
    =\frac12\int \frac{\abs{ \PP'(o)-\QQ'(o) }^2}{\paren{\sqrt{\PP'(o)}+\sqrt{\QQ'(o)}}^2}do.
\end{align*}
Hence, it holds that
\begin{align*}
    \frac{(\ea-1)^2}{8e^{\alpha} }\int \abs{ \PP[\lf_o]-\QQ[\lf_o] }^2 p(o)do \leq \DH{\PP',\QQ'}\leq \frac{(\ea-1)^2}{8}\int \abs{ \PP[\lf_o]-\QQ[\lf_o] }^2 p(o)do
\end{align*}
Notice that $\int p(o)do\in[e^{-\alpha},1]$, and hence we can normalize $p$ to a distribution $\op$ over $\cO$. The proof of \eqref{eqn:DH-LDP} is hence completed, and \eqref{eqn:KL-chi-LDP} follows similarly:
\begin{align*}
    \chis{\PP'}{\QQ'}
    =&~\int \frac{\abs{ \PP'(o)-\QQ'(o) }^2}{\QQ'(o)}do \\
    \leq&~ (\ea-1)^2\int \abs{ \PP[\lf_o]-\QQ[\lf_o] }^2 p(o)do \\
    \leq&~ (\ea-1)^2\EE_{o\sim \op} \abs{ \PP[\lf_o]-\QQ[\lf_o] }^2.
\end{align*}
\qed

\subsection{Proof of \cref{thm:pdec-lin-lower}}\label{appdx:p-dec-lin-q}
In this section, we provide a self-contained proof of \cref{thm:pdec-lin-lower}, following the approach of \citet{chen2024beyond} (see also \cref{appdx:DEC-general-lower}).
The proof is based on the following quantile-based \pDEC. 

\paragraph{Quantile-based \pDEC}
Given model class $\cM\subseteq (\Pi\to\DZ)$, for each $\eps>0$ and $\delta\in[0,1]$, we define the quantile-based \pDEC~as (slightly abusing the notation)
\begin{align}\label{def:p-dec-q-lin}
    \pdecql_{\eps,\delta}(\cM,\oM)\defeq \infpql\sup_{M\in\cM}\constr{ \hgm{p} }{ \EE_{(\pi,\lf)\sim q}\Dl^2( M(\pi),\oM(\pi) )\leq \eps^2 },
\end{align}
where $\hgm{p}$ is the \emph{$\delta$-quantile loss} of $p$, defined as
\begin{align*}
    \hgm{p}=\sup_{\Delta\geq 0}\set{\Delta: \PP_{\pi\sim q}(\LM{\pi}\geq \Delta)\geq \delta }.
\end{align*}
We also denote $\pdecql_{\eps,\delta}(\cM)\defeq \sup_{\oM\in\coM} \pdecql_{\eps,\delta}(\cM,\oM)$.
By definition, the quantile-based \pDEC~is always bounded by the original \pDEC:
\begin{align}\label{eqn:dec-q-to-c-trivial}
    \pdecql_{\eps}(\cM,\oM)-\delta\leq \pdecql_{\eps,\delta}(\cM,\oM)\leq \delta^{-1}\pdecql_{\eps}(\cM,\oM).
\end{align}
However, such a conversion can be loose in general.

\paragraph{Quantile lower bound}
Similar to \cref{appdx:DEC-general-lower}, we show that the quantile-based \pDEC~provides a lower bound regardless of the structure of the loss function.
\begin{proposition}[Quantile-based \pDEC~lower bound]\label{prop:p-dec-q-lin-lower}
For any $T\geq 1$ and constant $\delta\in[0,1)$, we denote $\ueps(T)\defeq \frac{1}{(\ea-1)}\sqrt{\frac{\delta}{2T}}$. Then, for any $T$-round \pLDP~algorithm $\alg$, there exists $\Ms\in\cM$ such that under $\PP\sups{\Ms,\alg}$,
\begin{align*}
    \riskdm(T)\geq \pdecql_{\ueps(T),\delta}(\cM),\qquad\text{with probability at least $\delta/2$.}
\end{align*}
\end{proposition}

Further, for reward-based loss function $L$, we can relate quantile-based \pDEC~to the original \pDEC~(following \citet[Proposition E.1]{chen2024beyond}).

\begin{lemma}\label{prop:p-dec-q-lin-to-c}
Suppose that the loss function $L$ is reward-based. Then, for any parameter $\eps>0, \delta\in[0,1)$, it holds that
\begin{align*}
    \pdecql_{\sqrt{2}\eps,\delta}(\cM)\geq \pdecl_{\eps}(\cM)-\frac{2\sqrt{2}\eps}{1-\delta}.
\end{align*}
\end{lemma}

Similarly, for metric-based loss function, we have the following lemma (following \cref{lem:p-dec-q-gen-metric}).
\begin{lemma}\label{lem:p-dec-lin-q-metric}
Suppose that the loss function $L$ is metric-based. Then, for any parameter $\eps>0, \delta\in[0,\frac12)$, it holds that
\begin{align*}
    \pdecql_{\eps,\delta}(\cM)\geq \frac12\pdecl_{\eps}(\cM).
\end{align*}
\end{lemma}

Therefore, the proof of \cref{thm:pdec-lin-lower} is completed by combining \cref{prop:p-dec-q-lin-lower} with \cref{prop:p-dec-q-lin-to-c} / \cref{lem:p-dec-lin-q-metric}.
\qed

\subsubsection{Proof of \cref{prop:p-dec-q-lin-lower}}\label{appdx:proof-p-dec-q-lin-lower}

We follow the strategy of \cref{appdx:proof-cDMSO-pac-lower-quantile}. 

Recall that an \pLDP~algorithm $\alg = \set{q\ind{t}}_{t\in [T]}\cup\set{p}$ is specified by a sequence of mappings, where the $t$-th mapping $q\ind{t}(\cdot\mid{}\Hy\ind{t-1})$ specifies the distribution of $(\pi\ind{t},\pr\ind{t})$ based on the history $\Hy\ind{t-1}$, and the final map $p(\cdot \mid{} \Hy\ind{T})$ specifies the distribution of the $\pihat$ based on $\Hy\ind{T}$. Therefore, for any model $M$, we define
\begin{align}\label{def:pM-qM-LDP}
q_{M,\alg}=\Emalg{ \frac{1}{T} \sum_{t=1}^{T} q\ind{t}(\cdot|\cH\ind{t-1}) }\in \DPP, \quad 
p_{M,\alg}=\Emalg{p(\cH\ind{T})} \in \DPi,
\end{align}
The distribution $q_{M,\alg}$ is the expected distribution of the average profile $(\pi\ind{1},\pr\ind{1},\cdots,\pi\ind{T},\pr\ind{T})$, and $p_{M,\alg}$ is the expected distribution of the output policy $\hpi$. 

Using the chain rule of KL divergence, for any model $M, \oM$,
\begin{align*}
\KLd{\PP\sups{\oM,\alg}}{\PP\sups{M,\alg}}  
=&~ \EE\sups{\oM,\alg}\brac{ \sum_{t=1}^T \KLd{ \pr\ind{t}\circ M(\pi\ind{t}) }{ \pr\ind{t}\circ M(\pi\ind{t}) } } \\
=&~ T\cdot \EE_{(\pi,\pr)\sim q_{\oM,\alg}} \KLd{ \pr\circ \oM(\pi) }{ \pr\circ M(\pi) }.
\end{align*}
Further, by \cref{prop:pLDP}, for any \pLDP~channel $\pr$, there exists a distribution $\tq_\pr\in\DL$ such that
\begin{align*}
    \KLd{ \pr\circ \PP_1 }{ \pr\circ\PP_2 }\leq (\ea-1)^2\EE_{\lf\sim \tq_\pr} \Dl^2(\PP_1, \PP_2).
\end{align*}

Therefore, for any model $M\in\cM$, we define $\tq_{M,\alg}\in\DPL$ to be the distribution of $(\pi,\lf)$, where $(\pi,\pr)\sim q_{M,\alg}$, and $\lf\sim \tq_{\pr}$. Then, our argument above gives
\begin{align}\label{eqn:KL-chain}
\KLd{\PP\sups{\oM,\alg}}{\PP\sups{M,\alg}}  
\leq (\ea-1)^2 T\cdot \EE_{(\pi,\lf)\sim \tq_{\oM,\alg}} \Dl^2( M(\pi) , \oM(\pi) ).
\end{align}
With this chain rule, we now present the proof of \cref{prop:p-dec-q-lin-lower}~(which is essentially following the analysis in \citet{chen2024beyond}).

\paragraph{Proof of \cref{prop:p-dec-q-lin-lower}}
We abbreviate $\eps=\ueps(T)$. Fix a $\Delta<\pdecql_{\eps,\delta}(\cM)$, and then there exists $\oM$ such that $\Delta<\pdecq_{\eps,\delta}(\cM,\oM)$. Hence, by the definition \cref{def:p-dec-q-lin}, we know that
\begin{align*}
    \Delta<\sup_{M\in\cM}\constr{ \hgm{p_{\oM,\alg}} }{ \EE_{(\pi,\lf)\sim \tq_{\oM,\alg}}\Dl^2( M(\pi) , \oM(\pi) )\leq \eps^2 }.
\end{align*}
Therefore, there exists $M\in\cM$ such that
\begin{align*}
    \EE_{(\pi,\lf)\sim \tq_{\oM,\alg}}\Dl^2( M(\pi) , \oM(\pi) )\leq \eps^2, \qquad
    \PP_{\pi\sim p_{\oM,\alg}}(\LM{\pi}>\Delta)\geq \delta.
\end{align*}
By \cref{eqn:KL-chain}, we know
\begin{align*}
    \KLd{\PP\sups{\oM,\alg}}{\PP\sups{M,\alg}} \leq (\ea-1)^2T\eps^2.
\end{align*}
By data-processing inequality, we have
\begin{align*}
    \DTV{ p_{\oM,\alg}, p_{M,\alg} } \leq \DTV{ \PP\sups{\oM,\alg}, \PP\sups{M,\alg} } \leq \sqrt{\frac{1}{2}(\ea-1)^2T\eps^2}.
\end{align*}
Therefore, combining the inequalities above, we have
\begin{align*}
    p_{M,\alg}(\pi: \LM{\pi}> \Delta )\geq  p_{\oM,\alg}(\pi: \LM{\pi}> \Delta )-\sqrt{\frac{1}{2}(\ea-1)^2T\eps^2}\geq \frac{\delta}{2}.
\end{align*}
By the definition of $p_{M,\alg}$, this gives $\PP\sups{M,\alg}\paren{ \LM{\hpi}> \Delta }\geq \frac\delta2$.
Letting $\Delta\to \pdecq_{\eps,\delta}(\cM)$ completes the proof.
\qed

\subsection{Proof of \cref{lem:linear-l2}}\label{appdx:proof-linear-l2-lower}

For each $\theta\in[-1,1]$, we denote $M_\theta\in\DZ$ to be the model given by
\begin{align*}
    (x,y)\sim M_\theta: \quad x\sim \nu, y\sim \Rad{\theta x}.
\end{align*}
Then it holds that
\begin{align*}
    \DTV{ M_\theta, M_0 }\leq \EE_{x\sim \nu}\abs{\theta x}=\abs{\theta} \EE|x|,
\end{align*}
and
\begin{align*}
    \LM[M_\theta]{\pi}=\EE_{x\sim p} \abs{ x(\theta-\pi) }^2=\abs{\theta-\pi}^2 \cdot \EE|x|^2.
\end{align*}
Therefore, for the model class $\cM=\set{M_\theta: \theta\in[-1,1]}$, we can consider $\theta=\min\sset{\frac{\eps}{\EE|x|},1}$, which gives
\begin{align*}
    \pdecl_\eps(\cM,M_0)\geq \pdecl_\eps(\set{M_\theta,M_0},M_0)\geq \frac{\EE|x|^2}{4} \min\sset{ \frac{\eps^2}{(\EE|x|)^2}, 1 }.
\end{align*}
Applying \cref{thm:pdec-lin-lower} gives the desired lower bound.
\qed

\subsection{Proof of \cref{thm:linear-upper}}\label{appdx:proof-linear-upper}

Following \cref{appdx:ExO-val-based} (\cref{lem:realizable-dec-o}), we consider the following \infosets~$\Psi=\Theta$:
\begin{align*}
    \cM_\theta=\sset{ M\in\cMlin: \thM=\theta }, \qquad
    \pi_\theta=\theta, \qquad
    \theta\in\Psi.
\end{align*}
Then, $\Psi$ is a \infosets~with respect to the model class $\cMlin$ and value function $\Vm(\pi)=-\Lonel{M}{\theta}$. It is clear that $\cMPs=\cMlin$, and hence we have the following guarantee of \LDPexo~(by \cref{thm:ExO-full-offset}). 
\begin{proposition}\label{prop:ExO-linear-upper}
Let $T\geq 1, \gamma>0$. Then, for linear regression under \Lone~loss, \LDPexo~(instantiated on $\Psi$ defined above) achieves \whp
\begin{align*}
    \EE_{\hth\sim \phat} \Lonel{\Mstar}{\hth}\leq \pdecol_{c\alpha^2\gamma}(\cMlin)+\frac{2\gamma\log (|\Theta|/\delta)}{T}.
\end{align*}
\end{proposition}
Note that for simplicity, we assume $\Theta$ is finite. By applying the argument on a covering of $\Theta\subset \Bone$, we can regard $\log|\Theta|\leq \tbO{d}$.

In the following, we denote $\cM\defeq \cMlin$, and it remains to upper bound $\pdecol_\gamma(\cM)$.
For simplicity of presentation, we assume that $\cX\subseteq \Bone\backslash\set{0}$ (without loss of generality).

\newcommand{\fom}{f\sups{\oM}}
\newcommand{\gu}{\mathsf{n}}
\newcommand{\lamz}{\lambda_0}
\newcommand{\epsm}[1]{\eps_{M,#1}}

Fix a reference model $\oM\in\DZ$.
Let $\onu\in\DX$ be the marginal distribution of $x$ under $(x,y)\sim \oM$, and let $\fom(x)=\EE\sups{\oM}[y|x]$ for $x\in\cX$. Note that $\fom$ is not necessarily a linear function. Further, for any $M\in\cM$, we let $\thM\in\Bone$ be the associated parameter so that $\EE\sups{M}[y|x]=\lr \thM, x\rr$.
In the following, we proceed to upper bound the offset \pDEC~\cref{def:p-dec-o} of $\cM$ with respect to $\oM$, which is defined as
\begin{align*}
    \pdecol_{\gamma}(\cM,\oM)\defeq \infpqll \sup_{M\in\cM}\sset{ \EE_{\pi\sim p}[\Lonel{M}{\pi}] -\gamma \EE_{\lf\sim q} \Dl^2( M, \oM  ) }.
\end{align*}

\paragraph{Construction of $(p,q)$}
The key observation is the following lemma.
\begin{lemma}\label{lem:U}
Suppose that $\lambda_0>0$ and $\onu\in\DX$ are given. Then there exists a PSD matrix $U\in \Rdd$ satisfies the following equation:
\begin{align}\label{def:U}
    \EE_{x\sim \onu}  \frac{Uxx^\top U}{\nrm{Ux}} +\lambda_0 U=\id_d.
\end{align}
In particular, by taking trace, it holds that $\EE_{x\sim \onu}\nrm{Ux}\leq d$.
\end{lemma}
We fix a $\lamz>0$ and invoke \cref{lem:U} to obtain a PSD matrix $U$ satisfies \cref{def:U}. Based on the matrix $U$, we define the normalization map  $\gu:\R^d\to\Bone$ as $\gu(v)=\frac{Uv}{\nrm{Uv}}$ for any vector $v\neq 0$. %
Then, \eqref{def:U} ensures that
\begin{align}\label{eqn:normalized-U}
    \EE_{x\sim \onu} \gu(x)x^\top +\lamz\id=U^{-1}.
\end{align}

To construct a distribution $q\in\DL$, we invoke the following lemma.
\renewcommand{\vf}{\mathbf{v}}
\newcommand{\BB}{\mathbb{B}}
\begin{lemma}\label{lem:dist-l}
Suppose that $\vf:\cZ\to \BB^D(1)$. Then there exists a distribution $Q(\vf)$ over $\cL=(\cZ\to[0,1])$, such that for any $\PP,\QQ\in\DZ$, it holds
\begin{align*}
    \EE_{\lf\sim Q(\vf)} \Dl^2(\PP,\QQ)\geq \nrm{ \PP[\vf]-\QQ[\vf] }^2,
\end{align*}
where we denote $\PP[\vf]\defeq \EE_{z\sim \PP} \vf(z)$.
\end{lemma}

To apply \cref{lem:dist-l}, we define maps %
\begin{align*}
    \vf_1(z)=\gu(x)\cdot y, \qquad \vf_2(z)=\vvec(\gu(x)x^\top), \qquad
    \vf_3(z)=\brac{ \indic{\nrm{Ux}<\gamma}\gamma^{-1}\nrm{Ux}, \indic{\nrm{Ux}\geq \gamma} },
\end{align*}
and then by \cref{lem:dist-l}, there exists a distribution $q\in\DL$ such that for any model $M\in\cM$,
\begin{align*}
    \EE_{\lf\sim q} \Dl^2(M,\oM)\geq \frac13 \sum_{i=1}^3 \nrm{  \EE_{z\sim M} \vf_i(z)-\EE_{z\sim \oM} \vf_i(z) }^2.
\end{align*}

For notational simplicity, in the following, we denote $\eps_{M,i}^2\defeq \nrm{  M[\vf_i]-\oM[\vf_i] }^2$ for each $i\in\set{1,2,3}$ and each $M\in\cM$. Then, by definition, we know
\begin{align*}
    &~\nrm{ \EE_{x\sim \onu} \gu(x) \paren{ \lr \thM, x \rr - \fom(x) } } \\
    \leq &~ \nrm{ \EE_{x\sim \muM} \gu(x) \lr \thM, x \rr -\EE_{x\sim \onu} \gu(x) \lr \thM, x \rr  }
    + \nrm{ \EE_{x\sim \muM} \gu(x) \lr \thM, x \rr - \EE_{x\sim \onu} \gu(x) \fom(x)  } \\
    =&~ \nrm{ \lr \EE_{x\sim \muM} \gu(x) x -\EE_{x\sim \onu} \gu(x) x, \thM \rr  }
    + \nrm{ \EE_{(x,y)\sim M} \gu(x) y - \EE_{(x,y)\sim \oM} \gu(x) y  } \\
    \leq&~ \nrmF{ \EE_{x\sim \muM} \gu(x) x -\EE_{x\sim \onu} \gu(x) x  }
    + \nrm{ \EE_{(x,y)\sim M} \gu(x) y - \EE_{(x,y)\sim \oM} \gu(x) y  } \\
    =&~ \epsm{2}+\epsm{1}.
\end{align*}
Therefore, we define
\begin{align*}
    \otheta=\argmin_{\theta\in\Bone} \nrm{ \EE_{x\sim \onu} \gu(x) \paren{ \lr \theta, x \rr - \fom(x) } }^2.
\end{align*}
Notice that the objective function above is a quadratic function of $\theta$, we know that for any $M\in\cM$,
\begin{align*}
    \nrm{ \EE_{x\sim \onu} \gu(x)  \lr \thM-\otheta, x \rr  } 
    \leq \nrm{ \EE_{x\sim \onu} \gu(x) \paren{ \lr \thM, x \rr - \fom(x) } }
    \leq \epsm{1}+\epsm{2}.
\end{align*}
Using \eqref{eqn:normalized-U}, we then have
\begin{align*}
    \nrm{U^{-1}(\thM-\otheta)}=\nrm{ \EE_{x\sim \onu} \gu(x)  \lr \thM-\otheta, x \rr + \lambda_0(\thM-\otheta)  }\leq \epsm{1}+\epsm{2}+2\lambda_0.
\end{align*}
We let $p$ be supported on $\otheta$.

\paragraph{Bounding the offset DEC risk}
Define $\cX_0\defeq \set{x: \nrm{Ux}<\gamma}$.
For any $M\in\cM$, we bound
\begin{align*}
    \Lonel{M}{\otheta}=&~\EE_{x\sim \muM} \abs{ \lr x, \otheta-\thM \rr } \\
    \leq&~ \EE_{x\sim \muM} \brac{ 2\cdot \indic{x\not\in\cX_0}+ \indic{x\in\cX_0}\abs{ \lr x, \otheta-\thM \rr } } \\
    \leq&~ 2 \EE_{x\sim \muM} \indic{x\not\in\cX_0} + \paren{\EE_{x\sim \muM}\indic{x\in\cX_0} \nrm{Ux}  }\sq \paren{ \EE_{x\sim \muM} \frac{\lr x, \otheta-\thM \rr^2}{\nrm{Ux}} }\sq
\end{align*}
First, notice that $\EE_{x\sim \onu} \indic{x\not\in\cX_0}\leq \frac1\gamma \EE_{x\sim \onu}\nrm{Ux}$, and hence
\begin{align*}
    &~ \EE_{x\sim \muM} \indic{x\not\in\cX_0} \\
    \leq&~ \abs{ \EE_{x\sim \muM} \indic{x\not\in\cX_0} - \EE_{x\sim \onu} \indic{x\not\in\cX_0}} +\EE_{x\sim \onu} \indic{x\not\in\cX_0} \leq \epsm{3}+\frac{d}{\gamma}.
\end{align*}
Similarly,
\begin{align*}
    &~ \EE_{x\sim \muM}\indic{x\in\cX_0} \nrm{Ux} \\
    \leq&~ \abs{ \EE_{x\sim \muM}\indic{x\in\cX_0} \nrm{Ux} - \EE_{x\sim \onu}\indic{x\in\cX_0} \nrm{Ux} } + \EE_{x\sim \onu} \indic{x\in\cX_0} \nrm{Ux}
    \leq \gamma\epsm{3}+d.
\end{align*}
Next, we denote $H_M=\EE_{x\sim \muM} \frac{xx^\top}{\nrm{Ux}}$, and we bound
\newcommand{\absb}[1]{\left|#1\right|}
\begin{align*}
&~ \EE_{x\sim \muM} \frac{\lr x, \otheta-\thM \rr^2}{\nrm{Ux}} 
= \nrm{\otheta-\thM}_{H_M}^2 \\
 =&~ \nrm{\otheta-\thM}_{H_M}^2 - \nrm{\otheta-\thM}_{H_{\oM}}^2 + \nrm{\otheta-\thM}_{H_{\oM}}^2 \\
\leq&~ \nrm{ (H_M-H_{\oM})(\otheta-\thM) } + \nrm{U^{-1}(\otheta-\thM)}^2,
\end{align*}
where the second inequality uses $H_{\oM}=\EE_{x\sim \onu} \frac{xx^\top}{\nrm{Ux}}\preceq U^{-2}$ (by \eqref{def:U}). Notice that
\begin{align*}
    \nrmF{U(H_M-H_{\oM})}=
    \nrmF{\EE_{x\sim \muM} \gu(x) x -\EE_{x\sim \onu} \gu(x) x}
    \leq \epsm{2},
\end{align*}
and hence
\begin{align*}
    \nrm{ (H_M-H_{\oM})(\otheta-\thM) }
    \leq \nrmF{U(H_M-H_{\oM})}\nrm{U^{-1}(\otheta-\thM)}.
\end{align*}
Combining the inequalities above, we can conclude that
\begin{align*}
    \Lonel{M}{\otheta}\leqsim \epsm{3}+\frac{d}{\gamma}+\sqrt{(d+\gamma\epsm{3})}(\epsm{1}+\epsm{2}+\lambda_0).
\end{align*}
Therefore, by applying the weighted AM-GM inequality, we have
\begin{align*}
    \Lonel{M}{\otheta}-\frac{\gamma}{3}(\epsm{1}^2+\epsm{2}^2+\epsm{3}^2)\leqsim \frac{d}{\gamma}+\sqrt{d+\gamma}\lambda_0.
\end{align*}
and hence $\pdecol_\gamma(\cM,\oM)\leq C\paren{\frac{d}{\gamma}+\sqrt{d+\gamma}\lambda_0}$ for some absolute constant $C$. Taking $\lambda_0\to 0$ gives $\pdecol_\gamma(\cM,\oM)\leq \frac{Cd}{\gamma}$.

\paragraph{Finalizing the proof}
We have shown that $\pdecol_\gamma(\cM)\leq C\frac{d}{\gamma}$. In particular, this implies $\pdecl_\eps(\cM)\leq \sqrt{Cd}\eps$ for any $\eps\in[0,1]$. 

Further, by \cref{prop:ExO-linear-upper}, \LDPexo~achieves \whp~that
\begin{align*}
    \EE_{\pi\sim \phat} \Lonel{\Mstar}{\pi}\leq \bigO{1}\cdot\brac{ \frac{d}{\alpha^2\gamma}+\frac{2\gamma\log (|\Theta|/\delta)}{T} }.
\end{align*}
Note that $\Lonel{\Mstar}{\pi}$ is a convex function with respect to $\pi\in\Bone$, and hence we can let \LDPexo~output $\hth=\EE_{\pi\sim \phat}[\pi]\in\Bone$. Then, by choosing $\gamma>0$ suitably, it is guaranteed that \whp
\begin{align*}
    \Lonel{\Mstar}{\hth}\leq \tbO{\sqrt{\frac{d^2\log(1/\delta)}{\alpha^2 T}}}.
\end{align*}
The proof of \cref{thm:linear-upper} is hence completed.
\qed

\subsubsection{Proof of Lemma~\ref{lem:U}}

Consider the compact, convex region $\cU\subset \R^{d\times d}$ given by
\begin{align}\label{eqn:region-U}
    \cU=\set{U: (1+\lambda_0)^{-1/2}\id_d \preceq U \preceq \lambda_0^{-1/2}\id_d }.
\end{align}
Define function $F:\cU\to\R^{d\times d}$ as follows:
\begin{align*}
    F(U)\defeq \paren{ \EE_{x\sim \onu} \frac{U\sq xx^\top U\sq }{\nrm{Ux}} +\lambda_0\id_d }^{-1}.
\end{align*}
Note that by definition, for any $U\in\cU$, $\lamz\id_d \preceq F(U)^{-1}\preceq (\lamz+1)\id_d$. 
Therefore, $F$ maps $\cU$ to itself.
Further, the map $(U,x)\mapsto \frac{U\sq xx^\top U\sq }{\nrm{Ux}}$ is uniformly continuous with respect to $U\in\cU$ and $x\neq 0$. Therefore, $F(U)$ is continuous in $U$, and Brouwer fixed-point theorem implies that there exists $U\in\cU$ such that $F(U)=U$, i.e., $U$ satisfies \eqref{def:U}.
\qed

\subsubsection{Proof of \cref{lem:dist-l}}
\newcommand{\signp}{\mathsf{s}}

Define $\signp(t)=1$ if $t\geq 0$, and $\signp(t)=0$ otherwise.

Fix the map $\vf:\cZ\to\BB^D(1)$.
For each $w\in \R^D$, we define $\lf_w\in\cL$ as
\begin{align*}
    \lf_w(z)=\nrm{\vf(z)}\signp(\lr \vf(z), w \rr), \qquad \forall z\in\cZ.
\end{align*}
Then, we define $Q(\vf)\in\DL$ to be the distribution of $\lf=\lf_w$ with $w\sim \normal{0,\id_D}$. 

Now, for any $x\in \R^D$ with $\nrm{x}=1$ and any fixed $z\in\cZ$, it holds
\begin{align*}
    \EE_{w} \lf_w(z) \lr w, x\rr = \lr \vf(z), x\rr,
\end{align*}
where $\EE_{w}$ is taken over $w\sim \normal{0,\id_D}$ and the equality follows from the rotational invariance of the Gaussian distribution.

Therefore, for the distribution $Q(\vf)$ defined above, we have
\begin{align*}
    \lr \PP[\vf]-\QQ[\vf], x \rr
    =&~ \PP[\lr \vf(z), x \rr] - \QQ[\lr \vf(z), x \rr] \\
    =&~ \EE_w \lr x,w \rr \paren{ \PP[\lf_w]-\QQ[\lf_w] } \\
    \leq&~ \paren{ \EE_w \lr x,w \rr^2 }^{\frac12} \paren{ \EE_w \Dl[\lf_w]^2(\PP,\QQ) }^{\frac12} \\
    =&~ \paren{ \EE_{\lf\sim Q(\vf)} \Dl^2(\PP,\QQ) }^{\frac12}.
\end{align*}
Hence, by the arbitrariness of $x$, we have
\begin{align*}
    \nrm{ \PP[\vf]-\QQ[\vf] }=\sup_{x\in\R^D: \nrm{x}=1} \lr \PP[\vf]-\QQ[\vf], x \rr
    \leq \paren{ \EE_{\lf\sim Q(\vf)} \Dl^2(\PP,\QQ) }^{\frac12}.
\end{align*}
\qed

\subsection{Lower bound for LDP learning linear models}\label{appdx:proof-linear-l1-lower}

Fix $d\geq 1$ and $\Delta\in[0,1]$. 
Let $\Theta=\set{-\Delta,\Delta}^d$ and $\cX=[d]$, and for each $\theta\in\Theta$, we define $f_\theta$ as
\begin{align*}
    f_\theta(i)=\theta_i, 
\end{align*} 
and let $M_\theta\in\DZ$ given by
\begin{align*}
    (x,y)\sim M_\theta: \quad x\sim \Unif(\cX), y|x\sim \Rad{f_\theta(x)}.
\end{align*}
Then, we let $\cF=\set{f_\theta: \theta\in\Theta}\subseteq (\cX\to[-1,1])$, and $\cM_d\defeq \set{M_\theta: \theta\in\Theta}$ is the class of well-specified models (with respect to $\cF$) with covariate distribution $\Unif(\cX)$.

Recall that for such a problem class, the decision space is $\Pi\subseteq (\cX\to [-1,1])$, which can be naturally identified as a subset of $[-1,1]^d$. Then, the loss function is given by
\begin{align*}
    \LM[M_\theta]{\pi}=\EE_{x\sim \Unif(\cX)} \abs{ \pi(x)-f_\theta(x) }=\frac1d\sum_{i=1}^d \abs{\pi(i)-\theta_i}.
\end{align*}
\begin{proposition}
Let $T\geq1$, $\Delta\in(0,1]$. Suppose that $\alg$ is a $T$-round \pLDP~algorithm, such that $\EE\sups{M,\alg}[\riskdm(T)]\leq \frac{\Delta}{4}$ for all $M\in\cM_d$. Then it holds that $T\geq \Om{\frac{d^2}{\alpha^2\Delta^2}}$.
\end{proposition}

\begin{proof}
For each $\theta\in\Theta$, we consider
\begin{align*}
    p_{\theta}\defeq \PP\sups{M_\theta,\alg}(\hpi=\cdot) \in\DPi.
\end{align*}
Recall the chain rule of KL divergence and \cref{prop:pLDP}: for any model $M,\oM$,
\begin{align*}
    \KLd{\PP\sups{M,\alg}}{\PP\sups{\oM,\alg}}=&~ \EE\sups{M,\alg}\brac{ \sum_{t=1}^T \KLd{ \pr^t\circ M }{ \pr^t\circ \oM } } \\
    \leq&~ T(\ea-1)^2\DTV{M,\oM}^2.
\end{align*}
Therefore, by data-processing inequality, we have
\begin{align*}
    \DTV{p_\theta,p_{\theta'}}\leq (\ea-1)\sqrt{T}\DTV{M_\theta,M_{\theta'}}
    =(\ea-1)\sqrt{T}\cdot \frac1d\nrm{\theta-\theta'}_1.
\end{align*}
We further denote $p_{\theta,i}=p_\theta(\pi(i)\geq 0)$. Then it holds that
\begin{align*}
    \EE\sups{M_\theta,\alg}[\riskdm(T)]=&~\EE_{\pi\sim p_\theta}\LM[M_\theta]{\pi}\\
    =&~ \frac1d\sum_{i=1}^d \EE_{\pi\sim p_\theta}\abs{\pi(i)-\theta_i}\\
    \geq&~ \frac{\Delta}d\paren{ \sum_{i: \theta_i=-\Delta} p_{\theta,i} + \sum_{i: \theta_i=+\Delta}(1-p_{\theta,i}) }.
\end{align*}
Thus,
\begin{align*}
    \sum_{\theta\in\Theta} \EE\sups{M_\theta,\alg}[\riskdm(T)]
    \geq&~ \frac{\Delta}d \sum_{i=1}^d \paren{ \sum_{\theta: \theta_i=-\Delta} p_{\theta,i} + \sum_{\theta: \theta_i=+\Delta}(1-p_{\theta,i}) } \\
    =&~ \frac{\Delta}d \sum_{i=1}^d \paren{ 2^{d-1} - \sum_{(\theta,\theta')\in\Theta_i} (p_{\theta,i}-p_{\theta',i}) },
\end{align*}
where $\Theta_i=\set{ (\theta,\theta'): \theta_i=+\Delta, \theta_i'=-\Delta, \text{and }\forall j\neq i, \theta_j=\theta_j' }$. Notice that for any $(\theta,\theta')\in\Theta_i$, we have
\begin{align*}
    \abs{p_{\theta,i}-p_{\theta',i}}\leq \DTV{p_\theta,p_{\theta'}}\leq \frac{2\Delta(\ea-1)\sqrt{T}}{d}.
\end{align*}
Therefore, we have
\begin{align*}
    \sum_{\theta\in\Theta} \EE\sups{M_\theta,\alg}[\riskdm(T)]\geq \Delta\cdot 2^{d-1}\paren{1-\frac{2\Delta(\ea-1)\sqrt{T}}{d}},
\end{align*}
which immediately implies
\begin{align*}
    \max_{M\in\cM} \EE\sups{M_\theta,\alg}[\riskdm(T)]\geq \frac{\Delta}{2}\paren{1-\frac{2\Delta(\ea-1)\sqrt{T}}{d}}.
\end{align*}
Hence, we must have $T\geq \frac{d^2}{4\Delta^2(\ea-1)^2}$.
\end{proof}

Choosing $\Delta=\min\sset{\frac{d}{4(\ea-1)\sqrt{T}}, \frac{1}{\sqrt{d}}}$ in the proof above (which ensures $\Theta\subset \Bone$), we have the following corollary.
\begin{corollary}\label{cor:linear-lower}
There exists a covariate distribution $\mu$ over $\Bone$, 
such that for the model class $\cM$ consisting of the linear models with covariate distribution $\mu$, 
any $T$-round \pLDP~algorithm with output $\hth$, it holds that
\begin{align*}
    \sup_{\Mstar\in\cM} \EE\sups{\Mstar,\alg} \Lonel{\Mstar}{\hth} \geqsim \min\sset{\frac{d}{\alpha\sqrt{T}}, \frac{1}{\sqrt{d}}}
\end{align*}
\end{corollary}
This lower bounds implies that the upper bound of \cref{thm:linear-upper} is nearly minimax-optimal (up to logarithmic factors).

\subsection{Proof of \cref{thm:LCB-dec-upper}}\label{appdx:proof-linear-CB-upper}

\newcommand{\locm}[1]{\cM_{#1}(\oM)}
\newcommand{\locmp}[1]{\locm{#1}\cup \set{\oM}}
\newcommand{\pix}{\phi(x,\pi)}

We claim that the \rDEC~of $\cM$ can be bounded as
\begin{align}\label{eq:proof-lin-cb-goal}
    \rdecl_\eps(\cM)=\sup_{\oM\in\coM} \rdecl_\eps(\cM\cup\set{\oM},\oM)\leq (20d+6)\eps, \qquad \forall \eps\in[0,1].
\end{align}
With \eqref{eq:proof-lin-cb-goal}, we may directly apply \cref{thm:CB-adv}, as $\log N_{\infty}(\cFlin,\Delta)\leq \bigO{d\log(1/\Delta)}$.

In the following, it remains to prove \eqref{eq:proof-lin-cb-goal}.
We only need to upper bound $\rdecl_\eps(\cM\cup\set{\oM},\oM)$ for any fixed reference model $\oM\in\coM$. Following the proof of \cref{thm:linear-upper} (\cref{appdx:proof-linear-upper}), we assume that $\phi(x,a)\neq 0$ for all $x\in\cX$, $a\in\cA$ without loss of generality.

Fix a reference model $\oM=\EE_{M\sim \mu}[M]\in\coM$ and $\eps\in[0,1]$. Note that $\oM$ is also associated with a mean reward function $\fom$ (not necessarily in $\cFlin$) and a context distribution $\onu\in\Delta(\cX)$. For any $M\in\cM$, we let $\thM\in\Bone$ be the associated parameter so that the mean reward function $\fm\in\cF$ is given by $\fm(x,a)=\lr \thM, \phi(x,a)\rr$.

In the following, we proceed to upper bound the \rDEC~:
\begin{align*}
    &~\rdecl_{\eps}(\cM\cup\set{\oM},\oM) \\
    \defeq&~ \inf_{\substack{p\in\DPL}}\sup_{M\in\cM\cup\set{\oM}}\constr{ \EE_{\pi\sim p}[\Vmm-\Vm(\pi)] }{ \EE_{(\pi,\lf)\sim p} \Dl^2( M(\pi), \oM(\pi) )\leq \eps^2 }.
\end{align*}
For notational simplicity, for any $p\in\DPL$, we denote
\begin{align*}
    \cM_{p,\eps^2}(\oM)\defeq \sset{M\in\cM: \EE_{(\pi,\lf)\sim p} \Dl^2( M(\pi), \oM(\pi) )\leq \eps^2}.
\end{align*}

Let $0<\lambda_0<\min\sset{\frac{\eps^2}{100},\frac{\eps}{\log|\cA|}}$ be a sufficiently small, fixed parameter. Let $\lambda=5\eps$. 

\newcommand{\Up}{U_\pi}

\newcommand{\gpi}{\mathsf{n}_\pi}
\newcommand{\Thom}{\Theta\subs{\oM}}
\newcommand{\thp}{\widehat\theta_P}

In the following, for notational simplicity, we denote $\pix\defeq \phi(x,\pi(x))\in\R^d$. 
We first invoke the following corollary of \cref{lem:U}.
\begin{lemma}\label{lem:U-pi}
Suppose that $\lambda_0>0$ and $\onu\in\DX$ are given. Then by \cref{lem:U}, for each $\pi\in\Pi$, there exists a PSD matrix $\Up\in \Rdd$ satisfies the following equation:
\begin{align}%
    \EE_{x\sim \onu}  \frac{\Up\pix\pix^\top \Up}{\nrm{\Up\pix}} +\lambda_0 \Up=\id_d.
\end{align}
We further define $\gpi(x)\defeq \brac{1;\frac{\Up\pix}{\nrm{\Up\pix}}}$.
\end{lemma}

\paragraph{Fixed point argument} Our proof strategy is that, for any distribution $P\in\DPi$, we define a ``refinement'' $F(P)\in\DPi$ of $P$. Then, the fixed point of $F$ is a distribution of good properties.

(1) Define the constrained set
\begin{align*}
    \Thom\defeq \set{ \theta\in\Bone: \abs{ \EE_{x\sim \onu}\lr \theta, \phi(x,\piom) \rr - \Vmm[\oM] }\leq 4\eps }. 
\end{align*}
Then, for each $P\in\DPi$, we define
\begin{align}\label{eqn:def-thp}
    \thp\defeq \argmin_{\theta\in\Thom} L_P(\theta)\defeq \EE_{\pi\sim P} \nrm{ \EE_{x\sim \onu} \gpi(x) \paren{ \lr \theta, \pix \rr - \fom(x,\pi(x)) } }^2 +\lambda_0^2\nrm{\theta}^2.
\end{align}
By the strong convexity of $L_P$, $\thp$ is a continuous function of $P\in\DPi$.

\newcommand{\sigp}{\Sigma_P}
\newcommand{\hfp}{\widehat{f}_P}

(2) Define $\sigp\defeq \EE_{\pi\sim P} \Up^{-2}$ and
\begin{align}\label{eqn:def-fP}
    \hfp(x,a)\defeq \lr \thp, \phi(x,a) \rr + 2 \cuto{\lambda \nrm{\phi(x,a)}_{\sigp^{-1}} }, \qquad \forall (x,a)\in\cX\times\cA.
\end{align}

(3) For each $x\in\cX$, we define
\begin{align}\label{eqn:def-QP}
    Q_P(\cdot|x)\defeq \argmax_{q\in\DA}~ \EE_{a\sim q} \hfp(x,a)+\lambda_0 H(q),
\end{align}
where $H(q)=-\sum_{a\in\cA} q(a)\log q(a)$ is the entropy of $q\in\DA$. Notice that the objective function is strongly concave with respect to $q\in\DA$, and hence $Q_P$ is continuous with respect to $P$.

(4) Finally, define $F(P)=\bigotimes_{x\in\cX} Q_P(\cdot|x)\in\DPi$. Formally, we define $F(P)\in\DPi$ as\footnote{Alternatively, we can also define $F(P)\in\DPi$ as the distribution of $\pi$ generated as $\pi(x)\sim Q_P(\cdot|x)$ independently for $x\in\cX$.}
\begin{align*}
    F(P)[\pi]\defeq \prod_{x\in\cX} Q_P(\pi(x)|x).
\end{align*}

By definition, $F:\DPi \to\DPi$ is continuous, and hence by \cref{thm:K-fixed-point}, there exists $P\in\DPi$ such that $F(P)=P$. In the following, we work with such a fixed-point distribution $P$.

We start with the following lemmas.
\begin{lemma}\label{lem:val-optimism}
(1) For any $M\in\cM$ with $\nrm{\thM-\thp}_{\sigp} \leq \lambda$, it holds that
\begin{align*}
    \Vmm-\EE_{\pi\sim P}\Vm(\pi)\leq 4 \EE_{\pi\sim P} \EE_{x\sim \muM} \cuto{ \lambda\nrm{\pix}_{\sigp^{-1}} }+\eps.
\end{align*}

(2) It holds that
\begin{align*}
    \EE_{\pi\sim P} \EE_{x\sim \onu} \nrm{\pix}_{\sigp^{-1}}\leq d.
\end{align*}
\end{lemma}

\paragraph{Construction of $p$} First, we define a distribution $\op\in\DPL$ as follows. Consider the maps
\begin{align*}
    \Bar{\vf}(z)=[r, \phi(x,\piom)],
\end{align*}
and \cref{lem:dist-l} implies that there exists $\oq\in \DL$
such that for any model $M\in\cM$, $\pi\in\Pi$,
\begin{align*}
    \EE_{\lf\sim \oq} \Dl^2(M(\pi),\oM(\pi))\geq \frac12 \nrm{  \EE_{z\sim M(\pi)} \bar{\vf}(z)-\EE_{z\sim \oM(\pi)} \bar{\vf}(z) }^2.
\end{align*}
We then choose $\op\in\DPL$ to be the distribution of $(\piom,\lf)$ under $\lf\sim \oq$.
\begin{lemma}\label{lem:op-empty}
Suppose that $M\in\locm{\op,2\eps^2}$. Then $\thM\in\Thom$ and $\abs{\Vm(\piom)-\Vm[\oM](\piom)}\leq 2\eps$.
\end{lemma}
Next, we define $\ps\in\DPL$ as follows. For each policy $\pi\in\Pi$, we define a map $\vf_\pi:\cZ\to \R^{(d+1)^2+1}$:
\begin{align*}
\vf_\pi(z)\defeq [ ~\gpi(x)\cdot r; ~\vvec(\gpi(x)\pix^\top); ~\cuto{\nrm{\pix}_{\sigp^{-1}}} ],
\end{align*}
and \cref{lem:dist-l} implies that there exists $q_\pi\in \DL$
such that for any model $M\in\cM$, $\pi\in\Pi$,
\begin{align}\label{eqn:def-qpi}
    \EE_{\lf\sim q_\pi} \Dl^2(M(\pi),\oM(\pi))\geq \frac15 \nrm{  \EE_{z\sim M(\pi)} \vf_\pi(z)-\EE_{z\sim \oM(\pi)} \vf_\pi(z) }^2.
\end{align}
We then define $\ps\in\DPL$ to be the distribution of $(\pi,\lf)$ under $\pi\sim P$, $\lf\sim q_\pi$. We summarize the properties of $\ps$ in the following lemma.
\begin{lemma}\label{lem:Dl-to-diff}
Suppose that $M\in\locm{\ps,2\eps^2}$. Then it holds that
\begin{align}
    &\EE_{\pi\sim P}\abs{ \Vm(\pi)-\Vm[\oM](\pi) }\leq 2\eps,\label{eqn:diff-val-pi} \\
    &\EE_{\pi\sim P} \nrm{ \EE_{x\sim \onu} \gpi(x) \paren{ \lr \theta, \pix \rr - \fom(x,\pi(x)) } }^2 \leq 20\eps^2, \label{eqn:diff-U-th}\\
    &\EE_{\pi\sim P}\abs{ \EE_{x\sim \muM} \cuto{\lambda \nrm{\pix}_{\sigp^{-1}}} - \EE_{x\sim \nu_{\oM}} \cuto{\lambda \nrm{\pix}_{\sigp^{-1}}} }\leq 4\eps.\label{eqn:diff-nrm-pix}
\end{align}
In particular, when $\thM\in\Thom$, it holds that $\nrm{\thM-\thp}_{\sigp} \leq 5\eps=\lambda$.
\end{lemma}

\newcommand{\epsp}{\eps_P}
\newcommand{\Mp}{M_P}
Now, we consider three cases. Define $\Mp\in\cM$ be the model with context distribution $\onu$ and parameter $\thp$, and let
\begin{align*}
    \epsp\defeq \EE_{\pi\sim P} \abs{ \Vm[\Mp](\pi)-\Vm[\oM](\pi) }.
\end{align*}

\textbf{Case 1: $\Thom=\emptyset$. } In this case, the set $\locm{\op,2\eps^2}=\emptyset$ by \cref{lem:op-empty}. Therefore, we can set $p=\op$ and bound
\begin{align*}
    \rdecl_\eps(\cM\cup\set{\oM},\oM)\leq&~ \sup_{ M\in \locmp{\op, \eps^2} } \EE_{\pi\sim \op} \brac{\Vmm-\Vm(\pi)}=\EE_{\pi\sim \op} \brac{\Vmm[\oM]-\Vm[\oM](\pi)}=0.
\end{align*}

\textbf{Case 2: $\Thom\neq \emptyset$ and $\epsp\leq 5\eps$. } In this case, we set $p=\frac{\op+\ps}{2}$. We proceed to upper bound $\Vmm-\EE_{\pi\sim p} \Vm(\pi)$ for any $M\in\locmp{p,\eps^2}$.

\textbf{Case 2(a): $M\in\locm{p,\eps^2}$.} Then, we know that $\nrm{\thM-\thp}_{\sigp} \leq 5\eps=\lambda$ by \cref{lem:Dl-to-diff}. Then, invoking \cref{lem:val-optimism} gives
\begin{align*}
    \Vmm-\EE_{\pi\sim P}\Vm(\pi)\leq&~ 4 \EE_{\pi\sim P} \EE_{x\sim \muM} \cuto{ \lambda\nrm{\pix}_{\sigp^{-1}} }+\eps \\
    \leq&~ 4\eps+5d\lambda,
\end{align*}
where the second inequality uses \eqref{eqn:diff-nrm-pix} and \cref{lem:val-optimism} (2). 

Therefore, it remains to upper bound $\Vmm-\Vm(\piom)$. Combining \eqref{eqn:diff-val-pi} and \cref{lem:op-empty} and the fact that $\Vm[\oM](\pi)\leq \Vmm[\oM]$, we have
\begin{align*}
    \EE_{\pi\sim P}\Vm(\pi)-\Vm(\piom)=&~ \EE_{\pi\sim P}[\Vm(\pi)-\Vm[\oM](\pi)+\Vm[\oM](\pi)-\Vmm[\oM]]+\Vmm[\oM]-\Vm(\piom) \leq 5\eps.
\end{align*}
To conclude, we have
\begin{align*}
    \Vmm-\EE_{\pi\sim p}\Vm(\pi)=&~\Vmm-\EE_{\pi\sim P}\Vm(\pi)+\frac{1}{2}\brac{ \EE_{\pi\sim P}\Vm(\pi)-\Vm(\piom) } \\
    \leq&~ (20d+6)\eps.
\end{align*}
\textbf{Case 2(b): $M=\oM$.} 
Consider the model $\Mp\in\cM$. Then, \cref{lem:val-optimism} implies that
\begin{align*}
    \Vmm[\Mp]-\EE_{\pi\sim P}\Vm[\Mp](\pi)\leq 4 \EE_{\pi\sim P} \EE_{x\sim \onu} \lambda\nrm{\pix}_{\sigp^{-1}}+\eps
    \leq 4d\lambda+\eps.
\end{align*}
Further, because $\thp\in\Thom$, we also have $\abs{\Vm[\Mp](\piom)-\Vmm[\oM]}\leq 2\eps$, and hence $\Vmm[\oM]\leq 2\eps+\Vm[\Mp](\piom)$. 

Therefore, combining the inequalities above, we have
\begin{align*}
    \Vmm[\oM]-\EE_{\pi\sim P}\Vm[\oM](\pi)
    \leq&~ \Vmm[\oM]-\Vmm[\Mp]+\Vmm[\Mp]-\EE_{\pi\sim P}\Vm[\Mp](\pi)+\EE_{\pi\sim P}[\Vm[\oM](\pi)-\Vm[\Mp](\pi)] \\
    \leq&~ 3\eps+4d\lambda+\epsp.
\end{align*}
Hence, using $\epsp\leq 5\eps$,
\begin{align*}
    \EE_{\pi\sim p} \brac{\Vmm[\oM]-\Vm[\oM](\pi)}
    =\frac12\paren{ \Vmm[\oM]-\EE_{\pi\sim P}\Vm[\oM](\pi) }\leq (10d+5)\eps.
\end{align*}

Combining the case (a) and (b), we conclude that
\begin{align*}
    \rdecl_\eps(\cM\cup\set{\oM},\oM)\leq (20d+6)\eps.
\end{align*}

\textbf{Case 3: $\Thom\neq \emptyset$ and $\epsp>5\eps$.} In this case, we set $b=\frac{5\eps}{2\epsp}<\frac12$, and $p=(1-b)\op+b\ps$. 

We first show that $\locm{p,\eps^2}=\emptyset$. Otherwise, there exists $M\in\locm{p,\eps^2}$, and hence $M\in\locm{\ps,\eps^2/b}$ and $M\in\locm{\op,2\eps^2}$. The latter implies $\thM\in\Thom$ (by \cref{lem:op-empty}), and the former implies (by \cref{lem:Dl-to-diff})
\begin{align*}
    \EE_{\pi\sim P} \nrm{ \EE_{x\sim \onu} \gpi(x) \paren{ \lr \thM, \pix \rr - \fom(x,\pi(x)) } }^2 \leq \frac{10\eps^2}{b},
\end{align*}
and hence by the definition \eqref{eqn:def-thp},
\begin{align*}
    \EE_{\pi\sim P} \nrm{ \EE_{x\sim \onu} \gpi(x) \paren{ \lr \thp, \pix \rr - \fom(x,\pi(x)) } }^2 \leq \frac{10\eps^2}{b}+\lambda_0.
\end{align*}
Notice that the first coordinate of $\gpi(x)$ is always 1, and hence
\begin{align*}
    \EE_{\pi\sim P} \abs{ \Vm[\Mp](\pi)-\Vm[\oM](\pi) }^2 \leq \frac{10\eps^2}{b}+\lambda_0^2, \quad\Rightarrow\quad
    \epsp^2\leq \frac{10\eps^2}{b}+\lambda_0^2.
\end{align*}
By our choice of $b$, this is a contradiction.

Therefore, it remains to bound $\EE_{\pi\sim \op} \brac{\Vmm[\oM]-\Vm[\oM](\pi)}$. Notice that
\begin{align*}
    \EE_{\pi\sim p} \brac{\Vmm[\oM]-\Vm[\oM](\pi)}
    =b\cdot\paren{ \Vmm[\oM]-\EE_{\pi\sim P}\Vm[\oM](\pi) },
\end{align*}
and our argument in Case 2(b) also applies here:
\begin{align*}
    \Vmm[\oM]-\EE_{\pi\sim P}\Vm[\oM](\pi)
    \leq&~ 3\eps+4d\lambda+\epsp.
\end{align*}
Therefore, we also have
\begin{align*}
    \EE_{\pi\sim p} \brac{\Vmm[\oM]-\Vm[\oM](\pi)}\leq (10d+4)\eps.
\end{align*}

The proof is completed by combining the three cases above.
\qed

\subsubsection{Proof of \cref{lem:val-optimism}}
We denote $P_x\defeq Q_P(\cdot|x)\in\DA$ for each $x\in\cX$. Then, using the definition of $P=F(P)$, we have
\begin{align*}
    \EE_{\pi\sim P} \Vm(\pi)=\EE_{x\sim \muM} \EE_{a\sim P_x} \fm(x,a).
\end{align*}
Next, for a fixed $x\in\cX$, by the definition of $Q_P$, it holds
\begin{align*}
    \forall a'\in\cA, \qquad \hfp(x,a')\leq \EE_{a\sim P_x} \hfp(x,a)+\lamz \log|\cA|.
\end{align*}
Notice that $\nrm{\thM-\thp}_{\sigp} \leq \lambda$, and hence
\begin{align*}
    \abs{ \lr \thp, \phi(x,a) \rr - \lr \thM, \phi(x,a) \rr } \leq \lambda \nrm{\phi(x,a)}_{\sigp^{-1}},
\end{align*}
which implies (using $\fm, \hfp\in[-1,1]$)
\begin{align*}
    \fm(x,a)\leq \hfp(x,a) \leq \fm(x,a)+4\cuto{ \lambda\nrm{\pix}_{\sigp^{-1}} }.
\end{align*}
Therefore, we can now combine the inequalities above to obtain
\begin{align*}
    \fm(x,\pim(x))=&~\max_{a\in\cA} \fm(x,a)
    \leq \max_{a\in\cA} \hfp(x,a)
    \leq \EE_{a\sim P_x} \hfp(x,a)+\eps \\
    \leq&~ \EE_{a\sim P_x} \brac{ \fm(x,a) + 4\cuto{ \lambda\nrm{\pix}_{\sigp^{-1}} }} +\eps,
\end{align*}
where we use $\lamz \log|\cA|\leq \eps$. Taking expectation over $x\sim \muM$ completes the proof of (1).

Now we proceed to prove (2). For any fixed $\pi\in\Pi$, by Cauchy inequality,
\begin{align*}
    \EE_{x\sim \onu} \nrm{\pix}_{\sigp^{-1}}
    =&~ \EE_{x\sim \onu} \sqrt{ \nrm{\Up\pix} \frac{ \nrm{\pix}_{\sigp^{-1}}^2}{\nrm{\Up\pix}} } \\
    \leq&~ \sqrt{ \EE_{x\sim \onu} \nrm{\Up\pix} \cdot \EE_{x\sim \onu} \frac{ \nrm{\pix}_{\sigp^{-1}}^2}{\nrm{\Up\pix}}  }. 
\end{align*}
Notice that by the definition of $\Up$ (\cref{lem:U-pi}), it holds that $\EE_{x\sim \onu} \nrm{\Up\pix}\leq d$, and
\begin{align*}
    \EE_{x\sim \onu}  \frac{\pix\pix^\top}{\nrm{\Up\pix}} \preceq \Up^{-2}.
\end{align*}
Hence,
\begin{align*}
    \EE_{x\sim \onu} \frac{ \nrm{\pix}_{\sigp^{-1}}^2}{\nrm{\Up\pix}}
    =\EE_{x\sim \onu} \llr \sigp^{-1} , \frac{\pix\pix^\top}{\nrm{\Up\pix}} \rrr
    =\llr \sigp^{-1} , \EE_{x\sim \onu} \frac{\pix\pix^\top}{\nrm{\Up\pix}} \rrr
    \leq \llr \sigp^{-1} , \Up^{-2} \rrr,
\end{align*}
where we recall for matrix $A, B\in\R^{d\times d}$, the Frobenius inner product is defined as $\lr A, B\rr=\tr(A^\top B)$. 

Therefore,
\begin{align*}
    \EE_{\pi\sim P}\EE_{x\sim \onu} \nrm{\pix}_{\sigp^{-1}}
    \leq&~ \EE_{\pi\sim P} \sqrt{ d \cdot \llr \sigp^{-1} , \Up^{-2} \rrr } \\
    \leq&~ \sqrt{ d \cdot \EE_{\pi\sim P} \llr \sigp^{-1} , \Up^{-2} \rrr } \\
    =&~ \sqrt{d \llr \sigp^{-1},  \EE_{\pi\sim P} \Up^{-2} \rrr}
    =d,
\end{align*}
where the last line follows from $\EE_{\pi\sim P} \Up^{-2}=\sigp$ and $\tr(\id_d)=d$.
\qed

\subsubsection{Proof of Lemma~\ref{lem:op-empty}}\label{appdx:proof-op-empty}

Suppose $M\in\locm{\op,2\eps^2}$. Then by the definition of $\op$, we have
\begin{align*}
    \abs{\EE_{r\sim M(\piom)} r-\EE_{r\sim \oM(\piom)} r}^2 + \nrm{ \EE_{x\sim \muM} \phi(x,\piom)-\EE_{x\sim \muM} \phi(x,\piom) }^2 \leq 4\eps^2.
\end{align*}
Notice that $V^M(\piom)=\EE_{r\sim M(\piom)} r$, and hence $\abs{\Vm(\piom)-\Vm[\oM](\piom)}\leq 2\eps$ follows immediately.

Further, we also have
\begin{align*}
    \Vm(\piom)=\EE_{r\sim M(\piom)} r=\llr \EE_{x\sim \muM} \phi(x,\piom), \thM \rrr,
\end{align*}
and hence
\begin{align*}
    &~\abs{ \EE_{x\sim \onu}\lr \theta, \phi(x,\piom) \rr - \Vmm[\oM] } \\
    \leq&~ \abs{ \Vm(\piom) - \Vmm[\oM] }+\abs{ \llr \EE_{x\sim \muM} \phi(x,\piom), \thM \rrr - \llr \EE_{x\sim \onu} \phi(x,\piom), \thM \rrr } \\
    \leq&~ \abs{ \Vm(\piom) - \Vmm[\oM] }+\nrm{\EE_{x\sim \muM} \phi(x,\piom)-\EE_{x\sim \muM} \phi(x,\piom)}\leq 4\eps.
\end{align*}
This immediately implies $\thM\in\Thom$.
\qed

\subsubsection{Proof of Lemma~\ref{lem:Dl-to-diff}}

Fix any $M\in\locm{\ps,2\eps^2}$. Then by the definition of $\ps$, we know
\begin{align*}
    \EE_{\lf\sim q_\pi} \Dl^2(M(\pi),\oM(\pi)) \leq 2\eps^2.
\end{align*}
The inequality \eqref{eqn:diff-val-pi} and \eqref{eqn:diff-nrm-pix} follows immediately from \eqref{eqn:def-qpi} (notice that the first coordinate of $\gpi(x)$ is 1). 

The inequality \eqref{eqn:diff-U-th} follows similarly from the proof of \cref{lem:op-empty} (see \cref{appdx:proof-op-empty}):
\begin{align*}
    &~ \nrm{ \EE_{x\sim \onu} \gpi(x) \paren{ \lr \thM, \pix \rr - \fom(x,\pi(x)) } } \\
    \leq&~ \nrm{ \EE_{x\sim \onu} \gpi(x) \lr \thM, \pix \rr - \EE_{x\sim \muM} \gpi(x) \lr \thM, \pix \rr  }
    + \nrm{ \EE_{x\sim \muM} \gpi(x) \lr \thM, \pix \rr - \EE_{x\sim \onu} \fom(x,\pi(x)) } \\
    \leq&~ \nrm{ \EE_{x\sim \onu} \gpi(x) \pix^\top  - \EE_{x\sim \muM} \gpi(x) \pix^\top  }_F
    + \nrm{ \EE_{z\sim M(\pi)} r - \EE_{z\sim \oM(\pi)} r },
\end{align*}
where the second inequality follows from $\nrm{\thM}\leq 1$ and the fact that $\EE^{M(\pi)}[r|x]=\lr \thM, \pix\rr$ and $\EE^{\oM(\pi)}[r|x]=\fom(x,\pi(x))$. Therefore, 
\begin{align*}
    \nrm{ \EE_{x\sim \onu} \gpi(x) \paren{ \lr \thM, \pix \rr - \fom(x,\pi(x)) } }^2
    \leq 2\nrm{  \EE_{z\sim M(\pi)} \vf_\pi(z)-\EE_{z\sim \oM(\pi)} \vf_\pi(z) }^2,
\end{align*}
and \eqref{eqn:diff-U-th} follows immediately.

Finally, we bound $\nrm{\thM-\thp}_{\sigp}$ assuming $\thM\in\Thom$. Using \eqref{eqn:diff-U-th}, we know $L_P(\thM)\leq 20\eps^2+\lambda_0^2$, where the quadratic loss function $L_P$ is defined in \eqref{eqn:def-thp}. Therefore, using $\thp=\argmin_{\theta\in\Thom} L_P(\theta)$, we have
\begin{align*}
    \frac12\nrm{\thM-\thp}_{\nabla^2 L_P }^2 \leq L_P(\thM)-L_P(\thp)\leq 20\eps^2+\lambda_0^2,
\end{align*}
where using the definition of $L_P$, we also have
\begin{align*}
    \frac12\nrm{\theta}_{\nabla^2 L_P }^2
    =\EE_{\pi\sim P} \nrm{ \EE_{x\sim \onu} \gpi(x) \pix^\top \theta }^2.
\end{align*}
Notice that, for $\gpi$ defined as in \cref{lem:U-pi}, we have
\begin{align*}
    \EE_{x\sim \onu} \gpi(x) \pix^\top = \EE_{x\sim \onu} \EE_{x\sim \onu}  \frac{\Up\pix\pix^\top }{\nrm{\Up\pix}}
    =\Up^{-1}-\lamz \id_d.
\end{align*}
Therefore, we know
\begin{align*}
    \frac12\nrm{\theta}_{\nabla^2 L_P }^2
    =\EE_{\pi\sim P} \nrm{ \EE_{x\sim \onu} \gpi(x) \pix^\top \theta }^2
    \geq \EE_{\pi\sim P}\brac{ \frac{5}{6}\nrm{\Up^{-1}\theta}^2-5\lamz^2\nrm{\theta}^2 },
\end{align*}
where we use $a^2\geq \frac{5}{6}(a+b)^2-5b^2$ for scalar $a,b\geq 0$. Plugging in $\theta=\thM-\thp$ and re-arranging yield
\begin{align*}
    \nrm{\thM-\thp}_{\sigp}
    =\EE_{\pi\sim P}\nrm{\Up^{-1}(\thM-\thp)}^2
    \leq 24\eps^2+20\lamz^2
    \leq 25\eps^2.
\end{align*}
\qed

\subsection{Proof of \cref{prop:Lip-CB} and \cref{prop:L-C-CB}}\label{appdx:proof-Lip-CBs}
\newcommand{\cFb}{\cF_{\sf b}}
\newcommand{\cMb}{\cM_{\sf b}}
\newcommand{\NX}{N_X}
\newcommand{\NA}{N_A}

In this section, we prove the results of Lipschitz contextual bandits (\cref{prop:Lip-CB} and \cref{prop:L-C-CB}). We first state the general result for any Lipschitz contextual bandits under the following conditions on the value function class $\cF\subseteq (\cX\times\cA\to[-1,1])$:

(1) $\cX$ and $\cA$ are both metric space, and for any $f\in\cF$, $x\in\cX$, $a\in\cA$, the function $f(\cdot,a)$ is 1-Lipschitz over $\cX$, and the function $f(x,\cdot)$ is 1-Lipschitz over $\cA$.

(2) There is a convex function class $\cFb\subseteq (\cA\to[-1,1])$, such that for any $f\in\cF$, $x\in\cX$, we have $f(x,\cdot)\in\cFb$. 

\newcommand{\of}{\bar{f}}
(3) The offset DEC of $\cFb$ is defined as
\begin{align}
    \rdeco_\gamma(\cFb,\of)=\inf_{p\in\DA} \sup_{f\in\cFb} \EE_{a\sim p}\brac{ \max_{a'} f(a')-f(a) - \gamma_x (f(a)-\of(a))^2 },
\end{align}
and we define $\rdeco_\gamma(\cFb)=\sup_{\of\in\cFb}\rdeco_\gamma(\cFb,\of)$ following \citet{foster2021statistical}. We assume that there is an increasing function $d(\gamma)\geq 1$ such that
\begin{align*}
    \rdeco_\gamma(\cMb)\leq \frac{d(\gamma)}{\gamma}, \quad\forall \gamma>0.
\end{align*}

Under the above conditions, we prove that the offset DEC of $\cM=\cMFcb$ can be bounded, and \ExOp~can be suitably instantiated to achieve the desired regret.

\newcommand{\cXd}{\cX_\Delta}
\newcommand{\piv}{{\bar{\pi}}}
\newcommand{\cAd}{\cA_\Delta}

\paragraph{Specifications of \LDPexo}
To instantiate \LDPexo, we consider an \infosets~that is slightly different from the one considered in \cref{appdx:ExO-CB}. 

Fix a parameter $\Delta\geq0$, we denote $\NX\defeq N(\cX,\Delta), \NA\defeq N(\cA,\Delta)$.
Let $\cXd\subset \cX$ be a minimal $\Delta$-covering of $\cX$, and for each $x\in\cX$, we let $[x]\in\cXd$ such that $\rho(x,[x])\leq \Delta$. Similarly, we take a minimal $\Delta$-covering $\cAd$ of $\cA$, and for each $a\in\cA$, we define $[a]\in\cAd$ such that $\rho(a,[a])\leq \Delta$. We consider the space $\Pi_+\defeq \prod_{x\in\cXd} \cAd$, i.e., for each $\piv\in\Pi_+$, $\piv$ is a $\cAd$-valued vector indexed by $x\in\cXd$. We also identify $\Pi_+\subseteq \Pi$, by associating $\piv(x)=\piv([x])$ for all $x\in\cX$.

For $\piv\in\Pi_+$, we let $\cF_\piv$ be class of all reward functions $f$ such that $\piv$ is a near-optimal policy:
\begin{align*}
    \cF_\piv\defeq \sset{ f\in\cF: \forall x\in\cXd, f(x,\piv(x))\geq \max_{a\in\cA} f(x,a)-\Delta }.
\end{align*}
By definition, $\cF=\cup_{\piv\in\Pi_+} \cF_\piv$.

Then, we consider the (\tpb) \infosets~$\Psi=\Pi_+$, with
\begin{align*}
    \cMps\defeq \sset{ M_{\nu,f}: \nu\in\DX, f\in\cF_\psi }, \qquad
    \pi_\psi=\psi, \qquad 
    \psi\in\Psi=\Pi_+.
\end{align*}
Clearly, $\Psi$ is a valid \infosets~for the constraint set $\MPowcxt$ introduced in \cref{ssec:CBs}, and we have $\log|\Psi|=\NX\log\NA$. Therefore, it remains to upper bound the regret DEC $\rdecol_\gamma(\cMPs)$.

\begin{theorem}[Learning Lipschitz contextual bandits]\label{thm:lip-cb-general}
It holds that
\begin{align}\label{eq:Lip-CB-DEC-upper}
    \rdecol_\gamma(\cMFcb)\leq \rdecol_\gamma(\cMPs) \leq \Delta+ \frac{\NX(10d(\gamma)+5)}{\gamma}
\end{align}
Further, \LDPexo~(when instantiated on $\Psi$) achieve \whp
\begin{align*}
    \frac1T\regdm(T)\leqsim \Delta+\NX\paren{ \frac{d(\gamma)}{\alpha^2\gamma}+\frac{\gamma\log (\NA/\delta)}{T} }.
\end{align*}
\end{theorem}

The proof of \cref{thm:lip-cb-general} is deferred to \cref{appdx:proof-lip-sb-general}. Using \cref{thm:lip-cb-general}, we prove \cref{prop:Lip-CB} and \cref{prop:L-C-CB} as follows.

\paragraph{Proof of \cref{prop:Lip-CB} (upper bound)}
By \citet[Proposition 5.2]{foster2021statistical}, for any function class $\cFb\subseteq (\cA\to[-1,1])$, we have
\begin{align*}
    \rdeco_\gamma(\cFb)\leq \frac{|\cA|}{\gamma}.
\end{align*}
Therefore, we have
\begin{align*}
    \rdecl_\eps(\cM)\leq \inf_{\gamma>0}\paren{ \rdecol_\gamma(\cM)+\gamma\eps^2 }\leqsim \Delta+\sqrt{\NX|\cA|\eps}.
\end{align*}
Similarly, we can suitably choose the parameter $\gamma>0$ such that \ExOp~achieves 
\begin{align*}
    \regdm\leqsim T\Delta+\NX\sqrt{|\cA|T\log(|\cA|/\delta)}.
\end{align*}
\qed

\newcommand{\polylog}{\poly\log}
\paragraph{Proof of \cref{prop:L-C-CB}}
In this case, we have
\begin{align*}
    \cFb=\set{ f:\cA\to [-1,1]: f\text{ is concave and 1-Lipschitz over }\cA\subset \R^K\text{ under }\nrm{\cdot}_2 }.
\end{align*}
Suppose that the diameter of $\cA$ is bounded by $R\geq 1$.
Then, by \citet[Proposition 6.3]{foster2021statistical}, we have
\begin{align*}
    \rdeco_\gamma(\cFb)\leq \frac{K^4}{\gamma}\cdot \polylog(\gamma,R).
\end{align*}
Therefore, we can bound
\begin{align*}
    \rdecl_\eps(\cM)\leq \inf_{\gamma>0}\paren{ \rdecol_\gamma(\cM)+\gamma\eps^2 }\leqsim \Delta+\sqrt{\NX K^4\eps}\cdot \polylog(\eps^{-1},R).
\end{align*}
Also note that $\log \NA\leq K\log(R/\Delta)$.
Therefore, we can suitably choose the parameter $\gamma>0$ such that \ExOp~achieves \whp
\begin{align*}
    \regdm\leqsim T\Delta+\NX\sqrt{K^5T}\cdot \polylog(T,R,\delta^{-1},\Delta^{-1}).
\end{align*}
\qed

\subsubsection{Proof of Theorem~\ref{thm:lip-cb-general}}\label{appdx:proof-lip-sb-general}

Let \ExOp~be instantiated with the measurement class $\Phi=\Pcp$ and \infosets~$\Psi$. Then, by \cref{thm:ExO-full-offset} and our analysis in \cref{appdx:inst-LDP-pac-lower}, it holds that \whp
\begin{align*}
    \frac1T\regdm(T)
    \leq&~ 3\Delta+\rdecol_{c\alpha^2\gamma}(\cMPs)+\frac{2\gamma}{T}\brac{\log\DC[3\Delta]{\MPowcxt, \Psi}+\log(1/\delta)}.
\end{align*}

Therefore, it remains to upper bound the \dct~$\log\DC[3\Delta]{\MPowcxt, \Psi}$ and the \rDEC.

We first prove that $\log\DC[3\Delta]{\MPowcxt, \Psi}\leq \log|\Psi|$. Consider the prior $\qd=\Unif(\Psi)$. Note that for any $f\in\cF$, there exists $\piv\in\Pi_+$ such that $f\in\cF_{\piv}$. Then, for any $x\in\cX$,
\begin{align*}
    \max_{a\in\cA} f(x,a)\leq&~ \max_{a\in\cA} f([x],a)+\Delta \\
    \leq&~ f([x],\piv([x]))+2\Delta
    \leq f(x,\piv(x))+3\Delta,
\end{align*}
where we use the Lipschitzness of $f$ and $\piv(x)=\piv([x])$. Therefore, we know $\cP_f\subseteq \cM_{\piv}$ and for any $M\in\cP_f$
\begin{align*}
    \Vmm-\Vm(\piv)=\EE_{x\sim M}\brac{ \max_{a\in\cA} f(x,a)-f(x,\piv(x)) }\leq 3\Delta.
\end{align*}
This implies that $\DC[3\Delta]{\MPowcxt, \Psi}\leq |\Psi|$.

\newcommandx{\tfm}[1][1=M]{\Tilde{f}\sups{#1}}
\newcommand{\tfom}{\tfm[\oM]}
Therefore, it remains to prove \eqref{eq:Lip-CB-DEC-upper}.

\paragraph{Proof of \eqref{eq:Lip-CB-DEC-upper}}
To upper bound $\rdecol_\gamma(\cMPs)$, we fix a reference model $\oM\in\co(\cMPs)=\co(\cMFcb)=:\cM^+$. 

For any model $M\in\cM^+$, we consider the function
\begin{align*}
    \tfm(x_i,a)=\EE\sups{M}[r|[x]=x_i,a].
\end{align*}
Then, because $M\in\cM^+$, we know $\tfm(x_i,\cdot)$ is a convex combination of elements in $\cFb$, and hence $\tfm(x_i,\cdot)\in\cFb$ for all $x_i\in\cXd$. We also denote $\muM\in\Delta(\cXd)$ to be distribution of $[x]$, $x\sim M$.

Construction of the distribution $p\in\DPL$:
For each $x\in\cXd$, 
we consider $\of_x\defeq \tfm[\oM](x,\cdot)\in\cFb$, $\gamma_x=\frac{\gamma\muM[\oM](x)}{5}$, and
\begin{align*}
    p_x=\argmin_{p\in\DA} \sup_{f\in\cFb} \EE_{a\sim p}\brac{ \max_{a'} f(a')-f(a) - \gamma_x (f(a)-\of_x(a))^2 }
\end{align*}
Then we know
\begin{align*}
    \EE_{a\sim p_x}\brac{ \max_{a'} f(a')-f(a) - \gamma_x (f(a)-\of_x(a))^2 } \leq \frac{d(\gamma_x)}{\gamma_x}, \quad \forall f\in\cFb.
\end{align*}
Next, we consider the following maps $\bv_0, \bv: \cZ \to \R^{\NX}$:
\begin{align*}
    \bv_0(x,a,r)=\be_{[x]}, \qquad \bv(x,a,r)=r\be_{[x]},
\end{align*}
where $\be_{x_i}\in\R^{\NX}$ is the vector with the $i$-th coordinate being 1 and other coordinates being 0. Then, by \cref{lem:dist-l}, there exists a distribution $Q\in\DL$ such that for all $M\in\cM^+$,
\begin{align*}
    &~ 2\EE_{\lf\sim Q}\Dl(M(\pi),\oM(\pi))^2 \\
    \geq&~ \nrm{ \EE_{z\sim M(\pi)}[\bv_0(z)]-\EE_{z\sim \oM(\pi)}[\bv_0(z)] }^2+\nrm{ \EE_{z\sim M(\pi)}[\bv(z)]-\EE_{z\sim \oM(\pi)}[\bv(z)] }^2.
\end{align*}
Then, we let $p\in\DPL$ be the distribution of $(\pi,\lf)$ under $\lf\sim Q$, $\pi([x])\sim p_{[x]}$ independently for all $[x]\in\cXd$, and $\pi(x)=\pi([x])$ for all $x\in\cX$.

Then, by definition
\begin{align*}
    &~2\EE_{(\pi,\lf)\sim p}\Dl(M(\pi),\oM(\pi))^2 \\
    =&~ 2\EE_{\pi\sim p}\EE_{\lf\sim Q}\Dl(M(\pi),\oM(\pi))^2 \\
    \geq&~ \EE_{\pi\sim p}\brac{\sum_{x\in\cXd} \abs{ \muM(x)-\muM[\oM](x) }^2+\abs{\muM(x) \tfm(x,\pi(x))-\muM[\oM](x) \tfom(x,\pi(x))}^2 }\\
    \geq&~ \EE_{\pi\sim p}\brac{ \frac12 \sum_{x\in\cXd} \muM[\oM](x)^2 \abs{\tfm(x,\pi(x))-\tfom(x,\pi(x))}^2 } \\
    =&~ \frac12 \sum_{x\in\cXd} \muM[\oM](x)^2 \EE_{a\sim p_x}\abs{\tfm(x,a)-\tfom(x,a)}^2,
\end{align*}
where the last line follows from the definition of $\pi\sim p$, as $\pi([x])\sim p_{[x]}$ independently.

Next, for any $M\in\cMPs$, there exists $\piv\in\Pi_+$ such that $M\in\co(\cM_\piv)$, and by our argument above, we know $\Vmm-\Vm(\piv)\leq 3\Delta$, and thus
\begin{align*}
    &~ \EE_{\pi\sim p}\brac{ \Vmm-\Vm(\pi) } \\
    \leq&~ 3\Delta+\EE_{\pi\sim p}\brac{ \Vm(\piv)-\Vm(\pi) } \\
    =&~ 3\Delta+\EE_{\pi\sim p}\EE_{x\sim \muM}\brac{ \tfm(x,\piv(x))-\tfm(x,\pi(x)) } \\
    \leq&~ 3\Delta+\EE_{x\sim \muM}\EE_{a\sim p_{[x]}}\brac{ \max_{a'\in\cA}\tfm(x,a')-\tfm(x,a) } \\
    \leq&~ 3\Delta+\EE_{x\sim \muM[\oM]}\EE_{a\sim p_{[x]}}\brac{ \max_{a'\in\cA}\tfm(x,a')-\tfm(x,a) }+2\sum_{x\in\cXd} \abs{ \muM(x)-\muM[\oM](x) }.
\end{align*}
Combining the inequalities above with the AM-GM inequality $2a\leq \frac{a^2}{5}+\frac{5}{\gamma}$, we know
\begin{align*}
&~ \EE_{\pi\sim p}\brac{ \Vmm-\Vm(\pi) }-\gamma\EE_{(\pi,\lf)\sim p}\Dl(M(\pi),\oM(\pi))^2 \\
\leq&~ 3\Delta+\sum_{x\in\cXd}\paren{ \muM[\oM](x)\EE_{a\sim p_{[x]}}\brac{ \max_{a'\in\cA}\tfm(x,a')-\tfm(x,a) } + 2\abs{ \muM(x)-\muM[\oM](x) } }\\
&~-\gamma \sum_{x\in\cXd}\paren{ \frac{1}{5}\muM[\oM](x)^2 \EE_{a\sim p_x}\abs{\tfm(x,a)-\tfom(x,a)}^2+\frac{1}{10} \abs{ \muM(x)-\muM[\oM](x) }^2 } \\
\leq&~ 3\Delta+\frac{10}{\gamma}+\sum_{x\in\cXd}\muM[\oM](x)\EE_{a\sim p_{[x]}}\brac{ \max_{a'\in\cA}\tfm(x,a')-\tfm(x,a)  - \frac{\muM[\oM](x)\gamma}{5}\abs{\tfm(x,a)-\tfom(x,a)}^2 } \\
\leq&~ 3\Delta+\frac{10}{\gamma}+\sum_{x\in\cXd}\muM[\oM](x) \cdot \rdecol_{\gamma_x}(\cFb, \of_x) 
\leq  3\Delta+\frac{10}{\gamma}+\NX\cdot \frac{5d(\gamma)}{\gamma}, 
\end{align*}
where the last line follows from the choice of $p_x$.
This gives the desired upper bound on the offset DEC as
\begin{align*}
    \rdecol_\gamma(\cM,\oM)
    \leq \sup_{M\in\cMPs}\EE_{(\pi,\lf)\sim p}\brac{ \LM{\pi}-\gamma\Dl(M(\pi),\oM(\pi))^2 }
    \leq \frac{\NX(5d(\gamma)+10)}{\gamma}.
\end{align*}
Therefore, by the arbitrariness of $\oM$, the proof of \eqref{eq:Lip-CB-DEC-upper} is completed.
\qed

\subsection{Lower bounds for structured contextual bandits}\label{appdx:proof-lin-cb-lower}

The argument of \cref{appdx:proof-linear-l1-lower} also implies the following lower bound for contextual bandits.
\begin{proposition}\label{prop:CB-lower-proto}
Let $d\geq 1, \Delta\in(0,1]$. Consider the contextual bandits problem with context space $\cX=[d]$, action space $\cA=\set{0,1}$, reward function class
\begin{align*}
    \cF_d\defeq\set{ f_\theta: \forall i\in[d], f_\theta(i,0)=0, f_\theta(i,1)=\theta_i \Delta}_{\theta\in\Theta},
\end{align*}
where $\Theta=\set{-1,1}^d$. Let $\cM_d$ be the contextual bandits problem class with reward function in $\cF_d$ and context distribution $\mu=\Unif(\cX)$. Then, for any $T$-round \pLDP~algorithm, 
\begin{align*}
    \sup_{M\in\cM_d}\EE\sups{M,\alg}[\riskdm(T)]\geq \frac{\Delta}{4}, \qquad\text{unless }
    T\geqsim \frac{d^2}{\alpha^2\Delta^2}.
\end{align*}
\end{proposition}
Note that when $\Delta\leq \frac{1}{\sqrt{d}}$, $\cF_d$ is a class of linear functions, and hence \cref{prop:CB-lower-proto} immediately implies a regret lower bound for linear contextual bandits.

\begin{corollary}\label{cor:reg-lower-lin-cb}
Let $d\geq 1, T\geq 1$. Then for any $T$-round \pLDP~algorithm, 
\begin{align*}
    \sup_{M\in\cMlincb}\EE\sups{M,\alg}[\regdm(T)]\geq c\min\sset{ \frac{d\sqrt{T}}{\alpha},\frac{T}{\sqrt{d}} }.
\end{align*}
\end{corollary}

Similarly, we can prove the lower bound of \cref{prop:Lip-CB} as follows.

\paragraph{Proof of \cref{prop:Lip-CB} (lower bound)}
Fix a $\Delta>0$, and we set $d=N(\cX,2\Delta)$, $\cA=\set{0,1}$. By the duality of packing and covering, there exists $x_1,\cdots,x_d\in\cX$ such that $\rho(x_i,x_j)\geq 2\Delta, \forall i\neq j$. 

Then, for each $\theta\in\set{-1,1}^d$, we define $f_\theta\in\cFlip$ as follows: for any $x\in\cX$, we set $f_\theta(x,0)=0$ and
\begin{align*}
    f_\theta(x,1)\defeq \sum_{i=1}^d \theta_i\max\set{ \Delta-\rho(x,x_i), 0 }.
\end{align*}
By definition, $f_\theta(\cdot,1)$ is clearly 1-Lipschitz, because for any $x\in\cX$, there is at most one $i\in[d]$ such that $\rho(x,x_i)\leq \Delta$. 

Therefore, we have an inclusion $\iota: \cM_d\to \cblip$. Hence, \cref{prop:CB-lower-proto} implies that for any $T$-round \pLDP~algorithm $\alg$, we have
\begin{align*}
    \sup_{M\in\cblip}\EE\sups{M,\alg}[\riskdm(T)]\geq \frac{\Delta}{4}, \qquad\text{unless }
    T\geqsim \frac{d^2}{\alpha^2\Delta^2}.
\end{align*}
This is the desired result.
\qed

\subsection{Proof of Proposition~\ref{prop:SQ-lower-simple}}\label{appdx:proof-SQ-lower-simple}

\newcommand{\olf}{\Bar{\lf}}

Fix $\Delta>0$ and let $\eps_0=\crdim(\cM,\Delta)$. 

Then, there exists a reference model $\oM$ and a set of models $\set{M_1,\cdots,M_m}\subseteq \cM$, such that (1) $\set{M_1,\cdots,M_m}$ is $\eps$-correlated relative to $\oM$; (2) for any $\pi\in\DD$, there is at most $m/2$ indexes $i\in[m]$ such that $\LM[M_i]{\pi}\leq \Delta$.

In the following, we proceed to lower bound the quantile-based \pDEC~(cf. \cref{appdx:p-dec-lin-q}).

For any $\lf\in\Lc$, we have
\begin{align*}
    \sum_{i=1}^m \Dl(M_i,\oM)^2
    =\sum_{i=1}^m \abs{ M_i[\lf]-\oM[\lf] }^2
    =\sup_{w\in\R^m: \nrm{w}\leq 1} \abs{ \sum_{i=1}^m w_i\paren{ M_i[\lf]-\oM[\lf] } }^2.
\end{align*}
For any fixed $w\in\R^m$, we can consider the shifted $\olf(z)=\lf(z)-\frac12\in[-\frac12,\frac12]$ and bound
\begin{align*}
\abs{ \sum_{i=1}^m w_i\paren{ M_i[\lf]-\oM[\lf] } }^2
=&~
\abs{ \sum_{i=1}^m w_i\paren{ M_i[\olf]-\oM[\olf] } }^2\\
=&~
\abs{ \EE_{z\sim \oM}\brac{\paren{\sum_{i=1}^m w_i \paren{\frac{M_i[z]}{\oM[z]}-1}}\olf(z) }}^2 \\
\leq&~ \EE_{z\sim \oM}\paren{\sum_{i=1}^m w_i \paren{\frac{M_i[z]}{\oM[z]}-1}}^2 \cdot \EE_{z\sim \oM} \olf(z)^2  \\
\leq &~ \frac14 \sum_{i,j} \rho_{\oM}(M_i,M_j) w_iw_j,
\end{align*}
where the first inequality follows from the Cauchy inequality. Therefore,
\begin{align*}
    \sum_{i=1}^m \Dl(M_i,\oM)^2
    =&~\sup_{w\in\R^m: \nrm{w}\leq 1} \abs{ \sum_{i=1}^m w_i\paren{ M_i[\lf]-\oM[\lf] } }^2 \\
    \leq &~ \frac14\sup_{w\in\R^m: \nrm{w}\leq 1} \sum_{i,j} \abs{\rho_{\oM}(M_i,M_j)}\cdot \abs{w_iw_j} \\
    \leq&~ \frac{1}{4}\paren{ m\eps^2_0+\sqrt{m(m-1)}\eps^2_0 } \leq \frac{m\eps^2_0}{2},
\end{align*}
where the last inequality follows from the definition of $\eps$-correlation. 

Hence, we may consider $\mu=\Unif(\set{M_1,\cdots,M_m})\subseteq \DM$. Then, for any $p\in\DDD, q\in\DL$, we know that
\begin{align*}
    \PP_{M\sim \mu, \pi\sim p}\paren{ \LM{\pi}\geq \Delta }\geq \inf_{\pi} \PP_{M\sim \mu}\paren{ \LM{\pi}\geq \Delta }\geq \frac12.
\end{align*}
Therefore, there must exist $\cM_0\subset \cM$ such that $\mu(\cM_0)\geq \frac13$, and
\begin{align*}
    \PP_{\pi\sim p}\paren{ \LM{\pi}\geq \Delta }\geq \frac14, \qquad\forall M\in\cM_0.
\end{align*}
Then, we also know
\begin{align*}
    \mu(\cM_0) \min_{M\in\cM_0} \EE_{\lf\sim q} \Dl(M,\oM)^2
    \leq \EE_{M\sim \mu} \EE_{\lf\sim q} \Dl(M,\oM)^2 \leq \frac12\eps_0^2,
\end{align*}
and hence there exists $M\in\cM_0$ with $\EE_{\lf\sim q} \Dl(M,\oM)^2\leq \frac32\eps^2_0$. This gives
\begin{align*}
    \pdecql_{\sqrt{2}\eps_0,1/4}(\cM,\oM)\geq \Delta,
\end{align*}
which also implies $\pdecl_{\sqrt{2}\eps_0}(\cM,\oM)\geq \frac{\Delta}{4}$.
Hence, the desired lower bounds on $\pdecl_\eps(\cM)$ follows for all $\eps\geq \eps_0$.

Furthermore, applying \cref{prop:p-dec-q-lin-lower} with $\pdecql_{\sqrt{2}\eps_0,1/4}(\cM,\oM)\geq \Delta$, we also have the desired lower bound on sample complexity.
\qed

\subsection{Proof of Proposition~\ref{thm:LDP-exp-lower}}\label{appdx:proof-par-lower}

Fix a parameter $\lambda\in[0,\frac12]$ and denote $\cX_+=\cX\backslash\set{\bz}$, $\mu_+=\Unif(\cX_+)$.

Then, for each $S\subseteq [n]$, we consider $M_S\defeq M_{S,\lambda}$ the model with covariate distribution being $\mu_\lambda=(1-\lambda)\delta_{0}+\lambda\mu_+$, and $y=f_S(x)$ for $(x,y)\sim M_S$.
We then consider the subclass $\cM_\lambda=\set{M_{S,\lambda}}_{S\subseteq [n]}\subset\cM$, and let $\oM=\oM_\lambda\in\DZ$ be the reference model given by $x\sim \mu,$ and $y\sim \Rad{0}$ if $x\neq \bz$, and $y=0$ otherwise. 

By definition, the pairwise correlation is given by
\begin{align*}
    \rho_{\oM}\paren{ M_S, M_{S'} }=
    \lambda  \EE_{x\sim \mu_+}f_S(x)f_{S'}(x) .
\end{align*}
Further, we know $f_S(x)=(-1)^{\sum_{i\in S} x_i}$, and hence
\begin{align*}
    \EE_{x\sim \Unif(\ZZ_2^d)}f_S(x)f_{S'}(x)= \EE_{x\sim \Unif(\ZZ_2^d)} (-1)^{\sum_{i\in S\Delta S'} x_i}=\begin{cases}
        0, & S\neq S', \\
        1, & S=S'.
    \end{cases}
\end{align*}
Therefore, we have
\begin{align*}
    \rho_{\oM}\paren{ M_S, M_{S'} }=\begin{cases}
        -\frac{\lambda}{2^d-1}, & S\neq S', \\
        \lambda, & S=S'.
    \end{cases}
\end{align*}
We also know that for any $f\in(\cX\to\set{-1,1})$,
\begin{align*}
    \LM[M_S]{f}\geq \lambda\PP_{x\sim \mu_+}\paren{ f(x)\neq f_S(x) }=\lambda\EE_{x\sim \mu_+}\brac{\frac{1-f(x)f_S(x)}{2}}.
\end{align*}
Therefore, for any $S\neq S'$,
\begin{align*}
    \LM[M_S]{f}+\LM[M_{S'}]{f}\geq \PP_{x\sim \mu}\paren{ f_S(x)\neq f_{S'}(x) }
    \geq \frac{\lambda}{2}\paren{ 1-\EE_{x\sim \mu_+}f_S(x)f_{S'}(x) }
    =\frac{\lambda}{2}\paren{ 1-\frac{1}{2^d-1} }.
\end{align*}
Hence, for any $f\in(\cX\to\set{-1,1})$, there exists at most one model $M_S$ such that $\LM[M_S]{f}\leq \frac{\lambda}{8}$, and by definition of the minimum correlation (\cref{def:SQ-cor}), we know
\begin{align}\label{eqn:proof-parity-sq}
    \crdim(\cM_\lambda,\lambda/8)\leq \frac{\lambda}{2^d-1}.
\end{align}
Note that for $\lambda=\Theta(1)$, \eqref{eqn:proof-parity-sq} is enough for proving lower bound $\Omega(2^d)$ for constant sub-optimality: applying \cref{prop:SQ-lower-simple} immediately yields the desired result (for sub-optimality level $\eps=\Theta(1)$).

In the following, we use a slightly more careful argument to show the lower bound of \pDEC. Notice that \eqref{eqn:proof-parity-sq} implies that for $\lambda=\frac12$ and $\eps_0=\frac{1}{\sqrt{2^d-1}}$,
\begin{align*}
    \pdecl_{\eps_0}(\cM_{1/2},\oM_{1/2})\geq \frac{1}{16}.
\end{align*}
Further, notice that for $\lambda\leq \frac12$, $S\subseteq [n]$, we have
\begin{align*}
    \LM[M_{S,\lambda}]{f}=2\lambda \LM[M_{S,1/2}]{f}, \qquad
    \Dl(M_{S,\lambda},\oM_\lambda)=2\lambda\Dl(M_{S,1/2},\oM_{1/2}).
\end{align*}
Therefore, by the definition of \pDEC~\cref{def:p-dec-l},
\begin{align*}
    \pdecl_{2\lambda\eps_0}(\cM_{\lambda},\oM_{\lambda})=2\lambda\cdot \pdecl_{\eps_0}(\cM_{1/2},\oM_{1/2})\geq \frac{\lambda}{8},
\end{align*}
and hence for any $\eps\leq \eps_0$, we con set $\lambda(\eps)=\frac{\eps}{2\eps_0}$, and then
\begin{align*}
    \pdecl_\eps(\cM)\geq \pdecl_{\eps}(\cM_{\lambda(\eps)},\oM_{\lambda(\eps)})
    \geq \frac{\lambda(\eps)}{8}
    =\frac{1}{16}\sqrt{2^d-1}\eps.
\end{align*}
Applying \cref{thm:pdec-lin-lower} completes the proof, as $L$ is a metric-based loss.
\qed

\subsection{Proof of \cref{cor:alg-approx-to-pure}}\label{appdx:proof-alg-approx-to-pure}

\newcommand{\tPP}{\widetilde{\PP}}
\newcommand{\tPPp}{\tPP\sups{M, \sf pr}}
\newcommand{\tPPa}{\tPP\sups{M}}
Fix the $T$-round algorithm $\alg$ with rules $\set{q\ind{t}}\cup \set{p}$, we define $\algpr$ as follows: for each round $t\in[T]$,
\begin{itemize}
\item Sample $(\pi\ind{t},\tpr\ind{t})\sim q\ind{t}(\cdot | \pi_1,\tpr_1,z_1,\cdots,\pi_{t-1},\tpr_{t-1},z_{t-1})$.
\item Set $\bpi\ind{t}=(\pi\ind{t},\tprpr[t])$ and observe $z\ind{t}$.
\end{itemize}
Now, we define $\tPPp$ to be the joint distribution of $\cH=\set{(\pi\ind{t},\tpr\ind{t},\tprpr[t],z\ind{t})}_{t\in[T]}$ under $\algpr$ and model $M\in(\Pi\to\DZ)$.

As an intermediate step of proof, we also consider $\tPPa$, the distribution of $\cH=\set{(\pi\ind{t},\pr\ind{t},\prprt[t],z\ind{t})}_{t\in[T]}$ under $\alg$ and model $M\in\cM$. By data-processing inequality, we have
\begin{align*}
    \DTV{ \PP\sups{M,\alg}(\cH_\pi=\cdot), \PP\sups{M,\algpr}(\cH_\pi=\cdot) }
    \leq \DTV{ \tPPp, \tPPa }.
\end{align*}
Then, we may apply the chain rule of TV distance, which gives
\begin{align*}
    \DTV{ \tPPp, \tPPa }
    \leq&~ \sum_{t=1}^T \EE_{\tPPa} \DTV{ \tPPp(z\ind{t}=\cdot| \pi_{1:t},\pr_{1:t},\prprt[1:t],z_{1:t-1} ), \tPPa(z\ind{t}=\cdot| \pi_{1:t},\pr_{1:t},\prprt[1:t],z_{1:t-1} ) } \\
    =&~\sum_{t=1}^T \EE_{\tPPa} \DTV{ \tPPp(z\ind{t}=\cdot| \pi\ind{t}, \prprt[t]), \tPPa(z\ind{t}=\cdot| \pi\ind{t}, \pr\ind{t} ) } \\
    =&~ \sum_{t=1}^T \EE_{\tPPa} \DTV{ \prprt[t]\circ M(\pi\ind{t}), \pr\ind{t}\circ M(\pi\ind{t}) } \\
    \leq&~ T\cdot \sup_{\pr}\sup_{z\in\cZ}\DTV{ \prpr(\cdot|z), \pr(\cdot|z) } \\
    \leq&~ \frac{T\beta}{1+\ea-\beta},
\end{align*}
where the expectation $\EE_{\tPPa}$ is taken over the trajectory $\cH=\set{(\pi\ind{t},\pr\ind{t},\prprt[t],z\ind{t})}_{t\in[T]}\sim \tPPa$.
Combining the inequalities above completes the proof.
\qed

\section{Proofs from \cref{sec:connection} and \cref{appdx:connection-more}}\label{appdx:DC}
\subsection{Proof of \cref{thm:local-minimax}}\label{appdx:proof-local-minimax}

For simplicity, for any model class $\cM\subseteq \DZ$, we denote
\begin{align*}
    \RISK(\cM)\defeq \inf_{\alg} \sup_{M\in\cM} \EE\sups{M,\alg} \brac{ \Riskdm(T) },
\end{align*}
where the $\inf_{\alg}$ is taken over \pLDP~algorithms.
Then we know
\begin{align*}
    \RISKloc(\cM,M_0)\defeq \sup_{M_1\in\cM}\RISK(\set{M_1,M_0}).
\end{align*}

\paragraph{Proof of the upper bound} We only need to bound \PDEC~in terms of the local DEC, as follows.
\begin{lemma}\label{lem:lin-dec-to-loc}
For any 2-point model class $\cM_0=\set{M_1,M_0}\subseteq \cM$, it holds that
\begin{align*}
    \pdecl_\eps(\cM_0)\leq \begin{cases}
        \inf_{\pi\in\DD}\paren{  \LM[M_1]{\pi}+\LM[M_0]{\pi} }, & \text{if } \sup_{\pi} \DTV{ M_1(\pi), M_0(\pi) } \leq 2\eps, \\
        0, & \text{otherwise}.
    \end{cases}
\end{align*}
\end{lemma}

With \cref{lem:lin-dec-to-loc}, we know that
\begin{align*}
    \sup_{M_1\in\cM} \pdecl_\eps(\set{M_0, M_1})\leq \pdecloc_{2\eps}(\cM).
\end{align*}
Applying \cref{thm:pdec-lin-upper} gives
\begin{align*}
    \RISK(\set{M_1,M_0})\leq \pdecl_{\oeps_\delta(T)}(\set{M_0, M_1})+\delta.
\end{align*}
Therefore, we may combine the two inequalities above to obatin
\begin{align*}
    \RISKloc(\cM,M_0)=\sup_{M_1\in\cM} \RISK(\set{M_1,M_0})\leq \pdecloc_{2\oeps_\delta(T)}(\cM)+\delta.
\end{align*}
This is the desired upper bound.
\qed

\paragraph{Proof of the lower bound} Similar to the proof of upper bound, we can directly lower bound  $\sup_{M_1\in\cM} \pdecl_\eps(\set{M_1, M_0})$ by the local DEC $\pdecloc_\eps(\cM)$. However, the \pDEC~lower bound (\cref{thm:pdec-lin-lower}) requires certain structural assumptions on the loss function $L$, which is in fact artificial in this case. Therefore, in the following, we utilize the quantile DEC lower bound (\cref{appdx:p-dec-lin-q}) to obtain a better lower bound.

\begin{lemma}\label{lem:dec-q-to-loc}
For any $\eps>0$, it holds
\begin{align*}
    \sup_{M_1\in\cM} \pdecql_{\eps,1/2}(\set{M_1,M_0},M_0)\geq \frac12\pdecloc_\eps(\cM,M_0).
\end{align*}
\end{lemma}

Now, applying \cref{prop:p-dec-q-lin-lower} gives
\begin{align*}
    \RISK(\set{M_1,M_0})\geq \frac14 \pdecql_{\ueps(T),1/2}(\set{M_1,M_0}),
\end{align*}
and hence \cref{lem:dec-q-to-loc} yields
\begin{align*}
    \RISKloc(\cM,M_0)=\sup_{M_1\in\cM}\RISK(\set{M_1,M_0})
    \geq \frac14 \sup_{M_1\in\cM} \pdecql_{\ueps(T),1/2}(\set{M_1,M_0})
    \geq \frac18 \pdecloc_\eps(\cM,M_0).
\end{align*}
This is the desired result.
\qed

\paragraph{Proof of \cref{lem:lin-dec-to-loc}}
\newcommand{\piexp}{\pi^{\sf exp}}
Define $\eps_0=\sup_{\pi} \DTV{ M_1(\pi), M_0(\pi) }$ and
\begin{align*}
    \piexp=\argmax_{\pi\in\Pi} \DTV{ M_1(\pi), M_0(\pi) }.
\end{align*}
Further, we choose $\lf\in\Lc$ such that
\begin{align*}
    \Dl\paren{ M_1(\piexp), M_0(\piexp) }=\DTV{ M_1(\piexp), M_0(\piexp) }.
\end{align*}
Then, we consider the distribution $q\in\DPL$ supported on $(\piexp,\lf)$, and any reference model $\oM\in\co(\cM_0)$ given by $\oM=\lambda M_0+(1-\lambda)M_1$ (where $\lambda\in[0,1]$). There are two cases:

(1) If $\eps_0\leq 2\eps$, then, we choose
\begin{align*}
    \piout=\argmin_{\pi\in\DD}\paren{  \LM[M_1]{\pi}+\LM[M_0]{\pi} },
\end{align*}
and let $p\in\DPi$ be the distribution supported on $\piout$. Then, $p$ certifies that
\begin{align*}
    \pdecl_\eps(\cM_0, \oM)\leq
        \inf_{\pi\in\DD}\paren{  \LM[M_1]{\pi}+\LM[M_0]{\pi} }.
\end{align*}

(2) If $\eps_0>2\eps$, then using the fact that
\begin{align*}
    \Dl\paren{ M_0(\piexp), \oM(\piexp) }=&~(1-\lambda)\Dl\paren{ M_1(\piexp), M_0(\piexp) }=(1-\lambda)\eps_0, \\
    \Dl\paren{ M_1(\piexp), \oM(\piexp) }=&~\lambda\Dl\paren{ M_1(\piexp), M_0(\piexp) }=\lambda\eps_0,
\end{align*}
we know there is at most one index $i\in\set{0,1}$ such that $\Dl\paren{ M_i(\piexp), \oM(\piexp) }\leq \eps$.
If such an index does not exist, then we already have $\pdecl_\eps(\cM_0, \oM)=0$. Otherwise, given such an index $i$, we can take a decision $\piout$ such that $\LM[M_i]{\piout}=0$, which also certifies $\pdecl_\eps(\cM_0, \oM)=0$.
\qed

\paragraph{Proof of \cref{lem:dec-q-to-loc}}
We take $0\leq \Delta<\pdecloc_{\eps}(\cM,M_0)$. Then by definition, there exists $M_1\in\cM$ such that
\begin{align*}
    \sup_\pi\DTV{M_1(\pi), M_0(\pi)}\leq \eps, \qquad \inf_\pi\paren{ \LM[M_1]{\pi}+\LM[M_0]{\pi} }\geq \Delta.
\end{align*}
Then, for any $p\in\DPi$ and $q\in\DPL$, we have
\begin{align*}
    \EE_{(\pi,\lf)\sim q}\Dl^2(M_1(\pi),M_0(\pi))\leq \sup_\pi\DTV{M_1(\pi), M_0(\pi)}^2 \leq \eps^2,
\end{align*}
and we also know $\set{\pi: \LM[M_1]{\pi}<\Delta/2}$ and $\set{\pi: \LM[M_0]{\pi}<\Delta/2}$ are disjoint, which implies
\begin{align*}
    p\paren{ \pi: \LM[M_1]{\pi}\geq \Delta/2 }+p\paren{ \pi: \LM[M_0]{\pi}\geq \Delta/2 }\geq 1.
\end{align*}
Therefore, the quantile-based \pDEC~can be lower bounded as
\begin{align*}
    \pdecql_{\eps,1/2}(\set{M_1,M_0},M_0)\geq \frac{\Delta}{2}.
\end{align*}
This gives the desired result by letting $\Delta\to \pdecloc_{\eps}(\cM,M_0)$.
\qed

\subsection{Proof of Theorem~\ref{thm:DC-lower}}\label{appdx:proof-DC-lower}

We first recall the notations and results of \cref{appdx:proof-p-dec-q-lin-lower}. Using \eqref{eqn:KL-chain}, we know that for any models $M, \oM$,
\begin{align*}
\KLd{\PP\sups{M,\alg}}{\PP\sups{\oM,\alg}}  
\leq (\ea-1)^2 T\cdot \EE_{(\pi,\lf)\sim \tq_{M,\alg}} \Dl^2( M(\pi) , \oM(\pi) )
\leq (\ea-1)^2T.
\end{align*}
On the other hand, by the definition \cref{def:DC} of $\DC{\cM}$, we have that
\begin{align*}
    \frac1{\DC{\cM}}\defeq \sup_{p\in\Delta(\Pi)}\inf_{M\in\cM} ~~p(\pi: \LM{\pi}\leq \Delta).
\end{align*}
Therefore, we may fix a reference model $\oM\in\cM$, and it holds that
\begin{align*}
    \inf_{M\in\cM} p_{\oM,\alg}( \pi: \LM{\pi}\leq \Delta )\leq \frac1{\DC{\cM}},
\end{align*}
and hence there exists $M\in\cM$ such that
\begin{align*}
   p_{\oM,\alg}( \pi: \LM{\pi}\leq \Delta )\leq \frac1{\DC{\cM}}.
\end{align*}
On the other hand, the condition of \cref{thm:DC-lower} gives $p_{M,\alg}( \pi: \LM{\pi}\leq \Delta )\geq \frac12$.
Therefore, by data-processing inequality,
\begin{align*}
    \KLd{\PP\sups{M,\alg}}{\PP\sups{\oM,\alg}}
    \geq \KLd{ p_{M,\alg} }{p_{\oM,\alg}}\geq \frac{\log\DC{\cM}}{2}-1.
\end{align*}
Comparing the lower and upper bounds above complete the proof.
\qed

\subsection{Proof of \cref{prop:JDP-lower-interactive} and \cref{prop:JDP-lower}}

With the following lemma (which generalizes \citet{beimel2013characterizing}), the proof is essentially similar to \cref{appdx:proof-DC-lower}.
\begin{lemma}\label{lem:JDP}
Suppose that $\alg$ is a $T$-round \pJDP~algorithm. Then for any two models $M,\oM$, it holds that
\begin{align}
    \PP\sups{M,\alg}(\hpi\in E)\leq e^{T\alpha}\cdot \PP\sups{\oM,\alg}(\hpi\in E), \qquad \forall E.
\end{align}
\end{lemma}

Notice that by definition, for any model $M\in\cM$,
\begin{align*}
    \PP\sups{M,\alg}\paren{ \LM{\hpi}\leq \Delta }\geq \frac{1}{2},
\end{align*}
and hence
\begin{align*}
    \PP\sups{\oM,\alg}\paren{ \LM{\hpi}\leq \Delta }\geq e^{-T\alpha}\PP\sups{M,\alg}\paren{ \LM{\hpi}\leq \Delta }
    \geq \frac12e^{-T\alpha}, \qquad \forall M\in\cM.
\end{align*}
Then, by the definition of \dct~(\cref{def:DC}), we know for the distribution $p=\PP\sups{\oM,\alg}(\hpi=\cdot)\in\DPi$, it holds that
\begin{align*}
    2e^{T\alpha}
    \geq \sup_{M\in\cM}\frac{1}{p\paren{ \LM{\pi}\leq \Delta }}
    \geq \DC{\cM}.
\end{align*}
This gives the desired lower bound $\alpha T\geq \log\DC{\cM}-\log 2$.
\qed

\paragraph{Proof of \cref{lem:JDP}}
We first consider the setting of \statp, which is easier to analyze. In this case, by definition, for any sequence of observations $\Hyz=(z_1,\cdots,z_T)$, $\Hyz'=(z_1',\cdots,z_T')\in \cZ^T$, \pJDP~implies that
\begin{align}\label{eqn:proof-JDP-lower}
    \PP\sups{\alg}(\hpi\in E|\Hy\ind{T})\leq e^{T\alpha}\cdot \PP\sups{\alg}(\hpi\in E|\Hy\ind{T}'), \qquad\forall E.
\end{align}
Therefore, we may take expectation over $\Hyz=(z_1,\cdots,z_T)\sim M$ and $\Hyz'=(z_1',\cdots,z_T')\sim \oM$, which completes the proof of \cref{lem:JDP}.

More generally, for interactive learning, for any two sequences $\Hyz=(z\ind{1},\cdots,z\ind{T})$ and $\Hyz'=(z\ind{1}',\cdots,z\ind{T}')$, it holds that
\begin{align*}
    \PP\sups{\alg}\paren{ (\pi\ind{1},\cdots,\pi\ind{T},\pi\ind{T+1})=\cdot| \Hyz}\leq e^{T\alpha} \PP\sups{\alg}\paren{ (\pi\ind{1},\cdots,\pi\ind{T},\pi\ind{T+1})=\cdot| \Hyz'}.
\end{align*}
Therefore, for any fixed sequence $(\pi\ind{1},\cdots,\pi\ind{T},\pi\ind{T+1})$, we may take expectation over $z\ind{t}\sim M(\pi\ind{t}), z\ind{t}'\sim \oM(\pi\ind{t})$ independently (recursively for $t=T,T-1,\cdots,1$), which gives
\begin{align*}
    \PP\sups{M,\alg}\paren{ \pi\ind{1},\cdots,\pi\ind{T},\pi\ind{T+1}}\leq e^{T\alpha} \PP\sups{\oM,\alg}\paren{ \pi\ind{1},\cdots,\pi\ind{T},\pi\ind{T+1}}.
\end{align*}
Hence, the proof of \cref{lem:JDP} is completed.
\qed

\subsection{Proof of Proposition~\ref{thm:DC-upper}}\label{appdx:proof-DC-upper}

We consider the private analog of the algorithm of \citet{chen2024beyond}. For the simplicity of presentation, we focus on PAC learning.
\newcommandx{\indk}[1][1=k]{_{(#1)}}
\newcommand{\hk}{\widehat{k}}
\newcommand{\hr}{\widehat{r}}

\begin{algorithm}
\begin{algorithmic}[1]
\REQUIRE Model class $\cM$, decision space $\Pi$, parameter $\Delta,\delta>0$, $T\geq 1$.
\STATE Set $N=\DC{\cM}\log(1/\delta)$, $J=\frac{T}{N}$, and
\begin{align}
    \pds=\arg\min_{p\in\Delta(\Pi)}\sup_{M\in\cM} ~~\frac{1}{ p(\pi: \LM{\pi}\leq \Delta)}.
\end{align}
\FOR{$k=0,1,\cdots,N-1$}
\STATE Sample $\pi_{(k)}\sim \pds$.
\STATE Set $\lf_{(k)}=R(\cdot,\pi_{(k)})$ and $\pr_{(k)}=\bpr[\lf_{(k)}]$ be the binary channel (\cref{example:binary-pr}).
\FOR{$t=kJ+1,\cdots,(k+1)J$}
\STATE Select $\pi\ind{t}=\pi_{(k)}$ and $\pr\ind{t}=\pr_{(k)}$ and observes $o\ind{t}\sim \pr\ind{t}\circ \Mstar(\pi\ind{t})$.
\ENDFOR
\STATE Compute $\hr_{(k)}=\frac{1}{J}\sum_{t=kJ+1}^{(k+1)J} o\ind{t}$.
\ENDFOR
\STATE Set $\hk=\argmax_{k\in[N]} \hr_{(k)}$.
\ENSURE $\hpi=\pi_{(\hk)}$.
\end{algorithmic}
\caption{``Brute-Force'' Algorithm}
\label{alg:DC-upper}
\end{algorithm}

\paragraph{Analysis of \cref{alg:DC-upper}}
By definition,
\begin{align*}
    \pi\indk[1], \cdots, \pi\indk[N] \sim \pds~\text{independently.}
\end{align*}
Hence,
\begin{align*}
    \PP\paren{ \forall k\in[N], \LM[\Ms]{\pi\indk}> \Delta }
    =&~ \pds\paren{ \pi: \LM[\Ms]{\pi}>\Delta }^N \\
    \leq&~ \paren{1-\frac{1}{\DC{\cM}}}^N \\
    \leq &~ \exp\paren{ -\frac{N}{\DC{\cM}} }\leq \delta.
\end{align*}
Therefore, \whp, there exists $k\in[K]$ such that $\LM[\Ms]{\pi\indk}\leq\Delta$.

Furthermore, by the definition of $(\pi\indk,\pr\indk)$, we know that for $t\in[kJ+1,(k+1)J]$, the observation $o\ind{t}\sim \Rad{\ca \Vm[\Mstar](\pi\indk)}$ are generated independently. Therefore, \whp[\frac{\delta}{N}], it holds that
\begin{align*}
    \abs{\hr\indk-\ca\Vm[\Mstar](\pi\indk)}\leq \sqrt{\frac{2\log(2N/\delta)}{J}}=:\eps_J.
\end{align*}
Hence, taking the union bound, we know that \whp[2\delta],
\begin{align*}
    \ca \Vm[\Mstar](\pi\indk[\hk])\geq \ca \max_{k\in[N]}\Vm[\Mstar](\pi\indk[\hk])-2\eps_J
    \geq \ca\paren{\Vmm[\Mstar]-\Delta}-2\eps_J.
\end{align*}
Reorganizing yields
\begin{align*}
    \LM[\Mstar]{\hpi}
    =&~\Vmm[\Mstar]-\Vm[\Mstar](\pi\indk[\hk])
    \leq \Delta+\frac{2}{\ca}\sqrt{\frac{2N\log(T/\delta)}{T}} \\
    \leq&~ \Delta+\bigO{\frac{\log(T/\delta)}{\alpha}\sqrt{\frac{\DC{\cM}}{T}}}.
\end{align*}
This is the desired result.
\qed

\subsection{Proof of \cref{lem:DC-agn-regr} and \cref{lem:DC-real-regr}}\label{appdx:proof-agn-regr}

We first show that $\DC{\cMagn}\leq \DCF$ for any 1-Lipschitz loss $\loss$. For any $M\in\cMagn$, we denote
\begin{align*}
    \fm=\argmin_{f\in\cF}\EE_{(x,y)\sim M} \loss(y, f(x)),
\end{align*}
and then for any $f\in\cFp$,
\begin{align*}
    \LM{f}=\EE_{(x,y)\sim M} \brac{ \loss(y, f(x))-\loss(y, \fm(x)) }
    \leq \EE_{x\sim M}\abs{ f(x)-\fm(x) }.
\end{align*}
Therefore, we have $\DC{\cMagn}\leq \DCF$ for any $\Delta>0$.

We next consider the absolute loss $\Labs$. By definition, for any $M\in\cMreal$, we have $\LM{f}=\EE_{x\sim M} \abs{f(x)-\fm(x)}$, and hence it holds that $\DC{\cMreal}=\DCF$. Notice that $\cMreal\subseteq \cMagn$, and hence we have
\begin{align*}
    \DCF=\DC{\cMreal}\leq \DC{\cMagn}\leq\DCF.
\end{align*}
This gives the desired results.
\qed

\begin{remark}
Similarly, under the squared loss $\Lsq$, we can also show that
\begin{align*}
    \DCF[\sqrt{2\Delta}] \leq \DC{\cMreal}=\DC{\cMwell}\leq \DC{\cMagn} \leq \DCF.
\end{align*}
\end{remark}

\subsection{Proof of Proposition~\ref{prop:impr-lin}}\label{appdx:proof-impr-lin}

\newcommand{\cFlinC}{\cF_{\mathsf{Lin},C}}
\newcommand{\bSd}{\mathbb{S}^{d-1}}

\paragraph{Proof of the lower bound}
For any parameter $C\geq 1$, we define
\begin{align*}
    \cFlinC\defeq \set{ f_\theta(x)=\lr \theta, x\rr }_{\theta:\nrm{\theta}\leq C},
\end{align*}
We lower bound $\DCF[\Delta][\cFlin][\cFlinC]$ as follows. Denote $\cF\defeq \cFlin$ and $\cF_C\defeq \cFlinC$.

Fix any $p\in\Delta(\cF_C)$, and we bound
\begin{align*}
    \inf_{\mu\in\DX, \fs\in\cF} \PP_{f\sim p}\paren{ f: \EE_{x\sim \mu}|f(x)-\fs(x)| \leq \frac12 } 
    \leq&~ \EE_{x_0\sim \Unif(\bSd)} \PP_{f\sim p}\paren{ f: |f(x_0)-1| \leq \frac12 } \\
    =&~ \EE_{f\sim p} \PP_{x_0\sim \Unif(\bSd)}\paren{ f: |f(x_0)-1| \leq \frac12 }.
\end{align*}
Notice that for any fix $\theta\in\R^d$ and $x_0\sim \Unif(\bSd)$, we have $\lr \theta, x_0\rr=\nrm{\theta} t$, where the random variable $t\in[-1,1]$ has density function
\begin{align*}
	P(t) = \frac{\Gamma(d/2)}{\Gamma((d-1)/2)\sqrt{\pi}}(1-t^2)^{(d-3)/2},
\end{align*}
see e.g. \citet[Section 2]{bubeck2016testing}. Therefore,
\begin{align*}
    \PP_{x_0\sim \Unif(\bSd)}\paren{ f: |f(x_0)-1| \leq \frac12 }
    =&~\PP_{t\sim P}\paren{ t\in \brac{ \frac{1}{2\nrm{\theta}}, \frac{3}{2\nrm{\theta}}} } \\
    \leq&~ \frac{1}{\nrm{\theta}}\cdot \bigO{\sqrt{d}} \paren{1-\frac{4}{\nrm{\theta}^2}}^{(d-3)/2} \\
    \leq&~ \bigO{1} \frac{\sqrt{d}}{C}\exp\paren{ -\frac{d-3}{2C^2} },
\end{align*}
Therefore, as long as $C\leq c_0\sqrt{d}$, we have
\begin{align*}
    \inf_{\mu\in\DX, \fs\in\cF} \PP_{f\sim p}\paren{ f: \EE_{x\sim \mu}|f(x)-\fs(x)| \leq \frac12 } \leq \exp\paren{-c_1\frac{d}{C^2}}, \qquad \forall p\in\Delta(\cF_C),
\end{align*}
for some universal constants $c_1, c_0>0$. 
Therefore,
\begin{align*}
    \log\DCF[\Delta][\cFlin][\cFlinC]\geq c_1\frac{d}{C^2}, \qquad \forall C\in[1,c_0\sqrt{d}].
\end{align*}
In particular, this gives the desired lower bound by letting $C=1$.
\qed

\paragraph{Proof of the upper bound}
As the above lemma indicates, to upper bound $\DCF$, we must choose $p\in\DDD$ to be highly improper. We construct such a distribution of improper functions as follows.

Fix a parameter $\lambda\in[0,\lambda_0]$ for some small enough universal constant $\lambda_0$. We set $p$ to be the distribution of $f_\theta$ with $\theta\sim \normal{0,\lambda^2\id_d}$. We proceed to lower bound the probability
\begin{align*}
    \PP_{f\sim p}\paren{ f: \EE_{x\sim \mu}|f(x)-\fs(x)| \leq \Delta } 
\end{align*}
for arbitrary fixed $\fs=f_{\ths}\in\cF$ and distribution $\mu\in\DX$. Notice that for $\theta\in\R^d$, we have
\begin{align*}
    \paren{\EE_{x\sim \mu}|f_\theta(x)-\fs(x)|}^2
    \leq \EE_{x\sim \mu}|f_\theta(x)-\fs(x)|^2
    =\nrm{\theta-\ths}_\Sigma^2,
\end{align*}
where $\Sigma=\EE_{x\sim \mu}[xx^\top]$. By the rotational invariance, we may assume that $\Sigma=\diag(\lambda_1,\cdots,\lambda_d)$ with $\lambda_1\geq \cdots\geq \lambda_d\geq 0$. Notice that $\tr(\Sigma)\leq 1$, and hence we have $\sum_{i=1}^n \lambda_i\leq 1$ and $\lambda_k\leq \frac{1}{k}$. Therefore, we know
\begin{align*}
    \nrm{\theta-\ths}_\Sigma^2\leq \max_{1\leq i\leq k}\paren{ \theta_i-\ths_i }^2 + \sum_{i=k+1}^d \lambda_i\paren{ \theta_i-\ths_i }^2.
\end{align*}
Using the fact that $\theta_i\sim \normal{0,\lambda^2}$, we know
\begin{align*}
    \PP_\theta\paren{ \paren{ \theta_i-\ths_i }^2\leq \lambda^2 }\geq \frac{\lambda}{\sqrt{2\pi}}\exp\paren{ -\frac{\paren{\abs{\ths_i}+\lambda}^2}{2\lambda^2} }.
\end{align*}
Therefore, using the independence between $\theta_1,\cdots,\theta_k$, we have
\begin{align*}
    \PP_\theta\paren{ \forall i\in[k], \paren{ \theta_i-\ths_i }^2\leq \lambda^2 }\geq \paren{\frac{\lambda}{\sqrt{2\pi}}}^k\exp\paren{ -\sum_{i=1}^k\frac{\paren{\abs{\ths_i}+\lambda}^2}{2\lambda^2} }
    \geq \paren{\frac{\lambda}{\sqrt{2\pi}e}}^k \exp\paren{-\frac{1}{\lambda^2}}.
\end{align*}
Further, using the fact that 
\begin{align*}
    \EE_\theta\brac{ \sum_{i=k+1}^d \lambda_i\paren{ \theta_i-\ths_i }^2 }
    =\sum_{i=k+1}^d \lambda_i\paren{\abs{\ths_i}^2+\lambda^2}\leq \frac{1}{k}+\lambda^2,
\end{align*}
we know that
\begin{align*}
    \PP_\theta\paren{ \sum_{i=k+1}^d \lambda_i\paren{ \theta_i-\ths_i }^2 \leq \frac{2}{k}+2\lambda^2 }\geq \frac{1}{2}.
\end{align*}
Therefore, using the independence between $\theta_1,\cdots,\theta_d$, we have
\begin{align*}
    \PP_\theta\paren{ \nrm{\theta-\ths}_\Sigma^2\leq \frac{2}{k}+3\lambda^2 }  \geq \frac12\exp\paren{ -\frac{1}{\lambda^2}+ k\log(\sqrt{2\pi}e/\lambda) }.
\end{align*}
Setting $k=\frac{\Delta^2}{4}$ and $\lambda^2=\frac{\Delta^2}{6}$ gives
\begin{align*}
    \PP_\theta\paren{ \nrm{\theta-\ths}_\Sigma^2\leq \Delta^2 }\geq \exp\paren{ -\frac{C_0}{\Delta^2}\log\paren{\frac{1}{\Delta}} },
\end{align*}
where $C_0$ is a large universal constant. 
By the arbitrariness of $\mu$ and $\ths$, we have
\begin{align*}
    -\log \PP_{f\sim p}\paren{ f: \EE_{x\sim \mu}|f(x)-\fs(x)| \leq \Delta } \leq \frac{C_0}{\Delta^2}\log\paren{\frac{1}{\Delta}}, \qquad \forall \mu\in\DX, \ths\in\Bone.
\end{align*}
Therefore, $p$ certifies that $\log\DCF\leq \frac{C_0}{\Delta^2}\log\paren{\frac{1}{\Delta}}$, and the proof is hence completed.
\qed

\subsection{Proof of Proposition~\ref{prop:RDim-to-DC}}\label{appdx:proof-RDim-to-DC}

Suppose that $p\in\DF$ is given by
\begin{align*}
    p\defeq \arg\min_{p\in\DF} \sup_{\nu\in\DX, \fs\in\cF} ~\frac{1}{p\paren{ f: \EE_{x\sim \nu}|f(x)-\fs(x)| \leq \Delta } }.
\end{align*}
Then, for any given $\nu\in\DX, \fs\in\cF$, we have
\begin{align*}
    \PP_{f\sim p}\paren{ \PP_{x\sim \nu}(f(x)\neq \fs(x))\leq \Delta }\geq \frac{1}{\DC{\cF}}.
\end{align*}
Therefore, for $N\geq 1$, we consider the distribution $p_N$ over the subsets of $\DD$ given by
\begin{align*}
    \cH\sim p_N: ~\cH=\set{f_1,\cdots,f_N}, \qquad f_1,\cdots,f_N\sim p~\text{independently}.
\end{align*}
Then, we can bound
\begin{align*}
    \PP_{\cH\sim p_N}\paren{ \exists f\in\cH, \PP_{x\sim \nu}(f(x)\neq \fs(x))\leq \Delta }\geq 1-\paren{1-\frac{1}{\DC{\cF}}}^N.
\end{align*}
Choosing $N\geq \DC{\cF}\log(4)$ yields that $p_N$ is a $\Delta$-probabilistic representation of $\cF$, and hence
\begin{align*}
    \rdim_{\Delta}(\cF)\leq \log N\leq  \log\DC{\cF}+2.
\end{align*}

Conversely, suppose that $\scH$ is an optimal $\eps$-probabilistic representation of $\cF$, i.e. $\rdim_\eps(\cF)=\sz(\scH)$. Then $\scH$ induces a distribution $p_{\scH}\in\DDD$ as
\begin{align*}
    f\sim p_{\scH}: ~\cH\sim \scH, ~f\sim \Unif(\cH).
\end{align*}
Then, for any $\nu\in\DX, \fs\in\cF$,
\begin{align*}
    \PP_{f\sim p_{\scH}}\paren{ f: \EE_{x\sim \nu}|f(x)-\fs(x)| \leq \eps }
    \geq&~ \EE_{\cH\sim \scH}\brac{ \frac{1}{|\cH|}\indic{ \exists f\in\cH, \PP_{x\sim \nu}(f(x)\neq \fs(x)) \leq \eps } } \\
    \geq&~ \frac{3}{4}\frac{1}{\sup_{\cH\in \supp(\scH)} |\cH|}.
\end{align*}
Therefore, $p_{\scH}$ certifies that
\begin{align*}
    \log \DC[\eps]{\cF}\leq \sup_{\cH\in \supp(\scH)} \log |\cH|+\log(4/3)
    \leq \rdim_\eps(\cF)+1.
\end{align*}
Combining the inequalities above completes the proof.
\qed

\end{document}